%% file: dissertation.tex
\documentclass[12pt,oneside,letterpaper]{memoir}

\input{includes/packages}
\input{includes/preamble}
\input{includes/defs}
\input{includes/thesisdefs}
\svnidlong{$LastChangedBy$}{$LastChangedRevision$}{$LastChangedDate$}{$HeadURL: http://freevariable.com/dissertation/branches/diss-template/dissertation.tex $} 

\clearpage

\title{Detecting and Learning Out-of-Distribution Data in the Open world: \\ Algorithm and Theory}

\author{Yiyou Sun}
\department{Computer Sciences}

\date{2023}

\begin{document}
%%%%%%%%%%%%%%%%%%%%%%%%%%%%%%%%%%%%%%%%%%%%%%%%%%
%%%%%%%%%%%%%%% TO ADD CITATION %%%%%%%%%%%%%%%%%
%%%%%%%%%%%%%%%%%%%%%%%%%%%%%%%%%%%%%%%%%%%%%%%%%%
% For OOD detection, maybe we could include some recent works in citations too. I would also cite a few NeurIPS’22 paper from last year, including Fang Zhen’s theory work, Xuefeng’s SIREN, open OOD
% 1:38
% Lood should not be cited as “proceeding” (which implies conference main paper)
% 1:38
% Rheeya’s ACL’23 too, with a few other works on the language model part
% 1:39
% Rheeya’s ACL paper support KNN method
% 1:39
% Also NPOS is also related to non-parametric KNN
% 1:40
% we could cite generously of works from other groups too
%%%%%%%%%%%%%%%%%%%%%%%%%%%%%%%%%%%%%%%%%%%%%%%%%%
%%%%%%%%%%%%%%% TO ADD CITATION %%%%%%%%%%%%%%%%%
%%%%%%%%%%%%%%%%%%%%%%%%%%%%%%%%%%%%%%%%%%%%%%%%%%

%%% Uncomment the following if your .bib contains references that you will not 
%%% explicitly cite, but that should be in the final bibliography:
% \nocite{*}

\ifpdf
\DeclareGraphicsExtensions{.pdf, .jpg, .tif}
\else
\DeclareGraphicsExtensions{.eps, .jpg}
\fi

\maketitle

%% Add \part declarations if you want, but it's not necessary
%\part{Preliminaries}

\include{frontmatter/frontmatter}

%% Now include the tex files for each chapter, like so (I put these in separate dirs): 
\include{chapters/intro}

\include{chapters/background}

\include{chapters/1_react}
\include{chapters/2_dice}
\include{chapters/3_knn}
\include{chapters/4_nscl}
\include{chapters/5_sorl}
\include{chapters/6_opencon}
\include{chapters/conclusion}

\include{chapters/appendix}

\nocite{uppaal2023fine}
%% Do you have appendices?  If so, add them here, just like chapters.
% \begin{appendices}
% \include{backmatter/appendix1}
% \end{appendices}

%% Are you a big nerd with a colophon?  Add it here.
% \begin{colophon}
% \input{backmatter/colophon}
% \end{colophon}

%% McBride is a very nice style (some version is included in this distribution)
\bibliographystyle{mcbride}
\bibliography{thesis}

%% Want an index?  Neither did I.
%\printindex

\end{document}

%% file: includes/packages.tex
\usepackage{xcolor}
\usepackage{amsmath}
\usepackage{amssymb}
\usepackage{mathtools}

\usepackage{amsthm}
\usepackage{amsfonts}
\usepackage{bbm}
\usepackage[utf8]{inputenc} 
\usepackage[T1]{fontenc}    
\usepackage{url}            
\usepackage{nicefrac}

\usepackage{multirow}
\usepackage{mathtools}
\usepackage{algorithm}
\usepackage{algorithmic}
\usepackage{wrapfig}

\usepackage[colorlinks]{hyperref}
\definecolor{mydarkblue}{rgb}{0,0.08,0.45}
  \hypersetup{ 
    pdftitle={},
    pdfkeywords={},
    pdfborder=0 0 0,
    pdfpagemode=UseNone,
    colorlinks=true,
    linkcolor=mydarkblue,
    citecolor=mydarkblue,
    filecolor=mydarkblue,
    urlcolor=mydarkblue,
    }

\DeclareMathOperator{\esn}{ESN}

\usepackage{color, colortbl}
\definecolor{Gray}{gray}{0.9}
\usepackage{rotating}

\newcommand{\bx}{\mathbf{x}}
\newcommand{\bz}{\mathbf{z}}

\newcommand{\bM}{\mathbf{M}}
\def\*#1{\mathbf{#1}}

\usepackage[capitalize,noabbrev]{cleveref}

\usepackage{xspace}
\usepackage{array}
\usepackage{ragged2e}
\usepackage{tcolorbox}
\makeatletter
\@namedef{ver@everyshi.sty}{}
\makeatother
\usepackage{tikz}
\usetikzlibrary{arrows,shapes,positioning,shadows,trees,mindmap}
\usepackage[edges]{forest}
\usetikzlibrary{arrows.meta}
\colorlet{linecol}{black!75}
\usepackage{xkcdcolors} % xkcd colors

\usepackage{tikz}
\usetikzlibrary{backgrounds}
\usetikzlibrary{arrows,shapes}
\usetikzlibrary{tikzmark}
\usetikzlibrary{calc}
\newcommand{\highlight}[2]{\colorbox{#1!17}{$\displaystyle #2$}}
\renewcommand{\highlight}[2]{\colorbox{#1!17}{#2}}

\newcommand{\cube}[1]{%
\scalebox{#1}{
\begin{tikzpicture}
\pgfmathsetmacro{\cubex}{0.2}
\pgfmathsetmacro{\cubey}{0.2}
\pgfmathsetmacro{\cubez}{0.2}
\draw (0,0,0) -- ++(-\cubex,0,0) -- ++(0,-\cubey,0) -- ++(\cubex,0,0) -- cycle;
\draw (0,0,0) -- ++(0,0,-\cubez) -- ++(0,-\cubey,0) -- ++(0,0,\cubez) -- cycle;
\draw (0,0,0) -- ++(-\cubex,0,0) -- ++(0,0,-\cubez) -- ++(\cubex,0,0) -- cycle;
\end{tikzpicture}
}
}
\newcommand{\sphere}[2]{%
\scalebox{#1}{
\begin{tikzpicture}
  \shade[ball color = #2!40, opacity = 0.2] (0,0) circle (0.2cm);
  \draw (0,0) circle (0.2cm);
\end{tikzpicture}
}
}

\newcommand{\cylinder}[1]{%
\scalebox{#1}{
\begin{tikzpicture}
\node[cylinder, draw, shape border rotate = 90,minimum size = 0.4cm] (c) at (0,0) {};
\end{tikzpicture}
}
}

\definecolor{mydarkred}{rgb}{0.6,0,0}
\definecolor{myblue}{HTML}{268BD2}
\definecolor{mygreen}{HTML}{658354}

%% file: includes/preamble.tex
\IfFileExists{natbib.sty}{%
\usepackage{natbib}%
}{}

\IfFileExists{appendix.sty}{%
\usepackage{appendix}%
}{}

\DisemulatePackage{setspace}
\usepackage[onehalfspacing]{setspace}

\IfFileExists{enumitem.sty}{%
\usepackage[loadonly]{enumitem}[2007/06/30]%
\newlist{hanglist}{enumerate}{1}% 
\setlist[hanglist]{label=\arabic*.}%
\setlist[hanglist,1]{leftmargin=0pt}%
}{%
}

\IfFileExists{mathpartir.sty}{%
\usepackage{mathpartir}%
}{}

\IfFileExists{listings.sty}{%
\usepackage{listings}%
}{}

\hypersetup{pdfborder = 0 0 0}

\usepackage{boxedminipage}

\IfFileExists{booktabs.sty}{%
\usepackage{booktabs}%
}{}

\usepackage{ifpdf}

\usepackage{lmodern}
\usepackage[LY1]{fontenc}

\IfFileExists{tgpagella.sty}{%
\usepackage{tgpagella}%
}{}

\ifpdf
\usepackage{graphicx}
\else
\usepackage{graphicx}
\fi

\usepackage{makeidx}
\makeindex

{\theoremstyle{plain}

}
{\theoremstyle{remark}

}

\newtheoremstyle{customsty1}
{3pt}%
{3pt}%
{}% --- body font
{}% --- indent amount
{\bfseries}% --- Theorem head font
{:}% --- Punctuation after head
{.5em}% --- space after head
{}% --- theorem head spec (can be left empty, meaning 'normal')

{\theoremstyle{customsty1} 
}

\makechapterstyle{deposit}{%

}

\makepagestyle{deposit}
 
\makeatletter

\makeevenfoot{deposit}{}{}{}
\makeoddfoot{deposit}{}{}{}
\makeevenhead{deposit}{\thepage}{}{}
\makeoddhead{deposit}{}{}{\thepage}
\makeatother

\copypagestyle{chapter}{plain}
\makeoddfoot{chapter}{}{}{}
\makeevenhead{chapter}{\thepage}{}{}
\makeoddhead{chapter}{}{}{\thepage}
\makeatletter
\newenvironment{wb-bib}[1]{%
  \chapter*{references}
\ifnobibintoc\else 
\phantomsection 
\addcontentsline{toc}{chapter}{References} 
\fi 
\prebibhook
  \begin{bibitemlist}{#1}}{\end{bibitemlist}\postbibhook}

\AtBeginDocument{%
  \@ifpackageloaded{natbib}{% natbib is loaded
    \addtodef{\endthebibliography}{}{\vskip-\lastskip\postbibhook}
    \@ifpackagewith{natbib}{sectionbib}{% with sectionbib option
      \renewcommand{\bibsection}{\@memb@bsec}}%
      {\renewcommand{\bibsection}{\@memb@bchap}}}%
  {}
  \@ifpackagewith{chapterbib}{sectionbib}{%
    \renewcommand{\sectionbib}[2]{}
    \renewcommand{\bibsection}{\@memb@bsec}}{}
}
\makeatother

%% file: includes/defs.tex
\makeatletter
\newcommand{\wb@episource}{}

\makeatother

\IfFileExists{svn-multi.sty}{%
\usepackage{svn-multi}%
\newcommand{\vcinfo}{}%
}{%
\newcommand{\svnidlong}[4]{}%
\newcommand{\vcinfo}{}%
}

%% file: includes/thesisdefs.tex
\makeatletter

\newcommand\@thesistitlemedskip{0.2in}
\newcommand\@thesistitlebigskip{0.6in}
\newcommand{\degree}[1]{\gdef\@degree{#1}}
\newcommand{\project}{\gdef\@doctype{A masters project report}}
\newcommand{\prelim}{\gdef\@doctype{A preliminary report}}
\newcommand{\thesis}{\gdef\@doctype{A thesis}}
\newcommand{\dissertation}{\gdef\@doctype{A dissertation}}
\newcommand{\department}[1]{\gdef\@department{(#1)}}

\newenvironment{titlepage}
 {\@restonecolfalse\if@twocolumn\@restonecoltrue\onecolumn
  \else \newpage \fi \thispagestyle{empty}
}{\if@restonecol\twocolumn \else \newpage \fi}

\gdef\@degree{Doctor of Philosophy}    %Default is PhD
\gdef\@doctype{A dissertation}         %Default is dissertation

\gdef\@department{(Electrical Engineering)} % Default is Electical Engineering
\gdef\@defensedate{07/19/2023}% Default is a long time from now.
\gdef\@committee{
  Jerry Xiaojin Zhu, Professor, Computer Sciences \\
  Yong Jae Lee, Associate Professor, Computer Sciences\\
  Yingyu Liang, Associate Professor, Computer Sciences\\
  Yiqiao Zhong, Assistant Professor, Statistics\\
  Sharon Yixuan Li (Advisor), Assistant Professor, Computer Sciences
  }

\renewcommand{\maketitle}{%
  \begin{titlepage}
    \def\thanks##1{\typeout{Warning: `thanks' deleted from thesis titlepage.}}
    \let\footnotesize\small \let\footnoterule\relax \setcounter{page}{1}
    \begin{center}
      {\textbf{\expandafter\expandafter{\@title}}} \\
      [\@thesistitlemedskip] \vspace{0.5cm}
       by \\[\@thesistitlemedskip]
      \@author \\[\@thesistitlebigskip]
      \@doctype\ submitted in partial fulfillment of \\
      the requirements for the degree of\\[\@thesistitlebigskip]
      \@degree \\[\@thesistitlemedskip]
      \@department \\[\@thesistitlemedskip] \vspace{0.5cm}
      at the \\[\@thesistitlemedskip]
      UNIVERSITY OF WISCONSIN--MADISON\\[\@thesistitlemedskip]
      \@date
    \end{center}
    \hspace*{-0.7in}Date of final oral examination: \@defensedate \\
    [\@thesistitlemedskip]
    \hspace*{-0.7in}The dissertation is approved  by the following members of the 
     Oral Committee:\\
    \@committee
  \end{titlepage}
  \setcounter{footnote}{0}
  \setcounter{page}{1} %title page is NOT counted
  \let\thanks\relax
  \let\maketitle\relax \let\degree\relax \let\project\relax \let\prelim\relax
  \let\department\relax
  \gdef\@thanks{}\gdef\@degree{}\gdef\@doctype{}
  \gdef\@department{}
}

\def\abstract{
  \chapter*{Abstract}
  \addcontentsline{toc}{chapter}{Abstract}
  \relax\markboth{Abstract}{Abstract}}

\def\advisortitle#1{\gdef\@advisortitle{#1}}
\def\advisorname#1{\gdef\@advisorname{#1}}
\gdef\@advisortitle{Professor}
\gdef\@advisorname{Cheer E.\ Place}

\def\umiabstract{
             \thispagestyle{empty}
                  \addtocounter{page}{-1}
                \begin{center}
                  {\textbf{\expandafter\uppercase\expandafter{\@title}}}\\
                  \vspace{12pt}
                  \@author \\
                  \vspace{12pt}
                  Under the supervision of \@advisortitle\ \@advisorname\\
                  At the University of Wisconsin-Madison
                \end{center}
}

\def\endumiabstract{\vfill \hfill\@advisorname\par\newpage}

\def\verbatimfile#1{\begingroup \singlespace
                    \@verbatim \frenchspacing \@vobeyspaces
                    \input#1 \endgroup
}

\def\copyrightpage{
  \newpage
  \thispagestyle{empty}    % No page number
  \addtocounter{page}{-1}
  \chapter*{}            % Required for \vfill to work
  \begin{center}
   \vfill
   \copyright\ Copyright by \@author\ \@date\\
   All Rights Reserved
  \end{center}}

\def\dedication{
  \newpage
  \null\vfil
  \begin{center}}
\def\enddedication{\end{center}\par\vfil\newpage}

\newcount\@testday
\def\today{\@testday=\day
  \ifnum\@testday>30 \advance\@testday by -30
  \else\ifnum\@testday>20 \advance\@testday by -20
  \fi\fi
  \number\day\ \
  \ifcase\month\or
    January \or February \or March \or April \or May \or June \or
    July \or August \or September \or October \or November \or December
    \fi\ \number\year
}

\newtheorem{theorem}{Theorem}[chapter]
\newtheorem{proposition}[theorem]{Proposition}
\newtheorem{assumption}[theorem]{Assumption}

\newtheorem{definition}[theorem]{Definition}

\newtheorem{lemma}[theorem]{Lemma}

\newcommand{\hiddensubsection}[1]{
    \stepcounter{subsection}
    \subsection*{\arabic{chapter}.\arabic{section}.\arabic{subsection}\hspace{1em}{#1}}
}
\def\@bibchaptitle{Bibliography}
\def\altbibtitle{\def\@bibchaptitle{Bibliography}}
\def\thebibliography#1{
  \global\@bibpresenttrue
  \chapter*{\@bibchaptitle\markboth{\@bibchaptitle}{\@bibchaptitle}}
  \addcontentsline{toc}{chapter}{\@bibchaptitle}
  \vspace{0.375in}    % added to match 4 line requirement
  \interlinepenalty=10000 % added to prevent breaking of bib entries
  \singlespace\list
  {[\arabic{enumi}]}{\settowidth\labelwidth{[#1]}\leftmargin\labelwidth
    \advance\leftmargin\labelsep \usecounter{enumi}}
  \def\newblock{\hskip .11em plus .33em minus -.07em}
  \sloppy
  \sfcode`\.=1000\relax}
\let\endthebibliography=\endlist

\makeatother

%% file: frontmatter/frontmatter.tex
\svnidlong{$LastChangedBy$}{$LastChangedRevision$}{$LastChangedDate$}{$HeadURL: http://freevariable.com/dissertation/branches/diss-template/frontmatter/frontmatter.tex $}
\vcinfo{}

%%% SOME OF THIS CODE IS ADAPTED FROM THE VENERABLE withesis.cls

% COPYRIGHT PAGE
%  - To include a copyright page use \copyrightpage
\copyrightpage

% DEDICATION
% \begin{dedication}
% 	\emph{Please insert your dedication here.}
% \end{dedication}

%% BEGIN PAGESTYLE

%%% You can pick a pagestyle if you want; see the memoir class
%%% documentation for more info.  The default ``deposit'' option meets
%%% the UW thesis typesetting requirements but is probably
%%% unsatisfactory for making a version of your dissertation that
%%% won't be deposited to the graduate school (e.g. for web or a nice
%%% printed copy)

% \chapterstyle{deposit}
\pagestyle{deposit}

% ACKNOWLEDGMENTS
\begin{acks}
\input{frontmatter/acks}
\end{acks}

% CONTENTS, TABLES, FIGURES
\renewcommand{\printtoctitle}[1]{\chapter*{#1}}
\renewcommand{\printloftitle}[1]{\chapter*{#1}}
\renewcommand{\printlottitle}[1]{\chapter*{#1}}

\renewcommand{\tocmark}{}
\renewcommand{\lofmark}{}
\renewcommand{\lotmark}{}

\renewcommand{\tocheadstart}{}
\renewcommand{\lofheadstart}{}
\renewcommand{\lotheadstart}{}

\renewcommand{\aftertoctitle}{}
\renewcommand{\afterloftitle}{}
\renewcommand{\afterlottitle}{}

\renewcommand{\cftchapterfont}{\normalfont} 
\renewcommand{\cftsectionfont}{\itshape} 
\renewcommand{\cftchapterpagefont}{\normalfont} 
\renewcommand{\cftchapterpresnum}{\bfseries} 
\renewcommand{\cftchapterleader}{} 
\renewcommand{\cftsectionleader}{} 
\renewcommand{\cftchapterafterpnum}{\cftparfillskip} 
\renewcommand{\cftsectionafterpnum}{\cftparfillskip} 

% \captionnamefont{\small\sffamily} 
% \captiontitlefont{\small\sffamily} 

% \renewcommand{\contentsname}{contents}
% \renewcommand{\listfigurename}{list of figures}
% \renewcommand{\listtablename}{list of tables}

\tableofcontents

\clearpage
\listoftables

\clearpage
\listoffigures

\clearpage
% NOMENCLATURE
% \begin{conventions}
% % \begin{description}
% % \item{\makebox[0.75in][l]{term}
% %        \parbox[t]{5in}{definition\\}}
% % \end{description}
% \input{conventions}
% \end{conventions}

% \advisorname{Sharon Li}
% \advisortitle{Assistant Professor}
% % ABSTRACT
% \begin{umiabstract}
%   \input{frontmatter/abstract}
% \end{umiabstract}

\begin{abstract}

\input{frontmatter/abstract}
\end{abstract}

\clearpage\pagenumbering{arabic}

%%% END STUFF TAKEN FROM WITHESIS EXAMPLE FILE

%% file: frontmatter/acks.tex
I would like to express my deepest gratitude to everyone who has contributed to the completion of this doctoral thesis. This research journey has been a transformative and enlightening experience, and I am fortunate to have received guidance, support, and encouragement from all individuals and institutions.

Firstly, I owe a huge thanks to my advisor, Professor Sharon Li. Her knowledge, guidance, and dedication to high standards have really shaped my work. 
 She encouraged me to venture into the unknown, confront challenges, and dig deeper intellectually. When I faced significant setbacks, like when my papers were repeatedly rejected, she was there and continued to assure me that my work was valuable and helped me regain my confidence. Her support during these tough times was invaluable and helped me get back on track.

 I'm also hugely grateful to my dissertation committee, Prof. Jerry Zhu, Prof. Yong Jae Lee, Prof. Yiqiao Zhong, and Prof. Yingyu Liang. Their expertise, constructive feedback, and insights have greatly improved my thesis.  
I'm also deeply grateful to my collaborators and labmates, Chuan Guo, Zhenmei Shi, Yifei Ming, Xuefeng Du, and Haoyue Bai, who have supported me tremendously. They enriched my thought process and created a supportive research environment.
 % Big thanks to the University of Wisconsin-Madison for all the resources and facilities that I needed for my research. 

Finally, my heartfelt appreciation goes to my family especially to my mom and my wife for their understanding and constant cheering. Their patience, support, and unwavering belief in me have been the foundation of my achievements. 

To everyone who helped with this thesis, whether I mentioned you or not, your support and encouragement have been invaluable. I am deeply grateful for all of you being part of my life.

%% file: frontmatter/abstract.tex
This thesis makes considerable contributions to the realm of machine learning, specifically in the context of open-world scenarios where systems face previously unseen data and contexts. Traditional machine learning models are usually trained and tested within a fixed and known set of classes, a condition known as the closed-world setting. While this assumption works in controlled environments, it falls short in real-world applications where new classes or categories of data can emerge dynamically and unexpectedly.

To address this, our research investigates two intertwined steps essential for open-world machine learning: \textit{Out-of-distribution (OOD) Detection} and \textit{Open-world Representation Learning (ORL)}. OOD detection focuses on identifying instances from unknown classes that fall outside the model's training distribution. This process reduces the risk of making overly confident, erroneous predictions about unfamiliar inputs. Moving beyond OOD detection, ORL extends the capabilities of the model to not only detect unknown instances but also learn from and incorporate knowledge about these new classes. 

In the realm of OOD detection, our work first introduces pioneering methodologies, namely ReACT and DICE, that can effectively differentiate samples from known and unknown classes. ReACT truncates abnormally high unit activations during test time to reduce the model's overconfidence in the output, while DICE leverages a model's most contributing weights by sparsification for OOD detection. Moreover, we present a distance-based OOD detection method with the introduction of a non-parametric approach using K-nearest neighbor (KNN) distance, with a paradigm shift in eschewing rigid distributional assumptions about the underlying feature space.

Moving beyond OOD detection, ORL involves deeper exploration into learning the unknown, answering crucial research questions about the interplay between known and unknown classes, and the role of label information in shaping representations. Through rigorous investigations, we aim to illuminate how knowledge about known classes can help uncover previously unseen classes and how label information impacts the learning and representation of both known and novel classes. This exploration inspires the development of a comprehensive algorithmic framework (OpenCon) for ORL, underpinned by a theoretical interpretation from the Expectation-maximization perspective.

By delving into these research problems of open-world learning, this thesis paves the way for building machine learning models that are not only performant but also reliable in the face of the evolving complexities of the real world.

%% file: chapters/intro.tex
\chapter{Introduction}
\label{sec:intro}

Advances in machine learning have revolutionized numerous domains, including image classification~\citep{deng2009imagenet,he2016deep}, object detection~\citep{girshick2015fast,sun2017faster}, segmentation~\citep{chen2017deeplab}, video processing~\citep{kahou2016emonets}, and audio recognition~\citep{purwins2019deep}, driving innovation and transforming the way we interact with technology. Noticeably, the vast majority of learning algorithms have been driven by the \textbf{closed-world} setting. For example, the face recognition systems of border control assume the inputs are all face images, in which case the traditional methods are sufficient to satisfy the industrial requirements~\citep{boulkenafet2015face,li2016original}. These applications assume that the classes are stationary and unchanged. This assumption, however, rarely holds for models deployed in the wild.

% One important characteristic of an \textbf{open-world} is that the model will naturally encounter novel classes. Considering a realistic scenario, where a machine learning model for recognizing products in e-commerce may encounter brand-new products together with old products. Similarly, an autonomous driving model can run into novel objects on the road, in addition to known ones. As the demand for intelligent systems grows, the need for machine learning algorithms to handle open-world scenarios becomes increasingly paramount. 
 
 One important characteristic of the \textbf{open-world} is that the intelligent system will encounter new contexts and data that were not taught to the algorithms during training, therefore requiring safe handling and adaption to the novel data. Traditional ML algorithms are typically unreliable to such out-of-distribution (OOD) data and can fail catastrophically~\citep{nguyen2015deep} (e.g., blindly predicting an OOD sample from an unknown class into a known class with high confidence). Preventing disastrous and overconfident outcomes for safe decision-making is thus a critical problem within trustworthy and open-world machine learning. This area already has numerous applications in autonomous driving, cloud computing, voice-assisted smartphones, smart logistics, healthcare, insurance, e-commerce systems, and many other industries. For example, a medical machine learning system may encounter a new disease it has never seen~\citep{sun2023lood}; an e-commerce classifier may come across brand-new products in the market~\citep{ecomm}; an autonomous driving model can run into an unknown object on the road~\citep{templeton2020tesla}. As the demand for intelligent systems grows, the need for machine learning algorithms to handle open-world scenarios becomes increasingly paramount. Open-world machine learning is an upcoming frontier and has gained increasing interest within the computer science community in the last few years. 

\begin{figure}[htb]
    \centering
    \includegraphics[width=0.9\linewidth]{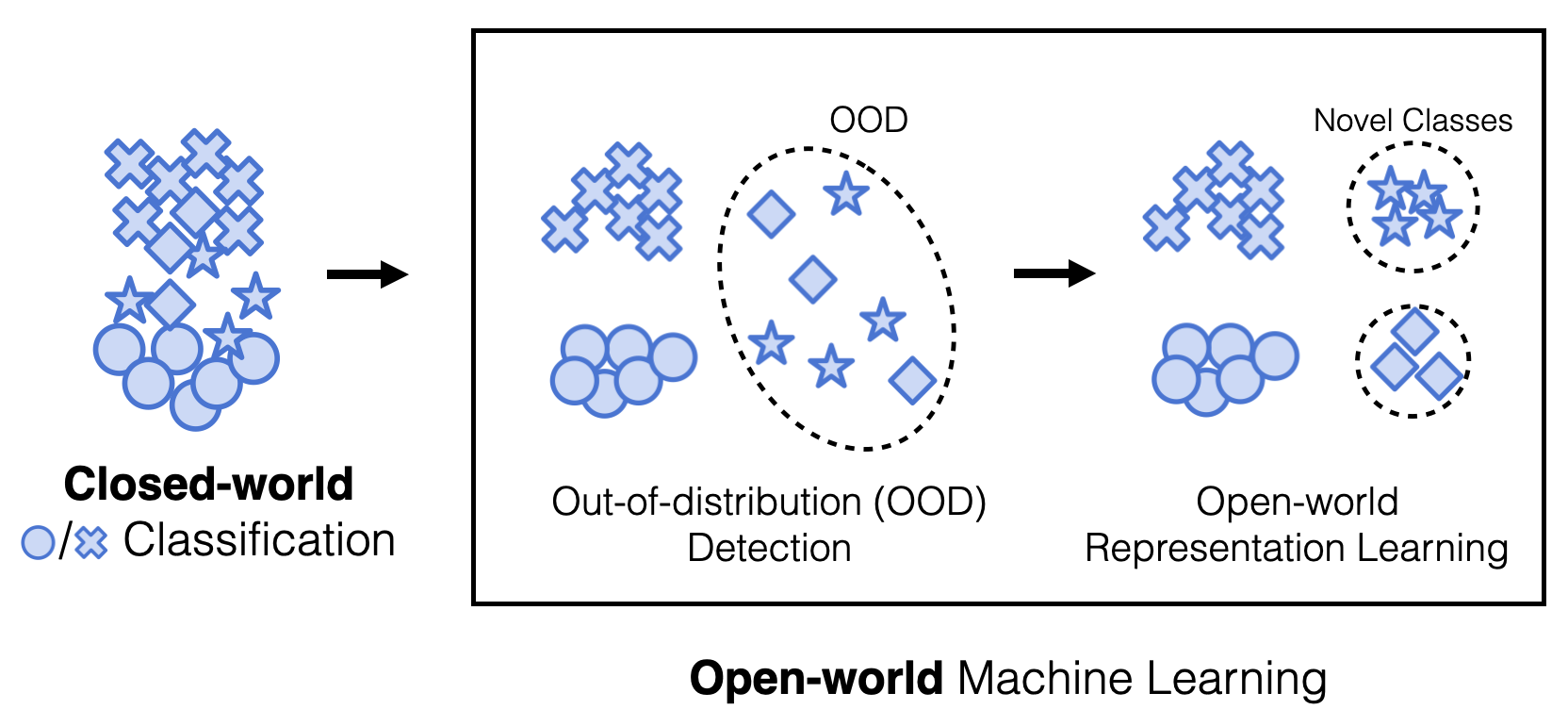}
    \caption[Illustration of two sub-problems of open-world machine learning and their relationship.]{The thesis focuses on the open-world machine learning problem,  composes of two sub-problems: \textit{Out-of-distribution (OOD) Detection} and \textit{Open-world Representation Learning (ORL)}. The figure uses a toy example with a binary classification problem of the circle and the cross. OOD detection aims to detect OOD samples (star and square) that are not in the training categories (circle and cross). ORL aims to learn distinguishable representations for all classes including known classes and OOD classes. }
    \label{fig:teaser}
\end{figure}

 Within this field, two prominent research topics emerge as the central areas of investigation: \textit{Out-of-distribution (OOD) Detection} and \textit{Open-world Representation Learning (ORL)}. The relationship between these research topics is illustrated in Figure~\ref{fig:teaser}. At a high level, OOD detection can be seen as the initial step in extending the closed-world classification problem to the open world. A reliable machine learning model should not only accurately classify in-distribution (ID) samples but also possess the capability to identify samples that lie outside the known distribution. Moving beyond OOD detection, the ORL problem further requires models to learn the hidden classes within OOD samples, in addition to the known classes. We delve deeper into this discussion in the subsequent paragraphs.

  \section{Out-of-distribution (OOD) Detection} The research revolves around effectively identifying instances from unknown classes or categories. In traditional machine learning, algorithms assume a closed-world setting with a fixed and known set of classes during training and inference~\citep{he2016deep,huang2017densely}. However, in open-world scenarios, where new classes can emerge dynamically, existing models often struggle to identify instances from previously unseen categories accurately~\citep{nguyen2015deep}. Developing robust techniques for Out-of-distribution Detection is crucial for reliably distinguishing between known and unknown classes, enabling more reliable machine learning systems. 

  A driving idea behind OOD detection is that the model should be much less confident about samples outside of its training distribution. However, modern neural networks can produce overconfident predictions on OOD inputs. This observation goes back to the early work by ~\cite{nguyen2015deep}. This phenomenon renders the separation of in-distribution (ID) and OOD data a non-trivial task, which attracts growing research attention in several thriving directions: 
  
(a) One line of work attempted to perform OOD detection by devising scoring functions, including OpenMax score~\citep{openworld}, maximum softmax probability~\citep{Kevin}, ODIN score~\citep{liang2018enhancing}, deep ensembles~\citep{lakshminarayanan2017simple}, Mahalanobis
distance-based score~\citep{lee2018simple}, energy score~\citep{liu2020energy,lin2021mood, wang2021canmulti, morteza2022provable}, gradient-based score~\citep{huang2021importance} and ViM score~\citep{wang2022vim}. 
% (2) Another promising line of work addressed OOD detection by training-time regularization~\cite{lee2017training, bevandic2018discriminative,  malinin2018predictive, hendrycks2018deep,  geifman2019selectivenet, hein2019relu, meinke2019towards, mohseni2020self, liu2020energy, jeong2020ood, van2020uncertainty, yang2021semantic, chen2021atom, hongxin2022logitnorm, ming2022posterior, katzsamuels2022training}. 
  On this line, this doctoral thesis includes two representative works (ReACT~\citep{sun2021react} and DICE~\citep{sun2022dice}) which push the boundaries of novel methodologies in detecting OOD data.

  % unveil the answer to the long-standing question of ``why OOD triggers overconfident predictions?'' which has remained an open problem for several years. 

  Specifically, ReACT~\citep{sun2021react} is proposed as a simple yet effective solution for reducing model overconfidence in OOD data. The key idea behind ReACT is to truncate the abnormally high unit activations during test-time OOD detection. Empirical and theoretical insights are provided to characterize and explain how ReACT improves OOD uncertainty estimation. By rectifying the activations, the outsized contribution of hidden units on OOD output can be attenuated, resulting in stronger separability from ID data. 

  The success of ReACT has led to a significant follow-up work called DICE~\citep{sun2022dice} which delves deeper into the detection of OOD data by investigating the influence of weights. DICE leverages the observation that a model's prediction for an ID class depends on only a subset of important units and their corresponding weights. Building on this observation, DICE introduces a novel idea of ranking weights based on their measure of contribution and selectively using the most contributing weights to derive the output for OOD detection. 

(b) Another avenue of exploration in OOD detection involves the adoption of distance-based approaches, which operate under the assumption that the test OOD samples are relatively far away from the ID data. In particular, {CSI}~\citep{tack2020csi} investigate the type of data augmentations that are particularly beneficial for OOD detection. Other works~\citep{winkens2020contrastive,2021ssd} verify the effectiveness of applying the off-the-shelf multi-view contrastive losses such as {SimCLR}~\citep{chen2020simclr} and {SupCon}~\citep{khosla2020supcon} for OOD detection. 
   
   Prior works commonly make a strong distributional assumption, assuming the underlying feature space follows a class-conditional Gaussian distribution. Unlike previous methods, this thesis introduces a non-parametric approach, specifically utilizing K-nearest neighbor (KNN) distance~\citep{sun2022knnood} and not relying on any specific distributional assumption about the underlying feature space. This crucial paradigm shift provides greater flexibility and generality in detecting OOD samples, as it does not impose rigid distributional assumptions. \\

  By advancing the understanding and techniques in OOD detection, this research contributes to the development of more reliable and robust machine learning models, paving the way for applications in open-world scenarios where the presence of unknown classes is a crucial challenge to overcome.

 \section{Open-world Representation Learning (ORL)}
 Beyond detecting the OOD data from unknown classes, an extended line of research lies in the ability to learn and incorporate knowledge in these unknown classes.  Concretely, the model has access to the training dataset with both labeled and unlabeled data. The labeled dataset contains samples that belong to a set of known classes, while the unlabeled dataset has a mixture of samples from both the known and novel classes. In practice, such unlabeled in-the-wild data can be collected almost for free upon deploying a model in the open world, and thus is available in abundance.  This gives rise to the pressing demand for the advancement of ORL algorithms, enabling more robust and adaptable open-world machine learning systems.

 The learning setting that considers both labeled and unlabeled data with a mixture of known and novel classes is first proposed in~\cite{cao2022openworld} and inspires a proliferation of follow-up works~\citep{pu2023dynamic,zhang2022promptcal,rizve2022openldn,vaze22gcd} advancing empirical success where most works put emphasis on learning high-quality representations~\citep{vaze22gcd,pu2023dynamic,zhang2022promptcal}.
The thesis further \emph{advances theoretical understanding} by answering two unresolved research questions~\citep{sun2023nscl,sun2023sorl} as well as providing effective empirical solutions~\citep{sun2023opencon,sun2023nscl,sun2023sorl}. 

 % An important part of the ORL's goal is called 
 % Novel Class Discovery (NCD)~\citep{Han2019dtc}, which aims to cluster novel classes by way of utilizing knowledge from the labeled data (of known classes).  Key to NCD  is harnessing the power of labeled data for possible knowledge sharing and transfer to the unlabeled data~\citep{hsu2017kcl, Han2019dtc, hsu2019mcl, zhong2021openmix, zhao2020rankstat, yang2022divide, sun2023opencon}. 
 
The first research question we aim to address is ``\textit{when and how does known class help discover unknown ones?}'' Recognizing the potential interplay between known and unknown classes is essential for effective open-world representation learning. By investigating this question, we seek to uncover the unsolved mystery in Novel Class Discovery (NCD)~\citep{hsu2017kcl, Han2019dtc, hsu2019mcl, zhong2021openmix, zhao2020rankstat, yang2022divide, sun2023opencon} by which knowledge about known classes can facilitate the discovery and recognition of previously unseen classes. Understanding these dynamics is crucial for designing algorithms that can leverage the relationships and similarities between known and unknown classes, leading to enhanced representation learning in open-world scenarios.

The second research question we explore is ``\textit{what is the role of label information in shaping representations for both known and novel classes?}'' In open-world representation learning, label information plays a vital role in guiding the formation of effective representations. 
By examining this question, we aim to shed light on how label information influences the learning process and the resulting representations for both known and novel classes. Investigating the impact of label information on the representation space can provide valuable insights into the change of the representations' discriminative power in known classes and how it generalizes to the novel class.

% % Different from self-supervised learning~\citep{van2018cpc,chen2020simclr,caron2020swav,he2019moco}, open-world
% % representation learning is required to effectively capture the underlying structures and characteristics of both known and unknown classes. This monograph seeks to offer a unified and comprehensive algorithm along with the theoretical underpinnings for two research questions: ``\textit{when and how does known class help discover unknown ones?}'' and ``\textit{what
% % is the role of the label information in shaping representations for both known and novel classes?}''.

Finally, the thesis tackles the empirical challenges in ORL.  
Different from self-supervised representation learning~\citep{van2018cpc,chen2020simclr,caron2020swav,he2019moco}, open-world representation learning is a distinct endeavor that goes beyond simply leveraging unlabeled data to uncover meaningful representations. It encompasses the challenging task of effectively capturing the underlying structures and characteristics of both known and unknown classes. In this monograph, we provide a unified and comprehensive algorithmic framework~\citep{sun2023opencon} accompanied by a theoretical interpretation from the Expectation-maximization (EM) perspective, tackling unique challenges within open-world representation learning. \\

Through rigorous theoretical analysis and algorithmic development, this monograph strives to address these research questions, contributing to the advancement of open-world representation learning. By elucidating the relationships between known and unknown classes and understanding the role of label information, we aim to unlock new avenues for representation learning in open-world settings, empowering machine learning systems to effectively capture the complexities and nuances of diverse and evolving real-world environments.

% allows harnessing the power of the labeled data for possible knowledge
% sharing and transfer to unlabeled data, and from known classes to novel classes

\section{Contribution and Thesis Outline}
\begin{figure}[htb]
    \centering
    \includegraphics[width=0.99\linewidth]{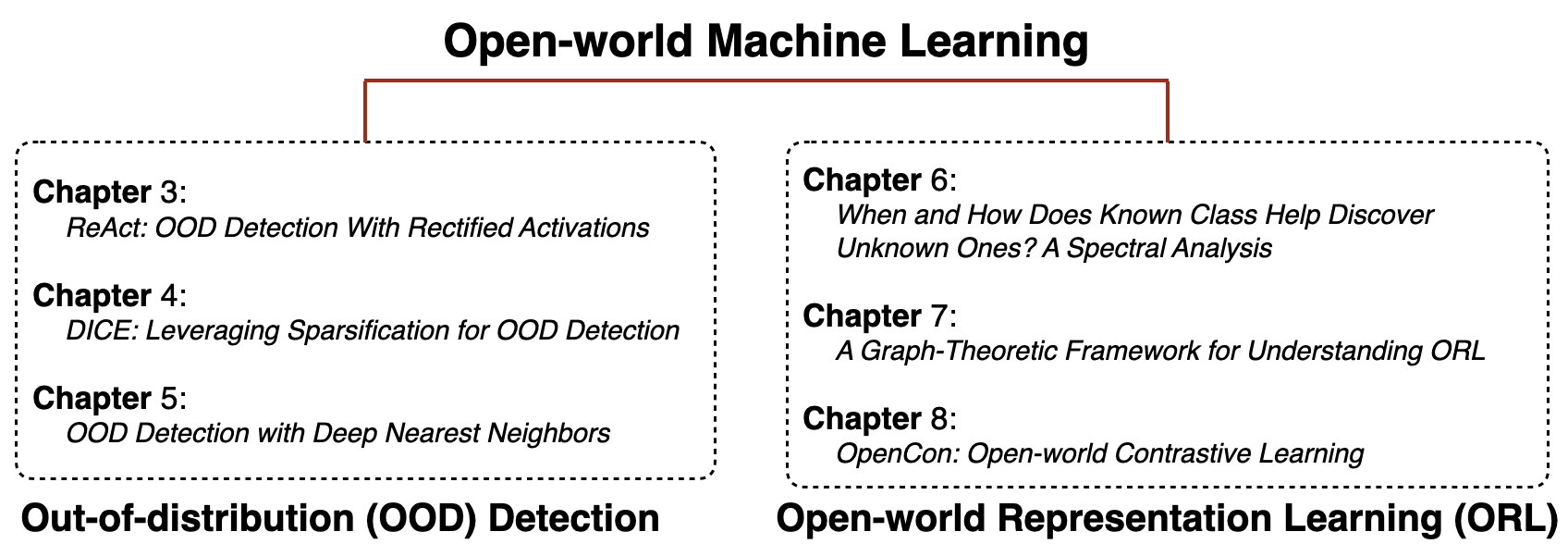}
    \caption[The outline of the thesis.]{Thesis outline including \textit{Out-of-Distribution detection} and \textit{Open-world Representation Learning}. }
    \label{fig:outline}
\end{figure}
This doctoral thesis makes significant contributions to the field of open-world machine learning, with a specific focus on two key research topics: \textit{Out-of-Distribution detection} and \textit{Open-world Representation Learning}, with outline illustrated in Figure~\ref{fig:outline}. The central contribution of this thesis lies in the development of novel methodologies and theoretical insights to address these challenges.

One of the primary contributions of this thesis lies in the development of several competitive algorithms for OOD detection, addressing the problem of model overconfidence in OOD inputs. This thesis unveils the mechanisms underlying overconfident predictions on OOD data, providing insights into why such phenomena occur. The introduction of the ReACT~\citep{sun2021react}, DICE~\citep{sun2022dice} and KNN~\citep{sun2022knnood} technique further offers practical post hoc solutions to reduce model overconfidence on OOD data. The empirical and theoretical insights gained from these algorithms shed light on the improvement of OOD detection performance and establish a solid ground for future work in this research topic.

Furthermore, this thesis highlights the importance of open-world representation learning as a crucial second step beyond the OOD detection process. Open-world representation learning is required to effectively capture the underlying structures and characteristics of both known and unknown classes, enabling the development of robust and adaptable machine learning systems in open-world scenarios, where new classes can emerge dynamically. This thesis presents novel methodologies ~\citep{sun2023opencon,sun2023sorl,sun2023nscl} for open-world representation learning, producing a compact representation space for both known and novel classes. The monograph also establishes the theoretical foundations of the ORL problem by elucidating the relationships between known and unknown classes and understanding
the role of label information, empowering machine learning system developers to effectively capture the complexities and nuances of diverse and evolving real-world environments. \\

In terms of the thesis outline:

\textbf{Chapter~\ref{sec:background}} offers a meticulous description of the problem setup, encompassing a comprehensive literature review that surveys the current body of research on out-of-distribution (OOD) detection and open-world representation learning. \\

Chapter~\ref{sec:react}, Chapter~\ref{sec:dice}, and Chapter~\ref{sec:knn} investigate the out-of-distribution detection problem:

\textbf{Chapter~\ref{sec:react}} presents the theoretical foundations and insights into the mechanisms of ``why model gets overconfidence in OOD data''. This chapter also introduces the ReACT technique with the key idea that truncates the abnormally high unit activations for test-time OOD detection. The content of this chapter is primarily based on ~\cite{sun2021react}.

\textbf{Chapter~\ref{sec:dice}} introduces DICE, which extended the research scope beyond the unit activations and investigated the influence of weight in OOD detection. DICE leverages the observation that a model’s prediction for an ID class depends on only a subset of important units (and corresponding weights). The idea is to rank weights based on the measure of contribution, and selectively use the most contributing weights to derive the output for OOD
detection. The content of this chapter is primarily based on ~\cite{sun2022dice}.

\textbf{Chapter~\ref{sec:knn}} describes a distance-based OOD detection by using $k$-th nearest neighbor distances, which operates under the assumption that the test OOD samples are relatively far away from the ID data. Importantly, it brings a crucial paradigm shift from a parametric to a non-parametric distance-based approach for OOD detection.
The content of this chapter is primarily based on ~\cite{sun2022knnood}. \\

Chapter~\ref{sec:nscl}, Chapter~\ref{sec:sorl}, and Chapter~\ref{sec:opencon} fall under the umbrella of open-world representation learning:

\textbf{Chapter~\ref{sec:nscl}} aims to answer an underexplored  
 research question ``\textit{when and how does known class help discover unknown ones?}'' Tailored to the problem, we introduce a graph-theoretic representation that can be learned by a novel NCD Spectral Contrastive
Loss (NSCL), which is appealing for practical usage while enjoying theoretical guarantees.
The content of this chapter is primarily based on ~\cite{sun2023nscl}.

\textbf{Chapter~\ref{sec:sorl}} investigates the second research question in open-world representation learning: ``\textit{what is the role of label information in shaping representations for both known and novel classes?}''  Our graph-theoretic framework (SORL) illuminates practical
algorithms and shed light on how label information influences the learning process and the resulting representations for both known and novel classes. 
The content of this chapter is primarily based on ~\cite{sun2023sorl}.

\textbf{Chapter~\ref{sec:opencon}} introduces OpenCon, a pioneering training framework for open-world representation learning. It establishes a contrastive loss framework that tackles unique challenges in the ORL problem: (a) the lack of clear separation between known vs. novel data in unlabeled data, and (b) the lack of supervision for data in novel classes.
The content of this chapter is primarily based on ~\cite{sun2023opencon}. \\

Finally, \textbf{Chapter~\ref{sec:conclusion}} concludes the thesis by summarizing the contributions, discussing the implications of the research findings, and outlining potential directions for future work. Through these contributions and the systematic exploration of OOD detection and open-world representation learning, this thesis advances the understanding and state-of-the-art in the field, providing valuable insights and practical methodologies to enhance the reliability and adaptability of machine learning models in open-world scenarios.

%% file: chapters/background.tex
\chapter{Background}
\label{sec:background}

\section{Problem Statement}
\label{sec:prob}

In this section, we introduce the problem setup of \textit{Out-of-distribution Detection} and \textit{Open-world Representation Learning}. We delve into the investigation of the OOD Detection problem in Chapter~\ref{sec:react}, Chapter~\ref{sec:dice}, and Chapter~\ref{sec:knn}. Furthermore, we explore the ORL problem in Chapter~\ref{sec:nscl}, Chapter~\ref{sec:sorl}, and Chapter~\ref{sec:opencon}. Note that Chapter~\ref{sec:nscl} specifically focuses on a sub-problem of ORL known as \textit{Novel Class Discovery} (NCD), which we will elaborate on in detail within Chapter~\ref{sec:nscl}.

\hiddensubsection{Out-of-distribution Detection} 

In OOD detection, we consider supervised multi-class classification, where $\mathcal{X}$ denotes the input space and $\mathcal{Y}_l=\{1,2,...,C\}$ denotes the label space. The training set $\mathcal{D}_{in} = \{(\*x_i, y_i)\}_{i=1}^n$ is drawn \emph{i.i.d.} from the joint data distribution $\mathcal{P}_{\mathcal{X}\mathcal{Y}_l}$. Let $\mathcal{P}_\text{in}$ denote the marginal distribution on $\mathcal{X}$. Let $f: \mathcal{X} \mapsto \mathbb{R}^{C}$ be a neural network trained on samples drawn from $\mathcal{P}_{\mathcal{X}\mathcal{Y}_l}$ to output a logit vector, which is used to predict the label of an input sample. 
 
 When deploying a machine model in the real world, {a reliable classifier should not only accurately classify known in-distribution (ID) samples, but also identify as ``unknown'' any OOD input}. This can be achieved by having an OOD detector, in tandem with the classification model $f$. OOD detection can be formulated as a binary classification problem. At test time, the goal of OOD detection is to decide whether a sample $\*x \in \mathcal{X}$ is from $\mathcal{P}_\text{in}$ (ID) or not (OOD). The decision can be made via a level set estimation:
\vspace{-0.1cm}
\begin{align*}
\label{eq:threshold}
	\mathcal{S}_{\lambda}(\*x)=\begin{cases} 
      \text{ID} & S(\*x)\ge \lambda \\
      \text{OOD} & S(\*x) < \lambda 
   \end{cases},
\end{align*}
where samples with higher scores $S(\*x)$ are classified as ID and vice versa, and  $\lambda$ is the threshold. In practice, OOD is often defined by a distribution that simulates unknowns encountered during deployment time, such as samples from an irrelevant distribution {whose label set has no intersection with $\mathcal{Y}$ and therefore should not be predicted by the model}.

\hiddensubsection{Open-world Representation Learning}

In addition to detecting out-of-distribution (OOD) samples, the open-world representation learning setting places significant emphasis on the objective of not only identifying new classes within OOD samples but also learning the existing classes in the wild. To formalize this, we provide a description of the data setup and the learning goal:

\noindent \textbf{Data setup.} We consider the training dataset $\mathcal{D} = \mathcal{D}_{l} \cup \mathcal{D}_{u}$ with two parts: 
\begin{enumerate}
\vspace{-0.2cm}
    \item The labeled set $\mathcal{D}_{l}=\left\{\bx_{i} , y_{i}\right\}_{i=1}^{n}$, with $y_i \in \mathcal{Y}_l$. The label set $\mathcal{Y}_l$ is  known. 
    \item The unlabeled set $\mathcal{D}_{u}=\left\{\bx_{i}\right\}_{i=1}^{m}$, where each sample $\bx_i \in \mathcal{X}$ can come from either known or novel classes\footnote{It generalizes the problem of Novel Class Discovery (NCD)~\citep{Han2019dtc}, which assumes the unlabeled set is purely from novel classes.}. Note that we do not have  access to the labels in $\mathcal{D}_{u}$. For mathematical convenience, we denote the underlying label set as $\mathcal{Y}_\text{all}$, where $\mathcal{Y}_l \subset \mathcal{Y}_\text{all}$ implies category shift and expansion. Accordingly, the set of novel classes is $\mathcal{Y}_n = \mathcal{Y}_\text{all} \backslash \mathcal{Y}_l$, where the \textit{subscript} $n$ stands for ``\textbf{n}ovel''. The model has no knowledge of the set $\mathcal{Y}_n$ nor its size.
\end{enumerate} 

\noindent \textbf{Goal.} Under the setting, the goal is to learn distinguishable representations \emph{for both known and novel classes} simultaneously.

\begin{table}[htb]
\caption{Comparison of problem settings related to the open-world representation learning. }
\centering
\scalebox{0.8}{
\begin{tabular}{llll} \toprule
\multirow{2}{*}{\textbf{Problem Setting}} & \multirow{2}{*}{\textbf{Labeled data}} & \multicolumn{2}{c}{\textbf{Unlabeled data}} \\ \cline{3-4} 
 & & \textbf{Known classes} & \textbf{Novel classes} \\ \hline
Semi-supervised learning & Yes & Yes & No \\
Robust semi-supervised learning & Yes & Yes & Yes (Reject) \\
Supervised learning & Yes & No & No \\
Novel class discovery & Yes & No & Yes (Discover) \\
Open-world representation  learning & Yes & Yes  & Yes (Cluster) \\
\hline
\end{tabular}}
\label{tab:set_diff}
\end{table}

\noindent \textbf{Difference w.r.t. existing problem settings.}  The open-world representation learning is a practical and relatively novel problem, which differs from existing problem settings (see Table~\ref{tab:set_diff} for a summary). In particular, (a) we consider \textit{both labeled data and unlabeled data} in training, and (b) we consider a mixture of \textit{both known and novel classes} in unlabeled data. Note that our setting generalizes traditional representation learning. For example, Supervised Contrastive Learning (SupCon)~\citep{khosla2020supcon} only assumes the labeled set $\mathcal{D}_l$, without considering the unlabeled data $\mathcal{D}_u$. Weakly supervised contrastive learning~\citep{zheng2021weakcl} assumes the same classes in labeled and unlabeled data, \emph{i.e.}, $\mathcal{Y}_l = \mathcal{Y}_\text{all}$, and hence remains closed-world. Self-supervised learning~\citep{chen2020simclr} relies completely on the unlabeled set $\mathcal{D}_{u}$ and does not assume the availability of the labeled dataset. The setup is also known as open-world semi-supervised learning (OSSL) or generalized category discovery (GCD), which is introduced in ~\cite{cao2022openworld} and ~\cite{vaze22gcd} respectively. Despite the similar setup, the learning goal of ORL is different: \citet{cao2022openworld} and ~\citet{vaze22gcd} focus on classification accuracy, while ORL aims to learn high-quality embeddings.

\section{Related Work}
\label{sec:related}
This section includes an introduction to the related works in \textit{Out-of-distribution Detection} and \textit{Open-world Representation Learning}. Additionally, each chapter includes discussions on other research areas relevant to its specific topic.

\hiddensubsection{Out-of-distribution Detection}

The phenomenon of neural networks' overconfidence in out-of-distribution data is first revealed in \cite{nguyen2015deep} with the learning theory established in recent work~\citep{fang2022out}. This research area attracts growing research attention in several thriving directions.   

\noindent\textbf{Training-based OOD Detection.} One promising line of work addressed OOD detection by training-time regularization~\citep{lee2017training, bevandic2018discriminative,  malinin2018predictive, hendrycks2018deep,  geifman2019selectivenet, hein2019relu, meinke2019towards, mohseni2020self, liu2020energy, jeong2020ood, van2020uncertainty, yang2021semantic, chen2021atom, hongxin2022logitnorm, ming2022posterior, katzsamuels2022training,du2022siren,tao2023non,bai2023feed}.
For example, models are encouraged to give predictions with uniform distribution~\citep{lee2017training,hendrycks2018deep} or higher energies~\citep{liu2020energy, ming2022posterior, du2022unknown, katzsamuels2022training} for outlier data. Most regularization methods require the availability of auxiliary OOD
data. VOS~\citep{du2022towards} alleviates the need by automatically synthesizing virtual outliers that can meaningfully regularize the model's decision boundary during training. 
Although these methods have demonstrated empirical success, their practical application scope is limited due to the requirement of a re-training process. Moreover, in the case of large models such as CLIP~\citep{radford2021learning}, the re-training process can be prohibitively expensive. The thesis does not encompass a discussion on this particular research direction but instead places a greater emphasis on the inference-based method, which we will introduce in the subsequent paragraph.

% To date, \emph{none} of the prior works investigated the non-parametric nearest neighbor approach for OOD detection. Our work bridges the gap by presenting the first study exploring the efficacy of using nearest neighbor distance for OOD detection.
% We demonstrate superior performance on several OOD detection benchmarks, and we hope our work draws attention to the strong promise of the non-parametric approach.  

\noindent\textbf{Inference-based OOD Detection.} 
This category of methods operates on a pre-trained network and detects OOD samples in a post hoc manner. They offer flexibility by allowing for plug-and-play functionality with most existing models. These methods can be broadly categorized into two branches: output-based and distance-based methods:

(a) \textit{Output-based methods}. This line of work attempted to perform OOD detection by devising scoring functions based on the model's output, including OpenMax score~\citep{openworld}, maximum softmax probability~\citep{Kevin}, ODIN score~\citep{liang2018enhancing}, deep ensembles~\citep{lakshminarayanan2017simple}, energy score~\citep{liu2020energy,lin2021mood, wang2021canmulti, morteza2022provable},  gradient-based score~\citep{huang2021importance}, MOS score~\citep{huang2021mos} and ViM score~\citep{wang2022vim}.
On this line, this doctoral thesis includes two representative works -- ReAct~\citep{sun2021react} in Chapter~\ref{sec:react} and DICE~\citep{sun2022dice} in Chapter~\ref{sec:dice} which push the boundaries of novel methodologies in detecting OOD data.

(b) \textit{Distance-based methods.} Another avenue of exploration in OOD detection involves the adoption of distance-based approaches, which operate under the assumption that the test OOD samples are relatively far away from the ID data.  {CSI}~\citep{tack2020csi} investigate the type of data augmentations that are particularly beneficial for OOD detection. Other works~\citep{winkens2020contrastive,2021ssd} verify the effectiveness of applying the off-the-shelf multi-view contrastive losses such as {SimCLR}~\citep{chen2020simclr} and {SupCon}~\citep{khosla2020supcon} for OOD detection. 
% These two works both use Mahalanobis distance as the OOD score, and make strong distributional assumptions by modeling the class-conditional feature space as multivariate Gaussian distribution. 
\citet{ming2023exploit} propose a prototype-based contrastive learning framework for OOD detection, which promotes stronger ID-OOD separability than SupCon loss. Prior works commonly make a strong distributional assumption, assuming the underlying feature space follows a class-conditional Gaussian distribution. Unlike previous methods, this thesis introduces a non-parametric approach, specifically utilizing K-nearest neighbor (KNN) distance~\citep{sun2022knnood} in Chapter~\ref{sec:knn} and not relying on any specific distributional assumption about the underlying feature space. Performance-wise, this method outperforms 13 competitive rivals according to a recent survey study~\citep{yang2022openood}.

\hiddensubsection{Open-world Representation Learning}
The learning setting that considers both labeled and unlabeled data with a mixture of known and novel classes is first proposed in~\cite{cao2022openworld} and inspires a proliferation of follow-up works~\citep{pu2023dynamic,zhang2022promptcal,rizve2022openldn,vaze22gcd} advancing empirical success. Most works put emphasis on learning high-quality embeddings~\citep{vaze22gcd,pu2023dynamic,zhang2022promptcal}. 
In particular, ~\citet{vaze22gcd} employs contrastive learning with both supervised and self-supervised signals. ~\citet{pu2023dynamic} improves clustering accuracy by learning conceptional representation and ~\citet{zhang2022promptcal} applies a two-stage approach that refines the embedding by an affinity graph after a pre-training stage. Different from prior works, the thesis further \emph{advancing theoretical understanding} by answering two unresolved research questions~\citep{sun2023nscl,sun2023sorl} in Chapter~\ref{sec:nscl} and Chapter~\ref{sec:sorl} as well as providing effective learning algorithms~\citep{sun2023opencon} in Chapter~\ref{sec:opencon}.

\clearpage
\section{Notations}
\label{sec:notation}
In this section, we define common notation that is shared throughout the thesis. Specific additional notations are defined within each respective chapter. It is crucial to recognize that the notations utilized in one chapter do not carry over to others.

\renewcommand{\arraystretch}{1.2}
\begin{table}[htb]
\caption{List of common math notations. }
\centering
\begin{tabular}{|c||l|}
\hline
% \textbf{Symbols} &  \textbf{Descriptions} \\ \midrule \hline
$[n]$    &  the set $\{1, ..., n\} $   \\ \hline
 $\|\cdot\|_1$  & $l_1$ norm of a matrix or a vector       \\ \hline
$\|\cdot\|_2$  & $l_2$ norm of a matrix or a vector  \\ \hline
 $\|\cdot\|_F$  & the Frobenius norm of a matrix \\ \hline
   $\mathbf{1}_n$ &  $n$-dimensional vector with all 1 \\ \hline
 $\mathbf{0}_n$ &  $n$-dimensional vector with all 0     \\ \hline
 $\mathbf{1}_{m\times n}$  &   $m$-by-$n$ matrix with all 1  \\ \hline
 $\mathbf{0}_{m\times n}$  &   $m$-by-$n$ matrix with all 0 \\ \hline
$I_n$ &   identity matrix with shape $n\times n$  \\ \hline
 $V_{(i,j)}/V_{ij}$ & the value at $i$-th row and $j$-th column of a matrix $V$  \\ \hline
 $V_{k, (i,j)}$ &  the value at $i$-th row and $j$-th column of a matrix $V_k$ \\ \hline
 $\*v_{(i)}/\*v_{i}$   & $i$-th value for a vector $\*v$  \\ \hline
 $\*v_{k, (i)}$  & $i$-th value for a subscripted vector $\*v_{k}$ \\ \hline
$\langle \*u, \*v \rangle$  & inner-production between $\*u$ and $\*v$  \\ \hline
 $V^{\dagger}$  & Moore-Penrose inverse of matrix $V$  \\ \hline
 %    &   \\ \hline
 % $A$  &  Adjacency Matrix  \\ \hline
 % \rule{0pt}{2pt}$\Dot{A}$  & Normalized Adjacency Matrix \\ \hline
\end{tabular}
\end{table}

%% file: chapters/1_react.tex
\part{Out-of-distribution Detection}

\chapter{ReAct: OOD Detection With
Rectified Activations}
\label{sec:react}

%%%%%%%%%%%%%%%%%%%%%%%%%%%%%%%%%%%%%%%%%%%%%%%%%%%%%%%%%%%%%%%%%%%%%
%%%%%%%%%%%%%%%%%%%%%%%%%%%%%%  INTRO %%%%%%%%%%%%%%%%%%%%%%%%%%%%%%%%
%%%%%%%%%%%%%%%%%%%%%%%%%%%%%%%%%%%%%%%%%%%%%%%%%%%%%%%%%%%%%%%%%%%  
\paragraph{Publication Statement.} This chapter is joint work with Chuan Guo and Yixuan Li. The paper version of this chapter appeared in NeurIPS21~\citep{sun2021react}. 

\noindent\rule{\textwidth}{1pt}

Out-of-distribution (OOD) detection has received much attention lately due to its practical importance in enhancing the safe deployment of neural networks. One of the primary challenges is that models often produce highly confident predictions on OOD data, which undermines the driving principle in OOD detection that the model should only be confident about in-distribution samples. In this chapter, we introduce \textbf{ReAct}---a simple and effective technique for reducing model overconfidence in OOD data. ReAct is motivated by a novel analysis of internal activations of neural networks, which displays highly distinctive signature patterns for  OOD distributions. ReAct can generalize effectively to different network architectures and different OOD detection scores. We empirically demonstrate that ReAct achieves competitive detection performance on a comprehensive suite of benchmark datasets, and give theoretical explication.

\section{Introduction}
\label{sec:react_intro}

Neural networks deployed in real-world systems often encounter out-of-distribution (OOD) inputs---unknown samples that the network has not been exposed to during training. Identifying and handling these OOD inputs can be paramount in safety-critical applications such as autonomous driving~\citep{filos2020can} and health care. For example, an autonomous vehicle may fail to recognize objects on the road that do not appear in its object detection model's training set, potentially leading to a crash. This can be prevented if the system identifies the unrecognized object as OOD and warns the driver in advance.

A driving idea behind OOD detection is that the model should be much more uncertain about samples outside of its training distribution. However, \citet{nguyen2015deep} revealed that modern neural networks can produce overconfident predictions on OOD inputs. This phenomenon renders the separation of in-distribution (ID) and OOD data a non-trivial task. Indeed, much of the prior work on OOD detection focused on defining more suitable measures of OOD uncertainty~\citep{godin2020CVPR, lakshminarayanan2017simple, liang2018enhancing, lee2018simple, liu2020energy, huang2021importance}. Despite the improvement, it is arguable that continued research progress in OOD detection requires insights into the fundamental cause and mitigation of model overconfidence on OOD data. %

\begin{figure}[t]
	\begin{center}
		\includegraphics[width=\linewidth]{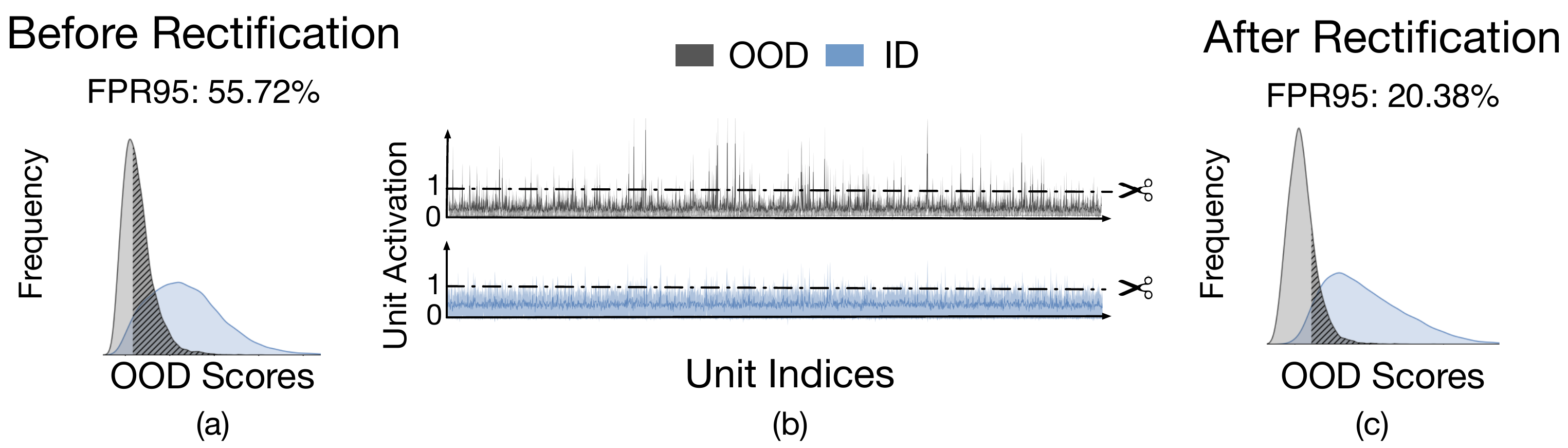}
	\end{center}
    % \vspace{-0.2cm}
	\caption[Plot of the distribution of per-unit activations and uncertainty scores of ID and OOD  before and after truncation.]{\small Plots showing (a) the distribution of ID (ImageNet~\citep{deng2009imagenet}) and OOD (iNaturalist~\citep{inat}) uncertainty scores before truncation, (b) the distribution of per-unit activations in the penultimate layer for ID and OOD data, and (c) the distribution of OOD uncertainty scores~\citep{liu2020energy} after rectification. Applying ReAct drastically improves the separation of ID and OOD data. See text for details. }
% 	\vspace{-0.2cm}
	\label{fig:react_teaser}
% 	\vspace{-0.2cm}
\end{figure}
In this chapter, we start by revealing an important observation that OOD data can trigger unit activation patterns that are significantly different from ID data. \autoref{fig:react_teaser}(b) shows the distribution of activations in the penultimate layer of ResNet-50 trained on ImageNet~\citep{deng2009imagenet}. Each point on the horizontal axis corresponds to a single unit. The mean and standard deviation are shown by the solid line and shaded area, respectively. The mean activation for ID data (blue) is well-behaved with a near-constant mean and standard deviation. In contrast, for OOD data (gray), the mean activation has significantly larger variations across units and is biased towards having sharp positive values (\emph{i.e.}, positively skewed). As a result, such high unit activation can undesirably manifest in model output, producing overconfident predictions on OOD data. A similar distributional property holds for other OOD datasets as well.

% \vspace{-0.2cm}
The above observation naturally inspires a simple yet surprisingly effective method---\textbf{Re}ctified \textbf{Act}ivations (dubbed \textbf{ReAct}) for OOD detection. In particular, the outsized activation of a few selected hidden units can be attenuated by rectifying the activations at an upper limit $c > 0$. 
Conveniently, this can be done on a pre-trained model without any modification to training. The dashed horizontal line in \autoref{fig:react_teaser}(b) shows the cutoff point $c$, and its effect on the OOD uncertainty score is shown in \autoref{fig:react_teaser}(c). After rectification, the output distributions for ID and OOD data become much more well-separated and the false positive rate (FPR) is significantly reduced from $55.72\%$ to $20.38\%$. Importantly, this truncation largely preserves the activation for in-distribution data, and therefore ensures the classification accuracy on the original task is largely comparable.  

We provide both empirical and theoretical insights, characterizing and explaining the mechanism by which ReAct improves OOD detection. We perform extensive evaluations and establish competitive performance on a suite of common OOD detection benchmarks, including CIFAR-10 and CIFAR-100, as well as a large-scale ImageNet dataset~\citep{deng2009imagenet}. ReAct outperforms Energy score~\citep{liu2020energy} by a large margin, reducing the average FPR95 by up to {25.05}\%. We further analyze our method theoretically and show that ReAct is more beneficial when OOD activations are more chaotic (\emph{i.e.}, having a larger variance) and positively skewed compared to ID activations, a behavior that is typical of many OOD datasets (\emph{cf.} \autoref{fig:react_teaser}). In summary, the \textbf{key results and contributions} for this chapter are:
\begin{enumerate}
    \item We introduce ReAct---a simple and effective \emph{post hoc} OOD detection approach that utilizes activation truncation. 
    We show that ReAct can generalize effectively to different network architectures and works with different OOD detection methods including MSP~\citep{Kevin}, ODIN~\citep{liang2018enhancing}, and energy score~\citep{liu2020energy}.
    \item We extensively evaluate ReAct on a suite of OOD detection tasks and establish a competitive performance among post hoc methods. Compared to the previous best method, ReAct achieves an FPR95 reduction of {25.05}\% on a large-scale ImageNet benchmark. 
    \item We provide both empirical ablation and theoretical analysis, revealing important insights that abnormally high activations on OOD data can harm their detection and how ReAct effectively mitigates this issue. Our insight inspires future research to further examine the internal mechanisms of neural networks for OOD detection.
\end{enumerate}

%%%%%%%%%%%%%%%%%%%%%%%%%%%%%%%%%%%%%%%%%%%%%%%%%%%%%%%%%%%%%%%%%%%%%
%%%%%%%%%%%%%%%%%%%%%%%%%%%%%%  METHOD %%%%%%%%%%%%%%%%%%%%%%%%%%%%%% %%%%%%%%%%%%%%%%%%%%%%%%%%%%%%%%%%%%%%%%%%%%%%%%%%%%%%%%%%%%%%%%%%%%%  

\section{Methodology}
\label{sec:react_method}

We introduce a simple and surprisingly effective technique,  \textbf{Re}ctified \textbf{Act}ivations (ReAct), for improving OOD detection performance. Our key idea is to perform \emph{post hoc} modification to the unit activation, so to bring the overall activation pattern closer to the well-behaved case. Specifically, we consider a pre-trained neural network parameterized by $\theta$, which encodes an input $\*x \in \mathbb{R}^d$ to a feature space with dimension $m$. We denote by $ h(\*x) \in \mathbb{R}^m$ the feature vector from the penultimate layer of the network. A weight matrix $\mathbf{W} \in \mathbb{R}^{m\times C}$ connects the feature $h(\*x)$ to the output $f(\*x)$, where $C$ is the total number of classes in $\mathcal{Y}=\{1,2,...,C\}$.

\vspace{0.1cm} \noindent \textbf{ReAct: Rectified Activation.} We propose the \texttt{ReAct} operation, which is applied on the penultimate layer of a network:
\begin{align}
\bar h(\*x) = \texttt{ReAct} (h(\*x); c),
\end{align}
where  $\texttt{ReAct}(x;c) = \min(x,c)$ and is applied element-wise to the feature vector $h(\*x)$. In effect, this operation truncates activations above $c$ to limit the effect of noise.  The model output after \emph{rectified activation} is given by:
\begin{align}
f^{\text{ReAct}}(\*x;\theta) = \*W^\top \bar h(\*x) + \mathbf{b},
\end{align}
where $\mathbf{b} \in \mathbb{R}^C$ is the bias vector. A higher $c$ indicates a larger threshold of activation truncation. When $c=\infty$, the output becomes equivalent to the original output $f(\*x;\theta)$ without rectification, where $f(\*x;\theta) = \*W^\top h(\*x) + \*b$.  Ideally, the rectification parameter $c$ should be  chosen to sufficiently preserve the activations for ID data while rectifying that of OOD data. In practice, we set $c$ based on the $p$-th percentile of activations estimated on the ID data. For example, when $p=90$, it indicates that 90\% percent of the ID activations are less than the threshold $c$. We discuss the effect of percentile in detail in Section~\ref{sec:react_experiments}.

\vspace{0.1cm} \noindent \textbf{OOD detection with rectified activation.} During test time, ReAct can be leveraged by a variety of downstream OOD scoring functions relying on $f^{\text{ReAct}}(\*x;\theta)$: 
\begin{align}
\label{eq:react_threshold}
	\mathcal{S}_{\lambda}(\*x; f^\text{ReAct})=\begin{cases} 
      \text{in } & S(\*x;f^\text{ReAct})\ge \lambda \\
      \text{out} & S(\*x;f^\text{ReAct}) < \lambda 
   \end{cases},
\end{align}
 where a thresholding mechanism is exercised to distinguish between ID and OOD during test time. To align with the convention,  
samples with higher scores $S(\*x;f)$ are classified as ID and vice versa. The threshold $\lambda$ is typically chosen so that a high fraction of ID data (\emph{e.g.,} 95\%) is correctly classified.  ReAct can be compatible with several commonly used OOD scoring functions  derived from the model output $f(\*x;\theta)$, including the softmax confidence~\citep{Kevin}, ODIN score~\citep{liang2018enhancing}, and the energy score~\citep{liu2020energy}.  In Section~\ref{sec:react_experiments}, we default to using the energy score (since it is hyperparameter-free and does not require fine-tuning), but demonstrate the benefit of using ReAct with other OOD scoring functions too.

%%%%%%%%%%%%%%%%%%%%%%%%%%%%%%%%%%%%%%%%%%%%%%%%%%%%%%%%%%%%%%%%%%%%%
%%%%%%%%%%%%%%%%%%%%%%%%%%%%%%  EXP %%%%%%%%%%%%%%%%%%%%%%%%%%%%%%
%%%%%%%%%%%%%%%%%%%%%%%%%%%%%%%%%%%%%%%%%%%%%%%%%%%%%%%%%%%%%%%%%%%%%  

\section{Experiment}
\label{sec:react_experiments}

In this section, we evaluate ReAct on a suite of OOD detection tasks. We first evaluate a on large-scale OOD detection benchmark based on ImageNet~\citep{huang2021mos} (Section~\ref{sec:react_imagenet}), and then proceed in Section~\ref{sec:react_common_benchmark} with CIFAR benchmarks~\citep{krizhevsky2009learning}.

\subsection{Evaluation on Large-scale ImageNet Task}
\label{sec:react_imagenet}

We first evaluate ReAct on a large-scale OOD detection benchmark developed in~\cite{huang2021mos}. Compared to the CIFAR benchmarks that are routinely used in literature, the ImageNet benchmark is more challenging due to a larger label space $(C=1,000)$. 
Moreover, such large-scale evaluation is more relevant to real-world applications, where the deployed models often operate on images that have high resolution and contain more classes than the CIFAR benchmarks. 
% \vspace{-0.3cm}

\vspace{0.1cm} \noindent \textbf{Setup.} 
We use a pre-trained ResNet-50 model~\citep{he2016identity} for ImageNet-1k. At test time, all images are resized to 224 $\times$ 224. 
We evaluate on four test OOD datasets from (subsets of) \texttt{Places365}~\citep{zhou2017places}, \texttt{Textures}~\citep{cimpoi2014describing}, \texttt{iNaturalist}~\citep{inat}, and \texttt{SUN}~\citep{sun} with non-overlapping categories w.r.t ImageNet. 
We use a validation set of Gaussian noise images, which are generated by sampling from $\mathcal{N}(0,1)$ for each pixel location. To ensure validity, we further verify the activation pattern under Gaussian noise, which exhibits a similar distributional trend with positive skewness and chaoticness; see Figure~\ref{fig:react_noise_act} in Appendix~\ref{sec:react_noise_act} for details.  We select $p$ from $\{10, 65, 80, 85, 90, 95, 99\}$ based on the FPR95 performance. The optimal $p$ is 90. 
\vspace{0.1cm} \noindent \textbf{Comparison with competitive OOD detection methods.} In Table~\ref{tab:main-results}, we compare ReAct with OOD detection methods that are competitive in the literature. For a fair comparison, all  methods use the pre-trained networks \emph{post hoc}.\@
We report performance for each OOD test dataset, as well as the average of the four. ReAct outperforms all baselines considered, including Maximum Softmax Probability~\citep{Kevin}, ODIN~\citep{liang2018enhancing}, Mahalanobis distance~\citep{lee2018simple}, and energy score~\citep{liu2020energy}.  Noticeably, ReAct reduces the FPR95 by \textbf{25.05}\% compared to ~\cite{liang2018enhancing} on ResNet. Note that Mahalanobis requires training a separate binary classifier, and displays limiting performance since the increased size of label space makes the class-conditional Gaussian density estimation less viable. In contrast, ReAct is much easier to use in practice, and can be implemented through a simple \emph{post hoc} activation rectification.

\begin{table}[t]
\caption[Comparison with ReAct and competitive {post hoc} out-of-distribution detection methods.]{\small \textbf{Main results.} Comparison with competitive \emph{post hoc} out-of-distribution detection methods. All methods are based on a model trained on \textbf{ID data only} (ImageNet-1k), without using any auxiliary outlier data. $\uparrow$ indicates larger values are better and $\downarrow$ indicates smaller values are better. The compared baselines include MSP~\citep{Kevin}, ODIN ~\citep{liang2018enhancing}, Mahalanobis~\citep{lee2018simple}, and Energy ~\citep{liu2020energy}. All values are percentages. %
}
\centering
\scalebox{0.58}{
\begin{tabular}{llcccccccccc}
\toprule
\multicolumn{1}{c}{\multirow{4}{*}{\textbf{Model}}} & \multicolumn{1}{c}{\multirow{4}{*}{\textbf{Methods}}} & \multicolumn{8}{c}{\textbf{OOD Datasets}} & \multicolumn{2}{c}{\multirow{2}{*}{\textbf{Average}}} \\ \cline{3-10}
\multicolumn{1}{c}{} & \multicolumn{1}{c}{} & \multicolumn{2}{c}{\textbf{iNaturalist}} & \multicolumn{2}{c}{\textbf{SUN}} & \multicolumn{2}{c}{\textbf{Places}} & \multicolumn{2}{c}{\textbf{Textures}} & \multicolumn{2}{c}{}\\
\multicolumn{1}{c}{} & \multicolumn{1}{c}{} & FPR95 & AUROC & FPR95 & AUROC & FPR95 & AUROC & FPR95 & AUROC & FPR95 & AUROC \\
\multicolumn{1}{c}{} & \multicolumn{1}{c}{} & \multicolumn{1}{c}{$\downarrow$} & \multicolumn{1}{c}{$\uparrow$} & \multicolumn{1}{c}{$\downarrow$} & \multicolumn{1}{c}{$\uparrow$} & \multicolumn{1}{c}{$\downarrow$} & \multicolumn{1}{c}{$\uparrow$} & \multicolumn{1}{c}{$\downarrow$} & \multicolumn{1}{c}{$\uparrow$} & \multicolumn{1}{c}{$\downarrow$} & \multicolumn{1}{c}{$\uparrow$} \\ \midrule
\multirow{6}{*}{ResNet} 
 & MSP  & 54.99 & 87.74 & 70.83 & 80.86 & 73.99 & 79.76 & 68.00 & 79.61 & 66.95 & 81.99 \\
 & ODIN & 47.66 & 89.66 & 60.15 & 84.59 & 67.89 & 81.78 & 50.23 & 85.62 & 56.48 & 85.41 \\
 & Mahalanobis & 97.00 & 52.65 & 98.50 & 42.41 & 98.40 & 41.79 & 55.80 & 85.01 & 87.43 & 55.47 \\
 & Energy & 55.72 & 89.95 & 59.26 & 85.89 & 64.92 & 82.86 & 53.72 & 85.99 & 58.41 & 86.17 \\
 &  \textbf{ReAct (Ours)} & \textbf{20.38} & \textbf{96.22 } & \textbf{24.20} & \textbf{94.20 } & \textbf{33.85} & \textbf{91.58} & \textbf{47.30} & \textbf{89.80} & \textbf{31.43} & \textbf{92.95} \\ \midrule
\multirow{6}{*}{MobileNet} 
 & MSP & 64.29 & 85.32 & 77.02 & 77.10 & 79.23 & 76.27 & 73.51 & 77.30 & 73.51 & 79.00 \\
 & ODIN & 55.39 & 87.62 & 54.07 & 85.88 & 57.36 & 84.71 & 49.96 & 85.03 & 54.20 & 85.81 \\
 & Mahalanobis & 62.11 & 81.00 & 47.82 & 86.33 & 52.09 & 83.63 & 92.38 & 33.06 & 63.60 & 71.01	\\
 & Energy & 59.50 & 88.91 & 62.65 & 84.50 & 69.37 & 81.19 & 58.05 & 85.03 & 62.39 & 84.91 \\
 & \textbf{ReAct (Ours)} & \textbf{42.40} & \textbf{91.53} & \textbf{47.69} & \textbf{88.16} & \textbf{51.56} & \textbf{86.64 } & \textbf{38.42 } & \textbf{91.53 } & \textbf{45.02} & \textbf{89.47 } \\ \bottomrule 
\end{tabular}
} 
\label{tab:main-results}

\end{table}

% \vspace{-0.3cm}
\vspace{0.1cm} \noindent \textbf{Effect of rectification threshold $c$.} 
We now characterize  the effect of the rectification parameter $c$, which can be modulated by the percentile $p$ described in Section~\ref{sec:react_method}.  
In Table~\ref{tab:c-ablation}, we summarize the OOD detection performance, where we vary $p=\{10, 65, 80, 85, 90, 95, 99\}$.  
This ablation confirms that over-activation does compromise the ability to detect OOD data, and ReAct can effectively alleviate this problem. Moreover, when $p$ is sufficiently large, ReAct can improve OOD detection while maintaining a comparable ID classification accuracy. Alternatively, once a sample is detected to be ID, one can always use the original activation $h(\*x)$, \emph{which is guaranteed to give identical classification accuracy}. When $p$ is too small, OOD performance starts to degrade as expected.%

 \begin{table}[t]
 \caption[Effect of rectification threshold for inference.]{\small Effect of rectification threshold for inference. Model is trained on ImageNet using ResNet-50~\citep{he2016deep}. All numbers are percentages and are averaged over 4 OOD test datasets. }
\centering
\scalebox{0.8}{
\begin{tabular}{c|ccccc}
\toprule
\textbf{Rectification percentile}               & \textbf{FPR95} \newline $~\downarrow$ & \textbf{AUROC} $~\uparrow$ & \textbf{AUPR} \newline$~\uparrow$ & \textbf{ID ACC.}$~\uparrow$  & \textbf{Threshold} $c$ \\ \midrule
No ReAct & 58.41                                  & 86.17                       & 96.88                              & 75.08   &  $\infty$                                    \\ 
$p=99$                              & 44.57                                  & 90.45                       & 97.96                              & {75.12}   & 2.25                                 \\
$p=95$                              & 35.39                                  & 92.39                       & 98.37                              & 74.76   & 1.50                                 \\
$p=90$                              & {31.43}                                  & {92.95}                       & {98.50}                              & 73.75   & 1.00                                 \\
$p=85$                              & 34.08                                  & 92.05                       & 98.35                              & 72.91   & 0.84                                 \\
$p=80$                              & 41.51                                  & 89.54                       & 97.91                              & 71.93   & 0.72                                 \\
$p=65$                              & 74.62                                  & 74.14                       & 94.39                              & 67.14   & 0.50                                 \\
$p=10$                              & 74.70                                  & 57.55                       & 86.06                              & 1.22    & 0.06                                 \\ 
\bottomrule
\end{tabular}
}
\label{tab:c-ablation}
\end{table}

\vspace{0.1cm} \noindent \textbf{Effect on other network architectures.} We show that ReAct is effective on a different architecture in Table~\ref{tab:main-results}. In particular, we consider a lightweight model MobileNet-v2~\citep{mobilenet2018CVPR}, which can be suitable for OOD detection in on-device mobile applications. Same as before, we apply ReAct on the output of the penultimate layer, with the rectification threshold chosen based on the $90$-th percentile. Our method reduces the FPR95 by \textbf{9.18}\% compared to the best baseline considered~\citep{liang2018enhancing}.

\vspace{0.1cm} \noindent \textbf{What about applying ReAct on other layers?} Our results suggest that applying ReAct on the penultimate layer is the most effective, since the activation patterns are most distinctive. To see this, we provide the activation and performance study for intermediate layers in Appendix~\ref{sec:react_diff_layers} (see Figure~\ref{fig:react_diff_layers} and Table~\ref{tab:diff_layers}). Interestingly, early layers display less distinctive signatures between ID and OOD data. This is expected because neural networks generally capture lower-level features in early layers (such as Gabor filters~\citep{zeiler2014visualizing} in layer 1), whose activations can be very similar between ID and OOD. The semantic-level features only emerge as with deeper layers, where ReAct is the most effective. %

\vspace{-0.2cm}
\subsection{Evaluation on CIFAR Benchmarks}
\label{sec:react_common_benchmark}

\vspace{0.1cm} \noindent \textbf{Datasets.}
We evaluate on CIFAR-10 and CIFAR-100~\citep{krizhevsky2009learning} datasets as in-distribution data, using the standard split with 50,000 training images and 10,000 test images. For OOD data, we consider six common benchmark datasets: \texttt{Textures}~\citep{cimpoi2014describing}, \texttt{SVHN}~\citep{netzer2011reading}, \texttt{Places365}~\citep{zhou2017places}, \texttt{LSUN-Crop}~\citep{yu2015lsun}, \texttt{LSUN-Resize}~\citep{yu2015lsun}, and \texttt{iSUN}~\citep{xu2015turkergaze}.

\vspace{0.1cm} \noindent \textbf{Experimental details.}
We train a standard ResNet-18~\citep{he2016deep} model on in-distribution data. The feature dimension of the penultimate layer is 512. For both CIFAR-10 and CIFAR-100, the models are trained for 100 epochs. The start learning rate is 0.1 and decays by a factor of 10 at epochs 50, 75, and 90. For threshold $c$, we use the 90-th percentile of activations estimated on the ID data. %

\begin{table}[b]
\caption[Performance of different OOD scoring functions equipped with ReAct.]{\small \textbf{Ablation results.} ReAct is compatible with different OOD scoring functions. For each ID dataset, we use the same model and compare the performance with and without ReAct respectively. $\uparrow$ indicates larger values are better and $\downarrow$ indicates smaller values are better. All values are percentages and are averaged over multiple OOD test datasets. Detailed performance for each OOD test dataset is available in Table~\ref{tab:detail-results}. }
\scalebox{0.7}{
\begin{tabular}{lccc|ccc|ccc}
\toprule
\multicolumn{1}{c}{\multirow{3}{*}{\textbf{Method}}} & \multicolumn{3}{c}{\textbf{CIFAR-10}}                                                                       & \multicolumn{3}{c}{\textbf{CIFAR-100}}                                                                      & \multicolumn{3}{c}{\textbf{ImageNet}}                                                                       \\
\multicolumn{1}{c}{}                        & \textbf{FPR95}                            & \textbf{AUROC}                          & \textbf{AUPR}                           & \textbf{FPR95}                            & \textbf{AUROC}                          & \textbf{AUPR}                           & \textbf{FPR95}                            & \textbf{AUROC}                          & \textbf{AUPR}                           \\
\multicolumn{1}{c}{}                        & \multicolumn{1}{c}{$\downarrow$} & \multicolumn{1}{c}{$\uparrow$} & \multicolumn{1}{c}{$\uparrow$} & \multicolumn{1}{c}{$\downarrow$} & \multicolumn{1}{c}{$\uparrow$} & \multicolumn{1}{c}{$\uparrow$} & \multicolumn{1}{c}{$\downarrow$} & \multicolumn{1}{c}{$\uparrow$} & \multicolumn{1}{c}{$\uparrow$} \\ \midrule

MSP & 56.71 & 91.17 & 79.11 & 80.72 & 76.83 & 78.41 & 66.95 & 81.99 & 95.76 \\
\rowcolor{Gray} MSP + ReAct & 53.81 & 91.70 & 92.11 & 75.45 & 80.40 & 84.28 & 58.28 & 87.06 & 97.22 \\ \midrule
Energy & 35.60 & 93.57 & 95.01 & 71.93 & 82.82 & 86.28 & 58.41 & 86.17 & 96.88 \\
\rowcolor{Gray} Energy+ReAct & {32.91} & \textbf{94.27 } & \textbf{95.53} & \textbf{59.61} & \textbf{87.48} & \textbf{89.63} & \textbf{31.43} & \textbf{92.95} & \textbf{98.50} \\
\midrule
 ODIN & 31.10 & 93.79 & 94.95 & 66.21 & 82.88 & 86.25 & 56.48 & 85.41 & 96.61 \\
\rowcolor{Gray} ODIN+ReAct & \textbf{28.81} & 94.04 & 94.82 & 59.91 & 85.23 & 87.53 & 44.10 & 90.70 & 98.04 \\ 
\bottomrule          
\end{tabular}
}
        \label{tab:ablation-results}
\end{table}

\vspace{0.1cm} \noindent \textbf{ReAct is compatible with various OOD scoring functions.} We show in Table~\ref{tab:ablation-results} that ReAct is a flexible method that is compatible with alternative scoring functions $S(\*x;f^\text{ReAct})$. To see this, we consider commonly used scoring functions, and compare the performance both with and without using ReAct respectively. In particular, we consider softmax confidence~\citep{Kevin}, ODIN score~\citep{liang2018enhancing} as well as energy score~\citep{liu2020energy}---all of which derive OOD scores directly from the output $f(\*x)$. In particular, using ReAct on energy score yields the best performance, which is desirable as energy is a hyperparameter-free OOD score and is easy to compute in practice. Note that Mahalanobis~\citep{lee2018simple} estimates OOD score using feature representations instead of the model output $f(\*x)$, hence is less compatible with ReAct. On all three in-distribution datasets, using ReAct consistently outperforms the counterpart without rectification. Results in Table~\ref{tab:ablation-results} are based on the average across multiple OOD test datasets. {Detailed performance for each OOD test dataset is provided in  Table~\ref{tab:detail-results}}.
%

%%%%%%%%%%%%%%%%%%%%%%%%%%%%%%%%%%%%%%%%%%%%%%%%%%%%%%%%%%%%%%%%%%%%%
%%%%%%%%%%%%%%%%%%%%%%%%%%%%%%  Theory %%%%%%%%%%%%%%%%%%%%%%%%%%%%%%%
%%%%%%%%%%%%%%%%%%%%%%%%%%%%%%%%%%%%%%%%%%%%%%%%%%%%%%%%%%%%%%%%%%%%  

\section{Theoretical Insight}
\label{sec:react_theory}

To better understand the effect of ReAct, we mathematically model the ID and OOD activations as rectified Gaussian distributions and derive their respective distributions after applying ReAct. These modeling assumptions are based on the activation statistics observed on ImageNet in Figure~\ref{fig:react_teaser}. In the following analysis, we show that {ReAct reduces mean OOD activations more than ID activations since OOD activations are more positively skewed} (see Section \ref{sec:react_theory_details} for derivations).

\vspace{0.1cm} \noindent \textbf{ID activations.} Let $h(\mathbf{x}) = (z_1,\ldots,z_m) =: \mathbf{z}$ be the activations for the penultimate layer. We assume each $z_i \sim \mathcal{N}^R(\mu, \sigma_\text{in}^2)$ for some $\sigma_\text{in} > 0$. 
Here $\mathcal{N}^R(\mu, \sigma_\text{in}^2) = \max(0, \mathcal{N}(\mu, \sigma_\text{in}^2))$ denotes the rectified Gaussian distribution, which reflects the fact that activations after ReLU have no negative components. Before truncation with ReAct, the expectation of $z_i$ is given by:
\begin{equation*}
\mathbb{E}_\text{in}[z_i] = \left[ 1 - \Phi \left( \frac{-\mu}{\sigma_\text{in}} \right) \right] \cdot \mu + \phi \left( \frac{-\mu}{\sigma_\text{in}} \right) \cdot \sigma_\text{in},
\end{equation*}
where $\Phi$ and $\phi$ denote the cdf and pdf of the standard normal distribution, respectively. After rectification with ReAct, the expectation of $\bar{z}_i = \min(z_i, c)$ is:
\begin{equation*}
\mathbb{E}_\text{in}[\bar{z}_i] = \left[ \Phi \left( \frac{c - \mu}{\sigma_\text{in}} \right) - \Phi \left( \frac{-\mu}{\sigma_\text{in}} \right) \right] \cdot \mu + \left[1 - \Phi \left( \frac{c - \mu}{\sigma_\text{in}} \right) \right] \cdot c + \left[ \phi \left( \frac{-\mu}{\sigma_\text{in}} \right) - \phi \left( \frac{c - \mu}{\sigma_\text{in}} \right) \right] \cdot \sigma_\text{in},
\end{equation*}
The reduction in activation after ReAct is:
\begin{equation}
    \label{eq:react_id_reduction}
    \mathbb{E}_\text{in} [z_i - \bar{z}_i] = \phi \left( \frac{c - \mu}{\sigma_\text{in}} \right) \cdot \sigma_\text{in} - \left[ 1 - \Phi \left( \frac{c - \mu}{\sigma_\text{in}} \right) \right] \cdot (c - \mu)
\end{equation}

\vspace{0.1cm} \noindent \textbf{OOD activations.} We model OOD activations as being generated by a two-stage process: Each OOD distribution defines a set of $\mu_i$'s that represent the mode of the activation distribution for unit $i$, and the activations $z_i$ given $\mu_i$ is represented by $z_i | \mu_i \sim \mathcal{N}^R(\mu_i, \tau^2)$ with $\tau > 0$. For instance, the dark gray line in Figure~\ref{fig:react_teaser} shows the set $\mu_i$'s on the iNaturalist dataset, and the light gray area depicts the distribution of $z_i | \mu_i$. One commonality across different OOD datasets is that the distribution of $\mu_i$ is \emph{positively skewed}. The assumption of positive skewness is motivated by our observation on real OOD data. Indeed, Figure~\ref{fig:react_mui} shows the empirical distribution of $\mu_i$ on an ImageNet pre-trained model for four OOD datasets, all of which display strong positive-skewness, \emph{i.e.}, the right tail has a much higher density than the left tail. This observation is surprisingly consistent across datasets and model architectures. Although a more in-depth understanding of the fundamental cause of positive skewness is important, for this work, we chose to rely on this empirically verifiable assumption and instead focus on analyzing our method ReAct.

\begin{figure}[htb]
    % \vspace{-6ex}
	\begin{center}
		\includegraphics[width=0.90\linewidth]{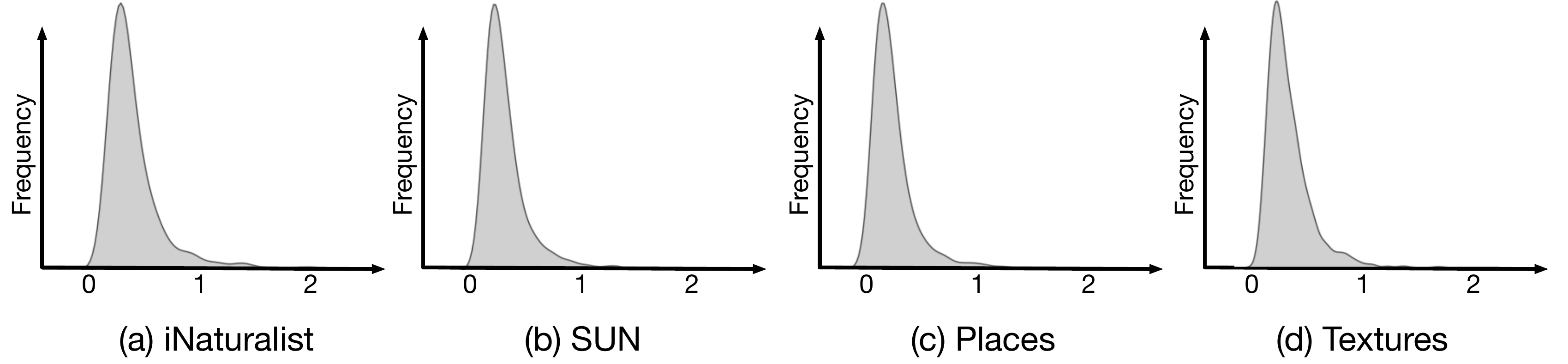}
	\end{center}
    \vspace{-2ex}
	\caption[Positively skewed distribution of $\mu_i$ (mean of each unit) in the penultimate layer for four OOD datasets.]{\small Positively skewed distribution of $\mu_i$ (mean of each unit) in the penultimate layer for four OOD datasets (iNaturalist~\citep{inat}, SUN~\citep{sun}, Places~\citep{zhou2017places}, Textures~\citep{cimpoi2014describing}). Model is trained on ImageNet. Note that the right tail has a higher density than the left tail. By left and right tail we refer to the samples that are to the left and right of the \emph{median}, which is close to the mode in the case of skewed distribution.}
	\label{fig:react_mui}
\end{figure}

Utilizing the positive-skewness property of $\mu_i$, we analyze the distribution of $z_i$ after marginalizing out $\mu_i$, which corresponds to averaging across different $\mu_i$'s induced by various OOD distributions. Let $x_i | \mu_i \sim \mathcal{N}(\mu_i, \tau^2)$ so that $z_i | \mu_i = \max(x_i | \mu_i, 0)$. Since $x_i | \mu_i$ is symmetric and $\mu_i$ is positively-skewed, the marginal distribution of $x_i$ is also positively-skewed\footnote{This can be argued rigorously using Pearson's mode skewness coefficient if the distribution of $\mu_i$ is unimodal.}, which we model with the epsilon-skew-normal (ESN) distribution~\citep{mudholkar2000epsilon}. Specifically, we assume that $x_i \sim \esn(\mu, \sigma_\text{out}^2, \epsilon)$, which has the following density function:
\begin{equation}
    q(x) = 
    \begin{cases}
    \phi((x - \mu) / \sigma_\text{out} (1+\epsilon)) / \sigma_\text{out} & \text{if } x < \mu, \\
    \phi((x - \mu) / \sigma_\text{out} (1-\epsilon)) / \sigma_\text{out} & \text{if } x \geq \mu.
    \end{cases}
    \label{eq:react_esn}
\end{equation}
with $\epsilon \in [-1,1]$ controlling the skewness.  In particular, the ESN distribution is positively-skewed when $\epsilon < 0$. It follows that $z_i = \max(x_i, 0)$, with expectation:
\begin{equation}
    \mathbb{E}_\text{out}[z_i] = \mu - (1+\epsilon) \Phi \left( \frac{-\mu}{(1+\epsilon) \sigma_\text{out}} \right) \cdot \mu + (1+\epsilon)^2 \phi \left( \frac{-\mu}{(1+\epsilon) \sigma_\text{out}} \right) \cdot \sigma_\text{out} - \frac{4 \epsilon}{\sqrt{2 \pi}} \cdot \sigma_\text{out}.
    \label{eq:react_ood_before_mean}
\end{equation}
Expectation after applying ReAct becomes:
\begin{align}
    \mathbb{E}_\text{out}[\bar{z}_i] &= \mu - (1 + \epsilon) \Phi \left( \frac{-\mu}{(1+\epsilon)\sigma_\text{out}} \right) \cdot \mu + (1 - \epsilon) \left[ 1 - \Phi \left( \frac{c - \mu}{(1-\epsilon)\sigma_\text{out}} \right) \right] \cdot (c - \mu) \nonumber \\
    &\qquad + \left[ (1+\epsilon)^2 \phi\left(\frac{-\mu}{(1+\epsilon)\sigma_\text{out}}\right) - (1-\epsilon)^2 \phi\left(\frac{c - \mu}{(1-\epsilon)\sigma_\text{out}}\right) - \frac{4 \epsilon}{\sqrt{2 \pi}} \right] \cdot \sigma_\text{out},
    \label{eq:react_ood_after_mean}
\end{align}
Hence:
\begin{equation}
    \label{eq:react_ood_reduction}
    \mathbb{E}_\text{out} [z_i - \bar{z}_i] = (1-\epsilon)^2 \phi\left(\frac{c - \mu}{(1-\epsilon)\sigma_\text{out}}\right) \cdot \sigma_\text{out} - (1 - \epsilon) \left[ 1 - \Phi \left( \frac{c - \mu}{(1-\epsilon)\sigma_\text{out}} \right) \right] \cdot (c - \mu),
\end{equation}
which recovers \autoref{eq:react_id_reduction} when $\epsilon = 0$ and $\sigma_\text{out} = \sigma_\text{in}$.

\begin{figure}[t]
\centering
\includegraphics[width=0.6\textwidth]{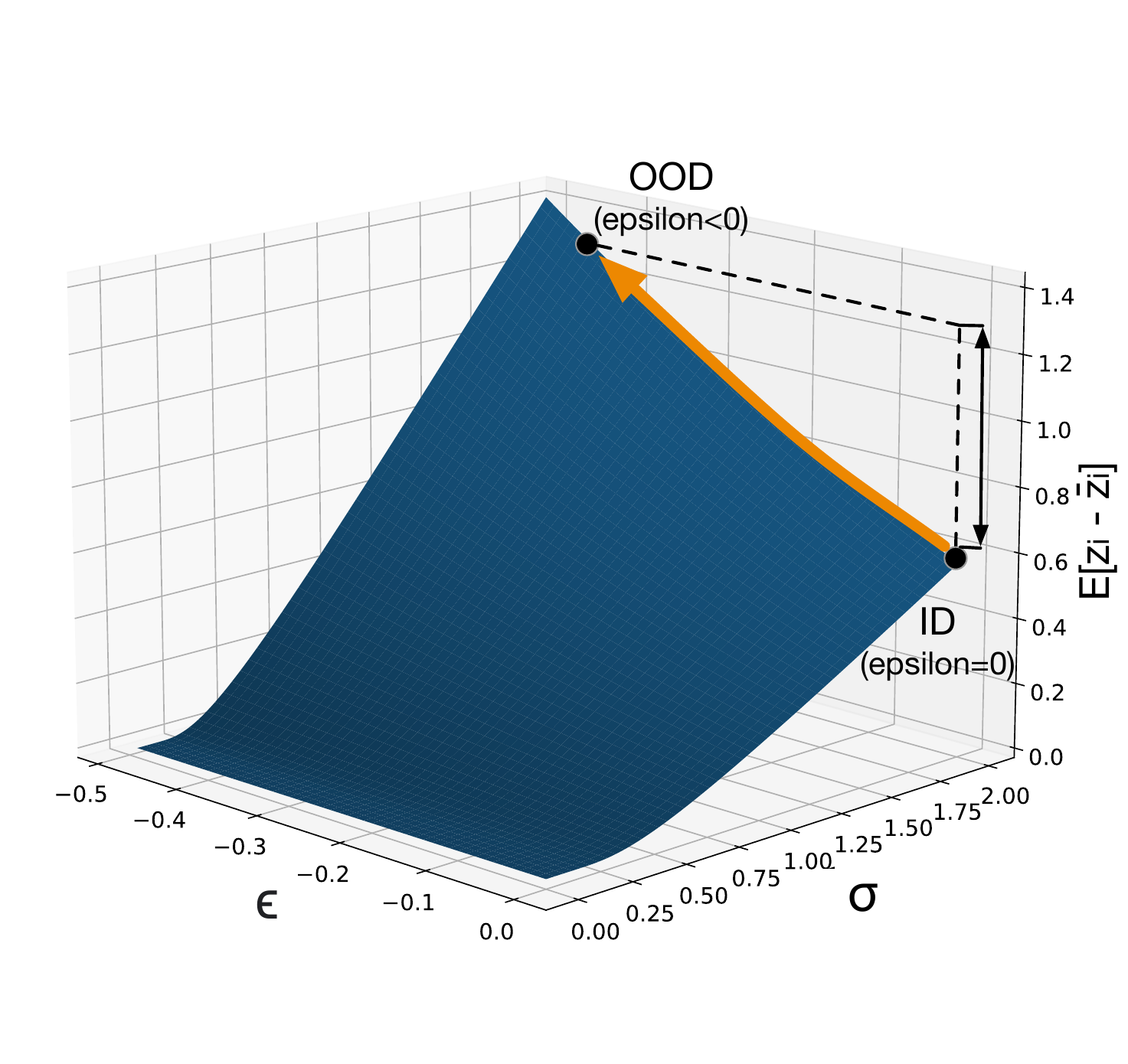}
\vspace{-0.5cm}
\caption[Plot showing the relationship between the skewness parameter $\epsilon$ and the chaotic-ness parameter $\sigma$ on activation reduction after applying ReAct.]{Plot showing the relationship between the skewness parameter $\epsilon$ and the chaotic-ness parameter $\sigma$ on activation reduction after applying ReAct. The function $\mathbb{E}[z_i - \bar{z}_i]$ is increasing in both $-\epsilon$ and $\sigma$, which suggests that ReAct has a greater reduction effect for activation distributions with positive skewness ($\epsilon<0)$ and chaotic-ness---two signature characteristics of OOD activation.}
\label{fig:react_act_reduction}
\end{figure}

\vspace{0.1cm} \noindent \textbf{Remark 1: Activation reduction on OOD is more than ID.} \autoref{fig:react_act_reduction} shows a plot of $\mathbb{E}[z_i - \bar{z}_i]$ for $\mu = 0.5$ and $c = 1$. Observe that decreasing $\epsilon$ (more positive-skewness) or increasing $\sigma$ (more chaotic-ness) leads to a larger reduction in the mean activation after applying ReAct. For example, under the same $\sigma$, a larger $\mathbb{E}_\text{out} [z_i - \bar{z}_i] - \mathbb{E}_\text{in}  [z_i - \bar{z}_i]$ can be observed by the gap of $z$-axis value between $\epsilon=0$ and $\epsilon<0$ (\emph{e.g.}, $\epsilon=-0.4$). This suggests that rectification on average affects OOD activations more severely compared to ID activations. 

\vspace{0.1cm} \noindent \textbf{Remark 2: Output reduction on OOD is more than ID.} To derive the effect on the distribution of model output, consider output logits $f(\mathbf{z}) = W \mathbf{z} + \mathbf{b}$ and assume without loss of generality that $W \mathbf{1} > 0$ element-wise. This can be achieved by adding a positive constant to $W$ without changing the output probabilities or classification decision. Let $\delta = \mathbb{E}_\text{out}[\mathbf{z} - \bar{\mathbf{z}}] - \mathbb{E}_\text{in}[\mathbf{z} - \bar{\mathbf{z}}] > 0$. Then:
\begin{align*}
    \mathbb{E}_\text{out}[f(\mathbf{z}) - f(\bar{\mathbf{z}})] & = \mathbb{E}_\text{out}[W (\mathbf{z} - \bar{\mathbf{z}})] = W \mathbb{E}_\text{out}[\mathbf{z} - \bar{\mathbf{z}}] = W \left( \mathbb{E}_\text{in}[\mathbf{z} - \bar{\mathbf{z}}] + \delta \mathbf{1} \right) \\ & = \mathbb{E}_\text{in}[W (\mathbf{z} - \bar{\mathbf{z}})] + \delta W \mathbf{1} \\ & > \mathbb{E}_\text{in}[f(\mathbf{z}) - f(\bar{\mathbf{z}})].
\end{align*}
Hence the increased separation between OOD and ID activations transfers to the output space as well. Note that the condition of $W \mathbf{1} > 0$ is sufficient but not necessary for this result to hold. In fact, our experiments in Section \ref{sec:react_experiments} do not require this condition. However, we verified empirically that ensuring $W \mathbf{1} > 0$ by adding a positive constant to $W$ and applying ReAct does confer benefits to OOD detection, which validates our theoretical analysis.

\vspace{0.1cm} \noindent \textbf{Why ReAct improves the OOD scoring functions?} Our theoretical analysis above shows that ReAct suppresses logit output for OOD data more so than for ID data. This means that for detection scores depending on the logit output (\emph{e.g.,} energy score~\citep{liu2020energy}), the gap between OOD and ID score will be enlarged after applying ReAct, which makes thresholding more capable of separating OOD and ID inputs; see Figure~\ref{fig:react_teaser}(a) and (c) for a concrete example showing this effect.

%%%%%%%%%%%%%%%%%%%%%%%%%%%%%%%%%%%%%%%%%%%%%%%%%%%%%%%%%%%%%%%%%%%%%
%%%%%%%%%%%%%%%%%%%%%%%%%%%%%%  Discussion %%%%%%%%%%%%%%%%%%%%%%%%%%%%%%%
%%%%%%%%%%%%%%%%%%%%%%%%%%%%%%%%%%%%%%%%%%%%%%%%%%%%%%%%%%%%%%%%%%%%  

\section{Discussion and Further Analysis}
\label{sec:react_discussion}

% \vspace{0.1cm} \noindent \textbf{Why do OOD samples trigger abnormal unit activation patterns?} 
\subsection{Why do OOD samples trigger abnormal unit activation patterns?}
So far we have shown that OOD data can trigger unit activation patterns that are significantly different from ID data, and that ReAct can effectively alleviate this issue (empirically in Section~\ref{sec:react_experiments} and theoretically in Section~\ref{sec:react_theory}). Yet a question left in mystery is \emph{why} such a pattern occurs in modern neural networks? 
Answering this question requires carefully examining the internal mechanism by which the network is trained and evaluated. Here we provide one plausible explanation for the activation patterns observed in Figure~\ref{fig:react_teaser}, with the hope of shedding light for future research.

\begin{table*}[htb]
\caption[Comparison with oracle using OOD's BN statistics.]{\small Comparison with oracle using OOD's BN statistics. The model is trained on ImageNet (see Section~\ref{sec:react_imagenet}). Values are AUROC. }
\centering
\scalebox{0.8}{
\begin{tabular}{lcccc}
\toprule
\textbf{Method}  &\textbf{iNaturalist}  & \textbf{Places}  & \textbf{SUN} & \textbf{Textures}   \\
\hline
Oracle (batch OOD for estimating BN statistics) & 99.59 & 99.09 & 98.32 & 91.43\\
 {ReAct} (single OOD) &  96.22 & 	94.20 & 	91.58 & 	89.80\\ 
  No ReAct~\citep{liu2020energy} & 89.95   & 85.89      & 82.86       & 85.99       \\ 
 \bottomrule
 \end{tabular}}
 \label{tab:bn-results}
\end{table*}

\begin{figure}[htb]
	\begin{center}
		\includegraphics[width=0.9\linewidth]{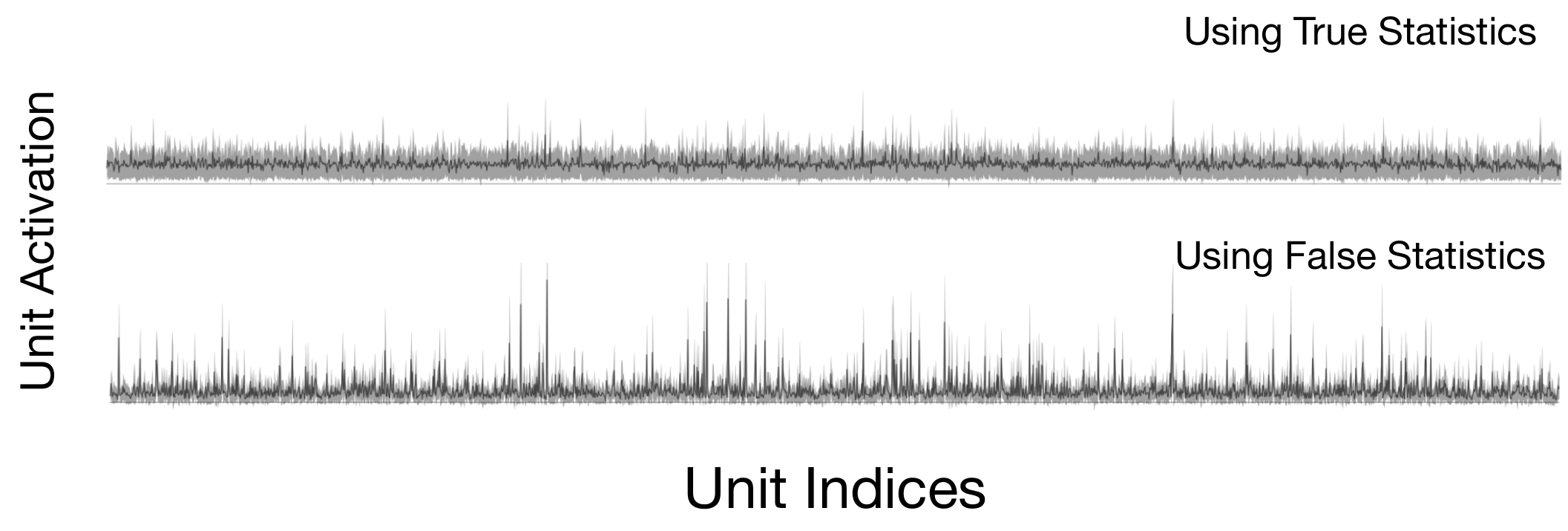}
	\end{center}
 	\vspace{-0.4cm}
	\caption[The distribution of per-unit activations in the penultimate layer for OOD data by using true vs. mismatched BatchNorm statistics for OOD data.]{\small The distribution of per-unit activations in the penultimate layer for OOD data (iNaturalist) by using \textit{true} (top) vs. \textit{mismatched} (bottom) BatchNorm statistics for OOD data.}
	\label{fig:react_trainval}
\end{figure}

Intriguingly, our analysis reveals an important insight that batch normalization (BatchNorm)~\citep{bn2015pmlr}---a common technique employed during model training---is in fact both a blessing (for ID classification) and a curse (for OOD detection). Specifically, for a unit activation denoted by $z$, the network estimates the running mean $\mathbb{E}_\text{in}(z)$ and variance $\text{Var}_\text{in}(z)$, over the entire ID training set during training. During inference time, the network applies BatchNorm statistics $\mathbb{E}_\text{in}(z)$ and $\text{Var}_\text{in}(z)$, which helps normalize the activations for the test data with the same distribution $\mathcal{D}_\text{in}$:
\begin{align}
    \text{BatchNorm}(z;\gamma, \beta, \epsilon) = \frac{z-\mathbb{E}_\text{in}[z]}{\sqrt{\text{Var}_\text{in}[z]+\epsilon}}\cdot \gamma + \beta
\end{align}
However, our key observation is that using \emph{mismatched} BatchNorm statistics---that are estimated on $\mathcal{D}_\text{in}$ yet blindly applied to the OOD $\mathcal{D}_\text{out}$---can trigger abnormally high unit activations. 
As a thought experiment, we instead apply the \emph{true} BatchNorm statistics estimated on a batch of OOD images and we observe well-behaved activation patterns with near-constant mean and standard deviations---just like the ones observed on the ID data  (see Figure~\ref{fig:react_trainval}, top). Our study therefore reveals one of the fundamental causes for neural networks to produce overconfident predictions for OOD data. After applying the true statistics (estimated on OOD), the output distributions between ID and OOD data become much more separable. While this thought experiment has shed some guiding light, the solution of estimating BatchNorm statistics on a batch of OOD data is not at all satisfactory and realistic. Arguably, it poses a strong and impractical assumption of having access to a batch of OOD data during test time. Despite its limitation, we view it as an oracle, which serves as an \emph{upper bound on performance} for ReAct. %

In particular, results in Table~\ref{tab:bn-results} suggest that our method favorably matches the oracle performance using the ground truth BN statistics. This is encouraging as our method does not impose any batch assumption and can be feasible for \emph{single-input} testing scenarios. %

\begin{table}[htb]
\caption[Effectiveness of ReAct for different normalization methods.]{\small \textbf{Effectiveness of ReAct for different normalization methods.} ReAct consistently improves OOD detection performance for the model trained with GroupNorm and WeightNorm. In-distribution is ImageNet-1k dataset. \textbf{Bold} numbers are superior results.  }
\centering
\scalebox{0.58}{
\begin{tabular}{llcccccccccc}
    \toprule
\multicolumn{1}{c}{\multirow{4}{*}{\textbf{}}} & \multicolumn{1}{c}{\multirow{4}{*}{\textbf{Methods}}} & \multicolumn{8}{c}{\textbf{OOD Datasets}}                                                                                      & \multicolumn{2}{c}{\multirow{2}{*}{\textbf{Average}}}  \\ \cline{3-10}
\multicolumn{1}{c}{}                       & \multicolumn{1}{c}{}                         & \multicolumn{2}{c}{\textbf{iNaturalist}} & \multicolumn{2}{c}{\textbf{SUN}} & \multicolumn{2}{c}{\textbf{Places}} & \multicolumn{2}{c}{\textbf{Textures}} & \multicolumn{2}{c}{}                       \\
\multicolumn{1}{c}{}                       & \multicolumn{1}{c}{}                         & FPR95          & AUROC          & FPR95      & AUROC      & FPR95        & AUROC       & FPR95         & AUROC        & FPR95                 & AUROC                        \\
\multicolumn{1}{c}{}                       & \multicolumn{1}{c}{}                         &        \multicolumn{1}{c}{$\downarrow$}         &     \multicolumn{1}{c}{$\uparrow$}            &      \multicolumn{1}{c}{$\downarrow$}       &    \multicolumn{1}{c}{$\uparrow$}   &        \multicolumn{1}{c}{$\downarrow$}       &    \multicolumn{1}{c}{$\uparrow$}    &       \multicolumn{1}{c}{$\downarrow$}         &    \multicolumn{1}{c}{$\uparrow$}    &         \multicolumn{1}{c}{$\downarrow$}               &       \multicolumn{1}{c}{$\uparrow$}                       \\    \hline

\multirow{2}{*}{GroupNorm} & w.o. ReAct & 65.38 & 88.45 & 65.11 & 85.52 & 65.46 & 84.34 & 69.17 & 83.22 & 66.28 & 85.38 \\
  & w/ ReAct    & \textbf{39.45} & \textbf{92.95} & \textbf{51.57} & \textbf{87.90} & \textbf{52.78} & \textbf{87.32} & \textbf{62.50} & \textbf{81.76} & \textbf{51.58} & \textbf{87.48} \\ \midrule
\multirow{2}{*}{WeightNorm}  & w.o. ReAct & 40.71 & 92.52 & 48.07 & 89.39 & 50.92 & 87.87 & 61.65 & 80.71 & 50.34 & 87.62 \\
  &  w/ ReAct    &\textbf{19.73} & \textbf{95.91} & \textbf{31.39} & \textbf{93.21} & \textbf{42.34} & \textbf{88.94} & \textbf{13.74} & \textbf{96.98} & \textbf{26.80} & \textbf{93.76} \\ \bottomrule 
\end{tabular}
} 
        \label{tab:norm-ablation}
\end{table}

\subsection{What about networks trained with different normalization mechanisms?}
% \vspace{0.1cm} \noindent \textbf{What about networks trained with different normalization mechanisms?} 
Going beyond batch normalization~\citep{bn2015pmlr}, we further investigate (1) whether networks trained with alternative normalization approaches exhibit similar activation patterns, and (2) whether ReAct is helpful there. To answer this question, we additionally evaluate networks trained with  WeightNorm~\citep{salimans2016weight} and GroupNorm~\citep{wu2018group}---two other well-known normalization methods. As shown in Figure~\ref{fig:react_abnormal}, the unit activations also display highly distinctive signature patterns between ID and OOD data, with more chaos on OOD data. In all cases, the networks are trained to adapt to the ID data, resulting in abnormal activation signatures on OOD data in testing. Unlike BatchNorm, there is no easy oracle (\emph{e.g.}, re-estimating the statistics on OOD data) to counteract the ill-fated normalizations.

We apply ReAct on models trained with WeightNorm and GroupNorm, and report results in Table~\ref{tab:norm-ablation}. Our results suggest that ReAct is consistently effective under various normalization schemes. For example, ReAct reduces the average FPR95 by \textbf{23.54}\% and \textbf{14.7}\% respectively. 
Overall, ReAct has shown broad efficacy and compatibility with different OOD scoring functions~(Section~\ref{sec:react_common_benchmark}).

\begin{figure}[htb]
	\begin{center}
		\includegraphics[width=0.90\linewidth]{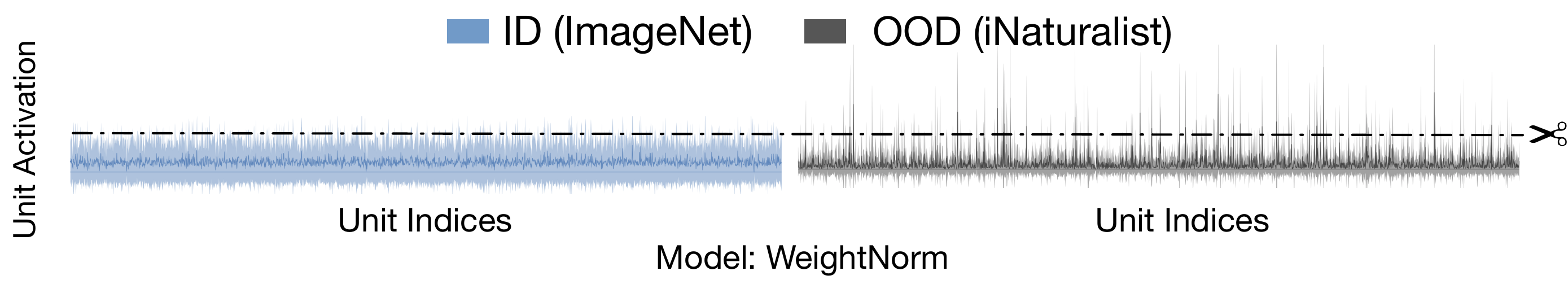}
	\end{center}
	\caption[The distribution of per-unit activations in the penultimate layer for ID  and OOD on model trained with weight normalization.]{\small The distribution of per-unit activations in the penultimate layer for ID (ImageNet~\citep{deng2009imagenet}) and OOD (iNaturalist~\citep{inat}) on model trained with weight normalization~\citep{salimans2016weight}.} % 
	\label{fig:react_abnormal}
 \vspace{-1.4cm}
\end{figure}

%%%%%%%%%%%%%%%%%%%%%%%%%%%%%%%%%%%%%%%%%%%%%%%%%%%%%%%%%%%%%%%%%%%%%
%%%%%%%%%%%%%%%%%%%%%%%%%%%%%%  RELATED %%%%%%%%%%%%%%%%%%%%%%%%%%%%%%%%%%%%%%%%%%%%%%%%%%%%%%%%%%%%%%%%%%%%%%%%%%%%%%%%%%%%%%%%%%%%%%%%%%%  

\section{Additional Related Work}
\label{sec:react_related}

\noindent \textbf{Neural network activation analysis.}
Neural networks have been studied at the granularity of the activation of individual layers~\citep{bn2015pmlr, morcos2018iclr, jason2014nips, zhou2018revisiting, sun2019adaptive}, or individual networks~\citep{li2015convergent}. In particular, ~\citet{li2015convergent} studied the similarity of activation space between two independently trained neural networks. Previously,  ~\citet{hein2019relu} showed that neural networks with ReLU activation can lead to arbitrary high activation for inputs far away from the training data. We show that using ReAct could efficiently alleviate this undesirable phenomenon. ReAct does not rely on auxiliary data and can be conveniently used for pre-trained models. The idea of rectifying unit activation~\citep{relu62010}, which is known as Relu6, was used to facilitate the learning of sparse features. In this chapter, we show that rectifying activation can drastically alleviate the overconfidence issue for OOD data, and as a result, improve OOD detection. 

\vspace{0.1cm} \noindent \textbf{Distributional shifts.} Distributional shifts have attracted increasing research interests~\citep{koh2021wilds}. It is important to recognize and differentiate various types of distributional shift problems. Literature in OOD detection is commonly concerned about model reliability and detection of label-space shifts, where the OOD inputs have disjoint labels \emph{w.r.t.} ID data and therefore \emph{should not be predicted by the model}. Meanwhile, some works considered {label distribution} shift~\citep{saerens2002adjusting, lipton2018detecting, shrikumar2019calibration, azizzadenesheli2019regularized, alexandari2020maximum, wu2021online}, where the label space is common between ID and OOD but the marginal label distribution changes, as well as covariate shift in the input space~\citep{hendrycks2019benchmarking, ovadia2019can}, where inputs can be corruption-shifted or domain-shifted~\citep{sun2020test, godin2020CVPR}.
It is important to note that our work focuses on the detection of shifts where the label space $\mathcal{Y}$ is different between ID and OOD data and hence the model should not make any prediction, instead of covariate shift where the model is expected to {generalize}.

%%%%%%%%%%%%%%%%%%%%%%%%%%%%%%%%%%%%%%%%%%%%%%%%%%%%%%%%%%%%%%%%%%%%%
%%%%%%%%%%%%%%%%%%%%%%%%%%%%%%  Summary %%%%%%%%%%%%%%%%%%%%%%%%%%%%%%%
%%%%%%%%%%%%%%%%%%%%%%%%%%%%%%%%%%%%%%%%%%%%%%%%%%%%%%%%%%%%%%%%%%%%  

\section{Summary}
\label{sec:react_summary}

This chapter provides a simple activation rectification strategy termed ReAct, which truncates the high activations during test time for  OOD detection. We provide both empirical and theoretical insights characterizing and explaining the mechanism by which ReAct improves OOD uncertainty estimation. By rectifying the activations, the outsized contribution of hidden units on OOD output can be attenuated, resulting in a stronger separability from ID data. Extensive experiments show ReAct can significantly improve the performance of OOD detection on both common benchmarks and large-scale image classification models.  Our insights have inspired future research to further examine the internal mechanisms of neural networks for OOD detection. %

%%%%%%%%%%%%%%%%%%%%%%%%%%%%%%%%%%%%%%%%%%%%%%%%%%%%%%%%%%%%%%%%%%%%%
%%%%%%%%%%%%%%%%%%%%%%%%%%%%%%  Supp %%%%%%%%%%%%%%%%%%%%%%%%%%%%%%%
%%%%%%%%%%%%%%%%%%%%%%%%%%%%%%%%%%%%%%%%%%%%%%%%%%%%%%%%%%%%%%%%%%%%  

% \section{Appendix}
% \label{sec:react_supp}

% \input{chapters/supp_react}

%% file: chapters/2_dice.tex
\chapter{DICE: Leverage Sparsification for OOD Detection }
\label{sec:dice}

\paragraph{Publication Statement.} This chapter is joint work with Yixuan Li. The paper version of this chapter appeared in ECCV22~\citep{sun2022dice}. 

\noindent\rule{\textwidth}{1pt}

ReAct led to subsequent research in this chapter that delves beyond unit activations, focusing on exploring the impact of weight and unit jointly in out-of-distribution (OOD) detection. In particular, conventional approaches often rely on an OOD score derived from the overparameterized weight space, while largely neglecting the significance of \textit{sparsification}. In this chapter, we reveal important insights that reliance on unimportant weights and units can directly attribute to the brittleness of OOD detection. To mitigate the issue, we propose a sparsification-based OOD detection framework termed \textbf{DICE}. Our key idea is to rank weights based on a measure of contribution, and selectively use the most salient weights to derive the output for OOD detection. We provide both empirical and theoretical insights, characterizing and explaining the mechanism by which DICE improves OOD detection. By pruning away noisy signals, DICE provably reduces the output variance for OOD data, resulting in a sharper output distribution and stronger separability from ID data. We demonstrate the effectiveness of
sparsification-based OOD detection on several
benchmarks and establish competitive performance.

%%%%%%%%%%%%%%%%%%%%%%%%%%%%%%%%%%%%%%%%%%%%%%%%%%%%%%%%%%%%%%%%%%%%%
%%%%%%%%%%%%%%%%%%%%%%%%%%%%%%  INTRO %%%%%%%%%%%%%%%%%%%%%%%%%%%%%%%%
%%%%%%%%%%%%%%%%%%%%%%%%%%%%%%%%%%%%%%%%%%%%%%%%%%%%%%%%%%%%%%%%%%%  

\section{Introduction}
\label{sec:dice_intro}

Deep neural networks deployed in real-world systems often encounter out-of-distribution (OOD) inputs---samples from unknown classes that the network has not been exposed to during training, and therefore should not be predicted by the model in testing. 
Being able to estimate and mitigate OOD uncertainty is paramount for safety-critical applications such as medical diagnosis~\citep{roy2021does,wang2017chestx} and autonomous driving~\citep{filos2020can}. For example, an autonomous vehicle may fail to recognize objects on the road that do not appear in its detection model’s training set, potentially leading to a crash. This gives rise to the importance of OOD detection, which allows the
learner to express ignorance and take precautions in the presence of OOD data.  %

The main challenge in OOD detection stems from the fact that modern deep neural networks can easily produce overconfident predictions on OOD inputs, making the separation between in-distribution (ID) and OOD data a non-trivial task. The vulnerability of machine learning to OOD data can be hard-wired in high-capacity models used in practice. In particular, modern deep neural networks can overfit observed patterns in the training data~\citep{zhang2016understanding}, and worse, {activate features on unfamiliar inputs}~\citep{nguyen2015deep}. 
To date, existing OOD detection methods commonly derive OOD scores using overparameterized weights, while largely overlooking the role of \emph{sparsification}. This chapter aims to bridge the gap.

\begin{figure}[t]
	\begin{center}
		\includegraphics[width=0.9\linewidth]{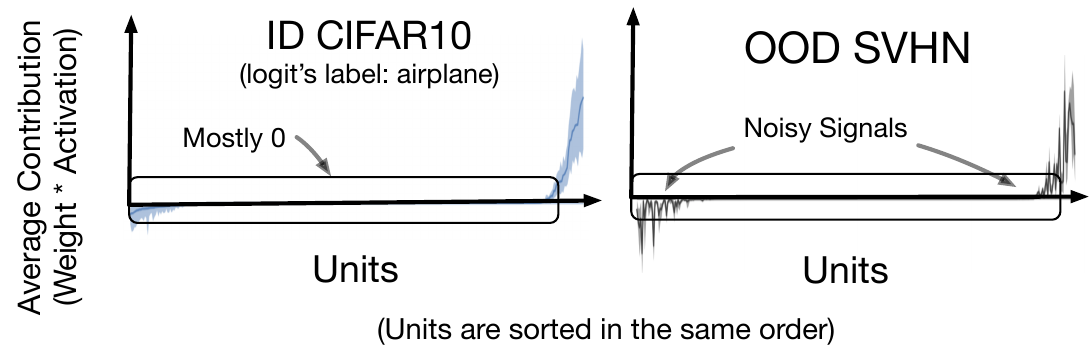}
	\end{center}
	\caption[Illustration of unit contribution to the class output.]{\small Illustration of unit contribution (\emph{i.e.}, \texttt{weight} $\times$ \texttt{activation}) to the class output. For class $c$, the output $f_c(\*x)$ is the summation of unit contribution from the penultimate feature layer of a neural network. \emph{Units are sorted in the same order}, based on the expectation of ID data's contribution (averaged over many CIFAR-10 samples) on the $x$-axis. \textbf{Shades indicate the variance for each unit}. \textbf{Left:} For in-distribution data (CIFAR-10, airplane), only a subset of units contributes to the model output.  \textbf{Right:} In contrast, out-of-distribution (OOD) data can trigger a non-negligible fraction of units with noisy signals, as indicated by the variances.}
	\label{fig:dice_whytopk}
\end{figure}

In this chapter, we start by revealing key insights that reliance on unimportant units and weights can directly attribute to the brittleness of OOD detection. Empirically on a network trained with CIFAR-10, we show that an OOD image can activate a non-negligible fraction of units in the penultimate layer (see Figure~\ref{fig:dice_whytopk}, right). Each point on the horizontal axis corresponds to a single unit. The y-axis measures the unit contribution (\emph{i.e.}, \texttt{weight} $\times$ \texttt{activation}) to the output of class \textsc{airplane}, with the solid line and the shaded area indicating the mean and variance, respectively. 
Noticeably, for OOD data (gray), we observe a non-negligible fraction of ``noisy'' units that display high variances of contribution, which is then aggregated to the model's output through summation. As a result, such noisy signals can undesirably manifest in model output---increasing the variance of output distribution and reducing the separability from ID data.

The above observation motivates a simple and effective method,
\emph{\textbf{Di}rected \textbf{S}parisification} (\textbf{DICE}), for OOD detection. %
DICE leverages the observation that a model's prediction for an ID class depends
on only a subset of important units (and corresponding weights), as evidenced in Figure~\ref{fig:dice_whytopk} (left). To exploit this, our novel idea is to rank weights based on the measure of {contribution}, and selectively use the most contributing weights to derive the output for OOD detection. As a result of the weight sparsification, we show that the model's output becomes more separable between ID and OOD data.   Importantly, DICE can be conveniently used by {{post hoc} weight masking} on a pre-trained network and therefore can preserve the ID classification accuracy. 
Orthogonal to existing works on sparsification for accelerating computation, our primary goal is to explore the sparsification approach for improved  OOD detection performance.%

We provide both empirical and theoretical insights characterizing and explaining the mechanism by which DICE improves OOD detection. We perform extensive evaluations and establish competitive performance on common OOD detection benchmarks, including CIFAR-10, CIFAR-100~\citep{krizhevsky2009learning}, and a large-scale ImageNet benchmark~\citep{huang2021mos}. Compared to the competitive post hoc method ReAct~\citep{sun2021react}, DICE reduces the FPR95 by up to {12.55}\%. Moreover, we perform ablation using various sparsification techniques and demonstrate the benefit of {directed sparsification} for OOD detection. 
Theoretically, by pruning away noisy signals from unimportant units and weights,
DICE  \emph{provably reduces the output variance}  and results in a sharper output distribution (see Section~\ref{sec:dice_theory}). The sharper distributions lead to a stronger separability between ID and OOD data and overall improved OOD detection performance (\emph{c.f.} Figure~\ref{fig:dice_teaser}).  The key results and contributions for this chapter are:
\begin{enumerate}
    \item (Methodology) We introduce DICE, a simple and effective approach for OOD detection utilizing {post hoc} weight sparsification. In the realm of OOD detection, DICE holds a significant place as it pioneered the exploration and exemplification of the efficacy of sparsification.
    
    \item (Experiments) We extensively evaluate DICE on common benchmarks and establish competitive performance among post hoc OOD detection baselines. DICE  outperforms the ReAct~\citep{sun2021react} by reducing the FPR95 by up to {12.55}\%. We show DICE can effectively improve OOD detection while preserving the classification accuracy of ID data.
    
    \item (Theory and ablations) We provide  ablation and theoretical analysis that improves understanding of a sparsification-based method for OOD detection. Our analysis reveals an important variance reduction effect, which probably explains the effectiveness of DICE. The aforementioned insights serve as a catalyst for further investigation into weight sparsification techniques aimed at out-of-distribution (OOD) detection.
\end{enumerate}

%%%%%%%%%%%%%%%%%%%%%%%%%%%%%%%%%%%%%%%%%%%%%%%%%%%%%%%%%%
%%%%%%%%%%%%%%%%%%%%%%%%%  METHOD %%%%%%%%%%%%%%%%%%%%%%%%  
%%%%%%%%%%%%%%%%%%%%%%%%%%%%%%%%%%%%%%%%%%%%%%%%%%%%%%%%%%

\section{Method}
\label{sec:dice_method}

\begin{figure*}[htb]
	\begin{center}
		\includegraphics[width=0.95\linewidth]{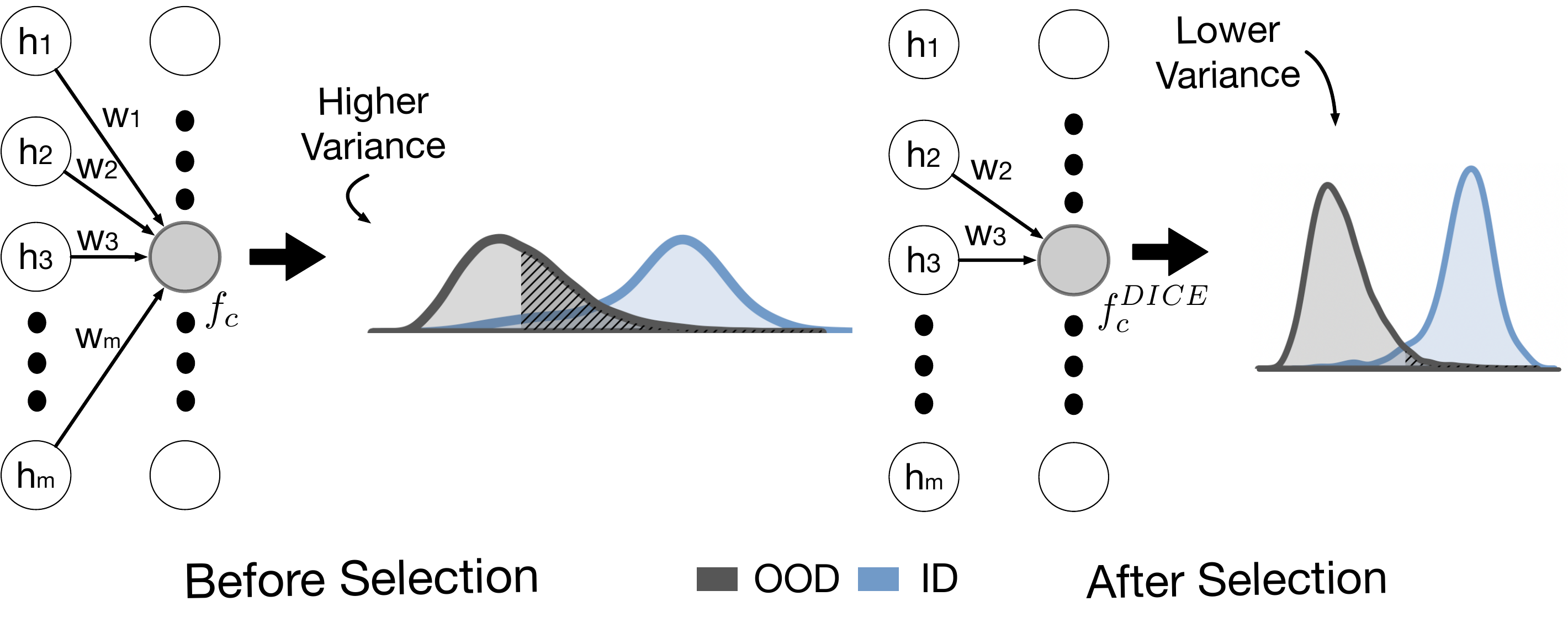}
	\end{center}
	\caption[Illustration of out-of-distribution detection using DICE.]{\small Illustration of out-of-distribution detection using \emph{Directed Sparsification} (\textbf{DICE}). We consider a pre-trained neural network, which encodes an input $\*x$ to a feature vector $h(\*x) \in \mathbb{R}^m$. \textbf{Left}: The logit output $f_c(\*x)$ of class $c$ is a linear combination of activation from \emph{all} units in the preceding layer, weighted by $w_i$. The full connection results in a high variance for OOD data's output, as depicted in the gray. \textbf{Right}: Our proposed approach leverages a selective subset of weights, which effectively reduces the output variance for OOD data, resulting in a sharper score distribution and stronger separability from ID data. The output distributions are based on CIFAR-10 trained network, with ID class label ``frog'' and SVHN as OOD.}
	\label{fig:dice_teaser}
\end{figure*}

\vspace{0.1cm} \noindent \textbf{Method overview.} %
The key idea of DICE is to selectively use a subset of important weights to derive the output for OOD detection. By utilizing sparsification, the network prevents adding irrelevant information to the output. We illustrate our idea in Figure~\ref{fig:dice_teaser}. Without DICE (\emph{left}), the final output is a summation of weighted activations across all units, which can have a high variance for OOD data (colored in gray). In contrast, with DICE (\emph{right}), the variance of output can be significantly reduced, which improves separability from ID data. We proceed with describing our method in detail, and provide the theoretical explanation later in Section~\ref{sec:dice_theory}.

\hiddensubsection{DICE: Directed Sparsification}
\label{sec:dice_dks}
We consider a deep neural network parameterized by $\theta$, which encodes an input $\*x \in \mathbb{R}^d$ to a feature space with dimension $m$. We denote by $ h(\*x) \in \mathbb{R}^m$ the feature vector from the penultimate layer of the network. 
A weight matrix $\mathbf{W} \in \mathbb{R}^{m\times C}$ connects the feature $h(\*x)$ to the output $f(\*x)$. 

\vspace{0.1cm} \noindent \textbf{Contribution matrix.} We perform a \emph{directed sparsification} based on a measure of contribution, and preserve the most important weights in $\*W$. To measure the contribution, we define a contribution matrix $\*V \in \mathbb{R}^{m\times C}$, where each column $\*v_c \in \mathbb{R}^m$ is given by:
\begin{align}
\*v_c = \mathbb{E}_{\*x\in \mathcal{D}_{in}} [\*w_c \odot h(\*x)],
\end{align}
where $\odot$ indicates the element-wise multiplication, and $\*w_c$ indicates weight vector for class $c$. Each element in $\*v_c \in \mathbb{R}^m$ intuitively measures the corresponding unit's average contribution to class $c$, estimated empirically on in-distribution data $\mathcal{D}_{in}$. A larger value indicates a higher contribution to the output $f_c(\*x)$ of class $c$. The vector $\*v_c$ is derived for all classes  $c \in \{1,2,...,C\}$, forming the contribution matrix $\*V$. Each element $\*v_{c}^i\in\*V$ measures the average contribution (\texttt{weight} $\times$ \texttt{activation}) from a unit $i$ to the output class $c\in \{1,2,...,C\}$. %

We can now select the top-$k$ weights based on the $k$-largest elements in $\*V$.  In particular, we define a masking matrix $\*M \in \mathbb{R}^{m\times C}$, which returns a matrix by setting $1$ for entries corresponding to the $k$ largest elements in $\*V$ and setting other elements to $0$.
The model output under \emph{contribution-directed sparsification} is given by
\begin{align}
f^{\text{DICE}}(\*x;\theta) = (\*M\odot\*W)^\top h(\*x) + \mathbf{b},
\end{align}
where $\mathbf{b} \in \mathbb{R}^C$ is the bias vector. The procedure described above essentially accounts for information from the most relevant units in the penultimate layer. Importantly, the sparsification can be conveniently imposed by \emph{post hoc} weight masking on the final layer of a pre-trained network, without changing any parameterizing of the neural network. Therefore one can improve OOD detection while preserving the ID classification accuracy. 

\vspace{0.1cm} \noindent \textbf{Sparsity parameter $p$.} To align with the convention in literature, we use the sparsity parameter $p= 1- \frac{k}{m\cdot C}$ in the remainder of this chapter. A higher $p$ indicates a larger fraction of weights dropped. When $p=0$, the output becomes equivalent to the original output $f(\*x;\theta)$ using dense transformation, where $f(\*x;\theta) = \*W^\top h(\*x) + \*b$. We provide ablations on the sparsity parameter later in Section~\ref{sec:dice_sparsification}.
\hiddensubsection{OOD Detection with DICE}
Our method DICE  in Section~\ref{sec:dice_dks} can be flexibly leveraged by the downstream OOD scoring function:
\begin{align}
\label{eq:dice_threshold_dice}
	\mathcal{S}_{\lambda}(\*x)=\begin{cases} 
      \text{in} & S_\theta(\*x)\ge \lambda \\
      \text{out} & S_\theta(\*x) < \lambda
   \end{cases},
\end{align}
 where a thresholding mechanism is exercised to distinguish between ID and OOD during test time. The threshold $\lambda$ is typically chosen so that a high fraction of ID data (\emph{e.g.,} 95\%) is correctly classified. Following recent work by Liu \emph{et. al}~\citep{liu2020energy}, we derive an energy score using the logit output $f^\text{DICE}(\*x;\theta)$ with contribution-directed sparsification. The function maps the logit outputs $f^\text{DICE}(\*x;\theta)$ to a scalar $E_\theta(\*x) \in \mathbb{R}$, which is relatively lower for ID data:
 \begin{equation}
     S_\theta(\*x) = -E_\theta(\*x) = \log \sum_{c=1}^C \exp(f_c^\text{DICE}(\*x;\theta)).
 \end{equation}
 The energy score can be viewed as the log of the denominator in softmax function:
 \begin{equation}
     p(y | \*x) = \frac{p(\*x,y)}{p(\*x)} = \frac{\exp(f_y(\*x;\theta))}{\sum_{c=1}^C \exp(f_c(\*x;\theta))},
 \end{equation}
and enjoys better theoretical interpretation than using posterior probability $p(y|\*x)$. 
% Note that DICE can also be compatible with an alternative scoring function such as maximum softmax probability (MSP)~\citep{Kevin}, though the performance of MSP is less competitive (see Appendix~\ref{sec:dice_msp}).
Later in Section~\ref{sec:dice_theory}, we formally characterize and explain why DICE improves the separability of the  scores between ID and OOD data. 
\label{sec:dice_ood}

%%%%%%%%%%%%%%%%%%%%%%%%%%%%%%%%%%%%%%%%%%%%%%%%%%%%%%%%%%
%%%%%%%%%%%%%%%%%%%%%%%%%%%  EXP %%%%%%%%%%%%%%%%%%%%%%%%%
%%%%%%%%%%%%%%%%%%%%%%%%%%%%%%%%%%%%%%%%%%%%%%%%%%%%%%%%%%

\section{Experiments}
\label{sec:dice_experiments}
In this section, we evaluate our method on a suite of OOD detection tasks. We begin with the CIFAR benchmarks that are routinely used in literature (Section~\ref{sec:dice_common_benchmark}).  In Section~\ref{sec:dice_imagenet}, we continue with a large-scale OOD detection task based on ImageNet. 
\subsection{Evaluation on CIFAR Benchmarks}
\label{sec:dice_common_benchmark}

\vspace{0.1cm} \noindent \textbf{Experimental details.}
We use  CIFAR-10~\citep{krizhevsky2009learning}, and CIFAR-100~\citep{krizhevsky2009learning} datasets as in-distribution data. We use the standard split with 50,000 training images and 10,000 test images. We evaluate the model on six common OOD benchmark datasets: \texttt{Textures}~\citep{cimpoi2014describing}, \texttt{SVHN}~\citep{netzer2011reading}, \texttt{Places365}~\citep{zhou2017places}, \texttt{LSUN-Crop}~\citep{yu2015lsun}, \texttt{LSUN-Resize}~\citep{yu2015lsun}, and \texttt{iSUN}~\citep{xu2015turkergaze}. We use DenseNet-101 architecture~\citep{huang2017densely} and train on in-distribution datasets. The feature dimension of the penultimate layer is 342. For both CIFAR-10 and CIFAR-100, the model is trained for 100 epochs with batch size 64, weight
decay 0.0001 and momentum 0.9. The start learning rate is 0.1 and decays by a factor of 10 at epochs 50, 75, and 90. We use the validation strategy in Appendix~\ref{sec:dice_val} to select $p$.

\vspace{0.1cm} \noindent \textbf{DICE vs. competitive baselines.} We show the results in Table~\ref{tab:dice_cifar-results}, where DICE outperforms competitive baselines. In particular, we compare with Maximum Softmax Probability~\citep{Kevin}, ODIN~\citep{liang2018enhancing}, Mahalanobis distance~\citep{lee2018simple}, Generalized ODIN~\citep{godin2020CVPR}, Energy score~\citep{liu2020energy}, and ReAct~\citep{sun2021react} (Chapter~\ref{sec:react}). For a fair comparison, all the methods derive the OOD score post hoc from the same pre-trained model, except for G-ODIN which requires model re-training. 

On CIFAR-100, we show that DICE reduces the average FPR95 by \textbf{18.73\%} compared to the vanilla energy score~\citep{liu2020energy} without sparsification. Moreover, our method also outperforms ReAct~\citep{sun2021react} (Chapter~\ref{sec:react}) by {12.55\%}. While {ReAct} only considers activation space, {DICE} examines \emph{both the {weights} and activation} values together---the multiplication of which {directly} determines the network's logit output. 
{Overall our method is more generally applicable}, and can be implemented through a simple {post hoc} weight masking. 

\begin{table}[t]
\centering
\caption[Comparison with out-of-distribution detection method on CIFAR benchmarks.]{\small Comparison with competitive \textit{post hoc} out-of-distribution detection method on CIFAR benchmarks. All values are percentages and are averaged over 6 OOD test datasets. The full results for each evaluation dataset are provided in Appendix~\ref{sec:dice_detailed-cifar}. We report standard deviations estimated across 5 independent runs. $^\S$ indicates an exception, where model retraining using a different loss function is required.} 
\scalebox{0.8}{
\begin{tabular}{lll|lll}
\toprule
 \multicolumn{1}{c}{\multirow{3}{*}{\textbf{Method}}} & \multicolumn{2}{c}{\textbf{CIFAR-10}} & \multicolumn{2}{c}{\textbf{CIFAR-100}} \\
 \multicolumn{1}{c}{} & \multicolumn{1}{l}{\textbf{FPR95}} & \multicolumn{1}{l|}{\textbf{AUROC}} & \multicolumn{1}{l}{\textbf{FPR95}} & \multicolumn{1}{l}{\textbf{AUROC}} & \\ 
  & $\downarrow$ & $\uparrow$ & $\downarrow$ & $\uparrow$ & \\ \midrule
 MSP~\citep{Kevin}  & 48.73 & 92.46 & 80.13 & 74.36 \\
 ODIN~\citep{liang2018enhancing}  & 24.57 & 93.71 & 58.14 & 84.49  \\
  GODIN$^\S$~\citep{godin2020CVPR}  & 34.25 & 90.61 & 52.87 & 85.24  \\
 Mahalanobis~\citep{lee2018simple}  & 31.42	 & 89.15 & 55.37 &	82.73 \\
 Energy~\citep{liu2020energy}  & 26.55 & 94.57 & 68.45 & 81.19  \\
 ReAct~\citep{sun2021react}  & 26.45	& 94.95 & 62.27 & 84.47  \\
\textbf{DICE} & \multicolumn{1}{c}{\textbf{20.83}$^{\pm{1.58}}$} & \textbf{95.24}$^{\pm{0.24}}$ & \multicolumn{1}{c}{\textbf{49.72}$^{\pm{1.69}}$} & \multicolumn{1}{c}{\textbf{87.23}$^{\pm{0.73}}$} \\ \bottomrule  
\end{tabular}}
\label{tab:dice_cifar-results}
\end{table}

\vspace{0.1cm} \noindent \textbf{ID classification accuracy.} Given the \emph{post hoc} nature of DICE, once the input image is marked as ID, one can always use the original fc layer, {which is guaranteed to give identical classification accuracy}. This incurs minimal overhead and results in optimal performance for both classification and OOD detection. We also measure the classification accuracy under 
different sparsification parameter $p$. Due to the space limit, the full results are available in Table~\ref{tab:dice_k-ablation} in Appendix.

\subsection{Evaluation on Large-scale ImageNet Task}
\label{sec:dice_imagenet}

\vspace{0.1cm} \noindent \textbf{Dataset.} We then evaluate DICE on a large-scale ImageNet classification model. Following MOS~\citep{huang2021mos}, we use four OOD test datasets from (subsets of) \texttt{Places365}~\citep{zhou2017places}, \texttt{Textures}~\citep{cimpoi2014describing}, \texttt{iNaturalist}~\citep{inat}, and \texttt{SUN}~\citep{sun} with non-overlapping categories \emph{w.r.t.} ImageNet. The evaluations span a diverse range of domains including fine-grained images, scene images, and textural images. OOD detection for the ImageNet model is more challenging due to both a larger feature space ($m=2,048$) as well as a larger label space $(C=1,000)$. In particular, the large-scale evaluation can be relevant to real-world applications, where the deployed models often operate on images that have high resolution and contain many class labels. Moreover, as the number of feature dimensions increases, noisy signals may increase accordingly, which can make OOD detection more challenging.  

\vspace{0.1cm} \noindent \textbf{Experimental details.} 
We use a pre-trained ResNet-50 model~\citep{he2016identity} for ImageNet-1k provided by Pytorch.
At test time, all images are resized to 224 $\times$ 224. We use the entire training dataset to estimate the contribution matrix and masking matrix $\*M$. We use the validation strategy in Appendix~\ref{sec:dice_val} to select $p$.

\begin{table*}[t]

\caption[Comparison with  out-of-distribution detection methods on ImageNet.]{\small \textbf{Main results.} Comparison with competitive \emph{post hoc} out-of-distribution detection methods. All methods are based on a discriminative model trained on {ImageNet}. $\uparrow$ indicates larger values are better and $\downarrow$ indicates smaller values are better. All values are percentages. \textbf{Bold} numbers are superior results. 
}
\label{tab:dice_main-results}
\scalebox{0.64}{
\begin{tabular}{lllllllllll}
    \toprule
  \multicolumn{1}{c}{\multirow{4}{*}{\textbf{Methods}}} & \multicolumn{8}{c}{\textbf{OOD Datasets}} & \multicolumn{2}{c}{\multirow{2}{*}{\textbf{Average}}} \\ \cline{2-9}
 \multicolumn{1}{c}{} & \multicolumn{2}{c}{\textbf{iNaturalist}} & \multicolumn{2}{c}{\textbf{SUN}} & \multicolumn{2}{c}{\textbf{Places}} & \multicolumn{2}{c}{\textbf{Textures}} & & \\
 \multicolumn{1}{c}{} & FPR95 & AUROC & FPR95 & AUROC & FPR95 & AUROC & FPR95 & AUROC & FPR95 & AUROC \\
 \multicolumn{1}{l}{} & \multicolumn{1}{l}{$\downarrow$} & \multicolumn{1}{l}{$\uparrow$} & \multicolumn{1}{l}{$\downarrow$} & \multicolumn{1}{l}{$\uparrow$} & \multicolumn{1}{l}{$\downarrow$} & \multicolumn{1}{l}{$\uparrow$} & \multicolumn{1}{l}{$\downarrow$} & \multicolumn{1}{l}{$\uparrow$} & \multicolumn{1}{l}{$\downarrow$} & \multicolumn{1}{l}{$\uparrow$} \\ \hline

MSP & 54.99 & 87.74 & 70.83 & 80.86 & 73.99 & 79.76 & 68.00 & 79.61 & 66.95 & 81.99 \\
ODIN  & 47.66 & 89.66 & 60.15 & 84.59 & 67.89 & 81.78 & 50.23 & 85.62 & 56.48 & 85.41 \\
GODIN & 61.91 & 85.40 & 60.83 & 85.60	 & 63.70 & 83.81 & 77.85 & 73.27 & 66.07 & 82.02 \\
Mahalanobis  & 97.00 & 52.65 & 98.50 & 42.41 & 98.40 & 41.79 & 55.80 & 85.01 & 87.43 & 55.47 \\
Energy & 55.72 & 89.95 & 59.26 & 85.89 & 64.92 & 82.86 & 53.72 & 85.99 & 58.41 & 86.17 \\

ReAct & 20.38 & 96.22 & \textbf{24.20} & \textbf{94.20} & \textbf{33.85} & \textbf{91.58} & 47.30 & 89.80 & 31.43 & 92.95 \\ \midrule
\textbf{DICE} & 25.63 & 94.49 & 35.15 & 90.83 & 46.49 & 87.48 & 31.72 & 90.30 & 34.75 & 90.77 \\

\textbf{DICE + ReAct} & \textbf{18.64} & \textbf{96.24} & 25.45 & 93.94 & 36.86 & 90.67 & \textbf{28.07} & \textbf{92.74} & \textbf{27.25} & \textbf{93.40}\\
\bottomrule
\end{tabular}
}

\end{table*}

\vspace{0.1cm} \noindent \textbf{Comparison with baselines.} In Table~\ref{tab:dice_main-results}, we compare DICE with competitive post hoc OOD detection methods. We report performance for each OOD test dataset, as well as the average of the four. 
We first contrast DICE with energy score~\citep{liu2020energy}, which allows us to see the direct benefit of using sparsification under the same scoring function. DICE reduces
the FPR95 drastically from 58.41\% to 34.75\%, a \textbf{23.66}\% improvement using sparsification.  
Second, we contrast with ReAct, which demonstrates strong performance on this challenging task using activation truncation. 
With the truncated activation proposed in ReAct, we show that DICE can further reduce the FPR95 by {5.78}\% with weight sparsification. 
Since the comparison is conducted on the same scoring function and feature activation, the performance improvement from ReAct to DICE+ReAct precisely highlights the benefit of using weight sparsification as opposed to the full weights.
Lastly, Mahalanobis displays limiting performance on ImageNet, while being computationally expensive due to estimating the inverse of the covariance matrix. In contrast, DICE is easy to use in practice, and can be implemented through simple {post hoc} weight masking. %

\section{Discussion and Ablations}
\label{sec:dice_sparsification}

\vspace{0.1cm} \noindent \textbf{Ablation on sparsity parameter $p$.} 
We now characterize the effect of the sparsity parameter $p$.  In Figure~\ref{fig:dice_sparsity}, we summarize the OOD detection performance for DenseNet trained on CIFAR-100, where we vary $p=\{0.1, 0.3, 0.5, 0.7, 0.9, 0.99\}$.  
Interestingly, we observe the performance improves with mild sparsity parameter $p$. A significant improvement can be observed from $p=0$ (no sparsity) to $p=0.1$. As we will theoretically later in Section~\ref{sec:dice_theory}, this is because the leftmost part of units being pruned has larger variances for OOD data (gray shade). Units in the middle part have small variances and contributions for both ID and OOD, therefore leading to similar performance as $p$ increases mildly. This ablation confirms that over-parameterization does compromise the OOD detection ability, and DICE can effectively alleviate the problem. In the extreme case when $p$ is too large (\emph{e.g.}, $p=0.99$), the OOD performance starts to degrade as expected. 

\begin{figure}[htb]
	\begin{center}
		\includegraphics[width=0.7\linewidth]{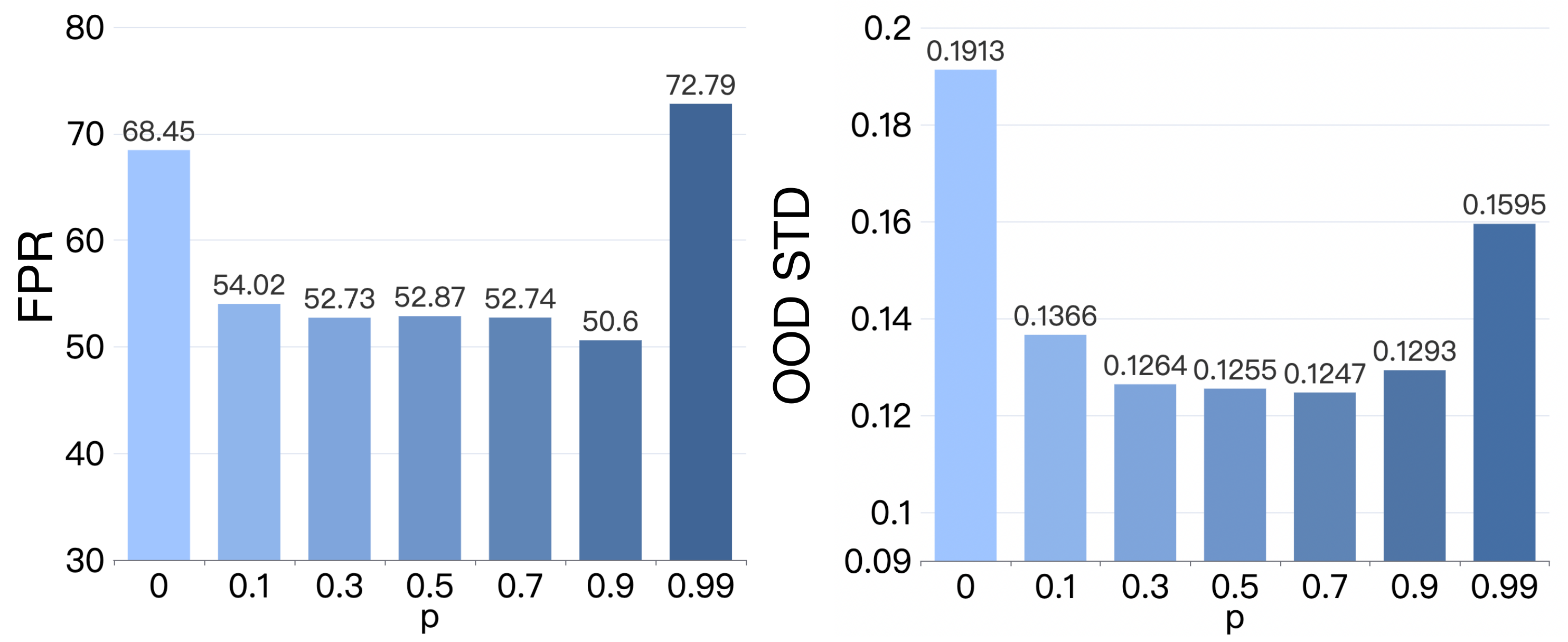}
	\end{center}
	\vspace{-0.3cm}
	\caption[Effect of varying sparsity parameter $p$ during inference time.]{\small Effect of varying sparsity parameter $p$ during inference time. Model is trained on CIFAR-100 using DenseNet101~\citep{huang2017densely}. 
	}
	\vspace{-0.4cm}
	\label{fig:dice_sparsity}
\end{figure}

\vspace{0.1cm} \noindent \textbf{Effect of variance reduction for output distribution.} 
Figure~\ref{fig:dice_teaser} shows that DICE has an interesting variance reduction effect on the output distribution for OOD data, and at the same time preserves the information for the ID data (CIFAR-10, class ``frog''). The output distribution without any sparsity ($p=0$) appears to have a larger variance, resulting in less separability from ID data (see left of Figure~\ref{fig:dice_teaser}). In contrast, sparsification with DICE results in a sharper distribution, which benefits OOD detection. 
In Figure~\ref{fig:dice_sparsity}, we also measure the standard deviation of energy score for OOD data (normalized by the mean of ID data's OOD scores in each setting). By way of sparsification, DICE can reduce the output variance. In Section~\ref{sec:dice_theory}, we formally characterize this and provide a theoretical explanation.

\begin{table}[t]
\caption[Effect of different sparsification methods for OOD detection with ImageNet as ID dataset.]{\small \textbf{Ablation results.} Effect of different \textit{post hoc} sparsification methods for OOD detection with ImageNet as ID dataset. All sparsification methods are based on the {same OOD scoring function}~\citep{liu2020energy}, with sparsity parameter $p=0.7$. All values are percentages and are averaged over multiple OOD test datasets.}
\label{tab:dice_prune-results}
\centering
\scalebox{0.9}{
        \begin{tabular}{lll}
    \toprule
     \multicolumn{1}{l}{\multirow{1}{*}{\textbf{Method}}} 
    & \multicolumn{1}{l}{\textbf{FPR95}}$\downarrow$ & \multicolumn{1}{l}{\textbf{AUROC}}$\uparrow$  \\
    \midrule
     Weight-Droput   & 76.28 &  76.55  \\
     Unit-Droput   & 83.91  & 64.98 \\
     Weight-Pruning &  52.81 & 87.08 \\
     Unit-Pruning   & 90.80 &  49.15 \\
     {{DICE (Ours)}} & \textbf{34.75} & \textbf{90.77} \\
    \bottomrule
    \end{tabular}
    }
\end{table}

\vspace{0.1cm} \noindent \textbf{Ablation on  pruning methods.} In this ablation, we evaluate OOD detection performance under the most common \emph{post hoc} sparsification methods. Here we primarily consider {post hoc} sparsification strategy which
operates conveniently on a \emph{pre-trained} network, instead of training with sparse regularization or architecture modification. The property is especially desirable for the adoption of OOD detection methods in real-world production environments, where the overhead cost of retraining can be sometimes prohibitive.   Orthogonal to existing works on sparsification, our primary goal is to explore the role of sparsification for improved  OOD detection performance, {rather than establishing a generic sparsification algorithm}. We consider the most common strategies, covering both unit-based and weight-based sparsification methods: (1) {unit dropout}~\citep{Nitish2014dropout} which randomly drops a fraction of units, (2) {unit pruning}~\citep{Hao2017pruneUnit} which
drops units with the smallest $l_2$ norm of the corresponding weight vectors, (3) {weight dropout}~\citep{Wan2013weightdropout} which randomly drops weights in the fully connected layer, and (4) {weight pruning}~\citep{Han2015prune} drops weights with the smallest entries under the $l_1$ norm. %
For consistency, we use the same OOD scoring function and the same sparsity parameter for all. 

Our ablation reveals several important insights shown in Table~\ref{tab:dice_prune-results}. 
First, in contrasting weight dropout vs. DICE, a salient performance gap of {41.53}\% (FPR95) is observed under the same sparsity. This suggests the importance of dropping weights \emph{directedly} rather than \emph{randomly}. Second, DICE outperforms a popular $l_1$-norm-based pruning method~\citep{Han2015prune} by up to {18.06}\% (FPR95). While it prunes weights with low magnitude, negative weights with large $l_1$-norm can be kept. The negative weights can undesirably corrupt the output with noisy signals (as shown in Figure~\ref{fig:dice_whytopk}). The performance gain of DICE over~\citep{Han2015prune} attributes to our contribution-directed sparsification, which is better suited for OOD detection.

\begin{table}
    \caption[Ablation on different strategies of choosing a subset of units. ]{\small Ablation on different strategies of choosing a subset of units.  Values are FPR95 (averaged over multiple test datasets).}
    \centering
    \scalebox{0.9}{
        \begin{tabular}{c|lll}
        \toprule
        Method  & CIFAR-10\textbf{$\downarrow$} & CIFAR-100 \textbf{$\downarrow$}\\
        \midrule
         Bottom-$k$ & 91.87 & 99.70 \\
         (Top+Bottom)-$k$  & 24.25 & 59.93\\
          Random-$k$ & 62.12 & 77.48 \\
         Top-$k$ (\textbf{DICE}) & \textbf{20.83}$^{\pm{1.58}}$ & \textbf{49.72}$^{\pm{1.69}}$ \\
         \bottomrule
        \end{tabular}}
    \label{tab:dice_topbot}
\end{table}

\vspace{0.1cm} \noindent \textbf{Ablation on unit selection.}
We have shown that choosing a subset of weights (with \emph{top-k} unit contribution) significantly improves the OOD detection performance. In this ablation, we also analyze those ``lower contribution units'' for OOD detection. Specifically, we consider: (1) \emph{Bottom-k} which only includes $k$ unit contribution with least contribution values, (2) \emph{top+bottom-k} which includes $k$ unit contribution with largest and smallest contribution values, (3) \emph{random-k} which randomly includes $k$ unit contribution and (4) \emph{top-k} which is equivalent to DICE method.  In Table~\ref{tab:dice_topbot}, we show that DICE  outperforms these variants.

%%%%%%%%%%%%%%%%%%%%%%%%%%%%%%%%%%%%%%%%%%%%%%%%%%%%%%%%%%
%%%%%%%%%%%%%%%%%%%%%%%%  THEORY %%%%%%%%%%%%%%%%%%%%%%%%%
%%%%%%%%%%%%%%%%%%%%%%%%%%%%%%%%%%%%%%%%%%%%%%%%%%%%%%%%%%

\section{Why does DICE improve OOD detection?}
\label{sec:dice_theory}

In this section, we formally explain the mechanism by which reliance on irrelevant units hurts OOD detection and how DICE effectively mitigates the issue. Our analysis highlights that DICE reduces the output variance for both ID and OOD data. Below we provide details. 

\vspace{0.1cm} \noindent \textbf{Setup.} For a class $c$, we consider the unit contribution vector $\*v$, the element-wise multiplication between the feature vector $h(\*x)$ and corresponding weight vector $\*w$. 
We contrast the two outputs with and without sparsity:
\begin{align*}
f_c= \sum_{i=1}^m v_i ~~~\text{(w.o sparsity)}  , \\
 f_c^\text{DICE}=\sum_{i\in \text{top units}} v_i ~~~ ~~~\text{(w. sparsity)},
\end{align*}
where $f_c$ is the output using the summation of all units' contribution, and $f_c^\text{DICE}$ takes the input from the top units (ranked based on the average contribution on ID data, see bottom of Figure~\ref{fig:dice_theory}).

\begin{figure}
    \centering
    \includegraphics[width=0.8\textwidth]{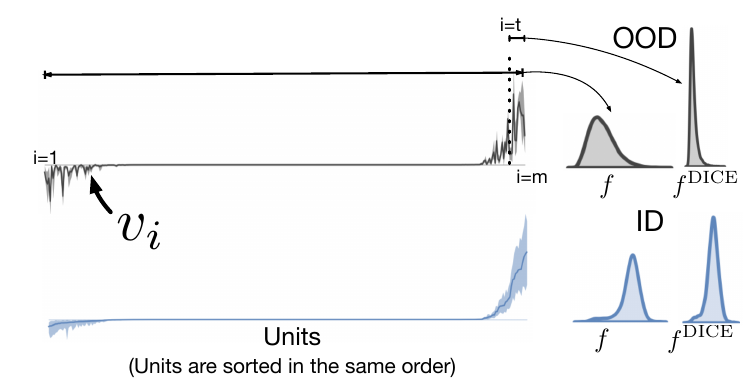}
   \caption[Plot of the average unit contribution in the penultimate layer in a sorted manner.]{\small Units in the penultimate layer are sorted based on the average contribution to a CIFAR-10 class (``airplane''). OOD data (SVHN) can trigger a non-negligible fraction of units with noisy signals on the CIFAR-10 trained model. }
  \label{fig:dice_theory}
\end{figure}

\vspace{0.1cm} \noindent \textbf{DICE reduces the output variance.} We consider the unit contribution vector for OOD data $\*v \in \mathbb{R}^m$, where each element is a \emph{random variable} $v_i$ with  mean $\mathbb{E}[v_i] = \mu_i$ and variance $\mathrm{Var}[v_i] = \sigma_i^2$.
For simplicity, we assume each component is independent, but our theory can be extended to correlated variables (see Remark 1). 
Importantly, {indices in $\*v$ are sorted based on \emph{the same order} of unit contribution on ID data}. %
By using units on the rightmost side, we now show the key result that DICE reduces the output variance.

\begin{proposition}
\label{prop:sum_gaussian}
Let $v_i$ and $v_j$ be two independent {random variables}. Denote the summation $r=v_i + v_j$, we have $\mathbb{E}[r] = \mathbb{E}[v_i] + \mathbb{E}[v_j]$ and $\mathrm{Var}[r] = \mathrm{Var}[v_i] + \mathrm{Var}[v_j]$.
\end{proposition}

\begin{lemma}
\label{lemma:dicev3}
When taking the top $m-t$ units, the output variable $f_c^\text{DICE}$ under sparsification has reduced variance:
$$\mathrm{Var} [f_c] - \mathrm{Var} [f_c^\text{DICE}] = \sum_{i=1}^{t} \sigma_i^2 $$
\end{lemma}
\noindent \textit{Proof}. The proof directly follows Proposition 1. 

\vspace{0.1cm} \noindent \textbf{Remark 1 (Extension to correlated variables).} We can show in a more general case with correlated variables, the variance reduction is:  %
$$\sum_{i=1}^{t} \sigma_i^2  + 2 \sum_{1\le i < j \le m} \mathrm{Cov}(v_i, v_j) - 2\sum_{t< i <j \le m} \mathrm{Cov}(v_i, v_j), $$
where $\mathrm{Cov}(\cdot ,\cdot )$ is the covariance. Our analysis shows that the covariance matrix primarily consists of 0, which indicates the independence of variables. Moreover, the summation of non-zero entries in the full matrix (i.e., the second term) is greater than that of the submatrix with top units (i.e., the third term), resulting in a larger variance reduction than in Lemma~\ref{lemma:dicev3}. 
See complete proof in Appendix~\ref{sec:dice_corr}. 

\vspace{0.1cm} \noindent \textbf{Remark 2.} 
 Energy score is compatible with DICE since it directly operates in the logit space. Our theoretical analysis above shows that DICE reduces the variance of each logit $f_c(\*x)$. 
This means that for detection scores such as energy score, the gap between OOD and ID score will be enlarged after applying DICE, which makes thresholding more capable of separating OOD and ID inputs and benefit OOD detection.

\begin{table}[t]
\centering
     \caption[Difference between the mean of ID's output and OOD's output.]{\small Difference between the mean of ID's output and OOD's output. Here we use CIFAR-100 as ID data and {$\Delta$=$\mathbb{E}_\text{in}[\max_c f_c^\text{DICE}]$ - $\mathbb{E}_\text{out}[\max_c f_c^\text{DICE}]$} is averaged over six common OOD benchmark datasets described in Section~\ref{sec:dice_experiments}.}
    \label{tab:dice_mean_shift}
\scalebox{0.9}{
    \centering
    \begin{tabular}{c|cccccc}
         \toprule 
          \textbf{Sparsity} &  $p=0.9$ & $p=0.7$ & $p=0.5$ & $p=0.3$ & $p=0.1$ & $p=0$\\ \midrule
          {$\Delta$ } & 7.92 & 7.28 & 7.99 & 8.04 & 7.36 &  6.67  \\ 
          \bottomrule
 \end{tabular}

    }
\end{table}

\vspace{0.1cm} \noindent \textbf{Remark 3 (Mean of output).}  Beyond variance, we further show in Table~\ref{tab:dice_mean_shift} the effect of sparsity on the mean of output: $\mathbb{E}_\text{in}[\max_c f_c^\text{DICE}]$ and $\mathbb{E}_\text{out}[\max_c f_c^\text{DICE}]$. The gap between the two directly translates into the OOD score separability. We show that DICE maintains similar (or even enlarges) differences in terms of mean as sparsity $p$ increases. Therefore, DICE overall benefits OOD detection due to both \emph{reduced output variances} and \emph{increased differences of mean}---the combination of both effects leads to stronger separability between ID and OOD. 

\vspace{0.1cm} \noindent \textbf{Remark 4 (Variance reduction on ID data).}
Note that we can also show the effect of variance reduction for ID data in a similar way. 
Importantly, DICE effectively preserves the most important information akin to the ID data, while reducing noisy signals that are harmful to OOD detection. Overall the variance reduction effect on both ID and OOD data leads to stronger separability.

%%%%%%%%%%%%%%%%%%%%%%%%%%%%%%%%%%%%%%%%%%%%%%%%%%%%%%%%%%
%%%%%%%%%%%%%%%%%%%%%%%%  RELATED %%%%%%%%%%%%%%%%%%%%%%%%
%%%%%%%%%%%%%%%%%%%%%%%%%%%%%%%%%%%%%%%%%%%%%%%%%%%%%%%%%%

\section{Additional Related Work}
\label{sec:dice_related}
\vspace{0.1cm} \noindent \textbf{Pruning and sparsification.} 
A great number of effort has been put into improving \emph{post hoc} pruning and training time regularization for deep neural networks~\citep{adadrop2013neurips,Mohammad2016NoiseOut,targetDropout,Han2016deepcomp,Han2015prune,Hao2017pruneUnit,Christos2018l0prune}. Many works obtain a sparse model by training with sparse regularization~\citep{adadrop2013neurips,Mohammad2016NoiseOut,Han2016deepcomp,Christos2018l0prune,sun2019adaptive} or architecture modification~\citep{targetDropout,Hao2017pruneUnit}, while our work primarily considers \emph{post hoc} sparsification strategy which
operates conveniently on a {pre-trained} network. On this line, two popular Bernoulli dropout techniques include unit dropout and weight dropout~\citep{Nitish2014dropout}. 
\emph{Post hoc} pruning strategies truncate weights with low magnitude~\citep{Han2015prune}, or drop units with low weight norms~\citep{Hao2017pruneUnit}.  %
In ~\citep{wong2021leveraging}, they use a sparse linear layer to help identify spurious correlations and explain misclassifications. 
Orthogonal to existing works, our goal is to improve the OOD detection performance rather than accelerate computation and network debugging. In this chapter, we first demonstrate that sparsification can be useful for OOD detection.
An in-depth discussion and comparison of these methods are presented in Section~\ref{sec:dice_sparsification}. 

%

%%%%%%%%%%%%%%%%%%%%%%%%%%%%%%%%%%%%%%%%%%%%%%%%%%%%%%%%%%%%%%%%%%%%%
%%%%%%%%%%%%%%%%%%%%%%%%%%%%%%  Summary %%%%%%%%%%%%%%%%%%%%%%%%%%%%%%%
%%%%%%%%%%%%%%%%%%%%%%%%%%%%%%%%%%%%%%%%%%%%%%%%%%%%%%%%%%%%%%%%%%%%  
\section{Summary}
This chapter provides a simple sparsification strategy termed DICE, which ranks weights based on a contribution measure and then uses the most significant weights to derive the output for OOD detection. We provide both empirical and theoretical insights characterizing and explaining the mechanism by which DICE improves OOD detection. By exploiting the most important weights, DICE provably reduces the output variance for OOD data, resulting in a sharper output distribution and stronger separability from ID data. Extensive experiments show DICE can significantly improve the performance of OOD detection for over-parameterized networks. We hope our research can raise more attention to the importance of weight sparsification for OOD detection.

%%%%%%%%%%%%%%%%%%%%%%%%%%%%%%%%%%%%%%%%%%%%%%%%%%%%%%%%%%%%%%%%%%%%%
%%%%%%%%%%%%%%%%%%%%%%%%%%%%%%  Supp %%%%%%%%%%%%%%%%%%%%%%%%%%%%%%%
%%%%%%%%%%%%%%%%%%%%%%%%%%%%%%%%%%%%%%%%%%%%%%%%%%%%%%%%%%%%%%%%%%%%  

% \clearpage
% \section{Appendix}
% \label{sec:dice_supp}

% \input{chapters/supp_dice}

%% file: chapters/3_knn.tex
\chapter{OOD Detection with Deep Nearest Neighbors}
\label{sec:knn}

\paragraph{Publication Statement.} This chapter is joint work with Yifei Ming, Xiaojin Zhu, and Yixuan Li. The paper version of this chapter appeared in ICML22~\citep{sun2022knnood}. 

\noindent\rule{\textwidth}{1pt}

In this chapter, we delve into an alternate pathway for Out-of-Distribution (OOD) detection, focusing on the utilization of distance-based methodologies. These techniques have shown significant potential, identifying test samples as OOD if their distance from in-distribution (ID) data considerably exceeds a set threshold. However, it is essential to note that previous approaches tend to carry a potent assumption about the distribution of the underlying feature space, an assumption that may not consistently hold true. Therefore, we examine the efficacy of a non-parametric approach using nearest-neighbor distance for OOD detection - an aspect that has hitherto received scant attention in the existing literature. Unlike prior works, our method does not impose any distributional assumption, hence providing stronger flexibility and generality. We demonstrate the effectiveness of nearest-neighbor-based OOD detection on several benchmarks and establish superior performance. Under the same model trained on ImageNet-1k, our method substantially reduces the false positive rate (FPR95) by {24.77}\% compared to a strong baseline SSD+, which uses a parametric approach Mahalanobis distance in detection.

%%%%%%%%%%%%%%%%%%%%%%%%%%%%%%%%%%%%%%%%%%%%%%%%%%%%%%%%%%%%%%%%%%%%%
%%%%%%%%%%%%%%%%%%%%%%%%%%%%%%  INTRO %%%%%%%%%%%%%%%%%%%%%%%%%%%%%%%%
%%%%%%%%%%%%%%%%%%%%%%%%%%%%%%%%%%%%%%%%%%%%%%%%%%%%%%%%%%%%%%%%%%%  

\section{Introduction}
\label{sec:knn_intro}

Modern machine learning models deployed in the open world often struggle with out-of-distribution (OOD) inputs---samples from a different distribution that the network has not been exposed to during training, and therefore should not be predicted at test time. A reliable classifier should not only accurately classify known in-distribution (ID) samples, but also identify as ``unknown'' any OOD input. This gives rise to the importance of OOD detection, which determines whether an input is ID or OOD and enables the
model to take precautions.

A rich line of OOD detection algorithms has been developed recently, among which distance-based methods demonstrated promise~\citep{lee2018simple, tack2020csi, 2021ssd}. Distance-based methods leverage feature embeddings extracted from a model, and operate under the assumption that the test OOD samples are relatively far away from the ID data. For example, \citet{lee2018simple} modeled the feature embedding space as a mixture of multivariate Gaussian distributions, and used the maximum Mahalanobis distance~\citep{mahalanobis1936generalized} to all class centroids for OOD detection. 
However, all these approaches make a strong distributional assumption of the underlying feature space being class-conditional Gaussian. 
As we verify, the learned embeddings can fail the Henze-Zirkler multivariate normality test~\citep{henze1990class}. This limitation leads to the open question:

\begin{center}
\emph{Can we leverage the non-parametric nearest neighbor approach for OOD detection?}
\end{center}

Unlike prior works, the non-parametric approach does not impose {any} distributional assumption about the underlying feature space, hence providing stronger \emph{flexibility and generality}. Despite its simplicity, the nearest neighbor approach has received scant attention. Looking at the literature on OOD detection in the past several years, there has not been any work that demonstrated the efficacy of a non-parametric nearest neighbor approach for this problem. {This suggests that making the seemingly simple idea work is non-trivial}. Indeed, we found that simply using the nearest neighbor distance derived from the feature embedding of a standard classification model is not performant. 

In this chapter, {we challenge the status quo by presenting the first study exploring and demonstrating the efficacy of the non-parametric nearest-neighbor distance for OOD detection}. To detect OOD samples, we compute the  $k$-th nearest neighbor (KNN) distance between the embedding of test input and the embeddings of the training set and use
a threshold-based criterion to determine if the input
is OOD or not. In a nutshell, we perform non-parametric level set estimation, partitioning the data into two sets (ID vs. OOD) based on the deep $k$-nearest neighbor distance.  KNN offers compelling advantages of being: (1) \textbf{distributional assumption free}, (2) \textbf{OOD-agnostic} (\emph{i.e.}, the distance threshold is estimated on the ID data only, and does not rely on information of unknown data),  (3) \textbf{easy-to-use} (\emph{i.e.}, no need to calculate the inverse of the covariance matrix which can be numerically unstable), and (4) \textbf{model-agnostic} (\emph{i.e.}, the testing procedure is applicable to different model architectures and training losses).

Our exploration leads to both empirical effectiveness (Section~\ref{sec:knn_exp} \& \ref{sec:knn_discussion}) and theoretical justification (Section~\ref{sec:knn_theory}). By studying the role of representation space, we show that a compact and normalized feature space is the key to the success of the nearest neighbor approach for OOD detection.  Extensive experiments show that KNN outperforms the parametric approach, and scales well to the large-scale dataset. Computationally, modern implementations of approximate nearest neighbor
search allow us to do this in milliseconds even when the database contains billions of images~\citep{faiss}. On a challenging ImageNet OOD detection benchmark~\citep{huang2021mos}, our KNN-based approach achieves superior performance under a similar inference speed as the baseline methods. The overall simplicity and effectiveness of KNN make it appealing for real-world applications. 
We summarize our contributions below:

\begin{enumerate}

\item We present the first study exploring and demonstrating the efficacy of non-parametric density estimation with nearest neighbors for OOD detection---a simple, flexible yet overlooked approach in literature. It draws attention to the strong promise of the non-parametric approach, which obviates data assumption on the feature space.

\item We demonstrate the superior performance of the KNN-based method on several OOD detection benchmarks, different model architectures (including CNNs and ViTs), and different training losses. Under the same model trained on ImageNet-1k, our method substantially reduces the false positive rate (FPR95) by \textbf{24.77}\% compared to a strong baseline {SSD+}~\citep{2021ssd}, which uses a parametric approach (\emph{i.e.}, Mahalanobis distance~\citep{lee2018simple}) for detection.

\item We offer new insights on the key components to make KNN effective in practice, including feature normalization and a compact representation space. Our findings are supported by extensive ablations and experiments. These insights are valuable to the community in carrying out future research.

\item We provide theoretical analysis, showing that KNN-based OOD detection can reject inputs equivalent to the Bayes optimal estimator. By modeling the nearest neighbor distance in
the feature space, our theory (1) directly connects to our method which also operates in the feature space, and (2) complements our experiments by considering the universality of OOD data. 

\end{enumerate}

\begin{figure*}[htb]
\centering
\includegraphics[width=0.99\linewidth]{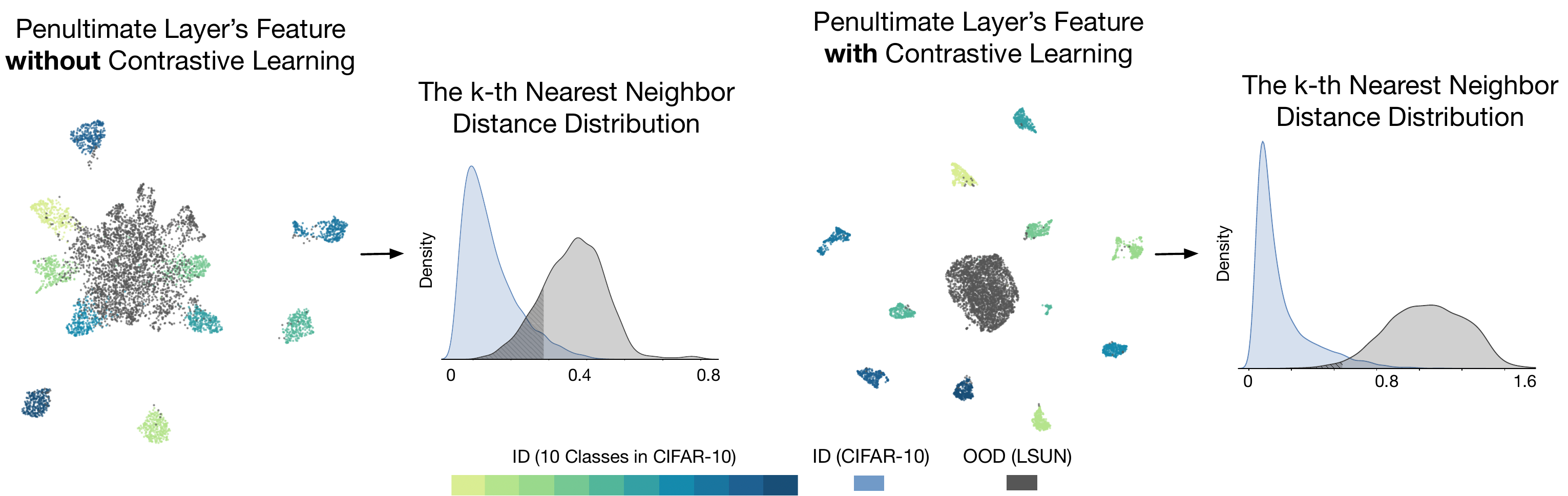}
\caption[Illustration of the framework using nearest neighbors for OOD detection.]{\small Illustration of our framework using nearest neighbors for OOD detection. KNN performs non-parametric level set estimation, partitioning the data into two sets (ID vs. OOD) based on the $k$-th nearest neighbor distance. The distances are estimated from the penultimate feature embeddings, visualized via UMAP~\citep{umap}.
Models are trained on ResNet-18~\citep{he2016identity} using cross-entropy loss (left) v.s. contrastive loss (right). The in-distribution data is CIFAR-10 (colored in non-gray colors) and  OOD data is LSUN (colored in gray). The shaded grey area in the density distribution plot indicates OOD samples that are misidentified as ID data. }
\label{fig:knn_umap}
\end{figure*}

%%%%%%%%%%%%%%%%%%%%%%%%%%%%%%%%%%%%%%%%%%%%%%%%%%%%%%%%%%%%%%
%%%%%%%%%%%%%%%%%%%%%%  METHOD SECTION %%%%%%%%%%%%%%%%%%%%%%%
%%%%%%%%%%%%%%%%%%%%%%%%%%%%%%%%%%%%%%%%%%%%%%%%%%%%%%%%%%%%%%
%
%

\section{Deep Nearest Neighbor for OOD detection}
\label{sec:knn_method}

In this section, we describe our approach using the deep $k$-Nearest Neighbor (KNN) for OOD detection. We illustrate our approach in Figure~\ref{fig:knn_umap}, which at a high level, can be categorized as a distance-based method. Distance-based methods leverage feature embeddings
extracted from a model and operate under the assumption that the test OOD samples are relatively far away from the ID data. Previous distance-based OOD detection methods employed parametric density estimation and modeled the feature embedding space as a mixture of multivariate Gaussian distributions~\citep{lee2018simple}. However, such an approach makes a strong distributional assumption of the learned feature space, which may not necessarily hold\footnote{We verified this by performing the Henze-Zirkler multivariate normality test~\citep{henze1990class} on the embeddings. The testing results show that the feature vectors for each class are not normally distributed at the significance level of 0.05.}.

In this chapter, we instead explore the efficacy of \emph{non-parametric density estimation using nearest neighbors} for OOD detection. Despite the simplicity, KNN approach is not systematically explored or compared in most
current OOD detection papers. Specifically, we compute the  $k$-th nearest neighbor distance between the
embedding of each test image and the training set, and use
a simple threshold-based criterion to determine if an input
is OOD or not. Importantly, we use the normalized penultimate feature $\*z= \phi(\bx) / \lVert \phi(\bx) \rVert_2$ for OOD detection, where $\phi: \mathcal{X} \mapsto \mathbb{R}^{m}$ is a feature encoder. Denote the embedding set of training data as $\mathcal{Z}_n = (\*z_1, \*z_2, ..., \*z_n )$. During testing, we derive the normalized feature vector $\*z^*$ for a test sample $\bx^*$, and calculate the Euclidean distances $\lVert\*z_i - \*z^*\rVert_2$ with respect to embedding vectors $\*z_i \in \mathcal{Z}_n$. We reorder $\mathcal{Z}_n$ according to the increasing distance $\lVert\*z_i - \*z^*\rVert_2$. 
Denote the reordered data sequence as $\mathcal{Z}_n' = (\*z_{(1)}, \*z_{(2)}, ..., \*z_{(n)})$. The decision function for OOD detection is given by: 
\begin{equation*}
    \mathcal{S}_{\lambda}(\*z^*;k) = \mathbf{1}\{-r_k(\*z^*) \ge \lambda\},
\end{equation*} 
where $r_k(\*z^*) = \lVert\*z^* - \*z_{(k)}\rVert_2$ is the distance to the $k$-th nearest neighbor ($k$-NN) and $\mathbf{1}\{\cdot\}$ is the indicator function. The threshold $\lambda$ is typically chosen so that a high fraction of ID data (\emph{e.g.,} 95\%) is correctly classified. The threshold does not depend on OOD data.

\begin{algorithm}[t]
\begin{algorithmic}
   \STATE {\textbf{Input:}} Training dataset $\mathcal{D}_{in}$, pre-trained neural network encoder $\phi$, test sample $\bx^*$, threshold $\lambda$\\
   
   \STATE For $\bx_i$ in the training data $\mathcal{D}_{in}$,  collect feature vectors $\mathcal{Z}_n = (\*z_1, \*z_2, ..., \*z_n )$

   \STATE \textbf{Testing Stage}: 
        \STATE Given a test sample, we calculate feature vector $\*z^* = \phi(\bx^*) / \lVert \phi(\bx^*) \rVert_2$
        \STATE Reorder $\mathcal{Z}_n$ according to the increasing value of $\lVert\*z_i - \*z^*\rVert_2$ as $\mathcal{Z}_n' = (\*z_{(1)}, \*z_{(2)}, ..., \*z_{(n)})$
    \STATE \textbf{Output: } OOD detection decision $\mathbf{1}\{-\lVert\*z^* - \*z_{(k)}\rVert_2 \ge \lambda\}$
\end{algorithmic}
\caption{OOD Detection with Deep Nearest Neighbors}
\label{alg}
\end{algorithm}

We summarize our approach in Algorithm~\ref{alg}. Noticeably, KNN-based OOD detection offers several compelling advantages:
\begin{enumerate}
\vspace{-0.2cm}
     \item \textbf{Distributional assumption free}: Non-parametric nearest neighbor approach does not impose distributional assumptions about the underlying feature space. KNN therefore provides stronger flexibility and generality, and is applicable even when the feature space does not conform to the mixture of Gaussians. 
       \item \textbf{OOD-agnostic}: The testing procedure does not rely on the information of unknown data. The distance threshold is estimated on the ID data only. 
     \item \textbf{Easy-to-use}: Modern implementations of approximate nearest neighbor search
allow us to do this in milliseconds even when the database contains billions of images~\citep{faiss}. In contrast, Mahalanobis distance requires calculating the inverse of the covariance matrix, which can be numerically unstable. 
\vspace{-0.2cm}
   \item \textbf{Model-agnostic}: The testing procedure applies to a variety of model architectures, including CNNs and more recent Transformer-based  ViT models~\citep{dosovitskiy2020image}. Moreover, we will show that KNN is agnostic to the training procedure as well, and is compatible with models trained under different loss functions (\emph{e.g.}, cross-entropy loss and contrastive loss). 
\end{enumerate}
  We proceed to show the efficacy of the KNN-based OOD detection approach in Section~\ref{sec:knn_exp}.

%%%%%%%%%%%%%%%%%%%%%%%%%%%%%%%%%%%%%%%%%%%%%%%%%%%%%%%%%%%%%%
%%%%%%%%%%%%%%%%%%%%  EXPERIMENT SECTION %%%%%%%%%%%%%%%%%%%%%
%%%%%%%%%%%%%%%%%%%%%%%%%%%%%%%%%%%%%%%%%%%%%%%%%%%%%%%%%%%%%%

\section{Experiments}
\label{sec:knn_exp}

The goal of our experimental evaluation is to answer the
following questions: (1) How does KNN fare against the parametric counterpart such as Mahalanobis distance for OOD detection? (2) Can KNN scale to a more challenging task when the training data is large-scale (\emph{e.g.}, ImageNet)? (3) Is KNN-based OOD detection effective under different model architectures and training objectives? (4) How do various design choices affect the performance?

\vspace{0.1cm} \noindent \textbf{Evaluation metrics.} 
We report the following metrics: (1) the false positive rate (\text{FPR}95) of OOD samples when the true positive rate of ID samples is at 95\%, (2) the area under the receiver operating characteristic curve (AUROC), (3) ID classification accuracy (ID ACC), and (4) per-image inference time (in milliseconds, averaged across test images).

\vspace{0.1cm} \noindent \textbf{Training losses.} In our experiments, we aim to show that KNN-based OOD detection is agnostic to the training procedure, and is compatible with models trained under different losses. We consider two types of loss functions, with and without contrastive learning respectively. We employ (1) cross-entropy loss which is the most commonly used training objective in classification, and (2) supervised contrastive learning (SupCon)~\citep{khosla2020supcon}--- the latest development for representation learning, which leverages the label information by aligning samples belonging to the same class in the embedding space.

\vspace{0.1cm} \noindent \textbf{Remark on the implementation.} 
All of the experiments are
based on PyTorch~\citep{pytorch}. Code is made publicly available online. We use Faiss~\citep{faiss}, a library for efficient nearest neighbor search. Specifically, we use \texttt{faiss.IndexFlatL2} as the indexing method with Euclidean distance. In practice, we pre-compute the embeddings for all images and store
them in a key-value map to make KNN search efficient. The embedding vectors for ID data only need to be extracted once after the training is completed. 

\subsection{Evaluation on CIFAR Benchmarks}
\label{sec:knn_common_benchmark}

%%%%%%%%%%%%%%%%%%%%%%%  TABLE CIFAR %%%%%%%%%%%%%%%%%%%%%%%

\begin{table*}[htb]
\caption[Comparison with OOD detection methods on CIFAR-10.]{ \textbf{Results on CIFAR-10.} Comparison with competitive OOD detection methods. All methods are based on a discriminative model trained on {ID data only}, without using outlier data.  $\uparrow$ indicates larger values are better and vice versa. } 
\centering
\scalebox{0.52}{

\begin{tabular}{llllllllllllll} \toprule
\multicolumn{1}{c}{\multirow{3}{*}{\textbf{Method}}} & \multicolumn{10}{c}{\textbf{OOD Dataset}} & \multicolumn{2}{c}{\multirow{2}{*}{\textbf{Average}}} & \multirow{3}{*}{\textbf{ID ACC}} \\
\multicolumn{1}{c}{} & \multicolumn{2}{c}{\textbf{SVHN}} & \multicolumn{2}{c}{\textbf{LSUN }} &  \multicolumn{2}{c}{\textbf{iSUN}} & \multicolumn{2}{c}{\textbf{Texture}} & \multicolumn{2}{c}{\textbf{Places365}} & \multicolumn{2}{c}{} & \\
\multicolumn{1}{c}{} & \textbf{FPR$\downarrow$} & \textbf{AUROC$\uparrow$} & \textbf{FPR$\downarrow$} & \textbf{AUROC$\uparrow$} &  \textbf{FPR$\downarrow$} & \textbf{AUROC$\uparrow$} & \textbf{FPR$\downarrow$} & \textbf{AUROC$\uparrow$} & \textbf{FPR$\downarrow$} & \textbf{AUROC$\uparrow$} & \textbf{FPR$\downarrow$} & \textbf{AUROC$\uparrow$} & \\ \midrule  &\multicolumn{13}{c}{\textbf{Without Contrastive Learning}}          \\
MSP & 59.66 & 91.25 & 45.21 & 93.80 & 54.57 & 92.12 & 66.45 & 88.50 & 62.46 & 88.64 & 57.67 & 90.86 & 94.21 \\
ODIN & 53.78 & 91.30 & 10.93  & 97.93 & 28.44 & 95.51 & 55.59 & 89.47 & 43.40 & 90.98 & 38.43 & 93.04 & 94.21 \\
Energy & 54.41 & 91.22 & 10.19 & 98.05 & 27.52 & 95.59 & 55.23 & 89.37 & 42.77 & 91.02 & 38.02 & 93.05 & 94.21 \\
GODIN & 18.72 & 96.10 & 11.52  & 97.12 & 30.02 & 94.02 & 33.58 & 92.20 & 55.25 & 85.50 & 29.82 & 92.97 & 93.64 \\ 
Maha. & 9.24  & 97.80 & 67.73 & 73.61 & 6.02  & 98.63 & 23.21 & 92.91 & 83.50 & 69.56 & 37.94 & 86.50 & 94.21 \\ 
KNN (Ours) & 27.97 & 95.48 & 18.50 & 96.84 & 24.68 & 95.52  & 26.74 & 94.96 & 47.84 & 89.93  & 29.15 & 94.55 & 94.21\\ \midrule
&\multicolumn{13}{c}{\textbf{With Contrastive Learning}}          \\
CSI & 37.38 & 94.69 & 5.88  & 98.86 & 10.36 & 98.01 & 28.85 & 94.87 & 38.31 & 93.04 & 24.16 & 95.89 &  94.38 \\
SSD+ & 1.51  & 99.68 & 6.09  & 98.48 & 33.60 & 95.16 & 12.98 & 97.70 & 28.41 & 94.72 & 16.52 & 97.15 & \textbf{95.07} \\
KNN+ & 2.42  & 99.52 & 1.78  & 99.48 & 20.06 & 96.74 & 8.09  & 98.56 & 23.02 & 95.36 & \textbf{11.07} & \textbf{97.93} & \textbf{95.07} \\ \bottomrule
\end{tabular}
}
\label{tab:knn_cifar_main}
\end{table*}

%%%%%%%%%%%%%%%%%%%%%%%  TABLE CIFAR %%%%%%%%%%%%%%%%%%%%%%%

\vspace{0.1cm} \noindent \textbf{Datasets.}
We begin with the CIFAR benchmarks that are routinely used in literature. We use the standard split with 50,000 training images and 10,000 test images. We evaluate the methods on common OOD datasets: \texttt{Textures}~\citep{cimpoi2014describing}, \texttt{SVHN}~\citep{netzer2011reading}, \texttt{Places365}~\citep{zhou2017places}, \texttt{LSUN-C}~\citep{yu2015lsun}, and \texttt{iSUN}~\citep{xu2015turkergaze}. All images are of size $32\times 32$.

\vspace{0.1cm} \noindent \textbf{Experiment details.}
We use ResNet-18 as the backbone for CIFAR-10. Following the original settings in \citet{khosla2020supcon}, models with {SupCon} loss are trained for 500 epochs, with the batch size of $1024$. The temperature $\tau$ is $0.1$. The dimension of the penultimate feature where we perform the nearest neighbor search is 512. The dimension of the projection head is 128. We use the cosine annealing learning rate~\citep{loshchilov2016sgdr} starting at 0.5. We use $k=50$ for CIFAR-10 and $k=200$ for CIFAR-100, which is selected from $k=\{1,10,20,50,100,200,500,1000,3000,5000\}$ using the validation method in ~\citep{hendrycks2018deep}. 
We train the models using stochastic gradient descent with momentum 0.9, and weight decay  $10^{-4}$. The model without contrastive learning is trained for 100 epochs. The start learning rate is 0.1 and decays by a factor of 10 at epochs 50, 75, and 90 respectively.

\vspace{0.1cm} \noindent \textbf{Nearest neighbor distance achieves superior performance.}
We present results in Table~\ref{tab:knn_cifar_main}, where non-parametric KNN approach shows favorable performance. Our comparison covers an extensive collection of competitive methods in the literature. For clarity, we divide the baseline methods into two categories: trained with and without contrastive losses. Several baselines derive OOD scores from a model trained with common softmax cross-entropy (CE) loss, including {MSP}~\citep{Kevin}, {ODIN}~\citep{liang2018enhancing}, {Mahalanobis}~\citep{lee2018simple}, and {Energy}~\citep{liu2020energy}. {GODIN}~\citep{godin2020CVPR} is trained using a DeConf-C loss, which does not involve contrastive loss either. 
For methods involving contrastive losses, we use the same network backbone architecture and embedding dimension, while only varying the training objective. These methods include 
{CSI}~\citep{tack2020csi} and
{SSD+}~\citep{2021ssd}. For terminology clarity, KNN refers to our method trained with CE loss, and KNN+ refers to the variant trained with SupCon loss. We highlight two groups of comparisons: 
\begin{itemize}
\vspace{-0.2cm}
\item \textbf{KNN vs. Mahalanobis} (without contrastive learning): Under the \emph{same} model trained with cross-entropy (CE) loss, our method achieves an average FPR95 of 29.15\%, compared to that of Mahalanobis distance 37.94\%. The performance gain precisely demonstrates the advantage of KNN over the parametric method Mahalanobis distance.
\vspace{-0.2cm}
\item \textbf{KNN+ vs. SSD+} (with contrastive loss):   KNN+ and {SSD+}  are fundamentally different in OOD detection mechanisms, despite both benefit from the contrastively learned representations.  {SSD+} modeled the feature embedding space as a multivariate Gaussian distribution for each class, and use Mahalanobis distance~\citep{lee2018simple} for OOD detection. Under the \emph{same} model trained with Supervised Contrastive Learning (SupCon) loss, our method with the nearest neighbor distance reduces the average FPR95 by ${5.45}\%$, which is a relatively \textbf{32.99}\% reduction in error. It further suggests the advantage of using nearest neighbors without making any distributional assumptions on the feature embedding space.  

\end{itemize}
	\vspace{-0.2cm}
The above comparison suggests that the nearest neighbor approach is compatible with models trained both with and without contrastive learning. In addition, KNN is also simpler to use and implement than {CSI}, which relies on sophisticated data augmentations and ensembling in testing. Lastly, as a result of the improved embedding quality, the ID accuracy of the model trained with {SupCon} loss is improved by ${0.86}\%$ on CIFAR-10 and 2.45\% on ImageNet compared to training with the {CE} loss. 
Due to space constraints, we provide results on  DenseNet~\citep{huang2017densely} in Appendix~\ref{sec:knn_other_arc}.

\vspace{0.1cm} \noindent \textbf{Contrastively learned representation helps.} While contrastive learning has been extensively studied in recent literature, {its role remains untapped when coupled with a non-parametric approach} (such as nearest neighbors) for OOD detection. We examine the effect of using supervised contrastive loss for KNN-based OOD detection. We provide both qualitative and quantitative evidence, highlighting advantages over the standard softmax cross-entropy ({CE}) loss. (1) We visualize the learned feature embeddings in Figure~\ref{fig:knn_umap} using UMAP~\citep{umap}, where the colors encode different class labels. A salient observation is that the representation with {SupCon} is more distinguishable and compact than the representation obtained from the {CE} loss. The high-quality embedding space indeed confers benefits for KNN-based OOD detection. (2) Beyond visualization, we also quantitatively compare the performance of KNN-based OOD detection using embeddings trained with {SupCon} vs {CE}. As shown in 
Table~\ref{tab:knn_cifar_main}, KNN+ with contrastively learned representations reduces the FPR95 on all test OOD datasets compared to using embeddings from the model trained with {CE} loss.

\vspace{0.1cm} \noindent \textbf{Comparison with other non-parametric methods.} In Table~\ref{tab:knn_nonparam}, we compare the nearest neighbor approach with other non-parametric methods. For a fair comparison, we use the same embeddings trained with {SupCon} loss. Our comparison  covers  an  extensive  collection of outlier detection methods in literature including: {IForest}~\citep{liu2008iforest}, {OCSVM}~\citep{bernhard2001ocsvm}, 
{LODA}~\citep{2016loda}, 
{PCA}~\citep{shyu2003pca}, and
{LOF}~\citep{breunig2000lof}. The parameter setting for these methods is available in Appendix~\ref{sec:knn_config}. We show that KNN+ outperforms alternative non-parametric methods by a large margin.

%%%%%%%%%%%%%%%%%%%%%%%  TABLE non-param %%%%%%%%%%%%%%%%%%%%%%%
\begin{table}[htb]
\centering
\caption[Comparison with other non-parametric methods.]{Comparison with other non-parametric methods. Results are averaged across all test OOD datasets. Model is trained on CIFAR-10.}
\vspace{0.2cm}
\scalebox{0.95}{
\begin{tabular}{lll} \toprule
 & \textbf{FPR95}$\downarrow$ & \textbf{AUROC}$\uparrow$ \\ \midrule
IForest~\citep{liu2008iforest} & 65.49 & 76.98 \\
OCSVM~\citep{bernhard2001ocsvm}   & 52.27 & 65.16 \\
LODA~\citep{2016loda}    & 76.38 & 62.59 \\
PCA~\citep{shyu2003pca}     & 37.26 & 83.13 \\
LOF~\citep{breunig2000lof}     & 40.06 & 93.47 \\
KNN+ (ours)    & \textbf{11.07} & \textbf{97.93} \\ \bottomrule
\end{tabular}}
\label{tab:knn_nonparam}
\end{table}

%%%%%%%%%%%%%%%%%%%%%%%  TABLE non-param %%%%%%%%%%%%%%%%%%%%%%%

\vspace{0.1cm} \noindent \textbf{Evaluations on hard OOD tasks.} Hard OOD samples are particularly
challenging to detect. To test the limit of the  non-parametric KNN approach, we follow CSI~\citep{tack2020csi} and evaluate on several hard OOD datasets: LSUN-FIX, ImageNet-FIX, ImageNet-R, and CIFAR-100. The results are summarized in Table~\ref{tab:knn_hardood}. Under the same model, {KNN+ consistently outperforms {SSD+}}.

\begin{table*}[htb]
\centering
\caption[Results on hard OOD detection tasks.]{Evaluation (FPR95) on hard OOD detection tasks. The model is trained on CIFAR-10 with SupCon loss. }
\scalebox{0.9}{
\begin{tabular}{ccccc} \toprule
    & \textbf{LSUN-FIX} & \textbf{ImageNet-FIX} & \textbf{ImageNet-R} & \textbf{C-100} \\ \midrule
{SSD+} & 29.86      & 32.26          &  45.62             & 45.50     \\
{KNN+ (Ours)} & \textbf{21.52}       & \textbf{25.92}          & \textbf{29.92}            & \textbf{38.83} \\ \bottomrule 
\end{tabular}}
\label{tab:knn_hardood}
\end{table*}

\subsection{Evaluation on Large-scale ImageNet Task}
\label{sec:knn_imagenet}

We evaluate on a large-scale OOD detection task based on ImageNet~\citep{deng2009imagenet}. Compared to the CIFAR benchmarks above, the ImageNet task is more challenging due to a large amount of training data. Our goal is to verify KNN's performance benefits and whether it scales computationally with millions of samples.

\vspace{0.1cm} \noindent \textbf{Setup.} We use a ResNet-50 backbone~\citep{he2016identity} and train on ImageNet-1k~\citep{deng2009imagenet} with resolution $224 \times 224$. Following the experiments in \citet{khosla2020supcon}, models with {SupCon} loss are trained for 700 epochs, with a batch size of $1024$. The temperature $\tau$ is $0.1$. The dimension of the penultimate feature where we perform the nearest neighbor search is 2048. The dimension of the project head is 128. We use the cosine learning rate~\citep{loshchilov2016sgdr} starting at 0.5. We train the models using stochastic gradient descent with momentum 0.9, and weight decay  $10^{-4}$. We use $k=1000$ which follows the same validation procedure as before. When randomly sampling $\alpha\%$ training data for nearest neighbor search, $k$  is scaled accordingly to $1000 \cdot \alpha\%$.

Following the ImageNet-based OOD detection benchmark in MOS~\citep{huang2021mos}, we evaluate on four test OOD datasets that are subsets of: {Places365}~\citep{zhou2017places}, {Textures}~\citep{cimpoi2014describing}, {iNaturalist}~\citep{inat}, and {SUN}~\citep{sun} with non-overlapping categories \emph{w.r.t.} ImageNet. The evaluations span a diverse range of domains including fine-grained images, scene images, and textural images.

%%%%%%%%%%%%%%%%%%%%%%%  TABLE ImageNet %%%%%%%%%%%%%%%%%%%%%%%
\begin{table*}[htb]
\centering
\caption[Comparison of OOD baseline methods on ImageNet.]{\textbf{Results on ImageNet}. All methods are based on a model trained on ID data only (ImageNet-1k~\citep{deng2009imagenet}). We report the OOD detection performance, along with the per-image inference time.}
\scalebox{0.5}{ 
\begin{tabular}{lllllllllllll} \toprule
\multirow{4}{*}{\textbf{Methods}} & \multirow{4}{*}{\begin{tabular}[c]{@{}l@{}}\textbf{Inference} \\ \textbf{time (ms)}\end{tabular}} & \multicolumn{8}{c}{\textbf{OOD Datasets}} & \multicolumn{2}{c}{\multirow{2}{*}{\textbf{Average}}} & \multirow{4}{*}{\textbf{ID ACC}}\\
 & & \multicolumn{2}{c}{\textbf{iNaturalist}} & \multicolumn{2}{c}{\textbf{SUN}} & \multicolumn{2}{c}{\textbf{Places}} & \multicolumn{2}{c}{\textbf{Textures}} & \multicolumn{2}{c}{} \\
 & & \textbf{FPR95} & \textbf{AUROC} & \textbf{FPR95} & \textbf{AUROC} & \textbf{FPR95} & \textbf{AUROC} & \textbf{FPR95} & \textbf{AUROC} & \textbf{FPR95} & \textbf{AUROC} \\ 
 & & $\downarrow$ & $\uparrow$ & $\downarrow$ & $\uparrow$ & $\downarrow$ & $\uparrow$ & $\downarrow$ & $\uparrow$ & $\downarrow$ & $\uparrow$ \\ \midrule
 &\multicolumn{10}{c}{\textbf{Without Contrastive Learning}}          \\
MSP & 7.04 & 54.99 & 87.74 & 70.83 & 80.86 & 73.99 & 79.76 & 68.00 & 79.61 & 66.95 & 81.99 & 75.08 \\
ODIN & 7.05 & 47.66 & 89.66 & 60.15 & 84.59 & 67.89 & 81.78 & 50.23 & 85.62 & 56.48 & 85.41 & 75.08 \\
Energy & 7.04 & 55.72 & 89.95 & 59.26 & 85.89 & 64.92 & 82.86 & 53.72 & 85.99 & 58.41 & 86.17 & 75.08 \\
GODIN & 7.04 & 61.91 & 85.40 & 60.83 & 85.60 & 63.70 & 83.81 & 77.85 & 73.27 & 66.07 & 82.02 & 70.43 \\
Mahalanobis &35.83 & 97.00 & 52.65 & 98.50 & 42.41 & 98.40 & 41.79 & 55.80 & 85.01 & 87.43 & 55.47 & 75.08 \\
KNN ($\alpha=100\%$) & 10.31 & 59.77 & 85.89 & 68.88 & 80.08 & 78.15 & 74.10 & 10.90 & 97.42 & 54.68 & 84.37 & 75.08 \\
KNN ($\alpha=1\%$) & 7.04 & 59.08 & 86.20 & 69.53 & 80.10 & 77.09 & 74.87 & 11.56 & 97.18 & 54.32 & 84.59 & 75.08\\
\hline
&\multicolumn{10}{c}{\textbf{With Contrastive Learning}}          \\
SSD+ & 28.31 & 57.16 & 87.77 & 78.23 & 73.10 & 81.19 & 70.97 & 36.37 & 88.52 & 63.24 & 80.09 & \textbf{79.10}\\
KNN+ ($\alpha=100\%$) & 10.47 & 30.18 & 94.89 & 48.99 & 88.63 & 59.15 & 84.71 & 15.55 & 95.40 & \textbf{38.47} & \textbf{90.91} & \textbf{79.10} \\
KNN+ ($\alpha=1\%$) & 7.04 & 30.83 & 94.72 & 48.91 & 88.40 & 60.02 & 84.62 & 16.97 & 94.45 & 39.18 & 90.55 & \textbf{79.10} \\ \bottomrule
\end{tabular}}
\label{tab:knn_imagenet_main}
\end{table*}
%%%%%%%%%%%%%%%%%%%%%%%  TABLE ImageNet %%%%%%%%%%%%%%%%%%%%%%%

\vspace{0.1cm} \noindent \textbf{Nearest neighbor approach achieves superior performance without compromising the inference speed.} In Table~\ref{tab:knn_imagenet_main}, we compare our approach with OOD detection methods that are competitive in the literature. The baselines are the same as what we described in Section~\ref{sec:knn_common_benchmark} except for {CSI}\footnote{The training procedure of {CSI} is computationally  prohibitive on ImageNet, which takes three months on 8 Nvidia 2080Tis.}. We report both  OOD detection performance and the inference time (measured by milliseconds). We highlight three trends: (1) KNN+ outperforms the best baseline by \textbf{18.01}\% in FPR95. (2) Compared to {SSD+}, KNN+ substantially reduces the FPR95 by $\textbf{24.77}\%$ averaged across all test sets. The limiting performance of {SSD+} is due to the increased size of label space and data complexity, which makes the class-conditional Gaussian assumption less viable. In contrast, our non-parametric method does not suffer from this issue, and can better estimate the density of the complex distribution for OOD detection. (3) KNN+ achieves strong performance with a comparable inference speed as the baselines. In particular, we show that performing nearest neighbor distance estimation with only $1\%$ randomly sampled training data can yield a similar performance as using the full dataset.

\vspace{0.1cm} \noindent \textbf{Nearest neighbor approach is competitive on ViT.} Going beyond convolutional neural networks, we show in Table~\ref{tab:knn_vit} that the nearest neighbor approach is effective for transformer-based  ViT model~\citep{dosovitskiy2020image}. We adopt the ViT-B/16 architecture fine-tuned on the ImageNet-1k dataset using cross-entropy loss. Under the same ViT model, our non-parametric KNN method consistently outperforms Mahalanobis. 

\begin{table*}[htb]
\centering
\caption[Performance on ViT-B/16 model fine-tuned on ImageNet-1k.]{Performance comparison (FPR95) on ViT-B/16 model fine-tuned on ImageNet-1k.}
\scalebox{0.85}{
\begin{tabular}{ccccc} \toprule
    & \textbf{iNaturalist} & \textbf{SUN} & \textbf{Places} & \textbf{Textures} \\ \midrule
{Mahalanobis (parametric)} & 17.56 & 80.51 & 84.12 & 70.51  \\
{KNN (non-parametric)} & \textbf{7.30} & \textbf{48.40} & \textbf{56.46} & \textbf{39.91} \\ \bottomrule 
\end{tabular}}
\label{tab:knn_vit}
\end{table*}

%%%%%%%%%%%%%%%%%%%%%%% FIGURE K-FPR %%%%%%%%%%%%%%%%%%%%
\begin{figure*}[htb]
	\begin{center}
		\includegraphics[width=0.8\linewidth]{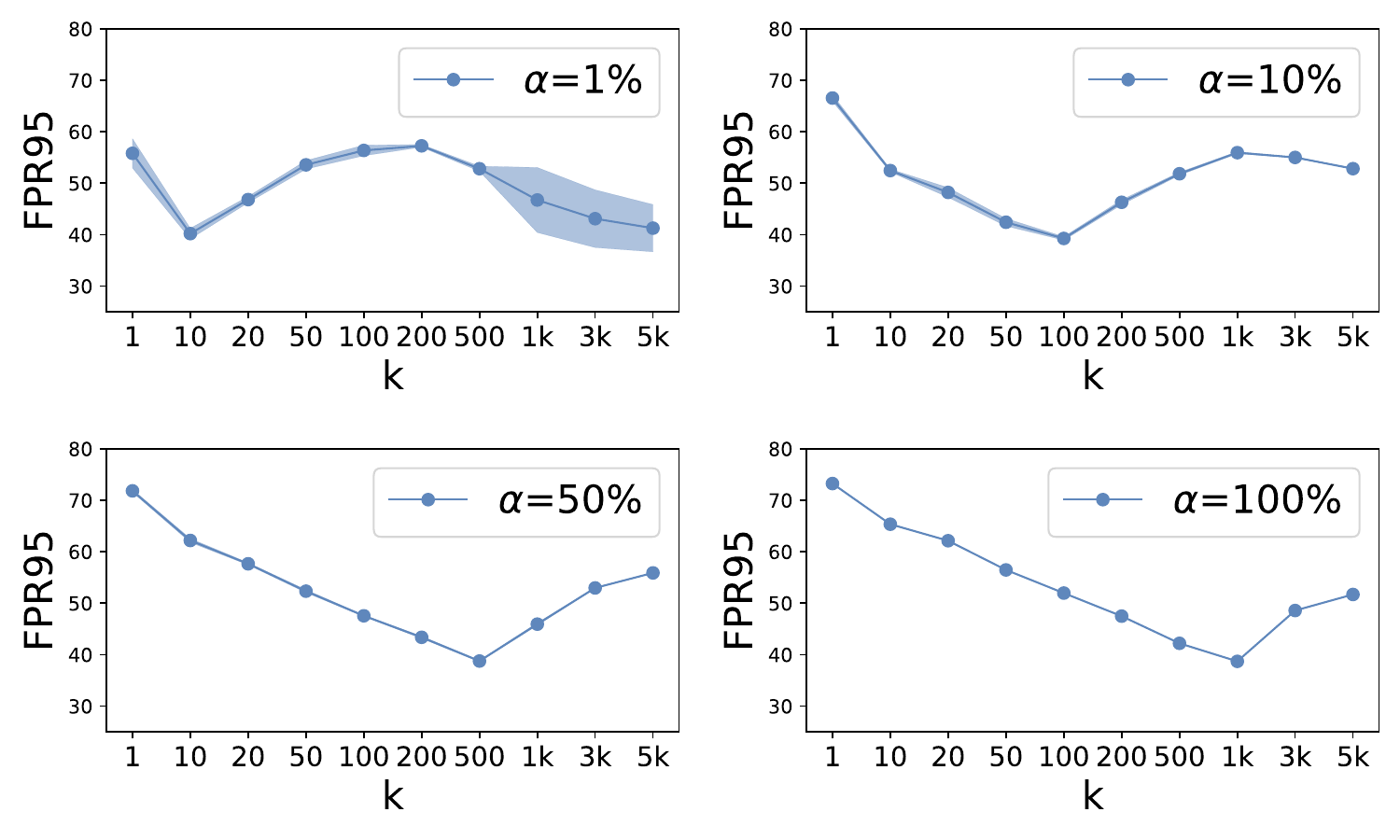}
	\end{center}
	\vspace{-0.5cm}
	\caption[Comparison with the effect of different $k$ and sampling ratio $\alpha$.]{Comparison with the effect of different $k$ and sampling ratio $\alpha$. We report an average FPR95 score over four test OOD datasets. The variances are estimated across 5 different random seeds. The solid blue line represents the averaged value across all runs and the shaded blue area represents the standard deviation. Note that the full ImageNet dataset ($\alpha=100\%$) has 1000 images per class. }
	\label{fig:knn_k_fpr}
\end{figure*}
%%%%%%%%%%%%%%%%%%%%%%% FIGURE K-FPR %%%%%%%%%%%%%%%%%%%%

\section{A Closer Look at KNN-based OOD Detection}
We provide further analysis and ablations to understand the behavior of KNN-based OOD detection. All the ablations are based on the ImageNet model trained with SupCon loss (same as in Section~\ref{sec:knn_imagenet}).

\label{sec:knn_discussion}

%%%%%%%%%%%%%%%%%%%%%%%  Figure  Ablation Many %%%%%%%%%%%%%%%%%%%%%%%

\begin{figure*}[htb]
	\begin{center}
		\includegraphics[width=.8\linewidth]{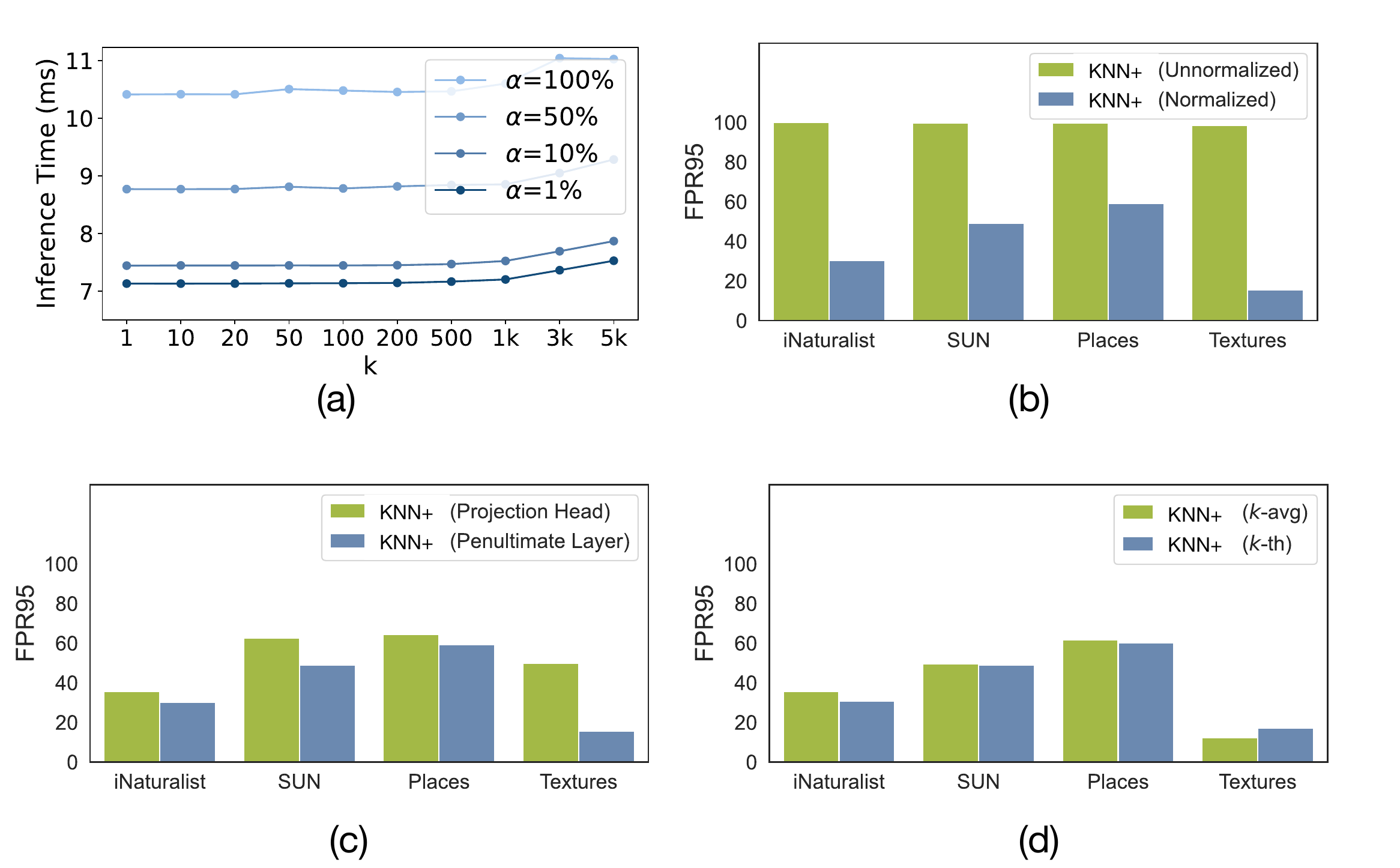}
	\end{center}
 \vspace{-0.5cm}
	\caption[Ablation results on different $k$, sampling ration $\alpha$, normalization, feature layer, etc.]{\small  \textbf{Ablation results.} In (a), we compare the inference speed (per-image) using different $k$ and sampling ration $\alpha$. For (b) (c) (d), the FPR95 value is reported over all test OOD datasets. Specifically, (b) compares the effect of using normalization in the penultimate layer feature vs. without normalization, (c) compares using features in the penultimate layer feature vs the projection head, and (d) compares the OOD detection performance using  $k$-th and averaged $k$ ($k$-avg) nearest neighbor distance. }
	\label{fig:knn_abl_many}
\end{figure*}

%%%%%%%%%%%%%%%%%%%%%%%  Figure  Ablation Many %%%%%%%%%%%%%%%%%%%%%%%

\vspace{0.1cm} \noindent \textbf{Effect of $k$ and sampling ratio.}
In Figure~\ref{fig:knn_k_fpr} and Figure~\ref{fig:knn_abl_many} (a), we systematically analyze the effect of $k$ and the dataset sampling ratios $\alpha$. We vary the number of neighbors $k=\{1,10,20,50,100,200,500,1000,3000,5000\}$ and random sampling ratio $\alpha = \{1\%,10\%,50\%,100\%\}$. We note several interesting observations: (1) The optimal OOD detection (measured by FPR95) remains \emph{similar} under different random sampling ratios $\alpha$. (2) The optimal $k$ is consistent with the one chosen by our validation strategy. For example, the optimal $k$ is 1,000 when $\alpha=100\%$; and the optimal $k$ becomes 10 when $\alpha=1\%$. (3) Varying $k$ does not significantly affect the inference speed when $k$ is relatively small (\emph{e.g.}, $k<1000$) as shown in Figure~\ref{fig:knn_abl_many} (a).

\vspace{0.1cm} \noindent \textbf{Feature normalization is critical.}
In this ablation, we contrast the performance of KNN-based OOD detection with and without feature normalization. The $k$-th NN distance can be derived by  $r_k(\frac{\phi(\bx)}{\lVert(\phi(\bx)\rVert})$ and $r_k(\phi(\bx))$, respectively. 
As shown in Figure~\ref{fig:knn_abl_many} (b), using feature normalization improved the FPR95 drastically by \textbf{61.05}\%, compared to the counterpart without normalization. To better understand this, we look into  the Euclidean distance $r=\lVert u - v \rVert_2$ between two vectors $u$ and $v$. The norm of the feature vector $u$ and $v$ could notably affect the value of the Euclidean distance. Interestingly, recent studies share the observation in Figure~\ref{fig:knn_norm} (a) that the ID data has a larger $L_2$ feature norm than OOD data~\citep{tack2020csi, huang2021importance}. Therefore, the Euclidean distance between ID features can be large (Figure~\ref{fig:knn_norm} (b)).  This contradicts the hope that ID data has a smaller $k$-NN distance than OOD data. 
Indeed, the normalization effectively mitigated this problem, as evidenced in Figure~\ref{fig:knn_norm} (c). Empirically, the normalization plays a key role in the nearest neighbor approach to be successful in OOD detection as shown in Figure~\ref{fig:knn_abl_many} (b).

%%%%%%%%%%%%%%%%%%%%%%%  Figure  Ablation Many %%%%%%%%%%%%%%%%%%%%%%%

\begin{figure}[t]
	\begin{center}
		\includegraphics[width=0.7\linewidth]{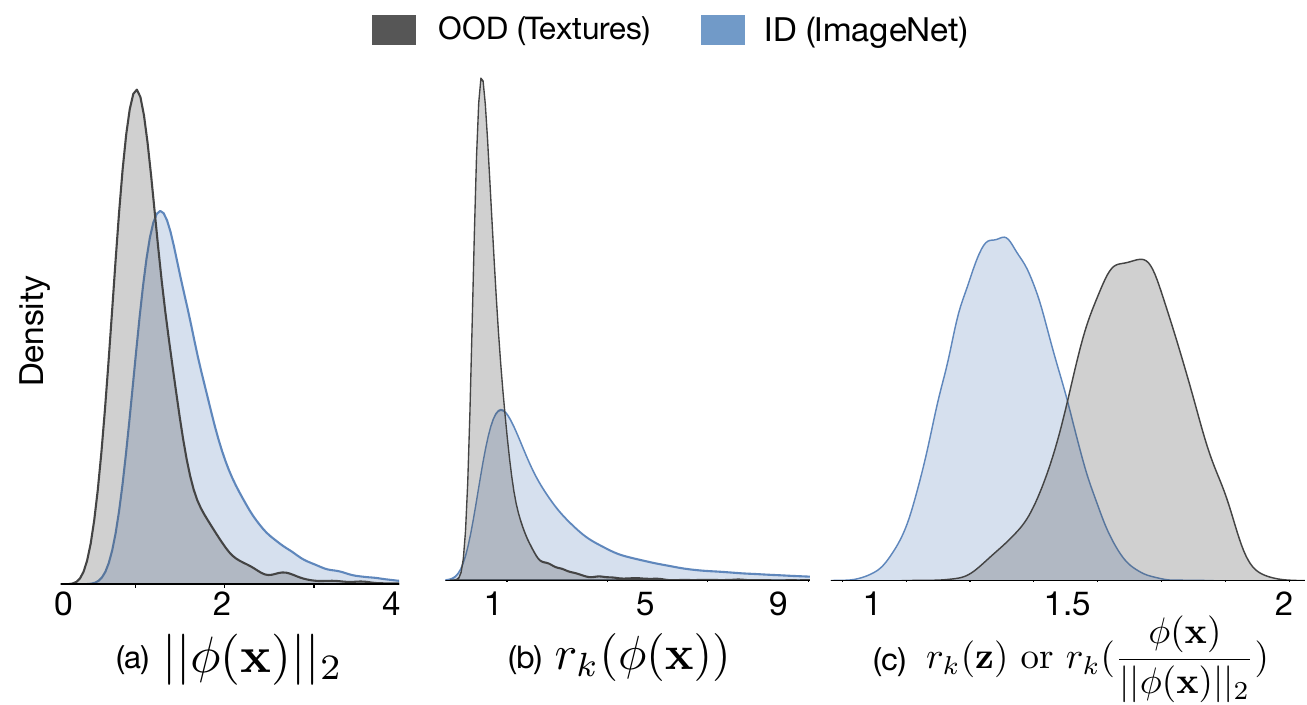}
	\end{center}	
	
	\caption[Distribution plot of features' norm and $k$-NN distance with or w/o. normalization.]{\small Distribution of (a) the $L_2$-norm of feature embeddings, (b) the $k$-NN distance with the \emph{unnormalized} feature embeddings, and (c) the $k$-NN distance with the \emph{normalized} features. }
	\label{fig:knn_norm}

\end{figure}

%%%%%%%%%%%%%%%%%%%%%%%  Figure  Ablation Many %%%%%%%%%%%%%%%%%%%%%%%

\vspace{-0.1cm}
\vspace{0.1cm} \noindent \textbf{Using the penultimate layer's feature is better than using the projection head.}
In this chapter, we follow the convention in {SSD+}, which uses features from the penultimate layer instead of the projection head. We also verify in Figure~\ref{fig:knn_abl_many} (c) that using the penultimate layer's feature is better than using the projection head on all test OOD datasets. This is likely due to the penultimate layer preserving more information than the projection head, which has much smaller dimensions.

\vspace{-0.1cm}
\vspace{0.1cm} \noindent \textbf{KNN can be further boosted by activation rectification.} We show that KNN+ can be made stronger with a ReAct~\citep{sun2021react} (Chapter~\ref{sec:react}). It was shown that the OOD data can have overly high activations on some feature dimensions, and this rectification is effective in suppressing the values. Empirically, we compare the results in Table~\ref{tab:knn_heu} by using the activation rectification and achieve improved OOD detection performance.

\begin{table}[htb]
    \centering
    \caption[Comparison of KNN-based method with and w/o. activation truncation.]{Comparison of KNN-based method with and without activation truncation. The ID data is ImageNet-1k. The value is averaged over all test OOD datasets.}
    \label{tab:knn_heu}
    \scalebox{0.8}{
    \begin{tabular}{c|cc}
    \toprule
    Method & {FPR95}\textbf{$\downarrow$} & {AUROC} \textbf{$\uparrow$} \\  \midrule
    KNN+ & 38.47 & 90.91 \\
     KNN+ (w. ReAct~\citep{sun2021react}) & \textbf{26.45} & \textbf{93.76} \\
     \bottomrule
    \end{tabular}}
\end{table}

\vspace{-0.1cm}
\vspace{0.1cm} \noindent \textbf{Using $k$-th and averaged $k$ nearest nerighbors' distance has similar performance.}
We compare two variants for OOD detection: $k$-th nearest neighbor distance vs. averaged $k$ ($k$-avg) nearest neighbor distance. The comparison is shown in Figure~\ref{fig:knn_abl_many} (d), where the average performance (on four datasets) is on par. The reported results are based on the full ID dataset ($\alpha=100\%$) with the optimal $k$ chosen for $k$-th NN and $k$-avg NN respectively. Despite the similar performance, using $k$-th NN distance has a stronger theoretical interpretation, as we show in the next section.

\section{Theoretical Justification}
\label{sec:knn_theory}

In this section, we provide a theoretical analysis of using KNN for OOD detection. By modeling the KNN in the feature space, our theory (1) directly connects to our method which also operates in the feature space, and (2) complements our experiments by considering the universality of OOD data. Our goal here is to analyze the average performance of our algorithm while being OOD-agnostic and training-agnostic.

\vspace{0.1cm} \noindent \textbf{Setup.}
We consider OOD detection task as a special binary classification task, where the negative samples (OOD) are only available in the testing stage. We assume the input is from feature embeddings space $\mathcal{Z}$ and the labeling set $\mathcal{G} = \{0 (\text{OOD}), 1 (\text{ID})\}$. In the inference stage, the testing set $\{(\*z_i, g_i)\}$ is drawn \textit{i.i.d.} from $P_{\mathcal{Z}\mathcal{G}}$. 

Denote the marginal distribution on $\mathcal{Z}$ as $\mathcal{P}$. We adopt the Huber contamination model~\citep{huber1964} to model the fact that we may encounter both ID and OOD data in test time:
$$
\mathcal{P} = \varepsilon \mathcal{P}_{out} + (1 - \varepsilon )\mathcal{P}_{in},
$$
where $\mathcal{P}_{in}$ and $\mathcal{P}_{out}$ are the underlying distributions of feature embeddings for ID and OOD data, respectively, and $\varepsilon$ is a constant controlling the fraction of OOD samples in testing. We use lower case $p_{in}(\*z_i)$ and $p_{out}(\*z_i)$ to denote the probability density function, where $p_{in}(\*z_i) = p(\*z_i| g_i=1)$ and $p_{out}(\*z_i) = p(\*z_i| g_i=0)$. 

A key challenge in OOD detection (and theoretical analysis) is the lack of knowledge on OOD distribution, which can arise universally outside ID data. We thus try to keep our analysis general and reflect the fact that we do not have any strong prior information about OOD. For this reason, we model OOD data with an equal chance to appear outside of the high-density region of ID data, $p_{out}(\*z) =c_0\mathbf{1}\{p_{in}(\*z) < c_1\}$\footnote{In experiments, as it is difficult to simulate the universal OOD, we approximate it by using a diverse yet finite collection of datasets. Our theory is thus complementary to our experiments and captures the universality of OOD data.}. The Bayesian classifier is known as the optimal binary classifier defined by $h_{Bay}(\*z_i) = \mathbf{1}\{p(g_i = 1|\*z_i) \ge \beta\}$\footnote{Note that $\beta$ does not have to be $\frac{1}{2}$ for the Bayesian classifier to be optimal. $\beta$ can be any value 
larger than $\frac{(1-\epsilon)c_1}{(1-\epsilon)c_1 + \epsilon c_0}$ when $\epsilon c_0 \ge (1-\epsilon)c_1$.}, assuming the underlying density function is given. 

Without such oracle information, our method applies $k$-NN as the distance measure which acts as a probability density estimation, and thus provides the decision boundary based on it. Specifically, KNN's hypothesis class $\mathcal{H}$ is given by $\{h:h_{\lambda, k, \mathcal{Z}_n}(\*z_i) = \mathbf{1}\{-r_k(\*z_i) \ge \lambda\}\}$, where $r_k(\*z_i)$ is the distance to the $k$-th nearest neighbor (\emph{c.f.} Section~\ref{sec:knn_method}).

\vspace{0.1cm} \noindent \textbf{Main result.} We show that our KNN-based OOD detector can reject inputs equivalent to the estimated Bayesian binary decision function. A small KNN distance $r_k(\*z_i)$ directly translates into a high probability of being ID, and vice versa. We depict this in the following Theorem.

\begin{theorem} 
\label{th:knn2bayes}
With the setup specified above, if $\hat{p}_{out}(\*z_i) = \hat{c}_{0}\mathbf{1}\{\hat{p}_{in}(\*z_i;k, n) < \frac{\beta\varepsilon\hat{c}_{0}}{(1-\beta)(1-\varepsilon)}\}$, and $\lambda = -\sqrt[m-1]{\frac{(1-\beta)(1 - \varepsilon) k}{\beta\varepsilon c_bn \hat{c}_{0}}}$, we have
$$
\mathbf{1}\{-r_k(\*z_i) \ge \lambda\} = \mathbf{1}\{\hat{p}(g_i = 1|\*z_i) \ge \beta\},
$$
\end{theorem}
where $\hat{p}(\cdot)$ denotes the empirical estimation. The proof is in Appendix~\ref{sec:knn_theory_sup}.

%%%%%%%%%%%%%%%%%%%%%%%%%%%%%%%%%%%%%%%%%%%%%%%%%%%%%%%%%%%%%%
%%%%%%%%%%%%%%%%%%%%%%%  RELATED SECTION %%%%%%%%%%%%%%%%%%%%%
%%%%%%%%%%%%%%%%%%%%%%%%%%%%%%%%%%%%%%%%%%%%%%%%%%%%%%%%%%%%%%

\section{Additional Related Work}
\label{sec:knn_related}

\vspace{0.1cm} \noindent \textbf{KNN for anomaly detection.} KNN has been explored for anomaly detection ~\citep{jing2014somknn, zhao2020analysis, liron2020knnanomly}, which aims to detect abnormal input samples from one class. We focus on OOD detection, which {requires additionally performing multi-class classification for ID data}. Some other recent works~\citep{dang2015knntabular, gu2019statknn, pires2020knntabular} explore the effectiveness of KNN-based anomaly detection for the tabular data. The potential of using KNN for OOD detection in deep neural networks is currently underexplored. Our work provides both new empirical insights and theoretical analysis of using the KNN-based approach for OOD detection.

%%%%%%%%%%%%%%%%%%%%%%%%%%%%%%%%%%%%%%%%%%%%%%%%%%%%%%%%%%%%%%%%%%%%%
%%%%%%%%%%%%%%%%%%%%%%%%%%%%%%  Summary %%%%%%%%%%%%%%%%%%%%%%%%%%%%%%%
%%%%%%%%%%%%%%%%%%%%%%%%%%%%%%%%%%%%%%%%%%%%%%%%%%%%%%%%%%%%%%%%%%%%  

\section{Summary}
\label{sec:knn_summary}

this chapter presents the first study exploring and demonstrating the efficacy of the non-parametric nearest-neighbor distance for OOD detection. Unlike prior works, the non-parametric approach does not impose {any} distributional assumption about the underlying feature space, hence providing stronger flexibility and generality.
We provide important insights that  
a high-quality feature embedding and a suitable distance measure are two indispensable components for the OOD detection task. Extensive experiments show  KNN-based method can notably improve the performance on several OOD detection benchmarks, establishing superior results. We hope our work inspires future research on using the non-parametric approach to OOD detection. 

%%%%%%%%%%%%%%%%%%%%%%%%%%%%%%%%%%%%%%%%%%%%%%%%%%%%%%%%%%%%%%%%%%%%%
%%%%%%%%%%%%%%%%%%%%%%%%%%%%%%  Supp %%%%%%%%%%%%%%%%%%%%%%%%%%%%%%%
%%%%%%%%%%%%%%%%%%%%%%%%%%%%%%%%%%%%%%%%%%%%%%%%%%%%%%%%%%%%%%%%%%%%  

% \section{Appendix}
% \label{sec:knn_supp}

% \input{chapters/supp_knn}

%% file: chapters/4_nscl.tex
\part{Open-world Representation Learning}
\chapter{When and How Does Known Class Help Discover Unknown Ones? A Spectral Analysis}
\label{sec:nscl}

\paragraph{Publication Statement.} This chapter is joint work with Zhenmei Shi, Yingyu Liang, and Yixuan Li. The paper version of this chapter appeared in ICML23~\citep{sun2023nscl}. 

\noindent\rule{\textwidth}{1pt}

The pivotal progression beyond recognizing OOD samples involves discovering latent classes within these samples. This unique task, known as Novel Class Discovery (NCD), is dedicated to the identification of new classes within an unlabeled dataset by leveraging pre-established knowledge from a labeled set of familiar classes. In the context of open-world representation learning, which accommodates unlabeled samples from both known and novel classes, NCD emerges as a distinct and significant sub-problem, specifically focusing on unveiling these novel classes.

Despite its importance, there is a lack of theoretical foundations for NCD. This chapter bridges the gap by providing an analytical framework to formalize and investigate \emph{when and how known classes can help discover novel classes}. 
Tailored to the NCD problem, we introduce a graph-theoretic representation that can be learned by a novel NCD Spectral Contrastive Loss (NSCL). Minimizing this objective is equivalent to factorizing the graph's adjacency matrix, which allows us to derive a provable error bound and provide the sufficient and necessary condition for NCD. Empirically, NSCL can match or outperform several strong baselines on common benchmark datasets, which is appealing for practical usage while enjoying theoretical guarantees.

\section{Introduction}
\label{sec:nscl_intro}

Though modern machine learning methods have achieved remarkable success~\citep{he2016deep, chen2020simclr, song2020score, wang2022pico}, the vast majority of learning algorithms have been driven by the closed-world setting, where the classes are assumed stationary and unchanged between training and testing. 
However, machine learning models in the open world will inevitably encounter novel classes that are outside the existing known categories~\citep{sun2021react,sun2022knnood,ming2022delving,ming2023exploit}.  Novel Class Discovery (NCD)~\citep{Han2019dtc} has emerged as an important problem, which aims to cluster similar samples in an unlabeled dataset (of novel classes) by way of utilizing knowledge from the labeled data (of known classes).  Key to NCD  is harnessing the power of labeled data for possible knowledge sharing and transfer to the unlabeled data~\citep{hsu2017kcl, Han2019dtc, hsu2019mcl, zhong2021openmix, zhao2020rankstat, yang2022divide, sun2023opencon}.

\begin{figure}[t]
    \centering
    \includegraphics[width=0.75\linewidth]{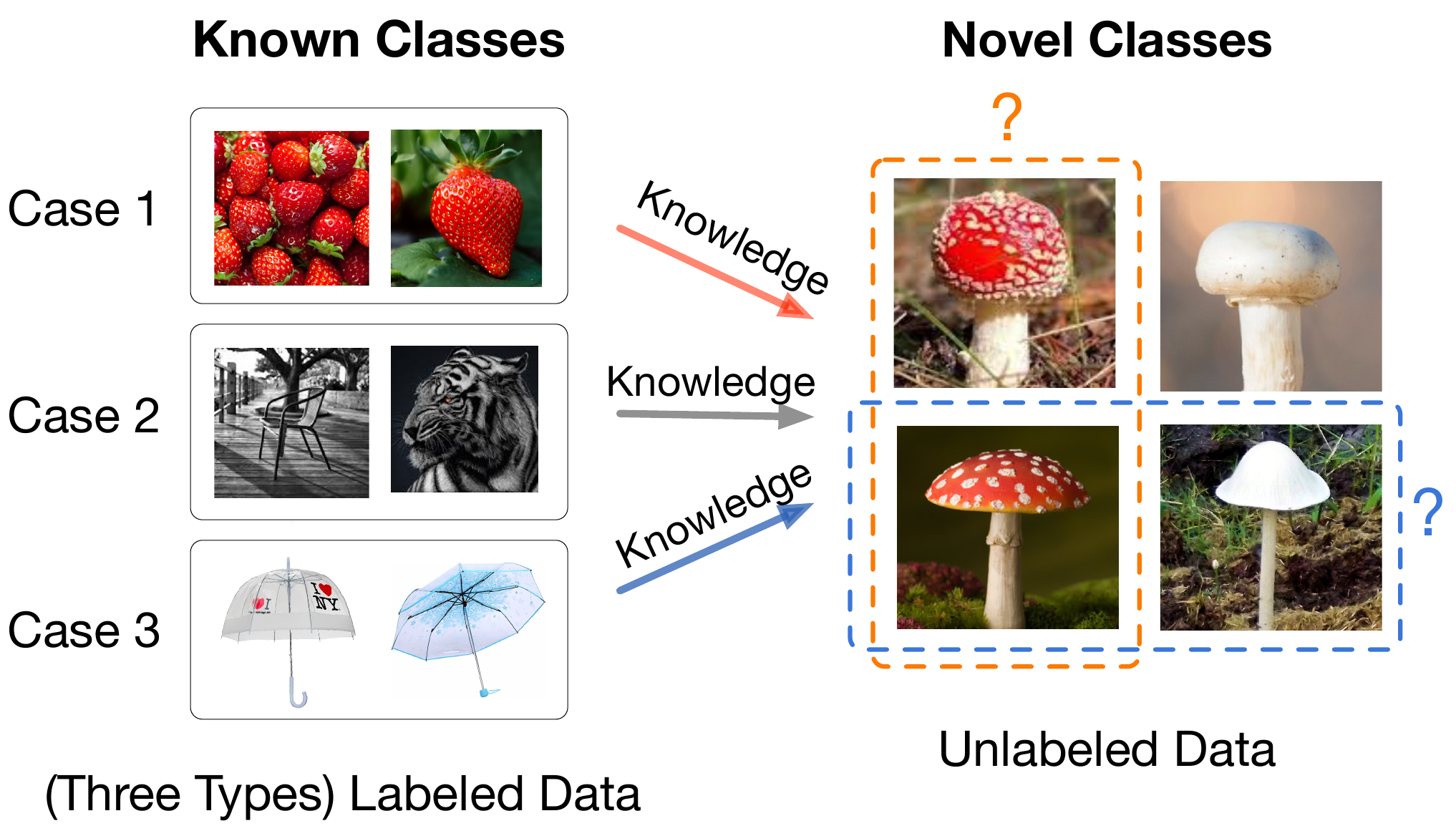}
    % \vspace{-0.3cm}
    \caption[Illustration of scenarios where different known classes could result in different novel clusters. ]{\textbf{ Novel Class Discovery (NCD)} aims to cluster similar samples in unlabeled
data (right), by way of utilizing knowledge from the labeled data (left). We illustrate scenarios where different known classes  could result in different novel clusters (e.g., red mushrooms or mushrooms with umbrella shapes). This chapter aims to provide a formal understanding.
    }
    \label{fig:nscl_teaser}
    % \vspace{-0.4cm}
\end{figure}

One promising approach for NCD is to  learn feature representation jointly from both labeled and unlabeled data, so that meaningful cluster structures emerge as novel classes. We argue that interesting intricacies can arise in this learning process---the resulting novel clusters may be very different, depending on the type of known class provided. We exemplify the nuances in Figure~\ref{fig:nscl_teaser}. In one scenario, the novel class ``red mushroom'' can be discovered, provided with the known class ``strawberry'' of a shared color feature. Alternatively, a different novel class can also emerge by grouping the bottom two images together (as ``mushroom with umbrella shape'' class), if the umbrella-shape images are given as a known class to the learner. We argue---perhaps obviously---that a formalized understanding of the intricate phenomenon is needed. This motivates our research:
\begin{center}
   \textit{\textbf{When and  how does the known class help discover novel classes?}}  
\end{center}

Despite the empirical successes in recent years, there is a limited theoretical understanding and formalization for novel class discovery. To the best of our knowledge, there is no prior work that investigated this research question from a rigorous theoretical standpoint or provided provable error bound. Our work thus complements the existing works by filling in the critical blank.

In this chapter, we start by formalizing a new learning algorithm that facilitates the understanding of NCD from a spectral analysis perspective. Our theoretical framework first introduces a graph-theoretic representation tailored for NCD, where the vertices are all the labeled and unlabeled data points, and classes form connected sub-graphs (Section~\ref{sec:nscl_graph_rep}). 
Based on this graph representation, we then 
introduce a new loss called NCD Spectral Contrastive Loss (NSCL) and show that minimizing our loss is equivalent to performing spectral decomposition on the graph (Section~\ref{sec:nscl_ncd-scl}). Such  equivalence allows us to derive the formal error bound for NCD based on the properties of the graph, which directly encodes the relations between known and novel classes.

We analyze the NCD quality by the linear probing performance on novel data, which is the least error of all possible linear classifiers with the learned representation. 
Our main result (Theorem~\ref{th:nscl_no_approx}) suggests that the linear probing  error can be significantly reduced (even to 0) when the linear span of known samples' feature covers the ``ignorance space'' of unlabeled data in discovering novel classes. 
Lastly, we verify that our theoretical guarantees can translate into empirical effectiveness. In particular, NSCL establishes competitive performance on common NCD benchmarks, outperforming the best baseline by \textbf{10.6}\% on the CIFAR-100-50 dataset (with 50 novel classes). 

Our \textbf{main contributions} are: 
\vspace{-0.2cm}
\begin{enumerate}
    \item We provide the first provable framework for the NCD problem, formalizing it by spectral decomposition of the graph containing both known and novel data. 
    Our framework  allows the research community to gain insights from a graph-theoretic perspective. 
    \vspace{-0.2cm}
    \item We propose a new loss called NCD Spectral Contrastive Loss
(NSCL) and show that minimizing our loss is equivalent to performing singular decomposition on the graph. The loss leads to strong empirical performance while enjoying theoretical guarantees.
    \vspace{-0.2cm}
    \item We provide theoretical insight by formally defining the semantic relationship between known and novel classes. Based on that, we derive an error bound of novel class discovery  and investigate the sufficient and necessary conditions for the perfect discovery results.
\end{enumerate}

% \vspace{-0.2cm}
\section{Setup~}
\label{sec:nscl_setup}

In Section~\ref{sec:prob}, we delineated the problem setup for Open-World Representation Learning, taking into account unlabeled samples from both known and novel classes. In this chapter,  we shift our focus to a more specific subset of the problem—Novel Class Discovery—which exclusively examines unlabeled samples from unidentified classes.
Formally, we describe the data setup and learning goal for novel class discovery (NCD).

\vspace{0.1cm} \noindent \textbf{Data setup.} We consider the empirical training set  $\mathcal{D}_{l} \cup \mathcal{D}_{u}$ as a union of labeled and unlabeled data. The labeled dataset is given by $\mathcal{D}_{l} = \{(\bar{x}_1,y_1),\ldots,(\bar{x}_i,y_i),\ldots\}$, where $y_i$ belongs to known class space $\mathcal{Y}_l$; and the unlabeled dataset is $\mathcal{D}_{u} = \{\bar{x}_1, \ldots,\bar{x}_j,\ldots\}$. We assume that each unlabeled sample $\bar x \in \mathcal{D}_u$ belongs to one of the \textbf{novel} classes, \emph{which do not overlap with the {known} classes $\mathcal{Y}_l$}.  We use $\mathcal{P}_{l}$ and $\mathcal{P}_{u}$ to denote the  marginal distributions of labeled and unlabeled data in the input space. Further, we let $\mathcal{P}_{l_i}$ denote the distribution of labeled samples with class label $i \in \mathcal{Y}_l$.

\vspace{0.1cm} \noindent \textbf{Learning goal.} We assume that there exists an underlying class space $\mathcal{Y}_{u} = \{1, ..., |\mathcal{Y}_u|\}$ for unlabeled data $\mathcal{X}_u$, which is not revealed to the learner. The goal of novel class discovery is to learn a clustering for the novel data, which can be mapped to $\mathcal{Y}_{u}$ with low  error. %

%

% \vspace{-0.3cm}
\section{Spectral Contrastive Learning for Novel Class Discovery}
\label{sec:nscl_method}

In this section, we introduce a new learning algorithm for NCD, from a graph-theoretic perspective. NCD is inherently a clustering problem---grouping similar points in unlabeled data $\mathcal{D}_u$ into the same cluster, by way of possibly utilizing helpful information from the labeled data $\mathcal{D}_l$. This clustering process can be fundamentally modeled by a graph, where the vertices are all the data points and classes form connected sub-graphs. Our novel framework first introduces a graph-theoretic representation for NCD, where edges connect similar data points (Section~\ref{sec:nscl_graph_rep}). We then 
propose a new loss that performs spectral decomposition on the similarity graph and can be
 written as a contrastive learning objective on neural net representations (Section~\ref{sec:nscl_ncd-scl}).   %

\subsection{Graph-Theoretic Representation for NCD} 
\label{sec:nscl_graph_rep}

We start by formally defining the augmentation graph and adjacency matrix. 
For notation clarity, we use $\bar x$ to indicate the natural sample (raw inputs without augmentation). Given an $\bar x$, we use $\mathcal{T}(x|\Bar{x})$ to denote the probability of $x$ being augmented from $\Bar{x}$. For instance, when $\Bar{x}$ represents an image, $\mathcal{T}(\cdot|\Bar{x})$ can be the distribution of common augmentations such as Gaussian blur, color distortion, and random cropping. 
The augmentation allows us to define a general population space $\mathcal{X}$, which contains all the original images along with their augmentations. In our case, $\mathcal{X}$ ($|\mathcal{X}|=N$) is composed of two parts $\mathcal{X}_l$ ($|\mathcal{X}_l| = N_l$), $\mathcal{X}_u$ ($|\mathcal{X}_u| = N_u$) which represents the division into labeled data with known classes and unlabeled data with novel classes respectively. 
Unlike unsupervised learning~\citep{chen2020simclr}, NCD has access to both labeled and unlabeled data. This leads to two cases where two samples $x$ and $x^+$ form a {\textbf{positive pair}} if: 
\begin{enumerate}[(a)]
    \item 
    % \vspace{-0.2cm}
    $x$ and $x^+$ are augmented from the same unlabeled image $\Bar{x}_u\sim \mathcal{P}_u$.
    
    \item $x$ and $x^+$ are augmented from two labeled samples $\Bar{x}_l$ and $\Bar{x}'_l$ \emph{with the same known class $i$}. In other words, both $\Bar{x}_l$ and $\Bar{x}'_l$ are drawn independently from $\mathcal{P}_{l_i}$.%
\end{enumerate}
We define the graph $G(\mathcal{X}, w)$ with vertex set $\mathcal{X}$ and edge weights $w$. For any two augmented data $x, x' \in \mathcal{X}$, $w_{x x'}$ is the marginal probability of generating the pair $(x,x')$:
\begin{align}
\begin{split}
w_{x x^{\prime}} &\triangleq \alpha \sum_{i \in \mathcal{Y}_l}\mathbb{E}_{\bar{x}_{l} \sim {\mathcal{P}_{l_i}}} \mathbb{E}_{\bar{x}'_{l} \sim {\mathcal{P}_{l_i}}} \tikzmarknode{c2}{\highlight{red}{$\mathcal{T}(x | \bar{x}_{l}) \mathcal{T}\left(x' | \bar{x}'_{l}\right)  $}} \\ &+ 
    \beta \mathbb{E}_{\bar{x}_{u} \sim {\mathcal{P}_u}} \tikzmarknode{c1}{\highlight{blue}{$ \mathcal{T}(x| \bar{x}_{u}) \mathcal{T}\left(x'| \bar{x}_{u}\right) $}},
    \label{eq:nscl_def_wxx}
    \vspace{1cm}
\end{split}
\end{align}
\begin{tikzpicture}[overlay,remember picture,>=stealth,nodes={align=left,inner ysep=1pt},<-]
    \path (c2.south) ++ (0,0.1em) node[anchor=north west,color=red!67] (c2t){\textit{ case (b)}};
    \draw [color=red!87](c2.south) |- ([xshift=-0.3ex,color=red] c2t.south east);
    \path (c1.south) ++ (0,0.1em) node[anchor=north west,color=blue!67] (c1t){\textit{ case (a)}};
    \draw [color=blue!87](c1.south) |- ([xshift=-0.3ex,color=blue]c1t.south east);
\end{tikzpicture}

where $\alpha,\beta$ modulates the importance between unlabeled and labeled data. 
The magnitude of $w_{xx'}$ indicates the ``positiveness'' or similarity between  $x$ and $x'$. 
We then use $w_x = \sum_{x' \in \mathcal{X}}w_{xx'}$ to denote the total edge weights connected to vertex $x$.

As a standard technique in graph theory~\citep{chung1997spectral}, we use the \textit{normalized adjacency matrix}:
\begin{equation}
    \dot{A}\triangleq D^{-1 / 2} A D^{-1 / 2},
    \label{eq:nscl_def}
\end{equation}
where  $A \in \mathbb{R}^{N \times N}$ is adjacency matrix with entries $A_{x x^\prime}=w_{x x^{\prime}}$ and $D \in \mathbb{R}^{N \times N}$ is a diagonal matrix with $D_{x x}=w_x.$ The normalization balances the degree of each node,  reducing the influence of vertices with very large degrees. The adjacency matrix defines the probability of $x$ and $x^{\prime}$  being considered as the positive pair from the perspective of augmentation, which helps derive the NCD Spectral Contrastive Loss as we show next.

% \vspace{-0.2cm}
\subsection{NCD Spectral Contrastive Learning}
\label{sec:nscl_ncd-scl}
In this subsection, we propose a formal definition of NCD Spectral Contrastive Loss, which can be derived from a spectral decomposition of $\dot{A}$. The derivation of the loss is inspired by ~\citep{haochen2021provable}, and allows us to theoretically show the equivalence between learning  feature embeddings and the projection on the top-$k$ SVD components of $\dot{A}$. Importantly, such equivalence facilitates the theoretical understanding based on the semantic relation between known and novel classes encoded in $\dot{A}$. 

Specifically, we consider low-rank matrix approximation:
\begin{equation}
    \min _{F \in \mathbb{R}^{N \times k}} \mathcal{L}_{\mathrm{mf}}(F, A)\triangleq\left\|\Dot{A}-F F^{\top}\right\|_F^2
    \label{eq:nscl_lmf}
\end{equation}
According to the Eckart–Young–Mirsky theorem~\citep{eckart1936approximation}, the minimizer of this loss function is $F^*\in \mathbb{R}^{N \times k}$ such that $F^* F^{*\top}$ contains the top-$k$ components of $\Dot{A}$'s SVD decomposition. 

Now, if we view each row $\*f_x^{\top}$ of $F$ as a learned feature embedding  $f:\mathcal{X}\mapsto \mathbb{R}^k$, the $\mathcal{L}_{\mathrm{mf}}(F, A)$ can be written as a form of the contrastive learning objective. We formalize this connection in Theorem~\ref{th:nscl_ncd-scl} below.

\begin{theorem}
\label{th:nscl_ncd-scl} 
We define $\*f_x = \sqrt{w_x}f(x)$ for some function $f$. Recall $\alpha,\beta$ are hyper-parameters defined in Eq.~\eqref{eq:nscl_def_wxx}. Then minimizing the loss function $\mathcal{L}_{\mathrm{mf}}(F, A)$ is equivalent to minimizing the following loss function for $f$, which we term \textbf{NCD Spectral Contrastive Loss (NSCL)}:
\begin{align}
\begin{split}
    \mathcal{L}_{nscl}(f) &\triangleq - 2\alpha \mathcal{L}_1(f) 
- 2\beta  \mathcal{L}_2(f) \\ & + \alpha^2 \mathcal{L}_3(f) + 2\alpha \beta \mathcal{L}_4(f) +  
\beta^2 \mathcal{L}_5(f),
\label{eq:nscl_def_nscl}
\end{split}
\end{align} where
\begin{align*}
    \mathcal{L}_1(f) &= \sum_{i \in \mathcal{Y}_l}\underset{\substack{\bar{x}_{l} \sim \mathcal{P}_{{l_i}}, \bar{x}'_{l} \sim \mathcal{P}_{{l_i}},\\x \sim \mathcal{T}(\cdot|\bar{x}_{l}), x^{+} \sim \mathcal{T}(\cdot|\bar{x}'_l)}}{\mathbb{E}}\left[f(x)^{\top} {f}\left(x^{+}\right)\right] , \\
    \mathcal{L}_2(f) &= \underset{\substack{\bar{x}_{u} \sim \mathcal{P}_{u},\\x \sim \mathcal{T}(\cdot|\bar{x}_{u}), x^{+} \sim \mathcal{T}(\cdot|\bar{x}_u)}}{\mathbb{E}}
\left[f(x)^{\top} {f}\left(x^{+}\right)\right], \\
    \mathcal{L}_3(f) &= \sum_{i \in \mathcal{Y}_l}\sum_{j \in \mathcal{Y}_l}\underset{\substack{\bar{x}_l \sim \mathcal{P}_{{l_i}}, \bar{x}'_l \sim \mathcal{P}_{{l_{j}}},\\x \sim \mathcal{T}(\cdot|\bar{x}_l), x^{-} \sim \mathcal{T}(\cdot|\bar{x}'_l)}}{\mathbb{E}}
\left[\left(f(x)^{\top} {f}\left(x^{-}\right)\right)^2\right], \\
    \mathcal{L}_4(f) &= \sum_{i \in \mathcal{Y}_l}\underset{\substack{\bar{x}_l \sim \mathcal{P}_{{l_i}}, \bar{x}_u \sim \mathcal{P}_{u},\\x \sim \mathcal{T}(\cdot|\bar{x}_l), x^{-} \sim \mathcal{T}(\cdot|\bar{x}_u)}}{\mathbb{E}}
\left[\left(f(x)^{\top} {f}\left(x^{-}\right)\right)^2\right], \\
    \mathcal{L}_5(f) &= \underset{\substack{\bar{x}_u \sim \mathcal{P}_{u}, \bar{x}'_u \sim \mathcal{P}_{u},\\x \sim \mathcal{T}(\cdot|\bar{x}_u), x^{-} \sim \mathcal{T}(\cdot|\bar{x}'_u)}}{\mathbb{E}}
\left[\left(f(x)^{\top} {f}\left(x^{-}\right)\right)^2\right].
\end{align*}
\end{theorem}
\begin{proof} (\textit{sketch})
We can expand $\mathcal{L}_{\mathrm{mf}}(F, A)$ and obtain
\begin{align*}
&\mathcal{L}_{\mathrm{mf}}(F, A) =\sum_{x, x^{\prime} \in \mathcal{X}}\left(\frac{w_{x x^{\prime}}}{\sqrt{w_x w_{x^{\prime}}}}-\*f_x^{\top} \*f_{x^{\prime}}\right)^2 = const + \\
&\sum_{x, x^{\prime} \in \mathcal{X}}\left(-2 w_{x x^{\prime}} f(x)^{\top} {f}\left(x^{\prime}\right)+w_x w_{x^{\prime}}\left(f(x)^{\top}{f}\left(x^{\prime}\right)\right)^2\right)
\end{align*} 
The form of $\mathcal{L}_{nscl}(f)$ is derived from plugging $w_{xx'}$ (defined in Eq.~\eqref{eq:nscl_def_wxx}) and $w_x$. 
We include the details in Appendix~\ref{sec:nscl_proof-nscl}. 
\end{proof}

\vspace{0.1cm} \noindent \textbf{Interpretation of $\mathcal{L}_{nscl}(f)$.} 
At a high level, $\mathcal{L}_1$ and $\mathcal{L}_2$ push the embeddings of \textbf{positive pairs} to be closer while $\mathcal{L}_3$, $\mathcal{L}_4$ 
 and $\mathcal{L}_5$ pull away the embeddings of \textbf{negative pairs}. In particular, $\mathcal{L}_1$ samples two random augmentation views of two images from labeled data with the \textbf{same} class label, and $\mathcal{L}_2$ samples two views from the same image in $\mathcal{X}_{u}$. For negative pairs, $\mathcal{L}_3$ uses two augmentation views from two samples in $\mathcal{X}_{l}$ with \textbf{any} class label. $\mathcal{L}_4$ uses two views of one sample in $\mathcal{X}_{l}$ and another one in $\mathcal{X}_{u}$. $\mathcal{L}_5$ uses two views from two random samples in $\mathcal{X}_{u}$. 

\section{Theoretical Analysis}
\label{sec:nscl_theory}
So far we have presented a spectral approach for NCD based on the augmentation graph. Under this formulation, we now formally investigate and analyze:
\emph{\textbf{when and how does the known class help discover novel class?}}  We start by showing that analyzing the linear probing performance is equivalent to analyzing the regression residual using singular vectors of $\Dot{A}$ in Sec.~\ref{sec:nscl_setup}. We then construct a toy example to illustrate and verify the key insight in Sec.~\ref{sec:nscl_toy}. We finally provide a formal theory for the general case in Sec.~\ref{sec:nscl_theory_main}.

\subsection{Theoretical Setup}
\label{sec:nscl_theory_setup}
\vspace{0.1cm} \noindent \textbf{Representation for unlabeled data.} We apply NCD spectral learning objective $\mathcal{L}_{nscl}(f)$ in Equation~\ref{eq:nscl_def_nscl} and assume the optimizer is capable to obtain the representation that minimizes the loss. We can then obtain the $F^*$ s.t. $F^*F^{*\top}$ are the top-$k$ components of $\Dot{A}$’s SVD decomposition. 
To ease the analysis, we will focus on the top-$k$ singular vectors $V^* \in \mathbb{R}^{N\times k}$ of $\Dot{A}$ such that $F^* = V^* \sqrt{\Sigma_k}$, where $\Sigma_k$ is the diagonal matrix with top-$k$ singular values ($\sigma_1, ..., \sigma_k$). 

Since we are primarily interested in the unlabeled data, we split $V^*$ into two parts: $U^* \in \mathbb{R}^{N_u\times k}$ for unlabeled data and $L^* \in \mathbb{R}^{N_l\times k}$ for labeled data, respectively. Assuming the first $N_l$ rows/columns in $\Dot{A}$ corresponds to the labeled data, we can conveniently rewrite $V^*$ as: 
\begin{equation}
    V^* = \left[\begin{array}{c}
         L^* (\text{labeled part})\\
         U^* (\text{unlabeled part})
    \end{array}\right]
\end{equation}
% \vspace{-0.4cm}

\vspace{0.1cm} \noindent \textbf{Linear probing evaluation.} With the learned representation for the unlabeled data, we can evaluate NCD quality by the linear probing performance. The strategy is commonly used in self-supervised learning~\citep{chen2020simclr}. Specifically, the weight of a linear classifier is denoted as $\*M \in \mathbb{R}^{k \times |\mathcal{Y}_u|}$.
The class prediction is given by $h(x;f, \*M) = \operatorname{argmax}_{i \in \mathcal{Y}_u} (f(x)^\top \*M)_i$. The linear probing performance is given by the least error of all possible linear classifiers:
\begin{equation}
\mathcal{E}(f)\triangleq\underset{{\*M}\in \mathbb{R}^{k \times |\mathcal{Y}_u|}}{\operatorname{min}}  \underset{{x \in \mathcal{X}_u}}{\sum} \mathbbm{1}\left[y(x) \neq h(x;f, \*M)\right],
\label{eq:nscl_def_error}
% % \vspace{-0.2cm}
\end{equation}

where $y(x)$ indicates the ground-truth class of $x$.

%%%%%%%%%%%%%%%%%%%%%% Figure toy setting  %%%%%%%%%%%%%%%%%%%%%% 
\begin{figure*}[t]
    \centering
    \includegraphics[width=0.95\linewidth]{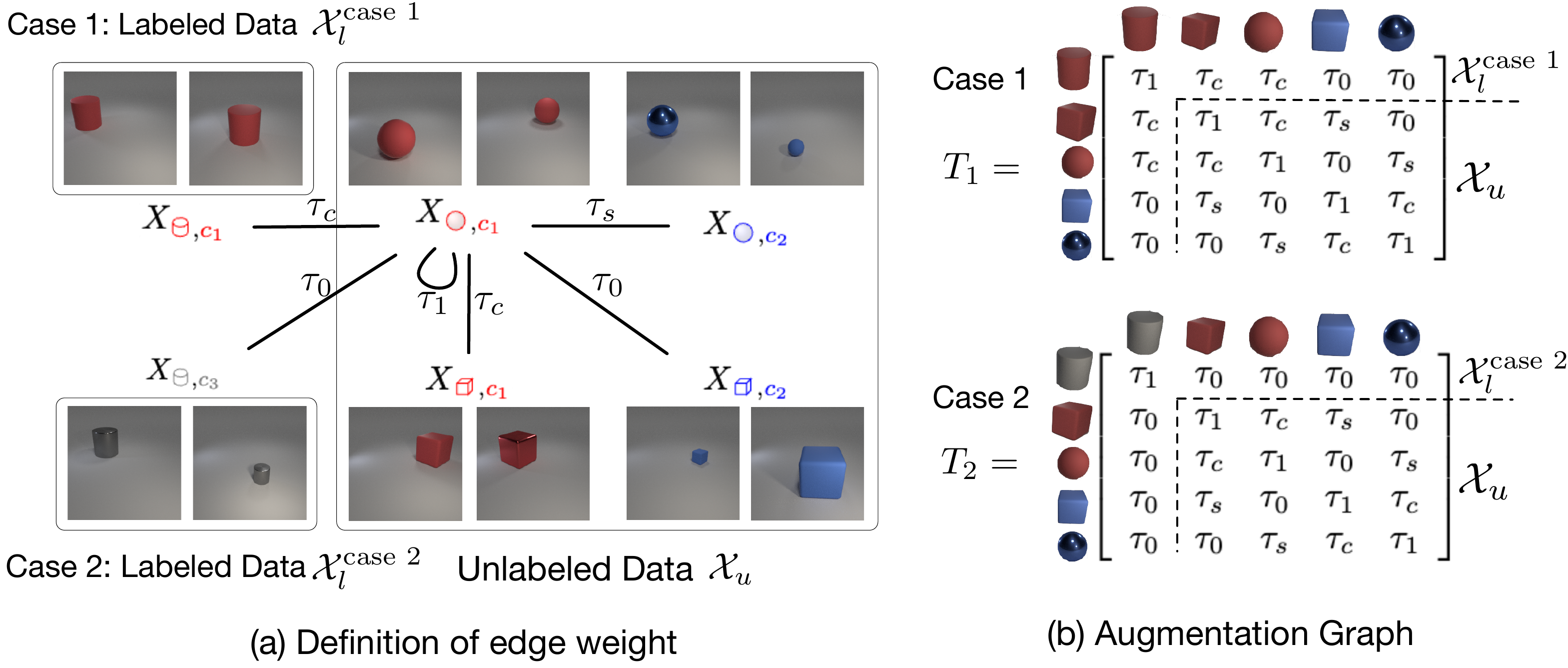}
    \caption[An illustrative example for theoretical analysis.]{An illustrative example for theoretical analysis. \textbf{(a)} The unlabeled data $\mathcal{X}_u$ consists of 3D objects of sphere/cube with red/blue colors. We consider two cases of labeled data: (1) Case 1 uses red cylinders  that are correlated with the target novel class (red). (2) Case 2 uses gray cylinders which have no correlation with $\mathcal{X}_u$. \textbf{(b)} The augmentation matrices for case 1 and case 2 respectively. See definition in Eq.~\eqref{eq:nscl_def_edge}. Best viewed in color.}
    % % \vspace{-0.5cm}
    \label{fig:nscl_toy_setting}
\end{figure*}
%%%%%%%%%%%%%%%%%%%%%%%%%%%%%%%%%%%%%%%%%%%%%%%%%%%%%%%%%%%%%%%%% 

\vspace{0.1cm} \noindent \textbf{Residual analysis.} With defined $U^*$, we can bound the linear probing error $\mathcal{E}(f)$  by the residual of the regression error $\mathcal{R}(U^*)$ as we show in Lemma~\ref{lemma:nscl_cls_bound} with proof in Appendix~\ref{sec:nscl_sup_cls_bound}. 
\begin{lemma}
Denote the $\mathbf{y}(x) \in \mathbb{R}^{|\mathcal{Y}_u|}$ as a one-hot vector whose $y(x)$-th position is 1 and 0 elsewhere. Let
$\mathbf{Y} \in \mathbb{R}^{N_u \times |\mathcal{Y}_u|}$ as a binary mask whose rows are stacked by $\mathbf{y}(x)$. We have: 
    $$\mathcal{R}(U^*) \triangleq \underset{{\*M}\in \mathbb{R}^{k \times |\mathcal{Y}_u|}}{\operatorname{min}} \|\mathbf{Y} - U^* \*M \|^2_F \geq \frac{1}{2}\mathcal{E}(f).$$
    % % \vspace{-0.4cm}
\label{lemma:nscl_cls_bound}
\end{lemma}
Note that we can rewrite $\mathcal{R}(U^*)$ as the summation of individual residual terms $\mathcal{R}(U^*, \Vec{y}_i)$: 
$
    \mathcal{R}(U^*) = \sum_{i \in \mathcal{Y}_u} \mathcal{R}(U^*, \Vec{y}_i),
$
where
$$ \mathcal{R}(U^*, \Vec{y}_i) \triangleq  \underset{{\Vec{\mu}_i}\in \mathbb{R}^{k}}{\operatorname{min}} \|\Vec{y}_i - U^* \Vec{\mu}_i \|^2_2, $$
and $\Vec{y}_i \in \mathbb{R}^{N_u}$ is the $i$-th column of $\mathbf{Y}$ and $\Vec{\mu}_i \in \mathbb{R}^{k}$ is the $i$-th column of $\*M$. Without losing the generality, our analysis will revolve around the residual term $\mathcal{R}(U^*, \Vec{y}_i)$ for specific class $i$. It is clear that if learned representation $U^*$ encodes more information of the label vector $\Vec{y}_i$, the residual $\mathcal{R}(U^*, \Vec{y}_i)$ becomes smaller\footnote{In an extreme case, if the first column of $U^*$ is exactly the same as $\Vec{y}_i$, one can set $\Vec{\mu}_i = [1, 0 , 0, ...]^{\top}$ to make residual zero.}. Such insight can be used to investigate which type of known class is more helpful for learning the representation of novel classes.

\subsection{An Illustrative Example}
\label{sec:nscl_toy}

We consider a toy example that helps illustrate the core idea of our theoretical findings. Specifically, the example aims to cluster 3D objects of different colors and shapes, as shown in Figure~\ref{fig:nscl_toy_setting} (a). These images are generated by a 3D rendering software~\citep{johnson2017clevr} with user-defined properties including colors, shape, size, position, etc. 

In what follows, we define two data configurations and corresponding graphs, where the labeled data is correlated  with the attribute of unlabeled data (\textbf{case 1}) vs. not (\textbf{case 2}). We are interested in contrasting the representations (in form of singular vectors) and residuals  derived from both scenarios. The proof of all theorems in this section is provided in Appendix~\ref{sec:nscl_proof-eigen}.

\vspace{0.1cm} \noindent \textbf{Motivation and data design.}  For simplicity, we focus on two main properties: color and shape. Formally, the images with shape $s$ and color $c$ are sampled from a generation procedure $\mathcal{G}$:  $$X_{s, c} \sim \mathcal{G}(s, c),$$ 
where $s \in \{\cube{1} (\text{cube}), \sphere{0.7}{gray} (\text{sphere}),\cylinder{0.6} (\text{cylinder})\}$, $c \in \{\textcolor{red}{c_1} (\text{red}), \textcolor{blue}{c_2} (\text{blue}), \textcolor{gray}{c_3} (\text{gray})\}$. We then construct our unlabeled dataset containing red/blue cubes/spheres as: 

$$\mathcal{X}_u \triangleq \{X_{\textcolor{red}{\cube{0.6}}, \textcolor{red}{c_1}}, X_{\textcolor{red}{\sphere{0.5}{red}}, \textcolor{red}{c_1}}, X_{\textcolor{blue}{\cube{0.6}}, \textcolor{blue}{c_2}}, X_{\textcolor{blue}{\sphere{0.5}{blue}}, \textcolor{blue}{c_2}}\}.$$

For simplicity, we assume each element in $\mathcal{X}_u$ is a single example. W.o.l.g, we also assume the red cube and red sphere form the target novel class. Then the corresponding labeling vector on $\mathcal{X}_u$ is defined by: $$\Vec{y} \triangleq \{1,1,0,0\}.$$ 
To answer \textit{``when and how does the known class help discover novel class?''}, we  construct two separate scenarios: one helps and the other one does not. Specifically, in the first case, we let the labeled data $\mathcal{X}_{l}^{\text{case 1}}$ be strongly correlated with the target class (red color) in unlabeled data:
$$\mathcal{X}_{l}^{\text{case 1}} \triangleq \{X_{\textcolor{red}{\cylinder{0.4}}, \textcolor{red}{c_1}}\} (\text{red cylinder}).$$  
In the second case, we construct the labeled data that has no correlation with any novel classes. We use gray cylinders which have no overlap in either shape and color: 
$$\mathcal{X}_{l}^{\text{case 2}} \triangleq \{X_{\textcolor{gray}{\cylinder{0.4}}, \textcolor{gray}{c_3}}\}  (\text{gray cylinder}).$$ 
Putting it together, our entire training dataset is 
$\mathcal{X}^{\text{case 1}} = \mathcal{X}_{l}^{\text{case 1}} \cup \mathcal{X}_{u}$ or  $\mathcal{X}^{\text{case 2}} = \mathcal{X}_{l}^\text{case 2} \cup \mathcal{X}_{u}$. We aim to verify the hypothesis that: 
\textit{the representation learned by $\mathcal{X}^{\text{case 1}}$ provides a much smaller regression residual to $\Vec{y}$ than $\mathcal{X}^{\text{case 2}}$ for color class. }

\vspace{0.1cm} \noindent \textbf{Augmentation graph.} 
Based on the data, we now define the probability of augmenting an image $X_{s, c}$ to another $X'_{s', c'}$:
\begin{align}
    \mathcal{T}\left(X'_{s', c'} \mid X_{s, c} \right)=\left\{\begin{array}{lll}
    \tau_1 & \text { if } & s=s', c=c', \\
    \tau_{s} & \text { if } & s=s', c \neq c', \\
    \tau_{c} & \text { if } & s\neq s', c=c', \\
    \tau_0 & \text { if } & s\neq s', c\neq c', \\
    \end{array}\right. 
    \label{eq:nscl_def_edge}
\end{align}
It is natural to assume the magnitude order that follows $\tau_1 \gg \max(\tau_{s},\tau_{c})$ and $\min(\tau_{s},\tau_{c}) \gg \tau_0$.  In two data settings $\mathcal{X}^\text{case 1}$ and $\mathcal{X}^\text{case 2}$, 
the corresponding augmentation matrices ${T}_1, {T}_2$ formed by $\mathcal{T}\left(\cdot|\cdot\right)$ are presented in Fig.~\ref{fig:nscl_toy_setting} (b). 
According to Eq.~\eqref{eq:nscl_def_wxx}, it can be verified that the adjacency matrices are $A_1 = T_1^2$ and $A_2 = T_2^2$ respectively. 

\vspace{0.1cm} \noindent \textbf{Main analysis.}  
We are primarily interested in analyzing the difference of the representation space derived from $A_1$ vs. $A_2$. Since $\tau_1 \gg \max(\tau_{s},\tau_{c})$, one can show that $A_1$ and $A_2$ are positive-definite. The singular vector is thus equivalent to the eigenvector. Also note that $A_1$ and their square root $T_1$ have the same eigenvectors and order. It is thus equivalent to analyzing the eigenvectors of $T_1$. Same with $A_2$ and $T_2$. In this toy example, we consider the eigenvalue problem of the unnormalized adjacency matrix\footnote{The normalized/unnormalized adjacency matrix corresponds to the NCut/RatioCut problem respectively~\citep{von2007tutorial}.} for simplicity. 

We put analysis on the top-$2$ eigenvectors $V^*_{1},V^*_{2}  \in \mathbb{R}^{5\times 2}$ for  $A_1$/$A_2$ ---- as we will see later, the top-$1$ eigenvector of $T_1/T_2$ usually functions at distinguishing known vs novel data, while the 2nd eigenvector functions at distinguishing color or shape. 

We let $U^*_1 \in \mathbb{R}^{4\times 2}$ contains the last 4 rows of $V^*_{1}$, and corresponds to the ``representation''  for the unlabeled data only. $U^*_2$ is defined in the same way \emph{w.r.t.} $A_2$. We have the following theorem:

\begin{theorem}
   Assume $\tau_1 = 1$, $\tau_0 = 0$, $\tau_s < 1.5\tau_c$. We have
$$
U^*_1=\left[\begin{array}{ccccc}
 a_1 & a_1 & b_1 & b_1 \\
 a_2 & a_2 & b_2 & b_2 \\
\end{array}\right]^{\top},
$$
where $a_1,b_1$ are some positive real numbers, and $a_2,b_2$ has different signs.  
$$U^*_2=\left\{\begin{array}{ll}   
\frac{1}{2}\left[\begin{array}{cccc}
1 & 1 & 1 & 1\\ 1 & 1 & -1 & -1\\
\end{array}\right]^{\top}, & \text{if } \tau_{s} < \tau_{c}, \\
\frac{1}{2}\left[\begin{array}{cccc}
1 & 1 & 1 & 1\\ -1 & 1 & -1 & 1\\
\end{array}\right]^{\top}, & \text{if } \tau_{s} > \tau_{c}, 
\end{array}\right. $$
With label vector $\Vec{y} = \{1,1,0,0\}$, we have 
\begin{equation}
    \mathcal{R}(U^*_1, \Vec{y}) = 0, \mathcal{R}(U^*_2, \Vec{y}) = \left\{\begin{array}{ll}    
    0, & \text{if } \tau_{s} < \tau_{c}\\
     1, &  \text{if } \tau_{s} > \tau_{c}.
    \end{array}\right. 
\end{equation}
    \label{th:nscl_toy_extreme}
\end{theorem}
\vspace{0.1cm} \noindent \textbf{Interpretation of Theorem~\ref{th:nscl_toy_extreme}:} 
The discussion can be divided into two cases: (1) 
$\tau_{s} < \tau_{c}$. (2) $\tau_{s} > \tau_{c}$. In the \textbf{first case} $\tau_{s} < \tau_{c}$, the connection between the same-color data pair is already stronger than the same-shape data pair. Thus the eigenvector corresponding to color information ($\frac{1}{2}[1,1,-1,-1]^\top$)  will be more prominent (and ranked higher in $U^*_2$) than ``shape eigenvector'' ($\frac{1}{2}[-1,1,-1,1]^\top$). 
Since the feature $U^*_2$ already encodes sufficient information (color) of the labeling vector $\Vec{y}$, fitting $\Vec{y}$ becomes easy and the residual $\mathcal{R}(U^*_2,\Vec{y})$ becomes 0.

In NCD, \textbf{we are more interested in the second case} ($\tau_{s} > \tau_{c}$), where unlabeled data indeed need some help from labeled data for better clustering. Such help comes from the semantic connection between labeled data and unlabeled data. In our toy example, the semantic connection comes from the first row/column of ${T}_1$ and ${T}_2$. However, the first row/column of ${T}_2$ is $[1,0,0,0,0]$, which means there is no extra information offered from $\mathcal{X}_{l}^\text{case 2}$. It is because $\mathcal{X}_{l}^\text{case 2}$ contains gray cylinders which have neither colors nor shapes connection to unlabeled data $\mathcal{X}_u$. Contrarily, $\mathcal{X}_{l}^\text{case 1}$ with red cylinder provides strong color prior. This allows the  ``color eigenvector'' ($[a_2, a_2, -b_2, -b_2]$) to become a main component in $U^*_1$, making the residual $\mathcal{R}(U^*_1,\Vec{y})=0$ even when $\tau_{s} > \tau_{c}$. 

\vspace{0.1cm} \noindent \textbf{Main takeaway.} In Theorem~\ref{th:nscl_toy_extreme}, we have verified the hypothesis that incorporating labeled data $\mathcal{X}_{l}^\text{case 1}$ (red  cylinder) can reduce the residual $\mathcal{R}(U^*_1,\Vec{y})$ more than using  $\mathcal{X}_{l}^\text{case 2}$, especially when color is a weaker signal than shape in unlabeled data. 

\vspace{0.1cm} \noindent \textbf{Extension: A more general result.} 
Note that $T_1$ and $T_2$ are special cases of the following $T(t)$ with $t \in [\tau_0, \tau_c]$:

\begin{equation*}
T(t)=\left[\begin{array}{ccccc}
\tau_1 & t & t & \tau_0 & \tau_0 \\
t & \tau_1 & \tau_{c} & \tau_{s} & \tau_0  \\
t & \tau_{c} & \tau_1 & \tau_0 & \tau_{s}  \\
\tau_0 & \tau_{s} & \tau_0 & \tau_1 & \tau_{c}  \\
\tau_0 & \tau_0 & \tau_{s} & \tau_{c} & \tau_1  \\
\end{array}\right],
\end{equation*}
where $t$ indicates the strength of the connection between labeled data and a novel class in unlabeled data. Let $U^*_t$ be the representation for unlabeled data derived from $T(t)$.  The following theorem indicates that the residual decreases when $t$ increases and the residual becomes 0 when $t$ is larger than a threshold $\Bar{t}$ depending on the gap between $\tau_s$ and $\tau_c$. 
\begin{theorem}
     Assume  $\tau_1 = 1$, $\tau_0 = 0$, $1.5\tau_c > \tau_s > \tau_c$. Let $\Bar{t} = \sqrt{\frac{2(\tau_s-\tau_c)^2\tau_c}{2\tau_c - \tau_s}}$, $r: \mathbb{R} \mapsto (0,1) $ be a real value function, we have 
     \begin{equation}
         \mathcal{R}(U^*_t, \Vec{y}) = \left\{\begin{array}{ll}    
     0, &  \text{if } t \in (\Bar{t}, \tau_s), \\
    r(t), & \text{if } t \in (0, \Bar{t}), \\ 
     1, &  \text{if } t = 0. \end{array}\right. 
     % % \vspace{-0.2cm}
     \end{equation}
    \label{th:nscl_toy_general}
\end{theorem}

\vspace{0.1cm} \noindent \textbf{Can adding labeled data be harmful?} We exemplify the scenario in Figure~\ref{fig:nscl_teaser}, where the umbrella images are
given as a known class, undesirably causing the “mushroom with
umbrella shape” to be grouped together. To formally analyze this case,  we construct \textbf{case 3}:
$$\mathcal{X}_{l}^{\text{case 3}} \triangleq \{X_{\textcolor{gray}{\cube{0.5}}, \textcolor{gray}{c_3}}\}  (\text{gray cube}).$$  
In this case, we have the following Lemma~\ref{th:nscl_toy_harmful}. 
\begin{lemma}
    If  $\frac{\tau_c}{\tau_s} \in (1, 1.5)$,  
$
    \mathcal{R}(U^*_3, \Vec{y})- \mathcal{R}(U^*_2, \Vec{y})=1. 
$
    \label{th:nscl_toy_harmful}
\end{lemma}
The residual in case 3 is now larger than in case 2, since the shape is treated as a more important feature than the color feature (which relates to the target class). The main takeaway of this lemma is that the labeled data can be harmful when its connection with unlabeled data is undesirably stronger in the spurious feature dimension.

\begin{figure}[t]
    \centering
    \includegraphics[width=0.75\linewidth]{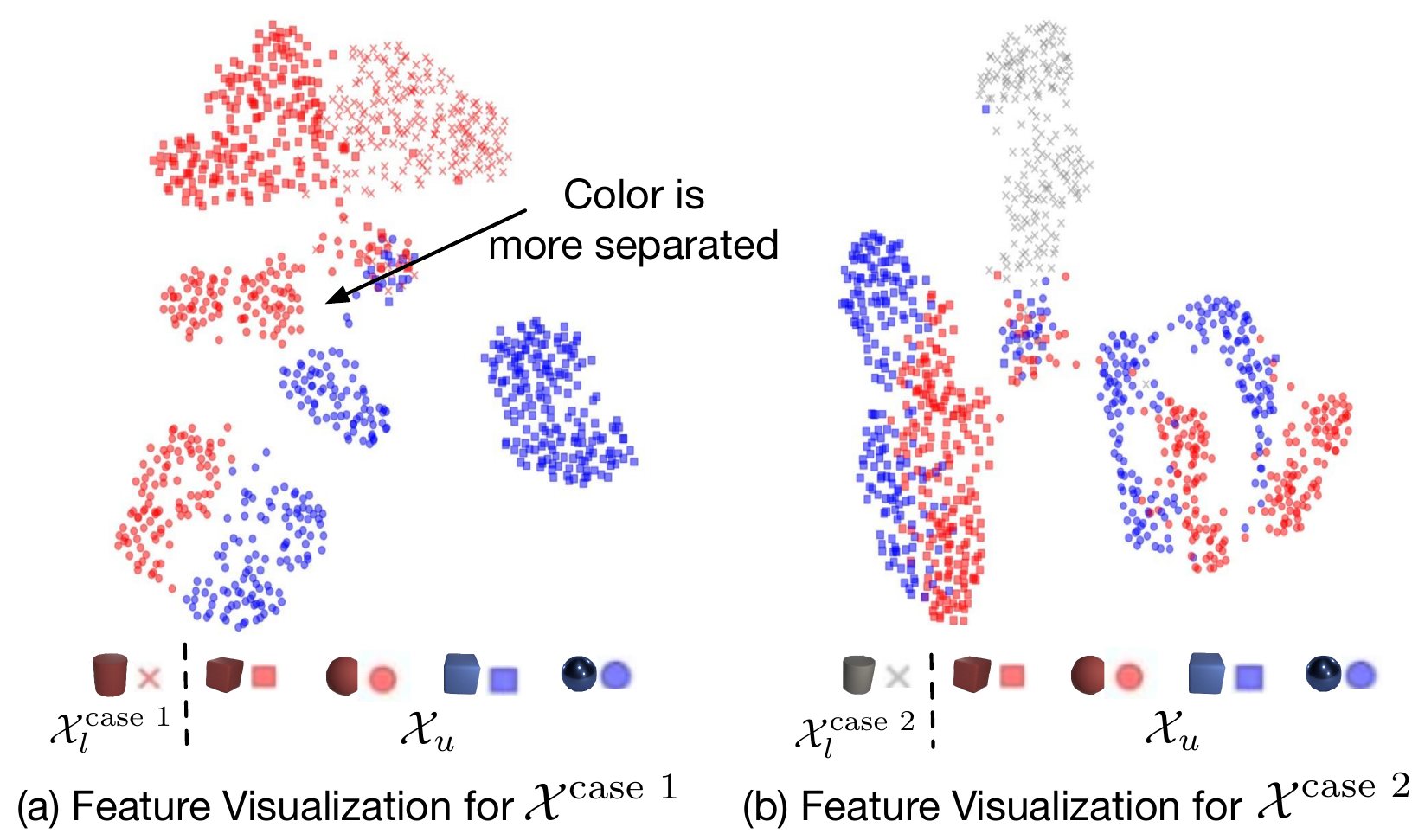}
    \caption[UMAP visualization of the feature embedding.]{UMAP~\citep{umap} visualization of the feature embedding learned from $\mathcal{X}^\text{case 1}$ and $\mathcal{X}^\text{case 2}$ respectively. The model is trained with NCD Spectral Contrastive Loss. }
    \label{fig:nscl_toy_vis}
\end{figure}
\vspace{0.1cm} \noindent \textbf{Qualitative results.} The theoretical results can be verified in our empirical results by visualization in Fig.~\ref{fig:nscl_toy_vis}. Due to the space limitation, we include experimental details in Appendix~\ref{sec:nscl_sup_exp_vis}. As seen in Fig.~\ref{fig:nscl_toy_vis} (a), the features of unlabeled data $\mathcal{X}_u$ jointly learned with red cylinder $\mathcal{X}_{l}^\text{case 1}$ are more distinguishable by color attribute,  as opposed to Fig.~\ref{fig:nscl_toy_vis} (b). 

\subsection{Main Theory}
\label{sec:nscl_theory_main}

The toy example offers an important insight that using the labeled data help reduce the residual when it provides the missing information of unlabeled data. In this section, we will formalize this insight by extending the toy example to a more general setting with $N$ samples. We start with the definition of notations. 

\vspace{0.1cm} \noindent \textbf{Notations.} Recall that $V^* \in \mathbb{R}^{N \times k}$ is defined as the top-$k$ singular vectors of $\Dot{A}$, which is further split into two parts $L^* = \left[l_1, l_2, \cdots, l_k\right] \in \mathbb{R}^{N_l \times k}$, $U^* = \left[u_1, u_2, \cdots, u_k\right] \in \mathbb{R}^{N_u \times k}$, for labeled and unlabeled samples respectively. Then we let $V^{\flat} \in \mathbb{R}^{N \times (N-k)}$ be the remaining singular vectors of $\Dot{A}$ except top-$k$. Similarly, we split $V^{\flat}$ into two parts ($L^{\flat} = \left[l_{k+1}, l_{k+2}, \cdots, l_N\right] \in \mathbb{R}^{N_l \times (N-k)}$, $U^{\flat} = \left[u_{k+1}, u_{k+2}, \cdots, u_N\right] \in \mathbb{R}^{N_u \times (N-k)}$). 

We now present our first main result in Theorem~\ref{th:nscl_no_approx}.
\begin{theorem}
    Denote the projection matrix $\mathsf{P}_{L^{\flat}} = L^{\flat\top}(L^{\flat}L^{\flat\top})^{\dag}L^{\flat}$, where $^{\dag}$ denotes the Moore-Penrose inverse. For any labeling vector $\Vec{y} \in \{0,1\}^{N_u}$, we have
    \begin{equation}
        \mathcal{R}(U^*, \Vec{y}) \leq \|(I-\mathsf{P}_{L^{\flat}}) U^{\flat\top} \Vec{y}\|^2_2.
        \label{eq:nscl_R_bound_no_approx}
    \end{equation} 
    \label{th:nscl_no_approx}
\end{theorem}
\vspace{0.1cm} \noindent \textbf{Interpretation of Theorem~\ref{th:nscl_no_approx}.} The bound of residual in Ineq.~\eqref{eq:nscl_R_bound_no_approx} is composed of two projections:  $U^{\flat\top}$ and $(I-\mathsf{P}_{L^{\flat}})$. 
We first consider the ignorance space formed by the first projection: 
$$\text{\textbf{ignorance space}} \triangleq U^{\flat\top} \Vec{y},$$ 
 which contains the information of the labeling vector $\Vec{y}$ that is not encoded in the learned representation $U^{*}$ of the unlabeled data. Intuitively, when $\mathcal{R}\left(U^*, \vec{y}\right)>0$, the labeling vector $\vec{y}$ does not lie in the span of the existing representation $U^*$. On the other hand, $\mathcal{R}\left([\begin{array}{cc}  U^* & U^{\flat}\end{array}], \vec{y}\right)=0$ since $U^*$ together with $U^{\flat}$ forms a full rank space. We also define a measure of the ignorance degree of the current feature space: 
 $\text{\textbf{ignorance degree}} \triangleq \mathfrak{T}(\Vec{y}) = \frac{\|U^{\flat\top} \Vec{y}\|_2}{\|\Vec{y}\|_2}.$ 

The second projection matrix $(I-\mathsf{P}_{L^{\flat}})$ is composed of $L^{\flat}$, which we deem as the  extra knowledge from known classes: 
\begin{equation*}
\text{\textbf{extra knowledge}} \triangleq L^{\flat}.
\end{equation*} Multiplying the second projection matrix $(I-\mathsf{P}_{L^{\flat}})$ further reduces the norm of the ignorance space by considering the extra knowledge from labeled data, since $\mathsf{P}_{L^{\flat}}$ is a projection matrix that projects a vector to the linear span of $L^{\flat}$. In the extreme case, when $U^{\flat\top} \Vec{y}$ fully lies in the linear span of $L^{\flat}$, the residual $\mathcal{R}(U^*, \Vec{y})$ goes 0.

Next, we present another main theorem that bounds the linear probing error  $\mathcal{E}(f)$ based on the relations between the known and novel classes. See Appendix~\ref{sec:nscl_sup_main_proof} for a detailed discussion and assumption.
\begin{theorem} 
Let $[\begin{array}{cc} A_{ul}\in \mathbb{R}^{N_u \times N_l}, A_{uu}\in \mathbb{R}^{N_u \times N_u} 
\end{array}]$ be the sub-matrix of the last $N_u$ rows of $\Dot{A}$, and $q_i$ be the $i$-th eigenvector of $A_{uu}$. 
The linear probing error can be bounded as follows:
    \begin{equation*}
        \mathcal{E}(f) \lesssim \frac{2N_u}{|\mathcal{Y}_u|}\left(\sum_i^{|\mathcal{Y}_u|} \tikzmarknode{igdegree}{\highlight{red}{$\mathfrak{T}(\Vec{y}_i)$}}(1-\tikzmarknode{kappa}{\highlight{blue}{$\kappa(\Vec{y}_i)^2$}}) + \frac{\|\Dot{A} - \bar{A}\|_2}{\sigma_{k} - \sigma_{k+1}} \right), 
        % \label{eq:nscl_R_bound_with_approx}
    \end{equation*}
\begin{tikzpicture}[overlay,remember picture,>=stealth,nodes={align=left,inner ysep=1pt},<-]
    \path (igdegree.north) ++ (0,0.5em) node[anchor=south west,color=red!67] (ig_arrow){\textit{ ignorance degree }};
    \draw [color=red!87](igdegree.north) |- ([xshift=-0.3ex,color=red] ig_arrow.south east);
    \path (kappa.south) ++ (0,-0.7em) node[anchor=north west,color=blue!67] (kappa_arrow){\textit{ knowledge coverage}};
    \draw [color=blue!87](kappa.south) |- ([xshift=-0.3ex,color=blue]kappa_arrow.south east);
\end{tikzpicture}
    where     
    $$\kappa(\Vec{y}) = \cos(\Bar{U}^{\flat\top} \Vec{y},\Bar{\mathfrak{l}}^{\flat}) 
    \gtrsim \min_{i > k, j > k} \frac{2\sqrt{\frac{\Vec{y}^{\top}q_i}{\Vec{\eta}_u^{\top}q_i}\frac{\Vec{y}^{\top}q_j}{\Vec{\eta}_u^{\top}q_j}}}{\frac{\Vec{y}^{\top}q_i}{\Vec{\eta}_u^{\top}q_i}+\frac{\Vec{y}^{\top}q_j}{\Vec{\eta}_u^{\top}q_j}},$$
    \label{th:nscl_main}
\end{theorem}
and $\Bar{A}$ is the approximation of $\Dot{A}$ by taking the expectation in the rows/columns of labeled samples (Appendix~\ref{sec:nscl_sup_with_approx}) with a similar motivation as the SBM model~\citep{holland1983stochastic}. In such condition,  $\Bar{U}^{\flat\top}$, $\Bar{\mathfrak{l}}^{\flat}$ and $\eta_u$ is the approximation to $U^{\flat\top}$,  $L^{\flat}$ and $A_{ul}$ accordingly.

\vspace{0.1cm} \noindent \textbf{Interpretation of $\kappa(\Vec{y})$.} We provide the detailed derivation of $\kappa(\Vec{y})$ in Lemma~\ref{lemma:nscl_sup_bound_kappa}. Intuitively, $\kappa(\Vec{y})$ measures the usefulness and relevance of knowledge from known classes for NCD. We formally call it coverage, which measures the cosine distance between the ignorance space and the extra knowledge: 
\begin{equation*}
    \text{\textbf{coverage}} \triangleq \kappa(\Vec{y}) = \cos(\tikzmarknode{ignorance}{\highlight{red}{$ \Bar{U}^{\flat\top} \Vec{y}$}}, \tikzmarknode{knowledge}{\highlight{blue}{$\Bar{\mathfrak{l}}^{\flat}$}}). 
\end{equation*}
\begin{tikzpicture}[overlay,remember picture,>=stealth,nodes={align=left,inner ysep=1pt},<-]
    \path (ignorance.south) ++ (0,0.1em) node[anchor=north east,color=red!67] (igtt){\textit{ ignorance space }};
    \draw [color=red!87](ignorance.south) |- ([xshift=-0.3ex,color=red] igtt.south west);
    \path (knowledge.north) ++ (0,0.5em) node[anchor=south east,color=blue!67] (kntt){\textit{ extra knowledge }};
    \draw [color=blue!87](knowledge.north) |- ([xshift=-0.3ex,color=blue]kntt.south west);
\end{tikzpicture}

Our Theorem~\ref{th:nscl_main} thus meaningfully shows that the linear probing error can be bounded more tightly as $\kappa(\Vec{y})$ increases (i.e., when labeled data provides more useful information for the unlabeled data). 

\vspace{0.1cm} \noindent \textbf{Implication of Theorem~\ref{th:nscl_main}.} Our theorem allows us to formalize answers to the ``When and How'' question. Firstly, the Theorem answers ``\textit{how the labeled data helps}''---because the knowledge from the known classes changes the representation of unlabeled data and reduces the ignorance space for novel class discovery. Secondly, the Theorem answers ``\textit{when the labeled data helps}''. Specifically, labeled data helps when the coverage between ignorance space and extra knowledge is nonzero. In the extreme case, if the extra knowledge fully covers the ignorance space, we get the perfect performance (0 linear probing error).

% % \vspace{-0.3cm}
\section{Experiments on Common Benchmarks~}
\label{sec:nscl_exp}
Beyond theoretical insights, we show empirically that our proposed NCD spectral  loss is effective on common benchmark datasets CIFAR-10 and CIFAR-100~\citep{krizhevsky2009learning}. Following the well-established NCD benchmarks ~\citep{Han2019dtc, han2020automatically, fini2021unified}, each dataset is divided into two subsets, the labeled set that
contains labeled images belonging to a set of known classes, and an unlabeled set with novel classes. Our comparison is on three benchmarks: \texttt{C10-5} means CIFAR-10 datasets split with 5 known classes and 5 novel classes and \texttt{C100-80} means CIFAR-100 datasets split with 80 known classes while \texttt{C100-50} has 50 known classes. The division is consistent with~\citet{fini2021unified}. %
We train the model by the proposed NSCL algorithm with details in Appendix~\ref{sec:nscl_sup_exp_cifar} and measure performance on the features in the penultimate layer of ResNet-18. 

 \vspace{0.1cm} \noindent \textbf{NSCL is competitive in discovering novel classes.} Our proposed loss NSCL is amenable to   the theoretical understanding of NCD, which is our primary goal of this work. Beyond theory, we show that NSCL is equally desirable in empirical performance. In particular, NSCL outperforms its rivals by a significant margin, as evidenced in Table~\ref{tab:nscl_main}. Our comparison covers an extensive collection of common NCD algorithms and baselines. In particular, on C100-50, we improve upon the best baseline ComEx by \textbf{10.6}\%. 
 This finding further validates that putting analysis on NSCL is appealing for both theoretical and empirical reasons.%

%%%%%%%%%%%%%%%%%%% Main Table %%%%%%%%%%%%%%%%%%%%%%%
\begin{table}[htb]
% % \vspace{-0.2cm}
\caption[Comaprison of NCD baselines in clustering accuracy on the novel set.]{Main Results. Results are reported in clustering accuracy (\%) on the \textit{training} split of the novel set. With the learned feature, we perform a K-Means clustering with the default setting in Python's \texttt{sklearn} package. The accuracy of the novel classes is measured by solving an optimal assignment problem using the Hungarian algorithm~\citep{kuhn1955hungarian}. ``C'' is short for CIFAR. SCL denotes training with Spectral Contrastive Loss purely on $\mathcal{D}_u$ while SCL$^\ddagger$ is trained on $\mathcal{D}_u \cup \mathcal{D}_l$ unsupervisedly.  } 
\centering
\scalebox{0.85}{
\begin{tabular}{llll}
\toprule
\textbf{Method} & \textbf{C10-5} & \textbf{C100-80} & \textbf{C100-50} \\ \midrule
 \textbf{KCL}~\citep{hsu2017kcl} & 72.3 & 42.1 & -\\
 \textbf{MCL}~\citep{hsu2019mcl} & 70.9 & 21.5 & - \\
\textbf{DTC}~\citep{Han2019dtc} & 88.7 & 67.3 & 35.9\\
\textbf{RS+}~\citep{zhao2020rankstat} & 91.7 & 75.2 & 44.1 \\
\textbf{DualRank}~\citep{zhao2021rankstat} & 91.6 & 75.3 & - \\
\textbf{Joint}~\citep{jia2021joint} & 93.4 & 76.4 & - \\
\textbf{UNO}~\citep{fini2021unified} & 92.6 & 85.0 & 52.9  \\
\textbf{ComEx}~\citep{yang2022divide} & 93.6 & 85.7 & 53.4
\\ \midrule
\textbf{SCL}~\citep{haochen2021provable}  & 92.4 & 72.7 & 51.8\\
\textbf{SCL}{$^\ddagger$}~\citep{haochen2021provable}  & 93.7  & 68.9 & 53.3\\
\textbf{NSCL} (Ours)  & \textbf{97.5}  &  \textbf{85.9} & \textbf{64.0}
\\ \bottomrule
\end{tabular}}
\label{tab:nscl_main}
% \vspace{-0.2cm}
\end{table}
%%%%%%%%%%%%%%%%%%%%%%%%%%%%%%%%%%%%%%%%%%%%%%%%%%%%%%%%

%%%%%%%%%%%%%%%%%%% Sup Table %%%%%%%%%%%%%%%%%%%%%%%
\begin{table*}[htb]
\caption[Comparison of baselines' results with overall/novel/known accuracy.]{Comparison of results reported in overall/novel/known accuracy (\%) on the \textit{test} split of CIFAR. The three metrics are calculated as follows. (1) \textbf{Known accuracy}: For the features from the labeled data, we train an additional linear head by linear probing and then measure classification accuracy based on the prediction $\Vec{h}_l$;  (2) \textbf{Novel accuracy}: For features from the unlabeled data, we perform a K-Means clustering with the default setting in Python's \texttt{sklearn} package, which produces the clustering prediction $\Vec{h}_u$. The clustering accuracy is further measured by solving an optimal assignment problem using the Hungarian algorithm~\citep{kuhn1955hungarian}; (3) \textbf{Overall accuracy}. The overall accuracy is measured by concatenating the prediction $\Vec{h}_l$ and $\Vec{h}_u$ and then solving the assignment problem.} 
\vspace{0.2cm}
\centering
\resizebox{0.85\linewidth}{!}{
\begin{tabular}{lllllll}
\toprule
\multirow{2}{*}{\textbf{Method}} & \multicolumn{3}{c}{\textbf{C10-5}} & \multicolumn{3}{c}{\textbf{C100-50}} \\
 & \textbf{All} & \textbf{Novel} & \textbf{Known} & \textbf{All} & \textbf{Novel} & \textbf{Known} \\ \midrule
 \textbf{DTC}~\citep{Han2019dtc} & 68.7 & 78.6 & 58.7 & 32.5 & 34.7 & 30.2 \\
\textbf{RankStats}~\citep{zhao2020rankstat} & 89.7 & 88.8 & 90.6 & 55.3 & 40.9 & 69.7 \\
\textbf{UNO}~\citep{fini2021unified} & \textbf{95.8} & 95.1 & 96.6  & 65.4  & 52.0 & 78.8 \\
\textbf{ComEx}~\citep{yang2022divide} & 95.0 & 93.2 & \textbf{96.7} & 67.2 & 54.5 & \textbf{80.1} \\ \midrule
\textbf{NSCL} (Ours) & 95.5 & \textbf{96.7} & 94.2  &  \textbf{67.4} & \textbf{57.1} &	77.4
\\ \bottomrule
\end{tabular}}
\label{tab:nscl_sup_main}
\end{table*}
%%%%%%%%%%%%%%%%%%%%%%%%%%%%%%%%%%%%%%%%%%%%%%%%%%%%%%%%

\vspace{0.1cm} \noindent \textbf{Ablation study on the unsupervised counterpart.} To verify whether the known classes indeed help discover new classes, we compare NSCL with the unsupervised counterpart (dubbed SCL) that is purely trained on the unlabeled data $\mathcal{D}_u$. Results show that the labeled data offers tremendous help and improves \textbf{13.2}\% in novel class accuracy. 

\vspace{0.1cm} \noindent \textbf{Supervision signals are important in the labeled data.} We also analyze how much the supervision signals in labeled data help. To investigate it, we compare our method NSCL with SCL trained on $\mathcal{D}_u \cup \mathcal{D}_l$ in a purely unsupervised manner. The difference is that SCL does not utilize the label information in $\mathcal{D}_l$. We denote this setting as SCL$^\ddagger$ in Table~\ref{tab:nscl_main}. Results show that NSCL provides stronger performance than SCL$^\ddagger$. The ablation suggests that relevant knowledge of known classes indeed provides meaningful help in novel class discovery. 

\vspace{0.1cm} \noindent \textbf{NSCL is competitive in the inductive setting.} 
We report performance comparison in Table~\ref{tab:nscl_sup_main}, comprehensively measuring three accuracy metrics for all/novel/known classes respectively. Different from Table~\ref{tab:nscl_main} which reports clustering results in a transductive manner, the performance in Table~\ref{tab:nscl_sup_main} is reported on the test split.   
For evaluation, we first collect the feature representations and then report overall/novel/known accuracy with inference details provided in the caption of Table~\ref{tab:nscl_sup_main}. 
We see that NSCL establishes comparable performance with baselines on the labeled data from known classes  and superior performance on novel class discovery. Notably, NSCL outperforms UNO~\citep{fini2021unified} on \texttt{C10-5} by 1.6\% and outperforms ComEx~\citep{yang2022divide} by 2.6\% on \texttt{C100-50} in terms of novel accuracy.

% \vspace{-0.3cm}
\section{Additional Related Work~}
\label{sec:nscl_related}

\vspace{0.1cm} \noindent \textbf{Novel class discovery.} 
Early works tackled novel category discovery (NCD) as a transfer learning problem, such as DTC~\citep{Han2019dtc}, KCL~\citep{hsu2017kcl}, MCL~\citep{hsu2019mcl}. 
Many subsequent works incorporate representation learning for NCD, including  RankStats~\citep{zhao2020rankstat}, NCL~\citep{zhong2021ncl} and UNO~\citep{fini2021unified}. CompEx~\citep{yang2022divide} further uses a novelty detection module to better separate novel and known. However, none of the previous works theoretically analyzed the key question: \textit{when and how do known classes help?} \citet{li2022closer} try to answer this question from an empirical perspective by comparing labeled datasets from different levels of semantic similarity. \citet{chi2021meta} directly define a solvable condition for the NCD problem but do not investigate the semantic relationship between known and novel classes. This chapter introduces the first work that systematically investigates the ``{when and how}'' questions by modeling the sample relevance
from a graph-theoretic perspective and providing a provable error bound for the NCD problem.

% \vspace{-0.3cm}
\vspace{0.1cm} \noindent \textbf{Spectral graph theory.} 
Spectral graph theory is a classic research problem~\citep{chung1997spectral,cheeger2015lower,kannan2004clusterings,lee2014multiway,mcsherry2001spectral}, which aims to partition the graph by studying the eigenspace of the adjacency matrix. The spectral graph theory is also widely applied in machine learning~\citep{ng2001spectral,shi2000normalized,blum2001learning,zhu2003semi,argyriou2005combining,shaham2018spectralnet}. Recently, ~\citet{haochen2021provable} derive a spectral contrastive loss from the factorization of the graph's adjacency matrix which facilitates theoretical study in unsupervised domain adaptation~\citep{shen2022connect,haochen2022beyond}.
The graph definition in existing works is purely formed by the unlabeled data, whereas {our graph and adjacency matrix is uniquely tailored for the NCD problem setting and consists of both labeled data from known classes and unlabeled data from novel classes}. We offer new theoretical guarantees and insights based on the relations between known and novel classes, which has not been explored in the previous literature.%

\vspace{0.1cm} \noindent \textbf{Theoretical analysis on contrastive learning.} 
 Recent works have advanced contrastive learning with empirical success~\citep{chen2020simclr,khosla2020supcon,zhang2021supporting,wang2022pico}, which necessitates  a theoretical foundation. ~\citet{arora2019theoretical,lee2021predicting,tosh2021contrastive,tosh2021contrastive2,balestriero2022contrastive, shi2023the} provided provable guarantees on the representations learned by contrastive learning for linear probing. ~\citet{shen2022connect,haochen2021provable,haochen2022beyond} further modeled the pairwise relation from the graphic view and provided error analysis of the downstream tasks. However, the existing body of work has mostly focused on \emph{unsupervised learning}. There is no prior theoretical work considering the NCD problem where both labeled and unlabeled data are presented. In this chapter, we systematically investigate how the label information can change the representation manifold and affect the downstream novel class discovery task.

%%%%%%%%%%%%%%%%%%%%%%%%%%%%%%%%%%%%%%%%%%%%%%%%%%%%%%%%%%%%%%%%%%%%%
%%%%%%%%%%%%%%%%%%%%%%%%%%%%%%  Summary %%%%%%%%%%%%%%%%%%%%%%%%%%%%%%%
%%%%%%%%%%%%%%%%%%%%%%%%%%%%%%%%%%%%%%%%%%%%%%%%%%%%%%%%%%%%%%%%%%%%  

\section{Summary~}
\label{sec:nscl_summary}
In this chapter, we present a theoretical framework of novel class discovery and provide new insight on the research question: ``\textit{when and how does the known class help discover novel classes?}''. Specifically, we propose a graph-theoretic representation that can be learned through a new NCD Spectral Contrastive Loss (NSCL). Minimizing this objective is equivalent to factoring the graph's adjacency matrix, which allows us to analyze the NCD quality by measuring the linear probing error on novel samples' features. Our main result (Theorem~\ref{th:nscl_no_approx}) suggests such error can be significantly reduced (even to 0) when the linear span of known samples' feature covers the ``ignorance space'' of unlabeled data in discovering novel classes. Our framework is also empirically appealing to use since it can achieve similar or better performance than existing methods on benchmark datasets. 
In summary, NSCL establishes a robust foundation for open-world representation learning by deciphering the influence of known classes in discovering new classes. This contributes profound theoretical and empirical impacts that stretch beyond conventional boundaries.

%%%%%%%%%%%%%%%%%%%%%%%%%%%%%%%%%%%%%%%%%%%%%%%%%%%%%%%%%%%%%%%%%%%%%
%%%%%%%%%%%%%%%%%%%%%%%%%%%%%%  Supp %%%%%%%%%%%%%%%%%%%%%%%%%%%%%%%
%%%%%%%%%%%%%%%%%%%%%%%%%%%%%%%%%%%%%%%%%%%%%%%%%%%%%%%%%%%%%%%%%%%%  
% \clearpage
% \section{Appendix}
% \label{sec:nscl_supp}

% \input{chapters/supp_nscl}

%% file: chapters/5_sorl.tex
\chapter{A Graph-Theoretic Framework for Understanding ORL~}
\label{sec:sorl}

\paragraph{Publication Statement.} This chapter is joint work with Zhenmei Shi, and Yixuan Li. The paper version of this chapter appeared
in NeurIPS23~\citep{sun2023sorl}. 

\noindent\rule{\textwidth}{1pt}

The preceding chapter delved into the analysis of the Novel Class Discovery (NCD) problem, which operates under the presumption that all training samples from known classes are labeled, thereby focusing its intent on discovering new classes within the unlabeled data. Open-world representation learning, on the other hand, pursues a more general objective, which endeavors to infer both known and novel classes in unlabeled data by leveraging prior knowledge from a labeled set. However, despite its significance, the theoretical underpinnings for this complex problem remain notably deficient, indicating a pressing need for further research and exploration in this domain.

This chapter bridges the gap by formalizing a graph-theoretic framework tailored for the open-world setting, where the clustering can be theoretically characterized by graph factorization.
Our graph-theoretic framework illuminates practical algorithms and provides guarantees. 
Specifically, utilizing our graph formulation, we present the algorithm --- Spectral Open-world Representation Learning (SORL). This technique, though bearing similarities to NSCL as discussed in Chapter~\ref{sec:nscl}, operates within a distinct problem domain. The process of minimizing the corresponding loss is fundamentally analogous to executing a spectral decomposition on the graph. 
Such equivalence allows us to derive a provable error bound on the clustering performance for both known and novel classes, and analyze rigorously when labeled data helps. Empirically, SORL can match or outperform several strong baselines on common benchmark datasets, which is appealing for practical usage while enjoying theoretical guarantees.

\section{Introduction~}
\label{sec:sorl_intro}

% \begin{wrapfigure}{r}{0.5\textwidth}
\begin{figure}[htb]
    \centering
\includegraphics[width=0.75\linewidth]{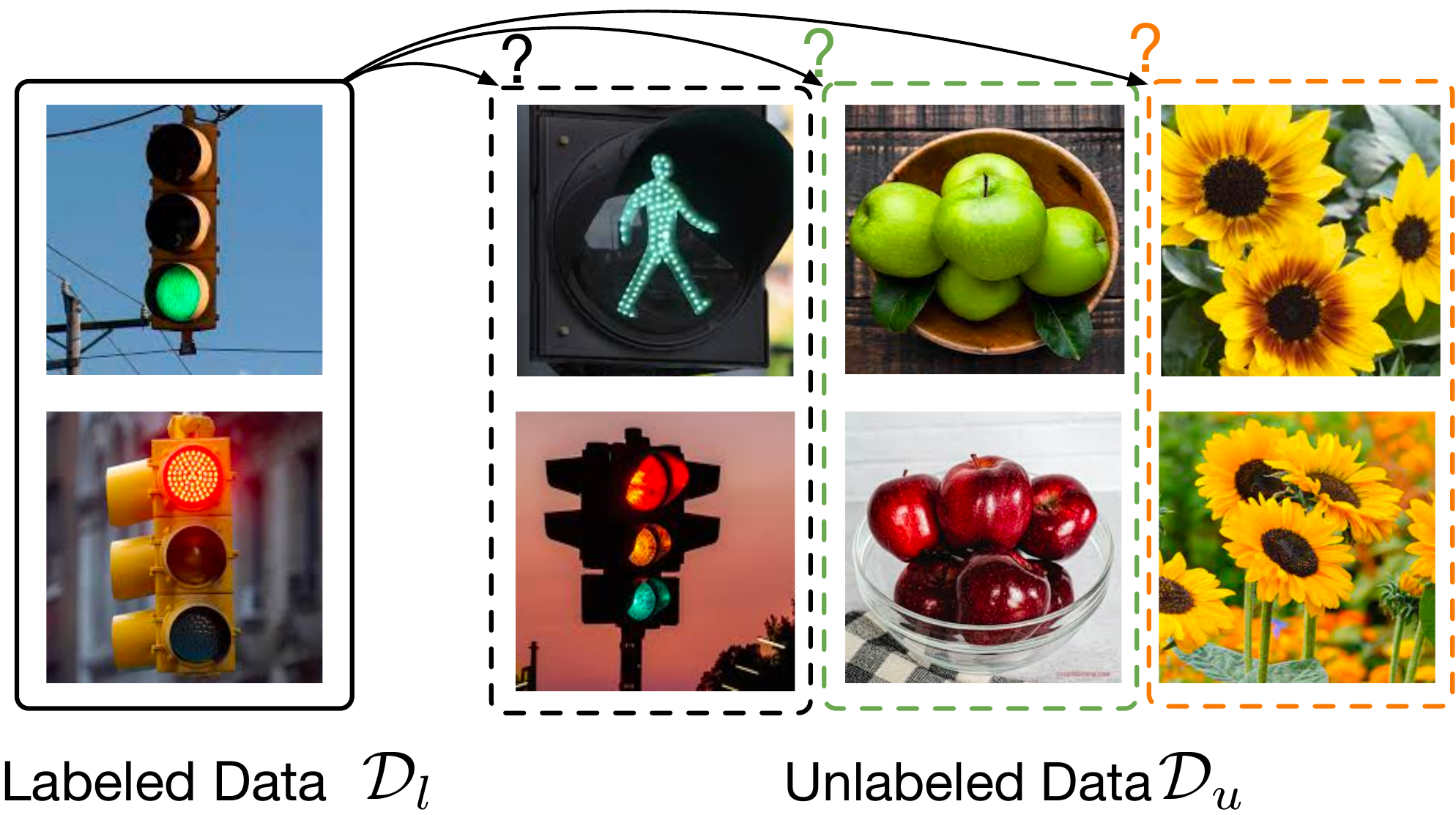}
 
\caption[Illustration of the research question in Open-world Representation Learning.]{{Open-world Representation Learning} aims to learn representations that can correctly cluster samples in the novel class and classify samples in the known classes by utilizing knowledge from the labeled data. An open question is \textit{``what
is the role of the label information in shaping representations for both known and novel classes?''} This chapter aims to provide a formal understanding.
    }
    \label{fig:sorl_teaser}
\end{figure}

Machine learning models in the open world inevitably encounter data from both known and novel classes.  
Traditional supervised machine learning models are trained on a closed set of labels, and thus can struggle to effectively cluster new semantic concepts. On the other hand, open-world representation learning approaches, such as those discussed in studies \citep{cao2022openworld, vaze22gcd, sun2023opencon},  enable models to distinguish \emph{both known and novel classes}, making them highly desirable for real-world scenarios. As shown in Figure~\ref{fig:sorl_teaser}, the learner has access to a labeled training dataset $\mathcal{D}_l$ (from known classes) as well as a large unlabeled dataset $\mathcal{D}_u$ (from both known and novel classes). By optimizing feature representations jointly from both labeled and unlabeled data, the learner aims to create meaningful cluster structures that correspond to either known or novel classes. With the explosive growth of data generated in various domains, open-world representation learning has emerged as a crucial problem in the field of machine learning. %

\vspace{0.1cm} \vspace{0.1cm} \noindent \textbf{Motivation.} Different from self-supervised learning~\citep{van2018cpc,chen2020simclr,caron2020swav,he2019moco,zbontar2021barlow,bardes2021vicreg,chen2021exploring,haochen2021provable}, open-world representation learning allows harnessing the power of the labeled data for possible knowledge sharing and transfer to unlabeled data, and from known classes to novel classes. In this joint learning process, we argue that interesting intricacies can arise---
the labeled data provided may be beneficial or unhelpful to the resulting  clusters.
We exemplify the nuances in Figure~\ref{fig:sorl_teaser}. In one scenario, when the model learns the labeled known classes (e.g., traffic light) by pushing red and green lights closer, such a relationship might transfer to help cluster green and red apples into a coherent cluster. Alternatively, when the connection between the labeled data and the novel class (e.g., flower) is weak, the benefits might be negligible. We argue---perhaps obviously---that a formalized understanding of the intricate phenomenon is needed.

\vspace{0.1cm} \vspace{0.1cm} \noindent \textbf{Theoretical significance.} To date, theoretical understanding of open-world representation learning is still in its infancy. In this chapter, we aim to fill the  critical blank by analyzing this important learning problem from a rigorous theoretical standpoint. Our exposition gravitates around the open question: \emph{what is the role of labeled data in shaping representations for both known and novel classes?} 
To answer this question, we formalize a graph-theoretic framework tailored for the open-world setting, 
where the vertices are all the data points and connected sub-graphs form classes (either known or novel). The edges are defined by a combination of supervised and self-supervised signals, which reflects the availability of both labeled and unlabeled data. Importantly, this graph facilitates the understanding of open-world representation learning from a
spectral analysis perspective, where the clustering can be theoretically characterized by graph factorization. Based on the graph-theoretic formulation, we derive a formal error bound by contrasting the clustering performance for all classes, before and after adding the labeling information. Our Theorem~\ref{th:sorl_main} reveals the sufficient condition for the improved clustering performance for a class. Under the K-means measurement, the unlabeled samples in one class can be better clustered, if their overall connection to the labeled data is stronger than their self-clusterability.

\vspace{0.1cm} \vspace{0.1cm} \noindent \textbf{Practical significance.}  Our graph-theoretic framework also illuminates practical algorithms with provided guarantees.  In particular, based on our graph formulation, we present the algorithm called Spectral Open-world Representation Learning (SORL) adapted from \citet{sun2023nscl}. Minimizing this loss is equivalent to performing spectral decomposition on the graph (Section~\ref{sec:sorl_orl-scl}), which brings two key benefits: (1) it allows us to analyze the representation space and resulting clustering performance in closed-form; (2) practically, it enables end-to-end training in the context of deep networks. We show that our learning algorithm leads to strong empirical
performance while enjoying theoretical guarantees. 
The learning objective can be effectively optimized using stochastic gradient descent on modern neural network architecture, making it desirable for real-world applications.

\section{Problem Setup~}
\label{sec:sorl_prob_setup}
While Chapter \ref{sec:prob} previously addressed the context of open-world representation learning, the subtle disparities in the notation for the purpose of theoretical analysis necessitate a re-visitation of the problem setup and notations. Formally, this section reiterates the data setup and learning objectives inherent in open-world representation learning.

\vspace{0.1cm} \noindent \textbf{Data setup.} We consider the empirical training set  $\mathcal{D}_{l} \cup \mathcal{D}_{u}$ as a union of labeled and unlabeled data. 
\begin{enumerate}

    \item The labeled set $\mathcal{D}_{l}=\left\{\Bar{x}_{i} , y_{i}\right\}_{i=1}^{n}$, with $y_i \in \mathcal{Y}_l$. The label set $\mathcal{Y}_l$ is  known. 
    \item The unlabeled set $\mathcal{D}_{u}=\left\{\Bar{x}_{i}\right\}_{i=1}^{m}$, where each sample $\Bar{x}_i$ can come from either known or novel classes\footnote{This generalizes the problem of Novel Class Discovery~\citep{Han2019dtc,hsu2017kcl,hsu2019mcl,zhao2021rankstat,zhong2021ncl,fini2021unified}, which assumes the unlabeled set is \emph{purely} from novel classes.}. Note that we do not have  access to the labels in $\mathcal{D}_{u}$. For mathematical convenience, we denote the underlying label set as $\mathcal{Y}_\text{all}$, where $\mathcal{Y}_l \subset \mathcal{Y}_\text{all}$. We denote $C = |\mathcal{Y}_\text{all}|$  the total number of classes. 

\end{enumerate} 

We use $\mathcal{P}_{l}$ and $\mathcal{P}$ to denote the  marginal distributions of labeled data and all data in the input space, respectively. Further, we let $\mathcal{P}_{l_i}$ denote the distribution of labeled samples with class label $i \in \mathcal{Y}_l$.

\vspace{0.1cm} \noindent \textbf{Learning goal.} Under the setting, the goal is to learn distinguishable representations \emph{for both known and novel classes} simultaneously. The representation quality will be measured using classic metrics, such as K-means clustering accuracy, which we will define mathematically in Section~\ref{sec:sorl_eval}.

\vspace{0.1cm} \noindent \textbf{Theoretical analysis goal.} We aim to comprehend the role of label information in shaping representations for both known and novel classes. It's important to note that our theoretical approach aims to understand the perturbation in the clustering performance by labeling existing, previously unlabeled data points within the dataset. By contrasting the clustering performance before and after labeling these instances, we uncover the underlying structure and relations that the labels may reveal. This analysis provides invaluable insights into how labeling information can be effectively leveraged to enhance the representations of both known and novel classes.

\section{A Spectral Approach for Open-world Representation Learning~}

In this section, we formalize and tackle the open-world representation learning (ORL) problem from a graph-theoretic view. Our fundamental idea is to formulate ORL as a clustering problem---where similar data points are grouped into the same cluster, by way of possibly utilizing helpful information from the labeled data $\mathcal{D}_l$. This clustering process can be modeled by a graph, where the vertices are all the data points and classes form connected sub-graphs.
 Specifically, utilizing our graph formulation, we present the algorithm — Spectral Open-world Representation Learning (SORL) in Section~\ref{sec:sorl_orl-scl}. The process of minimizing the corresponding loss is fundamentally analogous to executing a spectral decomposition on the graph.

\subsection{A Graph-Theoretic Formulation}
\label{sec:sorl_graph_def}
We start by formally defining the augmentation graph and adjacency matrix. 
For clarity, we use $\bar x$ to indicate the natural sample (raw inputs without augmentation). Given an $\bar x$, we use $\mathcal{T}(x|\Bar{x})$ to denote the probability of $x$ being augmented from $\Bar{x}$. For instance, when $\Bar{x}$ represents an image, $\mathcal{T}(\cdot|\Bar{x})$ can be the distribution of common augmentations \citep{chen2020simclr} such as Gaussian blur, color distortion, and random cropping. The augmentation allows us to define a general population space $\mathcal{X}$, which contains all the original images along with their augmentations. In our case, $\mathcal{X}$ is composed of augmented samples from both labeled and unlabeled data, with cardinality $|\mathcal{X}|=N$. We further denote  $\mathcal{X}_l$ as the set of  samples (along with augmentations) from the labeled data part.

 We define the graph $G(\mathcal{X}, w)$ with 
vertex set $\mathcal{X}$ and edge weights $w$. To define edge weights $w$, we decompose the graph connectivity into two components: (1) self-supervised connectivity $w^{(u)}$ by treating all points in $\mathcal{X}$ as entirely unlabeled, and (2) supervised connectivity $w^{(l)}$ by adding 
labeled information from $\mathcal{P}_l$ to the graph. We proceed to define these two cases separately.

First, by assuming all points as unlabeled, two samples ($x$, $x^+$) are considered a {\textbf{positive pair}} if:

\begin{center}
\textbf{Unlabeled Case (u):} \textit{$x$ and $x^+$ are augmented from the same image $\Bar{x} \sim \mathcal{P}$.}
\end{center}

For any two augmented data $x, x' \in \mathcal{X}$, $w^{(u)}_{x x'}$ denotes the marginal probability of generating the pair:
\begin{align}
\begin{split}
w^{(u)}_{x x^{\prime}} \triangleq \mathbb{E}_{\bar{x} \sim {\mathcal{P}}}  \mathcal{T}(x| \bar{x}) \mathcal{T}\left(x'| \bar{x}\right),
    \label{eq:sorl_def_wxx}
    \vspace{0.6cm}
\end{split}
\end{align}
which can be viewed as self-supervised connectivity~\citep{chen2020simclr,haochen2021provable}. However, different from self-supervised learning, ORL has access to the labeled information for a subset of nodes, which \emph{allows adding additional connectivity to the graph}. Accordingly, the positive pair can be defined as:
\begin{center}
    \textbf{Labeled Case (l):} \textit{$x$ and $x^+$ are augmented from two labeled samples $\Bar{x}_l$ and $\Bar{x}'_l$ \emph{with the same known class $i$}. In other words, both $\Bar{x}_l$ and $\Bar{x}'_l$ are drawn independently from $\mathcal{P}_{l_i}$}.
\end{center}

Considering both case (u) and case (l), the overall  edge weight for any pair of data $(x,x')$ is given by: 
\begin{align}
\begin{split}
w_{x x^{\prime}} = \eta_{u} w^{(u)}_{x x^{\prime}} + \eta_{l} w^{(l)}_{x x^{\prime}}, \text{where }
w^{(l)}_{x x^{\prime}} \triangleq \sum_{i \in \mathcal{Y}_l}\mathbb{E}_{\bar{x}_{l} \sim {\mathcal{P}_{l_i}}} \mathbb{E}_{\bar{x}'_{l} \sim {\mathcal{P}_{l_i}}} \mathcal{T}(x | \bar{x}_{l}) \mathcal{T}\left(x' | \bar{x}'_{l}\right),
    \label{eq:sorl_def_wxx_b}
\end{split}
\end{align}
and $\eta_{u},\eta_{l}$ modulates the importance between the two cases. 
The magnitude of $w_{xx'}$ indicates the ``positiveness'' or similarity between  $x$ and $x'$. 
We then use 
$w_x = \sum_{x' \in \mathcal{X}}w_{xx'}$
to denote the total edge weights connected to a vertex $x$. 

\vspace{0.1cm} \noindent \textbf{Remark: A graph perturbation view.}  With the graph connectivity defined above, we can now define the adjacency matrix $A \in \mathbb{R}^{N \times N}$ with entries $A_{x x^\prime}=w_{x x^{\prime}}$. Importantly, the adjacency matrix can be decomposed into two parts: 
\vspace{0.1cm}
\begin{equation}
    A = \eta_{u} A^{(u)} +  \tikzmarknode{Al}{\highlight{myblue}{$\eta_{l} A^{(l)}$}},
    \label{eq:sorl_adj_orl}
\end{equation}
\begin{tikzpicture}[overlay,remember picture,>=stealth,nodes={align=left,inner ysep=1pt},<-]
    \path (Al) ++ (0,0.7em) node[anchor=south west,color=blue!67] (Alt){\textit{ Perturbation by adding labels}};
    \draw [color=blue!87](Al.north) |- ([xshift=-0.3ex,color=blue]Alt.north east);
\end{tikzpicture}
which can be 
 regarded as the self-supervised adjacency matrix $A^{(u)}$ perturbed by additional labeling information encoded in $A^{(l)}$. This graph perturbation view serves as a critical foundation for our theoretical analysis of the clustering performance in Section~\ref{sec:sorl_theory}. 
As a standard technique in graph theory~\citep{chung1997spectral}, we use the \textit{normalized adjacency matrix}
of $G(\mathcal{X}, w)$:
\begin{equation}
    \Dot{A}\triangleq D^{-\frac{1}{2}} A D^{-\frac{1}{2}},
    \label{eq:sorl_def}
\end{equation}
where ${D} \in \mathbb{R}^{N \times N}$ is a diagonal matrix with ${D}^{}_{x x}=w^{}_x$. The normalization balances the degree of each node,  reducing the influence of vertices with very large degrees. The normalized adjacency matrix defines the probability of $x$ and $x^{\prime}$  being considered as the positive pair from the perspective of augmentation, which helps derive the new representation learning loss as we show next.

\subsection{SORL: Spectral Open-World Representation Learning}
\label{sec:sorl_orl-scl}
We introduce the algorithm called Spectral Open-world Representation Learning (SORL), which can be derived from a spectral decomposition of $\Dot{A}$. This technique, though
bearing similarities to NSCL as discussed in Chapter~\ref{sec:nscl}, operates within a different problem domain. 
The algorithm has both practical and theoretical values. First, it enables efficient end-to-end training in the context of modern neural networks. More importantly, it allows drawing a theoretical equivalence between learned representations and the top-$k$ singular vectors of $\Dot{A}$. Such equivalence facilitates theoretical understanding of the clustering structure encoded in $\Dot{A}$. 
Specifically, we consider low-rank matrix approximation:
\begin{equation}
    \min _{F \in \mathbb{R}^{N \times k}} \mathcal{L}_{\mathrm{mf}}(F, A)\triangleq\left\|\Dot{A}-F F^{\top}\right\|_F^2
    \label{eq:sorl_lmf}
\end{equation}
According to the Eckart–Young–Mirsky theorem~\citep{eckart1936approximation}, the minimizer of this loss function is $F_k\in \mathbb{R}^{N \times k}$ such that $F_k F_k^{\top}$ contains the top-$k$ components of $\Dot{A}$'s SVD decomposition. 

Now, if we view each row $\*f_x^{\top}$ of $F$ as a scaled version of learned feature embedding  $f:\mathcal{X}\mapsto \mathbb{R}^k$, the $\mathcal{L}_{\mathrm{mf}}(F, A)$ can be written as a form of the contrastive learning objective. We formalize it in Theorem~\ref{th:sorl_orl-scl}   below\footnote{Theorem~\ref{th:sorl_orl-scl} is primarily adapted from Theorem~\ref{th:nscl_ncd-scl} with a distinction in the data setting, as Chapter~\ref{sec:nscl} does not consider known class samples within the unlabeled dataset.}.

\begin{theorem}
\label{th:sorl_orl-scl} 
We define $\*f_x = \sqrt{w_x}f(x)$ for some function $f$. Recall $\eta_{u},\eta_{l}$ are coefficients defined in Eq.~\eqref{eq:sorl_def_wxx}. Then minimizing the loss function $\mathcal{L}_{\mathrm{mf}}(F, A)$ is equivalent to minimizing the following loss function for $f$, which we term \textbf{Spectral Open-world Representation Learning (SORL)}:
\begin{align}
\begin{split}
    \mathcal{L}_\text{SORL}(f) &\triangleq - 2\eta_{u} \mathcal{L}_1(f) 
- 2\eta_{l}  \mathcal{L}_2(f)  + \eta_{u}^2 \mathcal{L}_3(f) + 2\eta_{u} \eta_{l} \mathcal{L}_4(f) +  
\eta_{l}^2 \mathcal{L}_5(f),
\label{eq:sorl_def_SORL}
\end{split}
\end{align} where
\begin{align*}
    \mathcal{L}_1(f) &= \sum_{i \in \mathcal{Y}_l}\underset{\substack{\bar{x}_{l} \sim \mathcal{P}_{{l_i}}, \bar{x}'_{l} \sim \mathcal{P}_{{l_i}},\\x \sim \mathcal{T}(\cdot|\bar{x}_{l}), x^{+} \sim \mathcal{T}(\cdot|\bar{x}'_l)}}{\mathbb{E}}\left[f(x)^{\top} {f}\left(x^{+}\right)\right] , \\
    \mathcal{L}_2(f) &= \underset{\substack{\bar{x}_{u} \sim \mathcal{P},\\x \sim \mathcal{T}(\cdot|\bar{x}_{u}), x^{+} \sim \mathcal{T}(\cdot|\bar{x}_u)}}{\mathbb{E}}
\left[f(x)^{\top} {f}\left(x^{+}\right)\right], \\
    \mathcal{L}_3(f) &= \sum_{i, j\in \mathcal{Y}_l}
    \underset{\substack{\bar{x}_l \sim \mathcal{P}_{{l_i}}, \bar{x}'_l \sim \mathcal{P}_{{l_{j}}},\\x \sim \mathcal{T}(\cdot|\bar{x}_l), x^{-} \sim \mathcal{T}(\cdot|\bar{x}'_l)}}{\mathbb{E}}
\left[\left(f(x)^{\top} {f}\left(x^{-}\right)\right)^2\right], \\
    \mathcal{L}_4(f) &= \sum_{i \in \mathcal{Y}_l}\underset{\substack{\bar{x}_l \sim \mathcal{P}_{{l_i}}, \bar{x}_u \sim \mathcal{P},\\x \sim \mathcal{T}(\cdot|\bar{x}_l), x^{-} \sim \mathcal{T}(\cdot|\bar{x}_u)}}{\mathbb{E}}
\left[\left(f(x)^{\top} {f}\left(x^{-}\right)\right)^2\right], \\
    \mathcal{L}_5(f) &= \underset{\substack{\bar{x}_u \sim \mathcal{P}, \bar{x}'_u \sim \mathcal{P},\\x \sim \mathcal{T}(\cdot|\bar{x}_u), x^{-} \sim \mathcal{T}(\cdot|\bar{x}'_u)}}{\mathbb{E}}
\left[\left(f(x)^{\top} {f}\left(x^{-}\right)\right)^2\right].
\end{align*}
\end{theorem}
\begin{proof} (\textit{sketch})
We can expand $\mathcal{L}_{\mathrm{mf}}(F, A)$ and obtain
\begin{align*}
\mathcal{L}_{\mathrm{mf}}(F, A) &=\sum_{x, x^{\prime} \in \mathcal{X}}\left(\frac{w_{x x^{\prime}}}{\sqrt{w_x w_{x^{\prime}}}}-\*f_x^{\top} \*f_{x^{\prime}}\right)^2 \\
&= const + 
\sum_{x, x^{\prime} \in \mathcal{X}}\left(-2 w_{x x^{\prime}} f(x)^{\top} {f}\left(x^{\prime}\right)+w_x w_{x^{\prime}}\left(f(x)^{\top}{f}\left(x^{\prime}\right)\right)^2\right)
\end{align*} 
The form of $\mathcal{L}_\text{SORL}(f)$ is derived from plugging $w_{xx'}$ (defined in Eq.~\eqref{eq:sorl_def_wxx}) and $w_x$. 
Full proof is in Appendix~\ref{sec:sorl_proof-SORL}. 
\end{proof}

\textbf{Interpretation of $\mathcal{L}_\text{SORL}(f)$.} 
At a high level, $\mathcal{L}_1$ and $\mathcal{L}_2$ push the embeddings of \textbf{positive pairs} to be closer while $\mathcal{L}_3$, $\mathcal{L}_4$ 
 and $\mathcal{L}_5$ pull away the embeddings of \textbf{negative pairs}. In particular, $\mathcal{L}_1$ samples two random augmentation views of two images from labeled data with the \textbf{same} class label, and $\mathcal{L}_2$ samples two views from the same image in $\mathcal{X}$. For negative pairs, $\mathcal{L}_3$ uses two augmentation views from two samples in $\mathcal{X}_{l}$ with \textbf{any} class label. $\mathcal{L}_4$ uses two views of one sample in $\mathcal{X}_{l}$ and another one in $\mathcal{X}$. $\mathcal{L}_5$ uses two views from two random samples in $\mathcal{X}$.

\section{Theoretical Analysis~}
\label{sec:sorl_theory}
So far we have presented a spectral approach for open-world representation learning based on graph factorization. Under this framework, we now formally  analyze:
\emph{\textbf{how does the labeling information shape the representations for known and novel classes?}}

\subsection{An Illustrative Example}
\label{sec:sorl_theory_toy}
We consider a toy example that helps illustrate the core idea of our theoretical findings. Specifically, the example aims to distinguish 3D objects with different shapes, as shown in Figure~\ref{fig:sorl_toy_setting}. These images are generated by a 3D rendering software~\citep{johnson2017clevr} with user-defined properties including colors, shape, size, position, etc. We are interested in contrasting the representations (in the form of singular vectors), when  the label information is either incorporated in training or not. 

\vspace{0.1cm} \noindent \textbf{Data design.} Suppose the training samples come from three types, $\mathcal{X}_{\cube{1}}$, $\mathcal{X}_{\sphere{0.7}{gray}}$, $\mathcal{X}_{\cylinder{0.6}}$. Let $\mathcal{X}_{\cube{1}}$ be the sample space with \textbf{known} class, and $\mathcal{X}_{\sphere{0.7}{gray}}, \mathcal{X}_{\cylinder{0.6}}$ be the sample space with \textbf{novel} classes. Further, the two novel classes are constructed to have different relationships with the known class. Specifically, 
$\mathcal{X}_{\sphere{0.7}{gray}}$ shares some similarity with $\mathcal{X}_{\cube{1}}$ in color (red and blue); whereas another novel class $\mathcal{X}_{\cylinder{0.6}}$ has no obvious similarity with the known class.  
Without any labeling information, it can be difficult to distinguish $\mathcal{X}_{\sphere{0.7}{gray}}$ from $\mathcal{X}_{\cube{1}}$ since samples share common colors. We aim to verify the hypothesis that: \textit{adding labeling information to $\mathcal{X}_{\cube{1}}$ (i.e., connecting $\textcolor{blue}{\cube{0.7}}$ and $\textcolor{red}{\cube{0.7}}$) has a larger (beneficial) impact to cluster  $\mathcal{X}_{\sphere{0.7}{gray}}$ than $\mathcal{X}_{\cylinder{0.6}}$}.

\textbf{Augmentation graph.} Based on the data design, we formally define the augmentation graph, which encodes the probability of augmenting a source image $\bar{x}$ to the augmented view $x$:
\begin{align}
    \mathcal{T}\left(x \mid \bar{x} \right)=
    \left\{\begin{array}{ll}
    \tau_{1} & \text { if }  \text{color}(x) = \text{color}(\bar{x}), \text{shape}(x) = \text{shape}(\bar{x}); \\
    \tau_{c} & \text { if }  \text{color}(x) = \text{color}(\bar{x}), \text{shape}(x) \neq \text{shape}(\bar{x}); \\
    \tau_{s} & \text { if }  \text{color}(x) \neq \text{color}(\bar{x}), \text{shape}(x) = \text{shape}(\bar{x}); \\
    \tau_{0} & \text { if }  \text{color}(x) \neq \text{color}(\bar{x}), \text{shape}(x) \neq \text{shape}(\bar{x}). \\
    \end{array}\right.
    \label{eq:sorl_def_edge}
\end{align}
With Eq.~\eqref{eq:sorl_def_edge} and the definition of the adjacency matrix in Section~\ref{sec:sorl_graph_def}, we can derive the analytic form of $A^{(u)}$  and $A$, as shown in Figure~\ref{fig:sorl_toy_setting}(b). We refer readers to Appendix~\ref{sec:sorl_sup_toy} for the detailed derivation. The two matrices allow us to contrast the connectivity changes in the graph, before and after the labeling information is added.%

%%%%%%%%%%%%%%%%%%%%%% Figure toy setting  %%%%%%%%%%%%%%%%%%%%%% 
\begin{figure*}[htb]
    \centering
\includegraphics[width=0.98\linewidth]{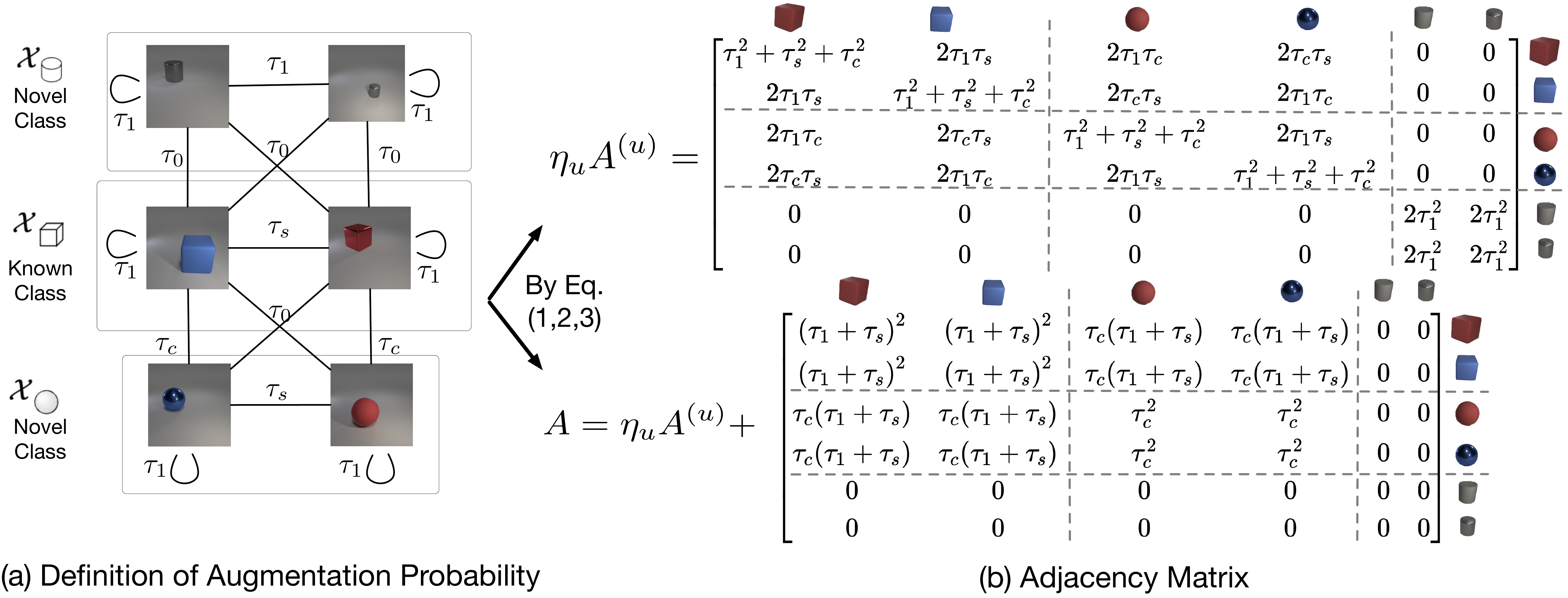}
    \caption[An illustrative example for theoretical analysis.]{An illustrative example for theoretical analysis. We consider a 6-node graph with one known class (cube) and two novel classes (sphere, cylinder). (a) The augmentation probabilities between nodes are defined by their color and shape in Eq.~\eqref{eq:sorl_def_edge}. (b) The adjacency matrix can then be calculated by Equations in Sec.~\ref{sec:sorl_graph_def} where we let $\tau_0=0, \eta_u=6, \eta_l=4$. The calculation details are in Appendix~\ref{sec:sorl_sup_toy}. The magnitude order follows $\tau_1 \gg \tau_{c} > \tau_{s} > 0$. }
    
    \label{fig:sorl_toy_setting}
\end{figure*}
%%%%%%%%%%%%%%%%%%%%%%%%%%%%%%%%%%%%%%%%%%%%%%%%%%%%%%%%%%%%%%%%% 

\vspace{0.1cm} \vspace{0.1cm} \noindent \textbf{Insights.} We are primarily interested in analyzing the
difference of the representation space derived from $A^{(u)}$ and $A$. We visualize the top-3 eigenvectors\footnote{When $\tau_1 \gg \tau_{c} > \tau_{s} > 0$, the top-3 eigenvectors 
% $V_3$ 
are almost equivalent to the feature embedding.}
% $Z$. } 
of the normalized adjacency matrix $\Dot{A}^{(u)}$ and $\Dot{A}$ in Figure~\ref{fig:sorl_toy_result}(a), where the results are based on the magnitude order  $\tau_1 \gg \tau_{c} > \tau_{s} > 0$. Our key {takeaway} is: \emph{adding labeling information to known class $\mathcal{X}_{\cube{1}}$ helps better distinguish the known class itself and the novel class $\mathcal{X}_{\sphere{0.7}{gray}}$, which has a stronger connection/similarity with $\mathcal{X}_{\cube{1}}$}. 

\vspace{0.1cm} \noindent \textbf{Qualitative analysis.} Our theoretical insight can also be verified empirically, by learning representations on over 10,000 samples using the loss defined in Section~\ref{sec:sorl_orl-scl}. Due to the space limitation, we include experimental details in Appendix~\ref{sec:sorl_sup_exp_vis}.
In Figure~\ref{fig:sorl_toy_result}(b), we visualize the learned features through UMAP~\citep{umap}. Indeed, we observe that samples become more concentrated around different shape classes after adding labeling information to the cube class.

%%%%%%%%%%%%%%%%%%%%%% Figure toy setting  %%%%%%%%%%%%%%%%%%%%%% 
\begin{figure*}[htb]
    \centering
\includegraphics[width=0.98\linewidth]{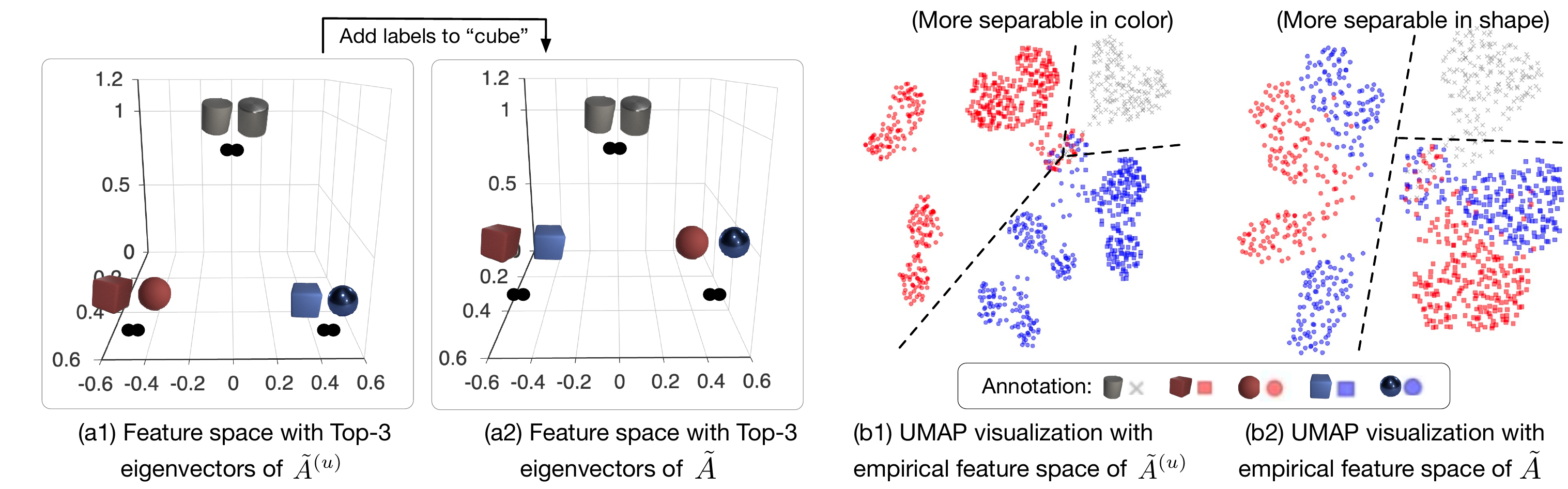}
    \caption[Visualization of representation space for toy example. ]{\small Visualization of representation space for toy example. (a) Theoretically contrasting the feature formed by top-3 eigenvectors of $\Dot{A}^{(u)}$ and $\Dot{A}$ respectively. (b) UMAP visualization of the features learned without (left) and with labeled information (right).  Details are in Appendix~\ref{sec:sorl_sup_toy} (eigenvector calculation) and Appendix~\ref{sec:sorl_sup_exp_vis} (visualization setting). }
    
    \label{fig:sorl_toy_result}
\end{figure*}
%%%%%%%%%%%%%%%%%%%%%%%%%%%%%%%%%%%%%%%%%%%%%%%%%%%%%%%%%%%%%%%%% 

%

\subsection{Main Theory}
\label{sec:sorl_th_main}
The toy example offers an important insight that the added labeled information is more helpful for the class with a stronger connection to the known class.
In this section, we
 formalize this insight by extending the toy example to
a more general setting. As a roadmap, we derive the result through three steps: \textbf{(1)} derive the closed-form solution of the learned representations; \textbf{(2)} define the clustering performance by the K-means measure; \textbf{(3)} contrast the resulting clustering performance before and after adding labels. We start by deriving the representations. 

\subsubsection{Learned Representations in Analytic Form}
\label{sec:sorl_feature}
\noindent\textbf{Representation without labels.} 
To obtain the representations, one can train the neural network $f:\mathcal{X}\mapsto \mathbb{R}^k$ using the spectral loss defined in Equation~\ref{eq:sorl_def_SORL}. We assume that the optimizer is capable to obtain the representation $Z^{(u)} \in \mathbb{R}^{N\times k}$
that minimizes the loss, where each row vector $\*z_i = f(x_i)^{\top}$. Recall that Theorem~\ref{th:sorl_orl-scl} allows us to derive a closed-form solution for the learned feature space by the spectral decomposition of the adjacency matrix, which is $\Dot{A}^{(u)}$ in the case without labeling information. Specifically, we have $F^{(u)}_k = \sqrt{D^{(u)}}Z^{(u)}$, where $F^{(u)}_k F^{(u)\top}_k$ contains the top-$k$ components of $\Dot{A}^{(u)}$'s SVD decomposition and $D^{(u)}$ is the diagonal matrix defined based on the row sum of $A^{(u)}$. We further define the top-$k$ singular vectors of $A^{(u)}$  as $V_k^{(u)} \in \mathbb{R}^{N\times k}$, so we have $F^{(u)}_k = V_k^{(u)} \sqrt{\Sigma_k^{(u)}}$, where $\Sigma_k^{(u)}$ is a diagonal matrix of the top-$k$ singular values of $A^{(u)}$.
{By equalizing the two forms of $F^{(u)}_k$}, the closed-formed solution of the learned feature space is given by $Z^{(u)} = [D^{(u)}]^{-\frac{1}{2}} V_k^{(u)} \sqrt{\Sigma_k^{(u)}}$. %

\vspace{0.1cm} \noindent \textbf{Representation perturbation by adding labels.} 
We now analyze how the representation is ``perturbed'' as a result of adding label information. We consider $|\mathcal{Y}_l| = 1$\footnote{To understand the perturbation by adding labels from more than one class, one can take the summation of the perturbation by each class.} to facilitate a better understanding of our key insight. We can rewrite $A$ in Eq.~\ref{eq:sorl_adj_orl} as:
\begin{equation*}
    A(\delta) \triangleq \eta_u A^{(u)} + \delta \mathfrak{l} \mathfrak{l}^{\top},
\end{equation*}
where we replace $\eta_l$ to $\delta$ to be more apparent in representing the perturbation and define $\mathfrak{l} \in \mathbb{R}^{N}, (\mathfrak{l})_x = \mathbb{E}_{\bar{x}_{l} \sim {\mathcal{P}_{l_1}}} \mathcal{T}(x | \bar{x}_{l})$. Note that $\mathfrak{l}$ can be interpreted as the vector of ``\textit{the semantic connection for sample $x$ to the labeled data}''. One can easily extend to $r$ classes by letting $\mathfrak{l} \in \mathbb{R}^{N \times r}$.

Here we treat the adjacency matrix as a function of the perturbation. In a similar manner as above, we can derive the normalized  adjacency matrix $\Tilde{A}(\delta)$ and the feature representation $Z(\delta)$ in closed-form. The details are included in Appendix~\ref{sec:sorl_sup_proof_main}.

\subsubsection{Evaluation Target}
\label{sec:sorl_eval}
With the learned representations, we can evaluate their quality by the clustering performance.
Our theoretical analysis of the clustering performance can well connect to empirical evaluation strategy in the literature~\citep{yang2022divide} using $K$-means clustering accuracy/error. 
Formally, we define the ground-truth partition of clusters by $\Pi = \{\pi_{1}, \pi_{2}, ..., \pi_{C}\}$, where $\pi_{i}$ is the set of samples' indices with underlying label $y_i$ and $C$ is the total number of classes (including both known and novel). We further let $\boldsymbol{\mu}_\pi = \mathbb{E}_{i \in \pi}\*z_i$ be the center of features in $\pi$, and the average of all feature vectors be $\boldsymbol{\mu}_\Pi = \mathbb{E}_{j \in [N]}\*z_j$.

The clustering performance of K-means depends on two measurements: \textbf{Intra-class} measure and \textbf{Inter-class} measure. Specifically, we let the intra-class measure be the average Euclidean distance from the samples' feature to the corresponding cluster center and we measure the inter-class separation as the distances between cluster centers: 
\begin{equation}
    \mathcal{M}_\text{intra-class}(\Pi, Z) \triangleq \sum_{\pi\in \Pi}  \sum_{i \in \pi}\left\|\*z_i-\boldsymbol{\mu}_\pi \right\|^2, \mathcal{M}_\text{inter-class}(\Pi, Z) \triangleq  \sum_{\pi\in \Pi}  |\pi| \left\|\boldsymbol{\mu}_\pi-\boldsymbol{\mu}_\Pi\right\|^2.
\label{eq:sorl_def_align_sep_measure}
\end{equation}
Strong clustering results translate into low $\mathcal{M}_\text{intra-class}$ and high $\mathcal{M}_\text{inter-class}$. Thus we define the \textbf{K-means measure} as:
\begin{equation}
    \mathcal{M}_{kms}(\Pi, Z) \triangleq \mathcal{M}_\text{intra-class}(\Pi, Z) / \mathcal{M}_\text{inter-class}(\Pi, Z). 
\label{eq:sorl_def_kms_measure}
\end{equation}
We also formally show in Theorem~\ref{th:sorl_sup_cluster_err_bound} (Appendix) that the K-means clustering error\footnote{ {It is theoretically inconvenient to directly analyze the clustering error since it is a non-differentiable target.} } is asymptotically equivalent to the K-means measure we defined above. 

\subsubsection{Perturbation in Clustering Performance}
With the evaluation target defined above, our main analysis will revolve around analyzing \textit{``how the extra label information help reduces} $\mathcal{M}_{kms}(\Pi, Z)$''. % 
Formally, we investigate the following error difference, as a result of added label information:
$$\Delta_{kms}(\delta) = \mathcal{M}_{kms}(\Pi, Z) - \mathcal{M}_{kms}(\Pi, Z(\delta)), $$
where the closed-form solution is given by the following theorem. Positive $\Delta_{kms}(\delta)$ means improved clustering, as a result of adding labeling information.

\begin{tcolorbox}[colback=gray!5!white]
\begin{theorem} (Main result.) Denote $V_{\varnothing}^{(u)} \in \mathbb{R}^{N \times (N-k)}$ as the \textit{null space} of $V_k^{(u)}$ and $\Tilde{A}_k^{(u)} = V_k^{(u)} \Sigma_k^{(u)} V_k^{(u)\top}$ as the rank-$k$ approximation for $\Tilde{A}^{(u)}$.  
 Given $\delta, \eta_{1} > 0 $ and let 
 $\mathcal{G}_k$ as the spectral gap between $k$-th and $k+1$-th singular values of $\Tilde{A}^{(u)}$, we have: 
\begin{align*}
    \Delta_{kms}(\delta) &=  \delta \eta_{1} \operatorname{Tr} \left( \Upsilon \left(V_k^{(u)} V_k^{(u)\top} \mathfrak{l} \mathfrak{l}^{\top}(I +   V_\varnothing^{(u)}  V_\varnothing^{(u)\top}) - 2\Tilde{A}_k^{(u)} diag(\mathfrak{l}) \right)\right) \\ &+ O(\frac{1}{\mathcal{G}_k} + \delta^2),
\end{align*}
where $diag(\cdot)$ converts the vector to the corresponding diagonal matrix and $\Upsilon \in \mathbb{R}^{N\times N}$ is a matrix encoding the \textbf{ground-truth clustering structure} in the way that $\Upsilon_{xx'} > 0$ if $x$ and $x'$ has the same label and  $\Upsilon_{xx'} < 0$ otherwise. The concrete form and the proof are in Appendix~\ref{sec:sorl_sup_proof_main}.  
\label{th:sorl_main}
\end{theorem}
\end{tcolorbox}
 Theorem~\ref{th:sorl_main} is more general but less intuitive to understand. To gain a better insight, we introduce Theorem~\ref{th:sorl_main_simp} which provides more direct implications. {We provide the justification of the  assumptions and the formal proof in Appendix~\ref{sec:sorl_sup_proof_main_simp}.}

\begin{tcolorbox}[colback=gray!5!white]
\begin{theorem} (Intuitive result.) 
Assuming the spectral gap $\mathcal{G}_k$ is sufficiently large and $\mathfrak{l}$ lies in the linear span of $V_k^{(u)}$. We also assume $\forall \pi_c \in \Pi, \forall i \in \pi_c, \mathfrak{l}_{(i)} =: \mathfrak{l}_{\pi_c}$ which represents the \textit{connection between class $c$ to the labeled data}.
 Given $\delta, \eta_{1}, \eta_{2} > 0 $, 
 we have: 
\begin{align*}
    &\Delta_{kms}(\delta) \geq  \delta \eta_{1}\eta_{2} \sum_{\pi_c \in \Pi}  |\pi_c| \mathfrak{l}_{\pi_c} \Delta_{\pi_c}(\delta),
\end{align*}
where 
\begin{equation*}
    \Delta_{\pi_c}(\delta) = \tikzmarknode{connect}{\highlight{myblue}{$(\mathfrak{l}_{\pi_c} - \frac{1}{N})$}} - (\frac{2N-2|\pi_c|}{N})(\tikzmarknode{cmp}{\highlight{red}{$\mathbb{E}_{i \in \pi_c} \mathbb{E}_{j \in \pi_c}\*z_i^{\top}\*z_j$}} - \tikzmarknode{div}{\highlight{orange}{$\mathbb{E}_{i \in \pi_c} \mathbb{E}_{j \notin \pi_c}\*z_i^{\top}\*z_j$}} ).~~~~~~~~
\end{equation*}
\begin{tikzpicture}[overlay,remember picture,>=stealth,nodes={align=left,inner ysep=1pt},<-]
    \path (connect.north) ++ (0,0.5em) node[anchor=south west,color=blue!67] (cnnt){\textit{Connection from class $c$ to the labeled data.}};
    \draw [color=blue!87](connect.north) |- ([xshift=-0.3ex,color=myblue]cnnt.south east);
    \path (cmp.south) ++ (0,0.0em) node[anchor=north east,color=red!67] (c){\textit{ Intra-class similarity }};
    \draw [color=red!87](cmp.south) |- ([xshift=-0.3ex,color=red] c.south west);
    \path (div.south) ++ (-0.7em,-1.3em) node[anchor=north east,color=orange!99] (c){\textit{ Inter-class similarity}};
    \draw [color=orange!87](div.south) |- ([xshift=-0.3ex,color=orange] c.south west);
\end{tikzpicture}
\label{th:sorl_main_simp}
\end{theorem}
\end{tcolorbox}
\textbf{Implications.} In Theorem~\ref{th:sorl_main_simp}, we define the \textbf{class-wise perturbation} of the K-means measure as $\Delta_{\pi_c}(\delta)$. This way, we can interpret the effect of adding labels for a specific class $c$. If we desire $\Delta_{\pi_c}(\delta)$ to be large, the sufficient condition is that 
\begin{center}
    \textit{connection of class c to the labeled data > intra-class similarity - inter-class similarity}.
\end{center}
\vspace{-0.15cm}
    We use examples in Figure~\ref{fig:sorl_teaser} to epitomize the core idea. Specifically, our unlabeled samples consist of  three underlying classes: traffic lights (known), apples (novel), and flowers (novel). \textbf{(a)} For unlabeled traffic lights from \emph{known classes} which are strongly connected to the labeled data, adding labels to traffic lights can largely improve the clustering performance; \textbf{(b)} For \emph{novel classes} like apples, it may also help  when they have a strong connection to the traffic light, and their intra-class similarity is not as strong (due to different colors); \textbf{(c)} However, labeled data may offer little improvement in clustering the flower class, due to the minimal connection to the labeled data and that flowers' self-clusterability is already strong.

\section{Empirical Validation of Theory~}
\label{sec:sorl_exp}

Beyond theoretical insights, we show empirically that SORL is effective on standard benchmark image
classification datasets CIFAR-10/100~\citep{krizhevsky2009learning}. 
Following the seminal work ORCA~\citep{cao2022openworld}, classes are divided into 50\% known and 50\% novel classes. We then use 50\% of samples from the known classes as the labeled dataset, and the rest as the unlabeled set. 
We follow the evaluation strategy in~\citep{cao2022openworld} and report the following metrics: (1) classification accuracy on known classes, (2) clustering accuracy on the novel data, and (3) overall accuracy on all classes. 
More experiment details are in Appendix~\ref{sec:sorl_sup_exp_details}. 

\begin{table}[htb]

\centering
\vspace{-0.6cm}
\caption[Main Results of comparing baselines in open-world representation learning.]{\small Main Results. Mean and std are estimated on five different runs. Baseline numbers are  from ~\citep{sun2023opencon,cao2022openworld}.}  
\resizebox{0.99\linewidth}{!}{
\begin{tabular}{lllllll}
\toprule
\multirow{2}{*}{\textbf{Method}} & \multicolumn{3}{c}{\textbf{CIFAR-10}} & \multicolumn{3}{c}{\textbf{CIFAR-100}} \\
 & \textbf{All} & \textbf{Novel} & \textbf{Known} & \textbf{All} & \textbf{Novel} & \textbf{Known} \\ \midrule
\textbf{FixMatch}~\citep{alex2020fixmatch} & 49.5 & 50.4 & 71.5  & 20.3 & 23.5 & 39.6 \\
\textbf{DS$^{3}$L}~\citep{guo2020dsl} & 40.2 & 45.3 & 77.6  & 24.0 & 23.7 & 55.1\\
\textbf{CGDL}~\citep{sun2020cgdl} & 39.7 & 44.6 & 72.3  & 23.6 & 22.5 & 49.3 \\
\textbf{DTC}~\citep{Han2019dtc} & 38.3 & 39.5 & 53.9 & 18.3 & 22.9 & 31.3  \\
\textbf{RankStats}~\citep{zhao2021rankstat} & 82.9 & 81.0 & 86.6  & 23.1 & 28.4 & 36.4 \\
\textbf{SimCLR}~\citep{chen2020simclr} & 51.7 & 63.4 & 58.3 & 22.3 & 21.2 & 28.6 \\ \midrule
\textbf{ORCA}~\citep{cao2022openworld} & 88.3$ ^{\pm{0.3}} $ & 87.5$ ^{\pm{0.2}} $ & 89.9$ ^{\pm{0.4}} $ & 47.2$ ^{\pm{0.7}} $ & 41.0$ ^{\pm{1.0}} $ & 66.7$ ^{\pm{0.2}} $  \\
\textbf{GCD}~\citep{vaze22gcd} & 87.5$ ^{\pm{0.5}} $ & 86.7$ ^{\pm{0.4}} $ & 90.1$ ^{\pm{0.3}} $ & 46.8$ ^{\pm{0.5}} $ & 43.4$ ^{\pm{0.7}} $ & \textbf{69.7}$ ^{\pm{0.4}} $   \\
\textbf{SORL (Ours)} & \textbf{93.5} $ ^{\pm{1.0}} $ & \textbf{92.5} $ ^{\pm{0.1}} $  & \textbf{94.0}$ ^{\pm{0.2}} $  & \textbf{56.1} $ ^{\pm{0.3}} $ & \textbf{52.0} $ ^{\pm{0.2}} $  & 68.2$ ^{\pm{0.1}} $
\\ \bottomrule
\end{tabular}}
\label{tab:sorl_main}
\end{table}
% \end{wraptable}

\vspace{0.1cm} \noindent \textbf{SORL achieves competitive performance.} 
Our proposed loss SORL is amenable to the theoretical understanding, which is our primary goal of this work. Beyond theory, we show that SORL is equally desirable in empirical performance. In particular, SORL displays competitive performance compared to existing methods, as evidenced in Table~\ref{tab:sorl_main}. Our comparison covers an extensive collection of very recent algorithms developed for this problem, including ORCA~\citep{cao2022openworld}, GCD~\citep{vaze22gcd}. We also compare methods in related problem domains: 
(1) Semi-Supervised Learning~\citep{alex2020fixmatch,guo2020dsl,sun2020cgdl}, (2) Novel Class Discovery~\citep{Han2019dtc,zhao2021rankstat}, (3) common representation learning method SimCLR~\citep{chen2020simclr}.
In particular, on CIFAR-100, we improve upon the best baseline ORCA by \textbf{8.9}\% in terms of overall accuracy.
 Our result further validates that putting analysis on SORL is appealing for both theoretical and empirical reasons.
% \vspace{0.1cm} \noindent \textbf{Verification on the Key Results.} 

\section{Broader Impact~}

From a theoretical perspective, our graph-theoretic framework can facilitate and deepen the understanding of other representation learning methods that commonly involve the notion of positive/negative pairs. In Appendix~\ref{sec:sorl_sup_simclr_analysis}, \textit{we exemplify how our framework can be potentially generalized to other common contrastive loss functions}~\citep{van2018cpc, khosla2020supcon, chen2020simclr}, and baseline methods that are tailored for the open-world representation learning problem (e.g., GCD~\citep{vaze22gcd}). Hence, we believe our theoretical framework has a broader utility and significance.

From a practical perspective, our work can directly impact and  benefit many real-world applications, where unlabeled data are produced at an incredible rate today. Major companies exhibit a strong need for making their machine learning systems and services amendable for the open-world setting but lack fundamental and systematic knowledge. Hence, our research advances the understanding of open-world machine learning and helps the industry improve ML systems by discovering insights and structures from unlabeled data.

\section{Additional Related Work~}
\label{sec:sorl_related}

\vspace{0.1cm} \noindent \textbf{Semi-supervised learning.} 
Semi-supervised learning (SSL) is a classic problem in machine learning. SSL typically assumes the same class space between labeled and unlabeled data, and hence remains closed-world. A rich line of empirical works~\citep{chapelle2006ssl, lee2013pseudo,sajjadi2016regularization,laine2016temporal,zhai2019s4l,rebuffi2020semi,alex2020fixmatch,guo2020dsl, chen2020semi, yu2020multi,park2021opencos, saito2021openmatch, huang2021trash, yang2022classaware,liu2010large} and theoretical efforts~\citep{oymak2021theoretical,sokolovska2008asymptotics,singh2008unlabeled,balcan2005pac,rigollet2007generalization,wasserman2007statistical,niyogi2013manifold} have been made to address this problem. An important class of SSL methods is to represent data as graphs and predict labels by aggregating proximal nodes' labels~\citep{zhu2002learning,zhang2009prototype,wang2006label,fergus2009semi,jebara2009graph,zhou2004semi,argyriou2005combining}. Different from classic SSL, we allow its
semantic space to cover both known and novel classes. Accordingly, we contribute a new graph-theoretic framework tailored to the open-world setting, and reveal new insights on how the labeled data can benefit the clustering performance on both known and novel classes. %

\vspace{0.1cm} \noindent \textbf{Spectral graph theory.} 
Spectral graph theory is a classic research problem~\cite{von2007tutorial,chung1997spectral,cheeger2015lower,kannan2004clusterings,lee2014multiway,mcsherry2001spectral}, which aims to partition the graph by studying the eigenspace of the adjacency matrix. The spectral graph theory is also widely applied in machine learning~\cite{ng2001spectral,shi2000normalized,blum2001learning,zhu2003semi,argyriou2005combining,shaham2018spectralnet,sun2023nscl}. Recently, ~\citet{haochen2021provable} derive a spectral contrastive loss from the factorization of the graph's adjacency matrix which facilitates theoretical study in unsupervised domain adaptation~\cite{shen2022connect,haochen2022beyond}. In these works, the graph's formulation is exclusively based on unlabeled data. Sun et al.~\cite{sun2023nscl} later expanded this spectral contrastive loss approach to cater to learning environments that encompass both labeled data from known classes and unlabeled data from novel ones. In this chapter, our adaptation of the loss function from ~\cite{sun2023nscl} is tailored to address the open-world representation learning challenge, considering known class samples within unlabeled data. 

\vspace{0.1cm} \noindent \textbf{Theory for self-supervised learning.} A proliferation of works in self-supervised representation learning demonstrates the empirical success~\citep{van2018cpc,chen2020simclr,caron2020swav,he2019moco,zbontar2021barlow,bardes2021vicreg,chen2021exploring,haochen2021provable} with the theoretical foundation by providing provable guarantees on the representations learned by contrastive learning for linear probing~\citep{arora2019theoretical,lee2021predicting,tosh2021contrastive,tosh2021contrastive2,balestriero2022contrastive,shi2023the}. From the graphic view, ~\citet{shen2022connect,haochen2021provable,haochen2022beyond}  model the pairwise relation by the augmentation probability and provided error analysis of the downstream tasks. 
The existing body of work has mostly focused on \emph{unsupervised learning}. 
In this chapter, we systematically investigate how the label information can change the representation manifold and affect the downstream clustering performance on both known and novel classes.

\section{Summary~}
\label{sec:sorl_summary}

In this chapter, we present a graph-theoretic framework for open-world representation learning. The framework facilitates the understanding of how representations change as a result of adding labeling information to the graph. Specifically, we learn representation through Spectral Open-world Representation Learning (SORL). Minimizing this objective is equivalent to factorizing the graph’s 
adjacency matrix, which allows us to analyze the clustering error difference between having vs. excluding labeled data. Our main results suggest that the clustering error can be significantly reduced if the
connectivity to the labeled data is stronger than their self-clusterability. Our framework is also empirically appealing to use since it achieves competitive performance on par with existing baselines. 
We also hope our framework and insights can inspire the broader representation learning community to understand the role of labeling prior.

%%%%%%%%%%%%%%%%%%%%%%%%%%%%%%%%%%%%%%%%%%%%%%%%%%%%%%%%%%%%%%%%%%%%%
%%%%%%%%%%%%%%%%%%%%%%%%%%%%%%  Supp %%%%%%%%%%%%%%%%%%%%%%%%%%%%%%%
%%%%%%%%%%%%%%%%%%%%%%%%%%%%%%%%%%%%%%%%%%%%%%%%%%%%%%%%%%%%%%%%%%%%  

% \newpage
% \section{Appendix}
% \label{sec:sorl_supp}
% \input{chapters/supp_sorl}

%% file: chapters/6_opencon.tex
\chapter{OpenCon: Open-world Contrastive Learning~}
\label{sec:opencon}

\paragraph{Publication Statement.} This chapter is a joint work with Yixuan Li. The paper version of this chapter appeared in TMLR23~\citep{sun2023opencon}. 

\noindent\rule{\textwidth}{1pt}

In the preceding chapters, we have established a solid theoretical foundation for open-world representation learning. As we delve deeper into this thesis, we unveil a pioneering learning framework, dubbed Open-World Contrastive Learning (OpenCon). This cutting-edge approach provides solutions to the empirical challenges stemming from the theoretical concepts previously discussed. OpenCon adeptly grapples with the complexities of constructing compact representations for both known and novel classes, while also facilitating novelty discovery along the way. The efficacy of OpenCon is demonstrated through rigorous testing on challenging benchmark datasets, where it exhibits superior performance.
On the ImageNet dataset, OpenCon significantly outperforms the current best method by {11.9}\% and {7.4}\% on the novel and overall classification accuracy, respectively. Theoretically, OpenCon can be rigorously interpreted from an EM algorithm perspective---minimizing our contrastive loss partially maximizes the likelihood by clustering similar samples in the embedding space.

%%%%%%%%%%%%%%%%%%%%%%%%%%%%%%%%%%%%%%%%%%%%%%%%%%%%%%%%%%%%%%%%%%%%%%%%%%
%%%%%%%%%%%%%%%%%%%%%%%%  INTRODUCTION SECTION %%%%%%%%%%%%%%%%%%%%%%%%%%%%
%%%%%%%%%%%%%%%%%%%%%%%%%%%%%%%%%%%%%%%%%%%%%%%%%%%%%%%%%%%%%%%%%%%%%%%%%%

\section{Introduction~}
\label{sec:opencon_intro}

Modern machine learning methods have achieved remarkable success~\citep{sun2017faster, van2018cpc, chen2020simclr, caron2020swav, he2019moco, zheng2021weakcl, wu2021ngc, cha2021co2l, cui2021parametriccl, jiang2021improving, gao2021fewshot, zhong2021ncl, zhao2021rankstat, fini2021unified,tsai2022wcl2, zhang2022semi,wang2022pico}. Noticeably, the vast majority of learning algorithms have been driven by the
closed-world setting, where the classes are assumed stationary and unchanged. This assumption, however, rarely holds for models deployed in the wild. 
 One important  characteristic of open world is that
the model will naturally encounter novel classes. 
Considering a realistic scenario, where a machine learning model for recognizing products in e-commerce may encounter brand-new products together with old products. Similarly, an autonomous driving
model can run into novel objects on the road, in addition to known ones. 
Under the setting, the model should ideally learn to distinguish not only the known classes, but also the novel categories. This problem is  proposed as open-world semi-supervised learning~\citep{cao2022openworld} or generalized category discovery~\citep{vaze22gcd}. Research efforts have only started  very recently to address this important and realistic problem.

Formally, we are given a labeled
training dataset $\mathcal{D}_l$ as well as an unlabeled dataset $\mathcal{D}_u$. The labeled dataset contains samples that belong to
a set of {known} classes, while the unlabeled dataset has a mixture of samples from \emph{both the known and novel classes}. In practice, such unlabeled in-the-wild data can be collected
almost for free upon deploying a model
in the open world, and thus is available in abundance. Under the setting, our goal is to learn distinguishable representations for both known and novel classes
simultaneously. While this setting naturally suits many real-world
applications, it also poses unique challenges due to: (a){ the lack of clear separation between known vs. novel data in} $\mathcal{D}_{u}$, and (b){ the lack of supervision for data in novel classes.} Traditional representation learning methods are not designed for this new setting. For example, supervised contrastive learning (SupCon)~\citep{khosla2020supcon} only assumes the labeled set $\mathcal{D}_l$, without considering the unlabeled data $\mathcal{D}_u$. Weakly supervised contrastive learning~\citep{zheng2021weakcl} assumes the same classes in labeled and unlabeled data, hence remaining closed-world and less generalizable to novel samples. Self-supervised learning~\citep{chen2020simclr} relies completely on the unlabeled set $\mathcal{D}_u$ and 
does not utilize the availability of the labeled dataset $\mathcal{D}_l$. 

\begin{figure}[t]
    \centering
    \includegraphics[width=0.97\linewidth]{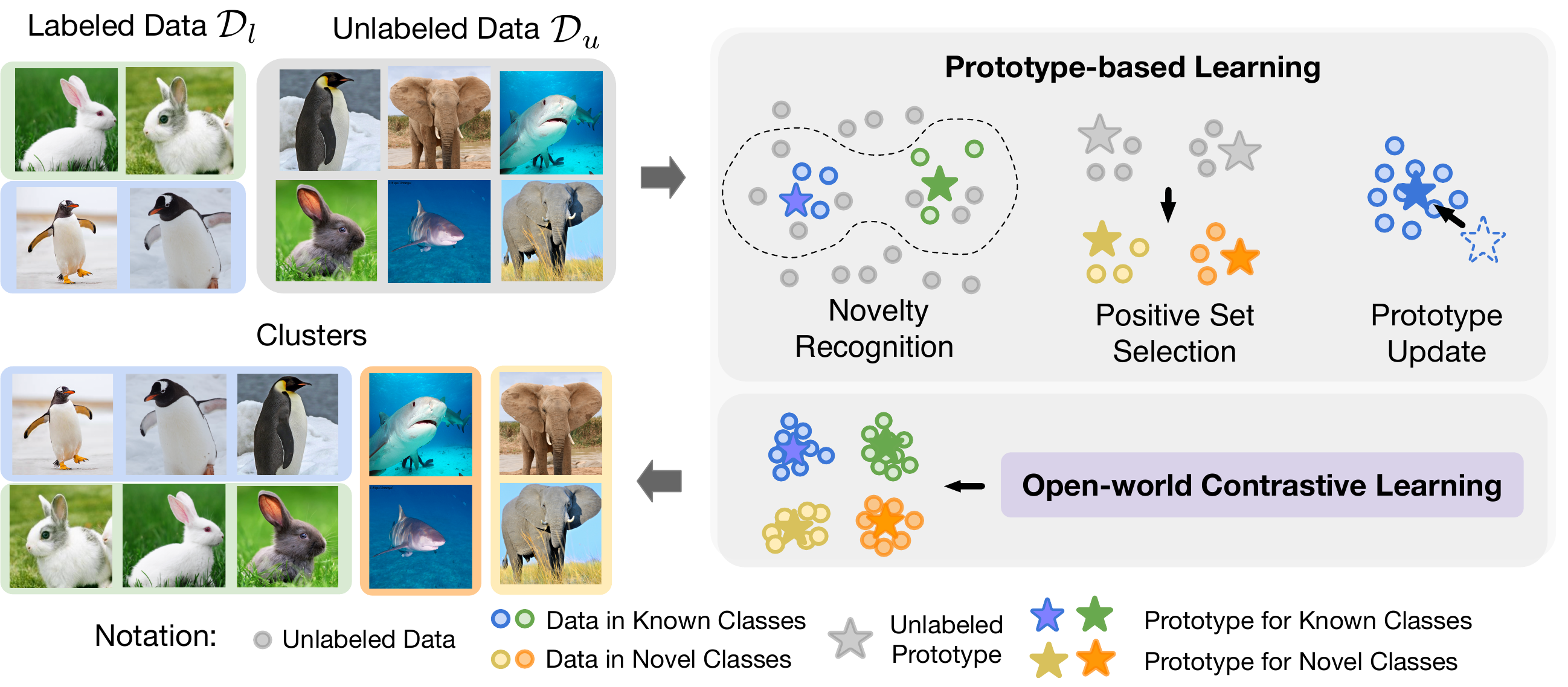}
    \caption[Illustration of our learning framework.]{Illustration of our learning framework \emph{Open-world Contrastive Learning} (OpenCon).  The model is trained on a labeled dataset $\mathcal{D}_l$ of known classes, and an unlabeled dataset $\mathcal{D}_u$ (with samples from both known and novel classes). OpenCon aims to learn distinguishable representations for both known (blue and green) and novel (yellow and orange) classes simultaneously. See Section~\ref{sec:opencon_method} for details.}
    \label{fig:opencon_teaser}
\end{figure}

Targeting these  challenges, we formally introduce a new learning framework, \emph{open-world contrastive learning} (dubbed \textbf{OpenCon}). OpenCon is designed to produce a compact representation space for both known and novel classes, and facilitates novelty discovery along the way. Key to our framework, we propose a novel prototype-based learning strategy, which encapsulates two components. First, we leverage the prototype vectors to separate known vs. novel classes in unlabeled data $\mathcal{D}_u$. The prototypes can be viewed as a set of representative embeddings, one for each class, and are updated 
by the evolving representations. 
Second, to mitigate the challenge of lack of supervision, we generate pseudo-positive pairs for contrastive comparison. We define the positive set to be those examples carrying the same {approximated} label, which is predicted based on the closest class prototype. In effect, the loss encourages closely aligned representations to {all} samples from the same {predicted} class, rendering a compact clustering of the representation. 

Our framework offers several compelling advantages. 
\textbf{(1)} Empirically, OpenCon establishes strong performance on challenging benchmark datasets, outperforming existing baselines by a significant margin (Section~\ref{sec:opencon_experiments}). OpenCon is also competitive without knowing the number of novel classes in advance---achieving similar or even slightly better performance compared to the oracle (in which the number of classes is given). 
\textbf{(2)} Theoretically, we demonstrate that our prototype-based learning can be rigorously interpreted from an Expectation-Maximization (EM) algorithm perspective. 
\textbf{(3)} Our framework is end-to-end trainable, and is compatible with both CNN-based and Transformer-based architectures. The \textbf{main contributions} are:
\begin{enumerate}

    \item We propose a novel framework, open-world contrastive learning (OpenCon), tackling a largely unexplored problem in representation learning. As an integral part
of our framework, we also introduce a prototype-based learning algorithm, which
facilitates novelty discovery and learning distinguishable  representations. 

\item Empirically, OpenCon establishes competitive performance on challenging tasks. For example, on the ImageNet dataset, OpenCon substantially outperforms the current best method ORCA~\citep{cao2022openworld} by \textbf{11.9}\% and \textbf{7.4}\% in terms of novel and overall accuracy.

\item  We provide insights through extensive ablations, showing the effectiveness of components in our framework. Theoretically, we show a formal connection with the EM algorithm---minimizing our contrastive loss partially maximizes the likelihood by clustering similar samples
in the embedding space.
\end{enumerate}

%%%%%%%%%%%%%%%%%%%%%%%%%%%%%%%%%%%%%%%%%%%%%%%%%%%%%%%%%%%%%%%%%%%%%%%%%%
%%%%%%%%%%%%%%%%%%%%%%%%  METHODOLOGY SECTION %%%%%%%%%%%%%%%%%%%%%%%%%%%%
%%%%%%%%%%%%%%%%%%%%%%%%%%%%%%%%%%%%%%%%%%%%%%%%%%%%%%%%%%%%%%%%%%%%%%%%%%

\section{Methodology~}
\label{sec:opencon_method}

We formally introduce a new learning framework, \emph{open-world contrastive learning} (dubbed OpenCon), which is designed to produce compact representation space for both known and novel classes. The open-world setting posits unique challenges for learning effective representations, namely due to (1) the lack of the separation between known vs. novel data in $\mathcal{D}_{u}$, (2) the lack of supervision for data in novel classes. Our learning framework targets these challenges. 

\subsection{Background: Generalized Contrastive Loss}

We start by defining a generalized contrastive loss that can characterize the family of contrastive losses. We will later instantiate the formula to define our open-world contrastive loss (Section~\ref{sec:opencon_prototype} and Section~\ref{sec:opencon_losses}). Specifically, we consider a deep neural network encoder $\phi:\mathcal{X}\mapsto \mathbb{R}^d$ that maps the input $\bx$ to a $L_2$-normalized feature embedding $\phi(\bx)$.  Contrastive losses operate on the normalized feature $\bz = \phi(\bx)$. In other words, the features have unit norm and lie on the unit hypersphere. For a given anchor point $\bx$, we define the per-sample contrastive loss:
\begin{equation}
    \label{eq:opencon_general_cl_loss}
    \mathcal{L}_\phi\big(\bx;\tau,\mathcal{P}(\bx),\mathcal{N}(\bx)\big) = -\frac{1}{|\mathcal{P}(\bx)|}\sum_{\bz^{+} \in \mathcal{P}(\bx)} \log \frac{\exp(\bz^{\top} \cdot \bz^+ / \tau)}{\sum_{\bz^- \in \mathcal{N}(\bx)} \exp (\bz^{\top} \cdot \bz^- / \tau)},
\end{equation}
where $\tau$ is the temperature parameter, $\bz$ is the $L_2$-normalized embedding vector of $\bx$, $\mathcal{P}(\bx)$ is the positive set of embeddings \emph{w.r.t.}  $\bz$, and $\mathcal{N}(\bx)$ is the negative set of embeddings. 

In open-world contrastive learning, the crucial challenge is how to construct $\mathcal{P}(\bx)$ and $\mathcal{N}(\bx)$ \emph{for different types of samples}. Recall that we have two broad categories of training data: (1) labeled data $\mathcal{D}_l$ with known class, and (2) unlabeled data $\mathcal{D}_u$ with both known and novel classes. In conventional supervised CL frameworks with $\mathcal{D}_l$ only, the positive sample pairs can be easily drawn according to the ground-truth labels~\citep{khosla2020supcon}. That is, $\mathcal{P}(\bx)$ consists of embeddings of samples that carry the same label as the anchor point $\bx$, and $\mathcal{N}(\bx)$ contains all the embeddings in the multi-viewed mini-batch excluding itself. However, this is not
straightforward in the open-world setting with novel classes. 

\subsection{Learning from Wild Unlabeled Data} 
\label{sec:opencon_prototype}
We now dive into the most challenging part of the data, $\mathcal{D}_u$, which contains both known and novel classes. We propose a novel prototype-based learning strategy that tackles the challenges of: (1) the separation between known and novel classes in $\mathcal{D}_u$, and (2) pseudo label assignment that can be used for positive set construction for novel classes. Both components facilitate the goal of learning compact representations, and enable  end-to-end training. 

Key to our framework, we keep a prototype embedding vector $\boldsymbol{\mu}_c$ for each class $c \in \mathcal{Y}_\text{all}$. Here $\mathcal{Y}_\text{all}$ contains both known classes $\mathcal{Y}_l$ and novel classes $\mathcal{Y}_n = \mathcal{Y}_\text{all}\backslash \mathcal{Y}_l$, and $\mathcal{Y}_l \cap \mathcal{Y}_n = \emptyset$. The prototypes can be viewed as a set of representative embedding vectors.  All the prototype vectors $\bM = [\boldsymbol{\mu}_1|...|\boldsymbol{\mu}_c|...]_{c\in \mathcal{Y}_\text{all}}$ are randomly initiated at the beginning of training, and will be updated along with learned embeddings. We will also discuss determining the cardinality $|\mathcal{Y}_\text{all}|$ (\emph{i.e.}, number of prototypes) in Section~\ref{sec:opencon_ablation}.

\vspace{0.1cm} \noindent \textbf{Prototype-based OOD detection.} We leverage the prototype vectors to perform out-of-distribution (OOD) detection, \emph{i.e.}, separate known vs. novel data in $\mathcal{D}_u$. For any given sample $\bx_i \in \mathcal{D}_u$, we measure the cosine similarity between its embedding $\phi(\bx_i)$ and prototype vectors of known classes $\mathcal{Y}_l$. If the sample embedding is far away from all the known class prototypes, it is more likely to be a novel sample, and vice versa. Formally, we propose the level set estimation:
\begin{equation}
\mathcal{D}_{n} = \{ \bx_{i} | \underset{j \in \mathcal{Y}_l}{\max}~~~\boldsymbol{\mu}_{j}^\top \cdot  \phi(\bx_i)  < \lambda\},
\end{equation}
where a thresholding mechanism is exercised to distinguish between known and novel samples during  training time. The threshold $\lambda$
can be chosen based on the labeled data $\mathcal{D}_l$. Specifically, one can calculate the scores $\max_{j \in \mathcal{Y}_l}\boldsymbol{\mu}_{j}^\top \cdot \phi(\bx_i)$ for all the samples in $\mathcal{D}_l$, and use the score at the $p$-percentile as the threshold. For example, when $p=90$, that means 90\% of labeled data is above the threshold. We provide ablation on the effect of $p$ later in Section~\ref{sec:opencon_ablation} and theoretical insights into why OOD detection helps open-world representation learning in Appendix~\ref{sec:opencon_whyood}.

\vspace{0.1cm} \noindent \textbf{Positive and negative set selection.} Now that we have identified novel samples from the unlabeled sample, we would like to facilitate learning compact representations for $\mathcal{D}_n$, where samples belonging to the same class are close to each other. As mentioned earlier, the crucial challenge is how to construct the positive
set, denoted as $\mathcal{P}_n(\bx)$. In particular, we do not have any supervision signal for unlabeled data in the novel classes. We propose utilizing the predicted label $\hat{y} =  \text{argmax}_{j \in \mathcal{Y}_\text{all} }\boldsymbol{\mu}_j^\top \cdot \phi(\bx)$ for positive set selection. 

For a mini-batch $\mathcal{B}_n$ with samples drawn from $\mathcal{D}_n$, we apply two random augmentations for each sample and generate a multi-viewed batch $\tilde{\mathcal{B}}_n$. We denote the embeddings of the multi-viewed batch as $\mathcal{A}_n$, where the cardinality $|\mathcal{A}_n| = 2 |\mathcal{B}_n|$. For any sample $\bx$ in the mini-batch $\tilde{\mathcal{B}}_n$, we propose selecting the positive and negative set of embeddings as follows:
\begin{align}
    \mathcal{P}_n(\bx)  & = \{ \bz' | \bz' \in \{\mathcal{A}_n \backslash \bz\}, \hat y' = \hat y\},\\
     \mathcal{N}_n(\bx) & = \mathcal{A}_n \backslash \bz,
\end{align}
where $\bz$ is the $L_2$-normalized embedding of $\bx$, and $\hat y'$
is the predicted label for the corresponding training example of $\bz'$. In other words, we define the
positive set of $\bx$ to be those examples carrying the same \emph{approximated} label prediction $\hat y$.

With the positive and negative sets defined, we are now ready to introduce our new contrastive loss for open-world data. We desire embeddings where samples assigned with the same pseudo-label can form a compact cluster. Following the general template in Equation~\ref{eq:opencon_general_cl_loss}, we define a novel loss function:
\begin{equation}
\label{eq:opencon_ln}
    \mathcal{L}_n = \sum_{\bx \in \tilde{\mathcal{B}}_n} \mathcal{L}_\phi \big(\bx; \tau_n, \mathcal{P}_n(\bx), \mathcal{N}_n(\bx)\big). 
\end{equation}
For each anchor, the loss encourages the network to align embeddings of its positive pairs while repelling the negatives. \emph{All} positives in a multi-viewed batch (\emph{i.e.}, the
augmentation-based sample as well as any of the remaining samples with the same label) contribute to the numerator. The loss encourages the encoder to give closely aligned representations to \emph{all} entries from the same \emph{predicted} class, resulting in a compact representation
space. We provide visualization in  Figure~\ref{fig:opencon_umap} (right).

\vspace{0.1cm} \noindent \textbf{Prototype update.} 
 The most canonical way to update the prototype embeddings is to compute
it in every iteration of training. However, this would extract a heavy computational toll and in turn cause unbearable training latency. Instead, we update the class-conditional prototype vector in a moving-average style~\citep{li2020mopro, wang2022pico}:

\begin{equation}
\label{eq:opencon_mov_avg}
\boldsymbol{\mu}_{c}:=\operatorname{Normalize}\left(\gamma \boldsymbol{\mu}_{c}+(1-\gamma) \bz \right), \text{for}~c = 
\begin{cases} 
y \text{ (ground truth label)}, & \text{if } \bz \in \mathcal{D}_{l}  \\
\text{argmax}_{j\in \mathcal{Y}_\text{n}} \boldsymbol{\mu}_j^\top \cdot \bz, & \text{if } \bz \in \mathcal{D}_{n}
\end{cases}
\end{equation}

Here, the prototype $\boldsymbol{\mu}_{c}$ of class $c$ is defined by the moving average of the normalized embeddings $\bz$, whose predicted class conforms to $c$. $\bz$ are embeddings of samples from $\mathcal{D}_l \cup \mathcal{D}_n$. $\gamma$  is a tunable hyperparameter. 

 \textbf{Remark:} We exclude samples in $\mathcal{D}_{u}\backslash \mathcal{D}_{n}$ because they may contain non-distinguishable data from known and unknown classes, which undesirably introduce noise to the prototype estimation. We verify this phenomenon by comparing the performance of mixing $\mathcal{D}_{u}\backslash \mathcal{D}_{n}$ with labeled data $\mathcal{D}_{l}$ for training the known classes. The results verify our hypothesis that the non-distinguishable 
 data would be harmful to the overall accuracy. We provide more discussion on this  in Appendix~\ref{sec:opencon_du_slash_dn}.

\subsection{Open-world Contrastive Loss} 
\label{sec:opencon_losses}
Putting it all together, we define the open-world contrastive loss (dubbed \textbf{OpenCon}) as the following:
\begin{equation}
\label{eq:opencon_overall_loss}
    \mathcal{L}_\text{OpenCon} = \lambda_n \mathcal{L}_n + \lambda_l\mathcal{L}_l + \lambda_u \mathcal{L}_u,
\end{equation}
where $\mathcal{L}_n$ is the newly devised contrastive loss for the novel data, $\mathcal{L}_l$ is the supervised contrastive loss~\citep{khosla2020supcon} employed on the labeled data $\mathcal{D}_l$, and $\mathcal{L}_u$ is the self-supervised contrastive loss~\citep{chen2020simclr} employed on the unlabeled data $\mathcal{D}_u$.  $\lambda$ are the coefficients of loss terms. Details of $\mathcal{L}_l$ and $\mathcal{L}_u$ are in Appendix~\ref{sec:opencon_lull}, along with the complete pseudo-code in Algorithm~\ref{alg:main} (Appendix).

\vspace{0.1cm} \noindent \textbf{Remark.} Our loss components work collaboratively to enhance the embedding quality in an open-world setting. The overall objective well suits the complex nature of our training data, which blends both labeled and unlabeled data. As we will show later in Section~\ref{sec:opencon_ablation}, a simple solution by combining supervised contrastive loss (on labeled data) and self-supervised loss (on unlabeled data) is suboptimal. Instead, having $\mathcal{L}_n$ is critical to encourage closely aligned representations to \emph{all} entries from the same {predicted} class, resulting in an overall more compact representation for novel classes.

%%%%%%%%%%%%%%%%%%%%%%%%%%%%%%%%%%%%%%%%%%%%%%%%%%%%%%%%%%%%%%%%%%%%%%%%%%
%%%%%%%%%%%%%%%%%%%%%%%%  THEORY SECTION %%%%%%%%%%%%%%%%%%%%%%%%%%%%
%%%%%%%%%%%%%%%%%%%%%%%%%%%%%%%%%%%%%%%%%%%%%%%%%%%%%%%%%%%%%%%%%%%%%%%%%%

\section{Theoretical Understandings~}
\label{sec:opencon_em}

\vspace{0.1cm} \noindent \textbf{Overview.} 
Our learning objective using wild data (\emph{c.f.} Section~\ref{sec:opencon_prototype}) can be rigorously interpreted from an Expectation-Maximization (EM) algorithm perspective. We start by introducing the high-level ideas of how our method can be decomposed into E-step and M-step respectively. At the \textbf{E-step}, we assign each data
example $\bx \in \mathcal{D}_n$ to one specific cluster. In OpenCon, it is estimated by using the prototypes: $\hat{y}_i =  \text{argmax}_{j \in \mathcal{Y}_\text{all} }\boldsymbol{\mu}_j^\top \cdot \phi(\bx_i)$.
 At the \textbf{M-step}, the EM algorithm aims to maximize the likelihood under the posterior class probability from the previous E-step. Theoretically, we show that minimizing our contrastive loss $\mathcal{L}_n$ (Equation~\ref{eq:opencon_ln}) partially maximizes the likelihood by clustering similar examples. In effect, our  loss concentrates similar data to the corresponding
  prototypes, encouraging the compactness of features.

\subsection{Analyzing the E-step}
In \textbf{E-step}, the goal of the EM algorithm is to maximize the likelihood with learnable feature encoder $\phi$ and prototype matrix $\bM = [\boldsymbol{\mu}_1|...|\boldsymbol{\mu}_c|...]$, which can be lower bounded:
\begin{align*}
    \sum_i^{|\mathcal{D}_n|} \operatorname{log} p(\bx_i| \phi, \bM) \geq
    \sum_i^{|\mathcal{D}_n|} q_i(c) \operatorname{log} \sum_{c \in \mathcal{Y}_\text{all}}  \frac{p(\bx_i, c| \phi, \bM)}{q_i(c)}, 
\end{align*}

where $q_i(c)$ is denoted as the density function of a possible distribution over $c$ for sample $\bx_i$. 
By using the fact that $\log(\cdot)$ function is concave, the inequality holds with equality when $\frac{p(\bx_i, c| \phi, \bM)}{q_i(c)}$ is a constant value, therefore we set: 
$$
q_i(c) = \frac{p(\bx_i, c| \phi, \bM)}{\sum_{c \in \mathcal{Y}_\text{all}} p(\bx_i, c| \phi, \bM)} = \frac{p(\bx_i, c| \phi, \bM)}{p(\bx_i| \phi, \bM)} = p(c| \bx_i, \phi, \bM),
$$  
which is the posterior class probability. To estimate $p(c| \bx_i, \phi, \bM)$, we model the data using  the von Mises-Fisher (vMF)~\citep{fisher1953dispersion} distribution since the {normalized} embedding locates in a high-dimensional hyperspherical space.

\begin{assumption}
\label{ass:vmf}
  The density function is given by $f\left(\bx | \boldsymbol{\mu}, \kappa\right) = c_{d}(\kappa) e^{\kappa \boldsymbol{\mu}^{\top} \phi(\bx)}$, where $\kappa$ is the concentration parameter and $c_{d}(\kappa)$ is a coefficient. 
\end{assumption} 

With the vMF distribution assumption in ~\ref{ass:vmf}, we have
$p(c| \bx_i, \phi, \bM) = \sigma_{c}(\bM^\top \cdot \phi(\bx_i)), $ where $\sigma$ denotes the softmax function and $\sigma_c$ is the $c$-th element. Empirically we take a one-hot prediction with $\hat{y}_i =  \text{argmax}_{j \in \mathcal{Y}_\text{all} }\boldsymbol{\mu}_j^\top \cdot \phi(\bx_i)$ since each example inherently belongs to exactly one prototype, so we let $q_i(c) = \mathbf{1}\{c = \hat{y}_i\}$.

\subsection{Analyzing the M-step}
In \textbf{M-step}, using the label distribution prediction $q_i(c)$ in the E-step, the optimization for the network $\phi$ and the prototype matrix $\bM$ is given by:
\begin{equation}
    \label{eq:opencon_m-target}
    \underset{\phi, \bM}{\operatorname{argmax\ }} \sum_{i=1}^{|\mathcal{D}_n|} \sum_{c \in \mathcal{Y}_\text{all}} q_i(c) \log \frac{p\left(\bx_{i}, c | \phi, \bM\right)}{q_i(c)} 
\end{equation}

The joint optimization target in Equation~\ref{eq:opencon_m-target} is then achieved by rewriting the Equation~\ref{eq:opencon_m-target} according to the following Lemma~\ref{lemma:mstep}  with proof in Appendix~\ref{sec:opencon_proof}:

\begin{lemma}
\label{lemma:mstep}{~\citep{zha2001spectral}}
We define the set of samples with the same prediction $\mathcal{S}(c) = \{\bx_i \in \mathcal{D}_n| \hat{y}_i = c\}$. The maximization step is equivalent to aligning the feature vector $\phi(\bx)$ to the corresponding prototype $\boldsymbol{\mu}_{c}$:

\begin{align*}
\underset{\phi, \bM}{\operatorname{argmax\ }} \sum_{i=1}^{|\mathcal{D}_n|} \sum_{c \in \mathcal{Y}_\text{all}} q_i(c) \log \frac{p\left(\bx_{i}, c | \phi, \bM\right)}{q_i(c)}
&= \underset{\phi, \bM}{\operatorname{argmax\ }} \sum_{c \in \mathcal{Y}_\text{all}} \sum_{\bx \in \mathcal{S}(c)}  \phi(\bx)^{\top} \cdot \boldsymbol{\mu}_{c}
\end{align*}
\end{lemma}

In our algorithm, the maximization step is achieved by optimizing $\bM$ and $\phi$ separately.

(a) \textbf{Optimizing $\bM$}: 

For fixed $\phi$, the optimal prototype is given by $\boldsymbol{\mu}^*_{c} = \operatorname{Normalize}(\mathbb{E}_{\bx \in \mathcal{S}(c)}[\phi(\bx)]).$ \textbf{This optimal form empirically corresponds to our prototype estimation in Equation~\ref{eq:opencon_mov_avg}.} Empirically, it is expensive to collect all features in $\mathcal{S}(c)$. We use the estimation of $\boldsymbol{\mu}_{c}$ by moving average: $$\boldsymbol{\mu}_{c}:=\operatorname{Normalize}\left(\gamma \boldsymbol{\mu}_{c}+(1-\gamma) \phi(\bx) \right), \forall \bx \in \mathcal{S}(c).$$ 

(b) \textbf{Optimizing $\phi$}: 

We then show that the contrastive loss $\mathcal{L}_n$ composed with the alignment loss part $\mathcal{L}_a$ encourages the closeness of features from positive pairs. {By minimizing $\mathcal{L}_a$, it is approximately maximizing the target in Equation~\ref{eq:opencon_m-target} with the optimal prototypes $\boldsymbol{\mu}_{c}^{*}$}. We can decompose the loss as follows:

\begin{align*}
    \mathcal{L}_n
    &= -\frac{1}{|\mathcal{P}(\bx)|}\sum_{\bz^{+} \in \mathcal{P}(\bx)} \log \frac{\exp(\bz^{\top} \cdot \bz^+ / \tau)}{\sum_{\bz^- \in \mathcal{N}(\bx)} \exp (\bz \cdot \bz^- / \tau)} 
    \\ &= \underbrace{-\frac{1}{|\mathcal{P}(\bx)|}\sum_{\bz^{+} \in \mathcal{P}(\bx)} (\bz^{\top} \cdot \bz^+ / \tau)}_{\mathcal{L}_{a}(\bx)} +  \underbrace{ \frac{1}{|\mathcal{P}(\bx)|} \sum_{\bz^{+}\in \mathcal{P}(\bx)} \log \sum_{\bz^- \in \mathcal{N}(\bx)} \exp (\bz^{\top} \cdot \bz^- / \tau)}_{\mathcal{L}_{b}(\bx)}.
\end{align*}
In particular, the first term $\mathcal{L}_{a}(\bx)$ is referred to as the alignment term ~\citep{wang2020understanding}, which encourages the compactness of features
from positive pairs. To see this, we have the following lemma~\ref{lemma:opt_phi} with proof in Appendix~\ref{sec:opencon_proof}. 

\begin{lemma}
\label{lemma:opt_phi}
Minimizing  $\mathcal{L}_{a}(\bx)$ is equivalent to the maximization step w.r.t. parameter $\phi$.
\begin{align*}
    \underset{\phi}{\operatorname{argmin\ }} \sum_{\bx \in \mathcal{D}_n}  \mathcal{L}_{a}(\bx)  &=  \underset{\phi}{\operatorname{argmax\ }} \sum_{c \in \mathcal{Y}_\text{all}} \sum_{\bx \in \mathcal{S}(c)}  \phi(\bx)^{\top} \cdot \boldsymbol{\mu}^*_{c} , 
\end{align*}
\end{lemma}

\vspace{0.1cm} \noindent \textbf{Summary.} These observations validate that our  framework learns representation for novel classes in an EM fashion. Importantly, we extend EM from a traditional learning setting to an open-world setting with the capability to handle real-world data arising in the wild. We proceed by introducing the empirical verification of our algorithm.

%%%%%%%%%%%%%%%%%%%%%%%%%%%%%%%%%%%%%%%%%%%%%%%%%%%%%%%%%%%%%%%%%%%%%%%%%%
%%%%%%%%%%%%%%%%%%%%%%%%  EXPERIMENT SECTION %%%%%%%%%%%%%%%%%%%%%%%%%%%%
%%%%%%%%%%%%%%%%%%%%%%%%%%%%%%%%%%%%%%%%%%%%%%%%%%%%%%%%%%%%%%%%%%%%%%%%%%

\section{Experimental Results~}
\label{sec:opencon_experiments}

\vspace{0.1cm} \noindent \textbf{Datasets.} We evaluate on standard benchmark image
classification datasets CIFAR-100~\citep{krizhevsky2009learning} and ImageNet~\citep{deng2009imagenet}. For the ImageNet, we sub-sample 100 classes, following the same setting as ORCA~\citep{cao2022openworld} for fair comparison. 
Note that we focus on these tasks, as they are much more challenging than toy datasets with fewer classes. The additional comparison on CIFAR-10 is in Appendix~\ref{sec:opencon_cifar-10}. By default, classes are divided into 50\% seen and 50\% novel classes. We then select 50\%
of known classes as the labeled dataset, and the rest as the unlabeled set.
The division is consistent with~\citet{cao2022openworld}, which allows us to compare the performance in a fair setting. Additionally, we explore different ratios of unlabeled data and novel classes (see Section~\ref{sec:opencon_ablation}).

\vspace{0.1cm} \noindent \textbf{Evaluation metrics.} We follow the evaluation strategy in~\citet{cao2022openworld} and report the following metrics: (1) classification accuracy on known classes, (2) classification accuracy  on the novel data, and (3) overall accuracy on all classes. The accuracy
of the novel classes is measured by solving an optimal assignment problem using the Hungarian algorithm~\citep{Kuhn1955thehungarian}. When reporting accuracy on all classes, we solve optimal assignments using both known and
novel classes.

\vspace{0.1cm} \noindent \textbf{Experimental details.}
We use ResNet-18 as the backbone
for CIFAR-100 and ResNet-50 as the backbone for ImageNet-100. 
The pre-trained backbones (no final FC layer) are identical to the ones in ~\citet{cao2022openworld}. To ensure a fair comparison, we follow the same practice in~\citet{cao2022openworld} and only update the parameters
of the last block of ResNet. 
In addition, we add a trainable two-layer MLP projection head that projects the feature from the penultimate layer to a lower-dimensional space $\mathbb{R}^d$ ($d=128$), which is shown to be effective for contrastive loss~\citep{chen2020simclr}. We use the same data augmentation strategies as SimCLR~\citep{chen2020simclr}. Same as in ~\citet{cao2022openworld}, we regularize the KL-divergence between the predicted label distribution $p\left(\hat{y}\right)$ and the class prior to prevent the network degenerating into a trivial solution in which all instances
are assigned to a few classes. We provide extensive details on the training configurations and all hyper-parameters in Appendix~\ref{sec:opencon_hyperparam}. 

\begin{table}[t]
\centering
\caption[Main Results of OpenCon and baseline methods' performance in open-world representation learning.]{Main Results. Asterisk ($^\star$) denotes that the original method can not recognize seen classes. Dagger ($^\dagger$) denotes the original method can not
detect novel classes (and we had to extend it). Results on ORCA, GCD and OpenCon  (mean and standard deviation) are averaged over five different runs. The ORCA results are reported based on the official repo~\citep{cao2022git}. } 
\scalebox{0.75}{
\begin{tabular}{lllllll}
\toprule
\multirow{2}{*}{\textbf{Method}} & \multicolumn{3}{c}{\textbf{CIFAR-100}} & \multicolumn{3}{c}{\textbf{ImagNet-100}} \\
 & \textbf{All} & \textbf{Novel} & \textbf{Seen} & \textbf{All} & \textbf{Novel} & \textbf{Seen} \\ \midrule
$^{\dagger}$\textbf{FixMatch}~\citep{alex2020fixmatch} & 20.3 & 23.5 & 39.6 & 34.9 & 36.7 & 65.8 \\
$^{\dagger}$\textbf{DS$^{3}$L}~\citep{guo2020dsl} & 24.0 & 23.7 & 55.1 & 30.8 & 32.5 & 71.2 \\
$^{\dagger}$\textbf{CGDL}~\citep{sun2020cgdl} & 23.6 & 22.5 & 49.3 & 31.9 & 33.8 & 67.3 \\
$^\star$\textbf{DTC}~\citep{Han2019dtc} & 18.3 & 22.9 & 31.3 & 21.3 & 20.8 & 25.6 \\
$^\star$\textbf{RankStats}~\citep{zhao2021rankstat} & 23.1 & 28.4 & 36.4 & 40.3 & 28.7 & 47.3 \\
$^\star$\textbf{SimCLR}~\citep{chen2020simclr} & 22.3 & 21.2 & 28.6 & 36.9 & 35.7 & 39.5 \\ \midrule
\textbf{ORCA}~\citep{cao2022openworld} & 47.2$ ^{\pm{0.7}} $ & 41.0$ ^{\pm{1.0}} $ & 66.7$ ^{\pm{0.2}} $ & 76.4$ ^{\pm{1.3}} $ & 68.9$ ^{\pm{0.8}} $ & 89.1$ ^{\pm{0.1}} $ \\
\textbf{GCD}~\citep{vaze22gcd} & 46.8$ ^{\pm{0.5}} $ & 43.4$ ^{\pm{0.7}} $ & \textbf{69.7}$ ^{\pm{0.4}} $ & 75.5$ ^{\pm{1.4}} $ & 72.8$  ^{\pm{1.2}} $  & \textbf{90.9} $ ^{\pm{0.2}} $  \\
\textbf{OpenCon (Ours)} & \textbf{52.7}$ ^{\pm{0.6}} $ & \textbf{47.8}$ ^{\pm{0.6}} $ & {69.1}$ ^{\pm{0.3}} $ & \textbf{83.8}$ ^{\pm{0.3}} $ & \textbf{80.8}$ ^{\pm{0.3}} $ & 90.6$ ^{\pm{0.1}} $ 
\\ \bottomrule
\end{tabular}}
\label{tab:opencon_main}
\end{table}

\vspace{0.1cm} \noindent \textbf{OpenCon achieves SOTA performance.} As shown in Table 1, OpenCon outperforms the rivals
by a significant margin on both CIFAR and ImageNet datasets. Our comparison covers an extensive collection of algorithms, including the best-performed methods to date. In particular, on ImageNet-100, we improve upon the best baseline by \textbf{7.4}\% in terms of overall accuracy. It is also worth noting that OpenCon improves the accuracy of novel classes by \textbf{11.9}\%. Note that the open-world representation learning is a relatively new setting.  Closest to our setting is the open-world semi-supervised learning (SSL) algorithms, namely  {ORCA}~\citep{cao2022openworld} and {GCD}~\citep{vaze22gcd}---that directly optimize the classification performance. While our framework emphasizes representation learning, we demonstrate the quality of learned embeddings by also measuring the classification accuracy. {This can be easily done by leveraging our learned prototypes on a converged model: $\hat{y} = \text{argmax}_{j \in \mathcal{Y}_\text{all}}~\boldsymbol{\mu}_j^\top \cdot \phi(\bx)$.}  We discuss the significance \emph{w.r.t.} existing works in detail:
\begin{itemize}

    \item \textbf{OpenCon vs. ORCA}  Our framework bears significant differences \emph{w.r.t.} ORCA in terms of learning goal and approach. (1) Our framework focuses on the representation learning problem, whereas ORCA optimizes for the classification performance using cross-entropy loss. Unlike ours, ORCA does not necessarily learn compact representations, as evidenced in Figure~\ref{fig:opencon_umap} (left). (2) We propose a novel open-world contrastive learning framework, whereas ORCA does not employ contrastive learning. ORCA uses a pairwise loss to predict similarities between pairs of instances, and does not consider negative samples. In contrast, our approach constructs both positive and negative sample sets, which encourage aligning representations to \emph{all} entries from the same ground-truth label or predicted pseudo label (for novel classes). (3) Our framework explicitly considers OOD detection, which allows separating known vs. novel data in $\mathcal{D}_u$. ORCA does not consider this and can suffer from noise in the pairwise loss (\emph{e.g.}, the loss may maximize the similarity between samples from known vs. novel classes).  
    
    \item \textbf{OpenCon vs. GCD} There are two key differences to highlight: (1) GCD~\citep{vaze22gcd} requires a  two-stage training procedure, whereas our learning framework proposes an end-to-end training strategy. Specifically, GCD applies the SupCon loss~\citep{khosla2020supcon} on the labeled data $\mathcal{D}_l$ and SimCLR loss~\citep{chen2020simclr} on the unlabeled data $\mathcal{D}_u$. The feature is then clustered separately by a semi-supervised K-means method. However, the two-stage method hinders the useful pseudo-labels to be incorporated into the training stage, which results in suboptimal performance. In contrast, our prototype-based learning strategy alleviates the need for a separate clustering process (\emph{c.f.} Section~\ref{sec:opencon_prototype}), which is therefore easy to use in practice and provides meaningful supervision for the unlabeled data.
    (2) We propose a contrastive loss $\mathcal{L}_n$ better utilizing the pseudo-labels during training, which facilitates learning a more compact representation space for the novel data . 
    From Table~\ref{tab:opencon_main}, we observe that OpenCon outperforms GCD by \textbf{8.3}\% (overall accuracy) on ImageNet-100, showcasing the benefits of our  framework. 
\end{itemize}

\begin{figure}
    \centering
    \includegraphics[width=0.95\linewidth]{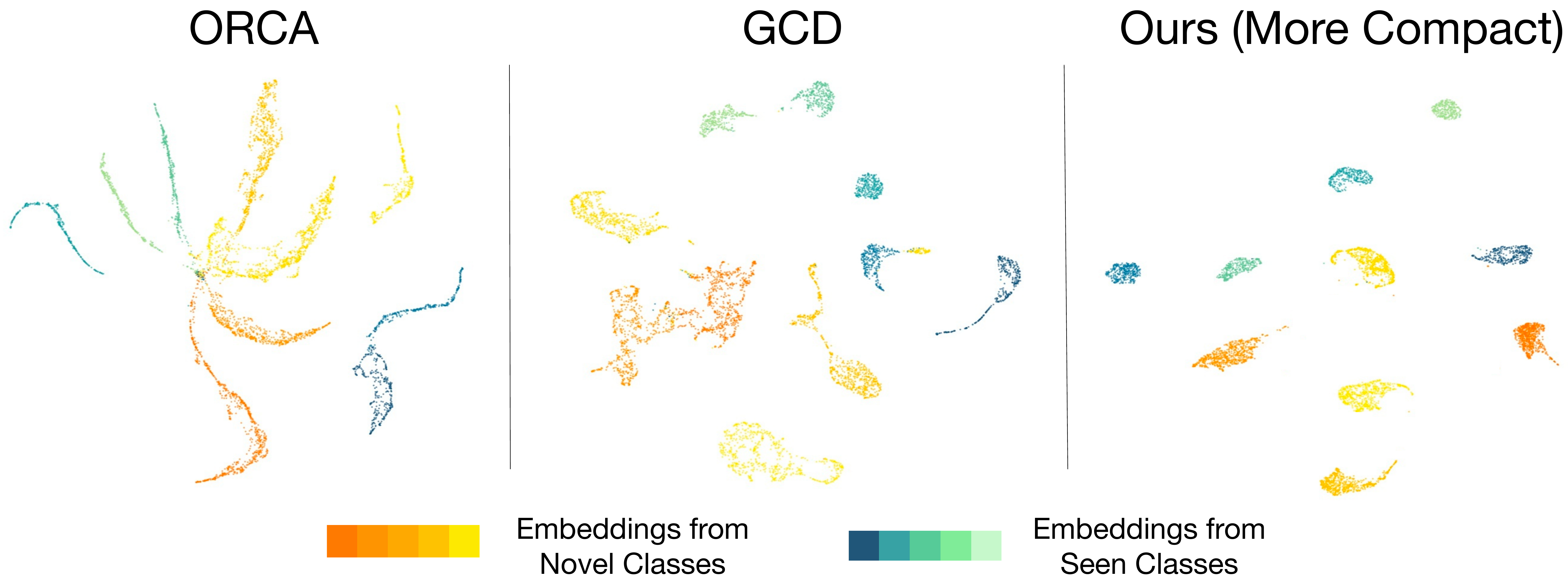}
    \caption[UMAP visualization of the feature embedding.]{UMAP~\citep{umap} visualization of the feature embedding from 10 classes (5 for seen, 5 for novel) when the model is trained on ImageNet-100 with ORCA~\citep{cao2022openworld}, GCD~\citep{vaze22gcd} and OpenCon (ours).}
    \label{fig:opencon_umap}
\end{figure}

Lastly, for completeness, we compare methods in related problem domains: (1) \textit{novel class detection}: \textsc{dtc}~\citep{Han2019dtc},  \textsc{RankStats}~\citep{zhao2021rankstat},
(2) \textit{semi-supervised learning}: \textsc{FixMatch}~\citep{alex2020fixmatch}, \textsc{ds$^{3}$l}~\citep{guo2020dsl} and \textsc{cgdl}~\citep{sun2020cgdl}. We also compare it with the common representation method \textsc{SimCLR}~\citep{chen2020simclr}. 
These methods are not designed for the Open-SSL task, therefore the performance is less competitive.

\begin{table}[htb]
\centering
\caption{Comparison of accuracy on ViT-B/16 architecture. Results are reported on ImageNet-100. }
\scalebox{0.9}{
    \begin{tabular}{ccccc} \toprule
         Methods & All & Novel & Seen \\ \midrule
         ORCA~\citep{cao2022openworld} & 73.5 & 64.6  & 89.3  \\
         GCD~\citep{vaze22gcd} &  74.1 & 66.3 &  89.8 \\ 
         {$k$-Means~\citep{macqueen1967classification}} &  {72.7} & {71.3} &  {75.5}  \\ 
         {RankStats+~\citep{zhao2021rankstat}} & {37.1} &   {24.8} & {61.6}   \\
         {UNO+~\citep{fini2021unified}} & {70.3} & {57.9} & \textbf{{95.0}}\\
         OpenCon (Ours) & \textbf{84.0} & \textbf{81.2} & 93.8 \\
            \bottomrule
    \end{tabular}}
    \label{tab:opencon_vit}
\end{table}

\noindent \textbf{OpenCon is competitive on ViT} 
Going beyond convolutional neural networks, we show in Table~\ref{tab:opencon_vit} that the OpenCon is competitive for transformer-based ViT model~\citep{dosovitskiy2020vit}. We adopt the ViT-B/16 architecture with DINO pre-trained weights ~\citep{caron2021emerging}, following the pipeline used in ~\citet{vaze22gcd}.  In Table~\ref{tab:opencon_vit}, we compare OpenCon's performance with {ORCA}~\citep{cao2022openworld}, {GCD}~\citep{vaze22gcd}, {$k$-Means~\citep{macqueen1967classification}, RankStats+~\citep{zhao2021rankstat} and UNO+~\citep{fini2021unified}} on ViT-B-16 architecture. On ImageNet-100, we improve upon the best baseline by 9.9 in terms of overall accuracy.

\noindent \textbf{OpenCon learns more distinguishable representations} 
We visualize feature embeddings using UMAP~\citep{umap} in Figure~\ref{fig:opencon_umap}. Different
colors represent different ground-truth class labels. For clarity, we use the ImageNet-100 dataset and visualize a subset of 10 classes. We can observe that OpenCon produces a better embedding space than GCD and ORCA. In particular, ORCA does not produce distinguishable representations for novel classes, especially when the number of classes increases. The features of GCD are improved, yet with some class overlapping (\emph{e.g.}, two orange classes). 
For reader's reference, we also include the version with a subset of 20 classes in Appendix~\ref{sec:opencon_umap20}, where OpenCon displays more distinguishable representations.

\section{A Comprehensive Analysis of OpenCon~} 
\label{sec:opencon_ablation}

\vspace{0.1cm} \noindent \textbf{Prototype-based OOD detection is important.} In Figure~\ref{fig:opencon_ood-p}, we ablate the
contribution of a key component in OpenCon: prototype-based OOD detection (\emph{c.f.} Section~\ref{sec:opencon_prototype}). To systematically analyze the effect, we report the performance under varying percentile $p \in \{0,10,30,50,70,90\}$. Each $p$ corresponds to 
a different threshold $\lambda$ for separating known vs. novel data in $\mathcal{D}_u$. In the extreme case with $p=0$, $\mathcal{D}_n$ becomes equivalent to $\mathcal{D}_u$, and hence the contrastive loss $\mathcal{L}_n$ is applied to the entire unlabeled data. 
We highlight two findings: (1) Without OOD detection ($p=0$), the unseen accuracy reduces by 2.4\%, compared to the best setting ($p=70\%$). This affirms the importance of OOD detection for better representation learning. (2) A higher percentile $p$, in general, leads to better performance. We also provide theoretical insights in Appendix~\ref{sec:opencon_whyood} showing OOD detection helps contrastive learning of novel classes by having fewer candidate classes.

\begin{table}[htb]
    \centering
    \includegraphics[width=0.5\linewidth]{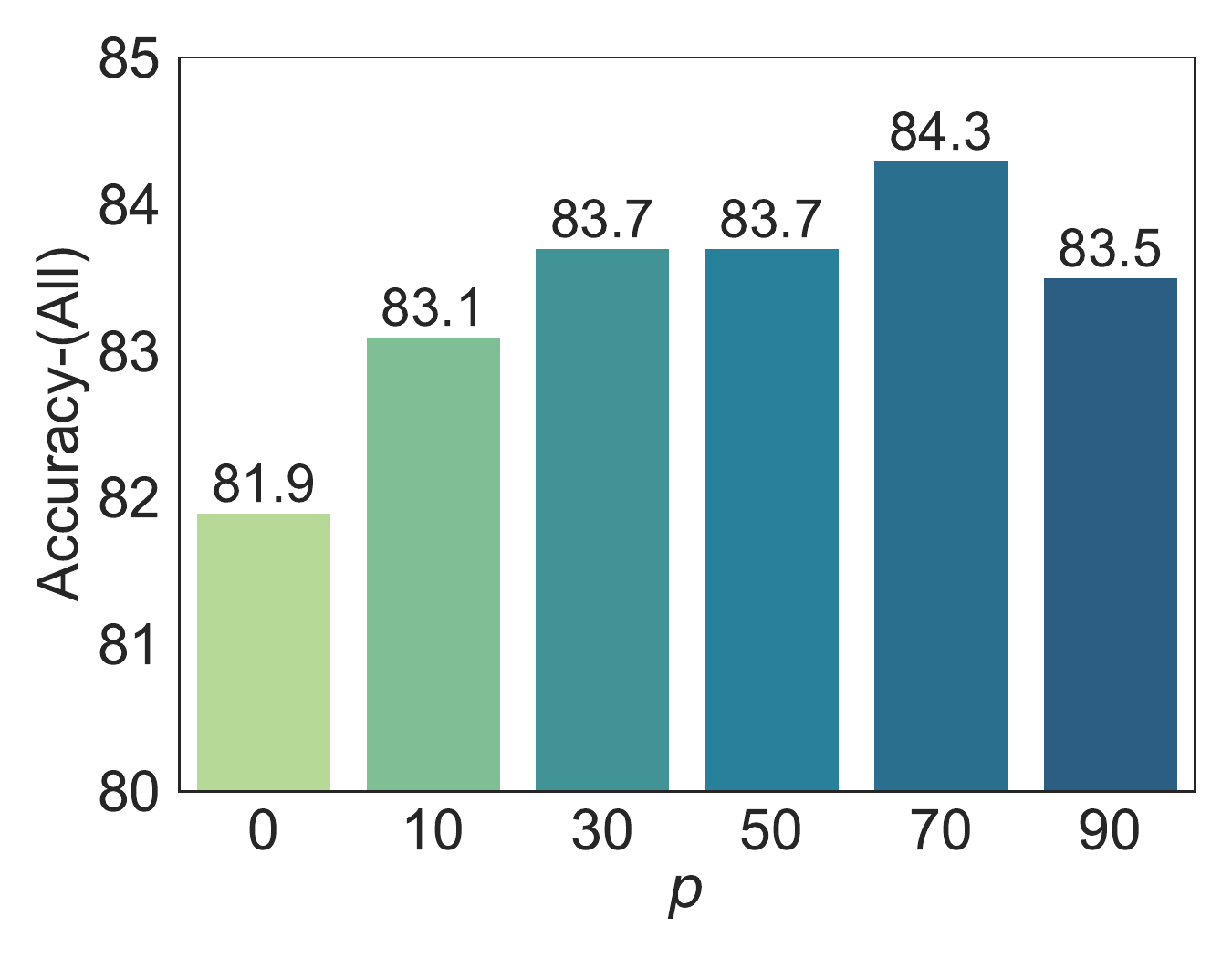}
    \caption[Effect of OOD detection threshold $p$. ]{\small Effect of $p$ on ImageNet-100 and CIFAR-100, measured by the overall accuracy. $p=0$ means no OOD detection.} 
    \label{fig:opencon_ood-p}
\end{table}

\vspace{0.1cm} \noindent \textbf{Ablation study on the loss components.} Recall that our overall objective function in Equation~\ref{eq:opencon_overall_loss} consists of three parts. We ablate the contributions of each component in Table~\ref{tab:opencon_loss}. Specifically, we modify OpenCon by removing: (i) supervised objective (\emph{i.e.}, w/o $\mathcal{D}_l$), (ii) unsupervised objective on the entire unlabeled data (\emph{i.e.}, w/o $\mathcal{D}_u$), and (iii) prototype-based contrastive learning on novel data $\mathcal{D}_n$.
We have the following key observations: (1) Both supervised objective $\mathcal{L}_l$ and unsupervised loss $\mathcal{L}_u$ are indispensable parts of open-world representation learning. This suits the complex nature of our training data, which requires learning on both labeled and unlabeled data, across both known and novel classes. (2) {Purely combining SupCon~\citep{khosla2020supcon} and SimCLR~\citep{chen2020simclr}---as used in GCD~\citep{vaze22gcd}---does not give competitive results}. For example, the overall accuracy is \textbf{9.3}\% lower than our method on the ImageNet-100 dataset. In contrast, having $\mathcal{L}_n$ encourages closely aligned representations to \emph{all} entries from the same \emph{predicted} class, resulting in a more compact representation
space for novel data. 
Overall, the ablation suggests that all losses in our framework work together synergistically to enhance the representation quality.

\begin{table}[htb]

\caption[Ablation study on loss component in OpenCon.]{Ablation study on loss component. }
\label{tab:opencon_loss}
\centering
\scalebox{0.8}{
\begin{tabular}{lllllll}
\hline
\multirow{2}{*}{\textbf{Loss Components}} & \multicolumn{3}{c}{\textbf{CIFAR-100}} & \multicolumn{3}{c}{\textbf{ImageNet-100}}\\
 & All & Novel & Seen & All & Novel & Seen \\  \midrule
w/o $\mathcal{L}_{l}$ & 43.3 & 47.1 & 38.8 & 68.6 & 73.4 & 59.1 \\
w/o $\mathcal{L}_{u}$ & 36.9 & 28.3 & 63.4 & 55.9 & 39.3 & 89.2 \\
w/o $\mathcal{L}_{n}$ & 46.6 & 42.2 & \textbf{70.3} & 74.5 & 70.7 & \textbf{91.0} \\ \midrule
OpenCon (ours) & \textbf{52.7}$ ^{\pm{0.6}} $ & \textbf{47.8}$ ^{\pm{0.6}} $ & {69.1}$ ^{\pm{0.3}} $ & \textbf{83.8}$ ^{\pm{0.3}} $ & \textbf{80.8}$ ^{\pm{0.3}} $ & 90.6$ ^{\pm{0.1}} $
\\ \bottomrule 
\end{tabular}}
\end{table}

\vspace{0.1cm} \noindent \textbf{Handling an unknown number of novel classes.}
In practice, we often do not know the number of classes $|\mathcal{Y}_\text{all}|$ in
advance. This is the dilemma faced by OpenCon and other baselines as well. 
In such cases, one can apply OpenCon by first estimating the number of classes. For a fair comparison, we use the same estimation technique\footnote{A clustering algorithm is performed on the combination of labeled data and unlabeled data. The optimal number of classes is chosen by validating clustering accuracy on the labeled data. } as in ~\citet{Han2019dtc, cao2022openworld}. On CIFAR-100, the estimated total number of classes is 124. At the beginning of training, we initialize the same number of prototypes accordingly. Results in Table~\ref{tab:opencon_unknown-cls} show that OpenCon outperforms the best baseline ORCA~\citep{cao2022openworld} by \textbf{7.3}\%.
Interestingly, despite the initial number of classes, the training process will converge to solutions that closely match the ground truth number of classes. 
For example, at convergence, we observe a total number of 109 actual clusters. The remaining ones have no samples assigned, hence can be discarded. 
Overall, with the estimated number of classes, OpenCon can achieve similar performance compared to the setting in which the number of classes is known.

\begin{table}[htb]
\caption[Performance with an unknown number of classes.]{Accuracy  on CIFAR-100 dataset with an unknown number of classes.}
    \centering
    \scalebox{0.9}{
    \begin{tabular}{ccccc} \toprule
         Methods & All & Novel & Seen \\ \midrule
         ORCA~\citep{cao2022openworld} & 46.4 & 40.0  & 66.3  \\
         GCD~\citep{vaze22gcd}& 47.2 & 41.9  & \textbf{69.8} \\ \midrule
         OpenCon (Known $|\mathcal{Y}_\text{all}|$) & 53.7 & \textbf{48.7} & \textbf{69.0} \\
         OpenCon (Unknown $|\mathcal{Y}_\text{all}|$) & \textbf{53.7} & 48.2  &  68.8 \\
            \bottomrule
    \end{tabular}}
    \label{tab:opencon_unknown-cls}

\end{table}

\begin{table}[htb]

\caption[Performance varying different labeling ratios and different numbers of known classes.]{Accuracy on CIFAR-100 under varying different labeling ratios (for labeled data) and different numbers of known classes ($|\mathcal{Y}_{l}|$). The number of novel classes is $100-|\mathcal{Y}_{l}|$. %
}
\label{tab:opencon_abl-label}
\centering
\scalebox{0.88}{
\begin{tabular}{lllllllll}
\hline
\multirow{2}{*}{\textbf{Labeling Ratio}} & \multirow{2}{*}{$|\mathcal{Y}_{l}|$} & \multicolumn{1}{c}{\multirow{2}{*}{\textbf{Method}}} & \multicolumn{3}{c}{\textbf{CIFAR-100}} & \multicolumn{3}{c}{\textbf{ImageNet-100}} \\
 & & \multicolumn{1}{c}{} & \multicolumn{1}{c}{All} & \multicolumn{1}{c}{Novel} & \multicolumn{1}{c}{Seen}  & \multicolumn{1}{c}{All} & \multicolumn{1}{c}{Novel} & \multicolumn{1}{c}{Seen} \\ \hline

\multirow{3}{*}{0.5} & \multirow{3}{*}{50} & ORCA & 47.0 & 41.3 & 66.2 & 76.6 & 69.0 & 88.9 \\
 & &  GCD & {47.2} & {43.6} & \textbf{{69.4}} & {75.1} & {73.2}  & \textbf{{90.9}}  \\
  & & {OpenCon} & \textbf{{53.7}} & \textbf{{48.7}} & {{69.0}} & \textbf{{84.3}} & \textbf{{81.1}} & {90.7} \\ \midrule
 
 \multirow{3}{*}{0.25} & \multirow{3}{*}{50} & ORCA & 47.6 & 42.5 & 61.8 & 72.2 & 64.2 & 87.3 \\
 & & GCD & 41.4 & 39.3 & \textbf{66.0} & 76.7 & 69.1 & 89.1 \\
 & & OpenCon & \textbf{51.3} & \textbf{44.6} & 65.5 & \textbf{82.0} & \textbf{77.1} & \textbf{90.6} \\ \midrule
\multirow{3}{*}{0.1} & \multirow{3}{*}{50} & ORCA & 41.2 & 37.7 & 54.6 & 68.7 & 56.8 & 83.4 \\
 & & GCD & 37.0 & 38.6 & 62.2 & 69.4 & 56.6 & \textbf{85.8} \\
 & & OpenCon & \textbf{48.2} & \textbf{44.4 } & \textbf{62.5 } & \textbf{75.4} & \textbf{66.8} & 85.2 \\ \midrule
\multirow{3}{*}{0.5} & \multirow{3}{*}{25} & ORCA & 40.4 & 38.8 & 66.0 & 57.8 & 54.0 & 89.4 \\
 & & GCD & 41.6 & 39.2 & 70.0 & 65.3 & 63.2 & 90.8 \\
 & & OpenCon & \textbf{43.9} & \textbf{41.9} & \textbf{70.2} & \textbf{74.5} & \textbf{72.7} & \textbf{91.2} \\ \midrule
\multirow{3}{*}{0.5} & \multirow{3}{*}{10} & ORCA & 34.3 & 33.8 & 67.4 & 45.1 & 43.6 & 93.0 \\
 & & GCD & 38.4 & 36.8 & 61.3 & 53.3 & 52.9 & 94.2 \\
 & & OpenCon & \textbf{40.9} & \textbf{40.5} & \textbf{69.9 } & \textbf{59.0} & \textbf{58.2} & \textbf{94.3}

 \\ \bottomrule
 
\end{tabular}}
\end{table}

% \vspace{0.3cm}
\vspace{0.1cm} \noindent \textbf{OpenCon is robust under a smaller number of labeled examples, and a larger number of novel classes.} 
We show that OpenCon's strong performance holds under more challenging settings with: (1) reduced fractions of labeled examples, and (2) different ratios of known vs. novel classes. The results are summarized in Table~\ref{tab:opencon_abl-label}. First, we reduce the labeling ratio from $50\%$ (default) to $25\%$ and $10\%$, while keeping the number of known classes to be the same (\emph{i.e.}, 50). With fewer labeled samples in the known classes, the unlabeled sample set size will expand accordingly. It posits more challenges for novelty discovery and representation learning. On ImageNet-100, OpenCon substantially improves the novel class accuracy by \textbf{10}\% compared to ORCA and GCD, when only 10\% samples are labeled. Secondly, we further increase the number of novel classes, from 50 (default) to 75 and 90 respectively.  On ImageNet-100 with 75 novel classes ($|\mathcal{Y}_l|=25$), OpenCon improves the novel class accuracy by \textbf{16.7}\% over ORCA~\citep{cao2022openworld}. Overall our experiments confirm the robustness of OpenCon under various settings.  
% \vspace{0.6cm}

\section{Additional Related Work~}

\vspace{0.1cm} \noindent \textbf{Contrastive learning.} A great number of works have explored the effectiveness of contrastive loss in unsupervised representation learning: InfoNCE~\citep{van2018cpc}, SimCLR~\citep{chen2020simclr}, SWaV~\citep{caron2020swav}, MoCo~\citep{he2019moco}, SEER~\citep{goyal2021seer} and ~\citep{li2020mopro,li2020prototypical, zhang2021supporting}. It motivates follow-up works to on weakly supervised learning tasks~\citep{zheng2021weakcl, tsai2022wcl2}, semi-supervised learning~\citep{chen2020simclrv2, li2021comatch, zhang2022semi, yang2022classaware}, supervised learning with noise ~\citep{wu2021ngc, karim2022unicon, Li2022SelCL}, continual learning~\citep{cha2021co2l}, long-tailed recognition~\citep{cui2021parametriccl, tian2021divide, jiang2021improving, tianhong2022targetedsupcon}, few-shot learning~\citep{gao2021fewshot}, partial label learning~\citep{wang2022pico}, novel class discovery~\citep{zhong2021ncl, zhao2021rankstat, fini2021unified}, hierarchical multi-label learning~\citep{shu2022hierarchical}. Under different circumstances, all works adopt different choices of the positive set, which is not limited to the self-augmented view in SimCLR~\citep{chen2020simclr}. 
Specifically, with label information available, SupCon~\citep{khosla2020supcon} improved representation quality by aligning features within the same class. Without supervision for the unlabeled data, ~\citet{dwibedi2021nncl} used the nearest neighbor as positive pair to learn a  compact embedding space. Different from prior works, we focus on the open-world representation learning problem, which is largely unexplored.

\vspace{0.1cm} \noindent \textbf{Novel category discovery.}
At an earlier stage, the problem of novel category discovery (NCD) is targeted as a transfer learning problem in DTC~\citep{Han2019dtc}, KCL~\citep{hsu2017kcl}, MCL~\citep{hsu2019mcl}. The learning is generally in a two-stage manner: the model is firstly trained with the labeled data and then transfers knowledge to learn the unlabeled data. OpenMix~\citep{zhong2021openmix} further proposes an end-to-end framework by mixing the seen and novel classes in a joint space. In recent studies, many researchers incorporate representation learning for NCD like RankStats~\citep{zhao2021rankstat}, NCL~\citep{zhong2021ncl} and UNO~\citep{fini2021unified}. In the ORL setting, the unlabeled test set consists of novel classes
but also classes are previously seen in the labeled data that need to be separated.

\vspace{0.1cm} \noindent \textbf{Semi-supervised learning.}  
A great number of early works~\citep{chapelle2006ssl, lee2013pseudo,sajjadi2016regularization,laine2016temporal,zhai2019s4l,rebuffi2020semi,alex2020fixmatch} have been proposed to tackle the problem of
semi-supervised learning (SSL). Typically, a standard cross-entropy loss is applied to the labeled data, and a consistency loss~\citep{laine2016temporal,alex2020fixmatch} or self-supervised loss~\citep{sajjadi2016regularization,zhai2019s4l,rebuffi2020semi} is applied to the unlabeled data. Under the closed-world assumption, SSL
methods achieve competitive performance which is close to the supervised methods. Later works ~\citep{oliver2018realistic, chen2020semi} point out that including novel classes in the unlabeled set can downgrade the performance. 
In ~\citet{guo2020dsl, chen2020semi, yu2020multi,park2021opencos, saito2021openmatch, huang2021trash, yang2022classaware}, OOD detection techniques are wielded to separate the OOD samples in the unlabeled data. Recent works~\citep{vaze22gcd, cao2022openworld, rizve2022openldn} further require the model to group samples from novel classes into semantically meaningful clusters. In our framework, we unify the novelty class detection and the representation learning and achieve competitive performance.

% 

%%%%%%%%%%%%%%%%%%%%%%%%%%%%%%%%%%%%%%%%%%%%%%%%%%%%%%%%%%%%%%%%%%%%%
%%%%%%%%%%%%%%%%%%%%%%%%%%%%%%  Summary %%%%%%%%%%%%%%%%%%%%%%%%%%%%%%%
%%%%%%%%%%%%%%%%%%%%%%%%%%%%%%%%%%%%%%%%%%%%%%%%%%%%%%%%%%%%%%%%%%%%  

\section{Summary~}
\label{sec:opencon_summary}

This chapter provides a new learning framework, \textit{open-world contrastive learning} (OpenCon) that learns highly distinguishable representations for both known and novel classes in an open-world setting.  Our open-world setting can generalize traditional representation learning and offers stronger flexibility. We provide important insights that the separation between known vs. novel data in the unlabeled data and the pseudo supervision for data in novel classes is critical. Extensive
experiments show that OpenCon can notably improve the accuracy on both known and novel classes compared to the current best method ORCA. {As a shared challenge by all methods, one limitation is that the prototype number in our end-to-end training framework needs to be pre-specified.  An interesting future work may include the mechanism to dynamically estimate and adjust the class number during the training stage.}

%%%%%%%%%%%%%%%%%%%%%%%%%%%%%%%%%%%%%%%%%%%%%%%%%%%%%%%%%%%%%%%%%%%%%
%%%%%%%%%%%%%%%%%%%%%%%%%%%%%%  Supp %%%%%%%%%%%%%%%%%%%%%%%%%%%%%%%
%%%%%%%%%%%%%%%%%%%%%%%%%%%%%%%%%%%%%%%%%%%%%%%%%%%%%%%%%%%%%%%%%%%%  

% \newpage
% \section{Appendix}
% \label{sec:opencon_supp}

% \input{chapters/supp_opencon}

%% file: chapters/conclusion.tex
\chapter{Conclusion~}
\label{sec:conclusion}

In conclusion, this thesis has provided significant contributions to the advancement of machine learning within open-world scenarios. Open-world learning, in contrast to traditional closed-world machine learning models, is confronted with novel, unseen data and contexts. This presents an unprecedented set of challenges that demand robust, innovative solutions.
The thesis tackled these challenges in two interconnected stages: Out-of-distribution (OOD) Detection and Open-world Representation Learning (ORL). 

The first stage, OOD detection, provides the foundation for identifying instances from previously unseen classes, thus reducing the risk of overconfident and potentially misleading predictions. We have developed pioneering methodologies, including ReACT (Chapter~\ref{sec:react}) and DICE (Chapter~\ref{sec:dice}), and introduced a non-parametric approach using K-nearest neighbor (KNN) distance (Chapter~\ref{sec:knn}), all of which enhance the effectiveness of OOD detection.

Building upon the OOD detection, the second stage, ORL, extends our capacity to learn from and incorporate knowledge about new classes. NSCL in Chapter~\ref{sec:nscl} and SORL in Chapter~\ref{sec:sorl} deepen our understanding of the complex interplay between known and unknown classes and the critical role of label information in shaping representations. This thesis answers these challenging questions, providing invaluable insights for the development of advanced ORL algorithms.
Moreover, we provided an empirical solution in the form of a comprehensive contrastive learning framework in Chapter~\ref{sec:opencon} for ORL, buttressed by theoretical interpretation from the Expectation-maximization perspective. This work is anticipated to significantly enhance machine learning's adaptability and reliability in open-world scenarios.

By navigating the complexities of open-world learning, this thesis contributes to shaping a new paradigm in machine learning, one that is responsive and adaptable to ever-evolving open-world contexts. The insights, methodologies, and theoretical developments presented here pave the way for future research in open-world machine learning, making strides toward unknown-aware and unknown-adaptable AI systems.

\section{Future Work~}

As we chart the future of open-world machine learning, there are several promising areas for further research. While this thesis has made significant strides in OOD detection and ORL, these areas of study are still in their nascent stages, and there are numerous unexplored avenues to investigate.
\begin{enumerate}
    \item \textbf{Fine-tuning of Models with Growing Data:} Future studies could consider more sophisticated techniques for fine-tuning machine learning models in the face of new data. This might involve developing dynamic models capable of evolving their parameters as they encounter new classes without forgetting existing knowledge, thereby boosting their adaptability and effectiveness in open-world scenarios.

    \item \textbf{Expanding the Scope of Open-world Learning:} Presently, open-world learning is most often associated with visual classification tasks. Future work can aim to extend this concept to other domains such as natural language processing, speech recognition, and recommendation systems.

    \item \textbf{Ethical and Legal Considerations:} As machine learning models become increasingly ubiquitous in society, future research must consider the ethical and legal aspects of open-world learning. This could involve developing frameworks to ensure fairness, transparency, and accountability in open-world learning models.

    \item \textbf{Real-world Testing and Implementation}: Lastly, while the results presented in this thesis exhibit considerable promise on benchmark datasets, the significant value lies in deploying and evaluating these models on real-world datasets. These datasets differ from benchmark collections in several ways; they may not have a balanced distribution of training samples across different classes and might not provide a clear separation between class sets. Implementing models in such environments helps identify practical challenges and areas for refinement that are not typically apparent in benchmark datasets. Conducting tests within this context, therefore, stands as a crucial step for future work, enabling us to fine-tune and validate our models under actual operating conditions and further improve their reliability and adaptability.
\end{enumerate}

Overall, the future of open-world machine learning holds vast potential. As this field of research progresses, we are likely to see even more innovative solutions and methodologies, moving us closer to truly intelligent, adaptive, and reliable AI systems.

%% file: chapters/appendix.tex
\newpage
\appendix 

\chapter{Appendix for Out-of-distribution Detection~}

\section{ReAct: OOD Detection With
Rectified Activations~}
\label{sec:react_supp}
\input{chapters/supp_react}

\clearpage
\section{DICE: Leverage Sparsification for OOD Detection~}
\label{sec:dice_supp}

\input{chapters/supp_dice}

\clearpage
\section{OOD Detection with Deep Nearest Neighbors~}
\label{sec:knn_supp}

\input{chapters/supp_knn}

\chapter{Appendix for Open-world Representation Learning~}

\section{When and How Does Known Class Help Discover Unknown Ones? A Spectral Analysis~}
\label{sec:nscl_supp}
\input{chapters/supp_nscl}

\clearpage
\section{A Graph-theoretic Framework for Understanding ORL~}
\label{sec:sorl_supp}

\input{chapters/supp_sorl}

\clearpage
\section{OpenCon: Open-world Contrastive Learning~}
\label{sec:opencon_supp}

\input{chapters/supp_opencon}

%% file: chapters/supp_react.tex
\subsection{Theoretical Details}
\label{sec:react_theory_details}

Here we derive \autoref{eq:react_ood_reduction} for $\epsilon \geq 0$ and $\sigma_\text{out} = \sigma > 0$. Since $\text{ESN}(\mu, \sigma^2, 0) = \mathcal{N}^R(\mu, \sigma)$, we can obtain \autoref{eq:react_id_reduction} for ID activation by specializing the result to $\epsilon = 0$. We begin with a useful lemma.

\begin{lemma}
\label{lem:esn_prob}
Let $X \sim \text{ESN}(0, \sigma^2, \epsilon)$ and let $a \leq b \leq 0$, $0 \leq c \leq d$. Then $\mathbb{P}(a \leq X \leq b) = (1+\epsilon) \left[ \Phi \left( \frac{b}{(1+\epsilon) \sigma} \right) - \Phi \left( \frac{a}{(1+\epsilon) \sigma} \right) \right]$ and $\mathbb{P}(c \leq X \leq d) = (1-\epsilon) \left[ \Phi \left( \frac{d}{(1-\epsilon) \sigma} \right) - \Phi \left( \frac{c}{(1-\epsilon) \sigma} \right) \right]$.
\end{lemma}

\begin{proof}
\begin{align*}
    \mathbb{P}(a \leq X \leq b) &= (1+\epsilon) \int_{x = a}^b \frac{1}{(1+\epsilon) \sigma} \phi \left( \frac{x}{(1+\epsilon) \sigma} \right) dx \\
    &= (1+\epsilon) \left[ \Phi \left( \frac{b}{(1+\epsilon) \sigma} \right) - \Phi \left( \frac{a}{(1+\epsilon) \sigma} \right) \right]
\end{align*}
since the integral that of a $\mathcal{N}(0, (1+\epsilon)^2 \sigma^2)$ distribution between $a$ and $b$. The result for $\mathbb{P}(c \leq X \leq d)$ follows analogously.
\end{proof}

Suppose that $X_\mu \sim \text{ESN}(\mu, \sigma^2, \epsilon)$ with $\mu > 0$ and let $Z_\mu = \max(X_\mu, 0)$. Define $X = X_\mu - \mu$ so that $Z_\mu = \max(X + \mu, 0) = \max(X, -\mu) + \mu$. We can derive the expectation of $Z := \max(X, -\mu)$:
\begin{align*}
    \mathbb{E}[Z] &= \underbrace{-\mu \cdot \mathbb{P}(X < -\mu)}_{\text{(I)}} + \underbrace{\int_{x = -\mu}^0 \frac{x}{\sigma} \phi \left( \frac{x}{(1+\epsilon) \sigma} \right) dx}_{\text{(II)}} + \underbrace{\int_{x = 0}^\infty \frac{x}{\sigma} \phi \left( \frac{x}{(1-\epsilon) \sigma} \right) dx}_{\text{(III)}}. \\
    \text{(I)} &= -\mu (1+\epsilon)\Phi \left(\frac{-\mu}{(1+\epsilon) \sigma}\right) \quad \text{by Lemma \ref{lem:esn_prob}}; \\
    \text{(II)} &= (1 + \epsilon) \int_{x = -\mu}^0 \frac{x}{1 + \epsilon) \sigma} \phi \left( \frac{x}{(1+\epsilon) \sigma} \right) dx = (1 + \epsilon)^2 \left[ \phi \left( \frac{-\mu}{(1+\epsilon) \sigma} \right) - \phi(0) \right] \sigma
\end{align*}
since the integral is the expectation of an un-normalized truncated Gaussian between $-\mu$ and $0$. Similarly, (III) is $(1-\epsilon)$ times the expectation of an un-normalized Gaussian between $0$ and $\infty$, thus $\text{(III)} = (1 - \epsilon)^2 \phi(0) \sigma$. Combining (I)-(III) gives:
\begin{align*}
    \mathbb{E}[Z] &= -\mu (1+\epsilon)\Phi \left(\frac{-\mu}{(1+\epsilon) \sigma}\right) + (1 + \epsilon)^2 \left[ \phi \left( \frac{-\mu}{(1+\epsilon) \sigma} \right) - \phi(0) \right] \sigma + (1 - \epsilon)^2 \phi(0) \sigma \\
    &= -\mu (1+\epsilon)\Phi \left(\frac{-\mu}{(1+\epsilon) \sigma}\right) + (1 + \epsilon)^2 \phi \left( \frac{-\mu}{(1+\epsilon) \sigma} \right) \sigma + \phi(0) \sigma [(1-\epsilon)^2 - (1 + \epsilon)^2] \\
    &= -\mu (1+\epsilon)\Phi \left(\frac{-\mu}{(1+\epsilon) \sigma}\right) + (1 + \epsilon)^2 \phi \left( \frac{-\mu}{(1+\epsilon) \sigma} \right) \sigma - 4 \epsilon \phi(0) \sigma \\
    &= -\mu (1+\epsilon)\Phi \left(\frac{-\mu}{(1+\epsilon) \sigma}\right) + (1 + \epsilon)^2 \phi \left( \frac{-\mu}{(1+\epsilon) \sigma} \right) \sigma - \frac{4 \epsilon}{\sqrt{2 \pi}} \sigma.
\end{align*}
\autoref{eq:react_ood_before_mean} follows since $\mathbb{E}[Z_\mu] = \mathbb{E}[Z] + \mu$. To derive \autoref{eq:react_ood_after_mean}, note that the expectation of $\bar{Z} := \min(Z, c - \mu)$ is given by:
\begin{equation*}
    \mathbb{E}[\bar{Z}] = \text{(I)} + \text{(II)} + \underbrace{\int_{x = 0}^{c - \mu} \frac{x}{\sigma} \phi \left( \frac{x}{(1-\epsilon) \sigma} \right) dx}_{\text{(IV)}} + \underbrace{(c - \mu) \cdot \mathbb{P}(X > c - \mu)}_{\text{(V)}}.
\end{equation*}
$\text{(IV)} = (1 - \epsilon)^2 \left[ \phi(0) - \phi \left( \frac{c-\mu}{(1-\epsilon) \sigma} \right) \right] \sigma$ follows from a similar argument as above, and
$$\text{(V)} = (c - \mu) (1 - \epsilon) \left[ 1 - \Phi \left( \frac{c - \mu}{(1 + \epsilon) \sigma} \right) \right]$$
can be derived using Lemma \ref{lem:esn_prob}. Combining (I),(II),(IV),(V) and observing that $\mathbb{E}[\min(Z_\mu, c)] = \mathbb{E}[\bar{Z}] + \mu$ gives \autoref{eq:react_ood_after_mean}.

\subsection{Ablation Study on Different Layers }
\label{sec:react_diff_layers}
We provide the activation patterns for intermediate layers in Figure~\ref{fig:react_diff_layers} and the OOD detection performance of applying ReAct to these layers in Table~\ref{tab:diff_layers}. In particular, there are four residual blocks in the original ResNet-50 network~\citep{he2016identity}. The four layers (denoted by \texttt{layer 1} - \texttt{layer 4}) are taken from the output of each residual block. Interestingly, early layers display less distinctive signatures between ID and OOD data and ReAct performs worse than the baseline~\citep{liu2020energy} when it is applied on \texttt{layer 1} - \texttt{layer 3}. This is expected because neural networks generally capture lower-level features in early layers (such as Gabor filters~\citep{zeiler2014visualizing}), whose activations can be very similar between ID and OOD. The semantic-level features only emerge as with deeper layers, where ReAct is the most effective. 

\begin{figure}[htb]
	\begin{center}
		\includegraphics[width=0.98\linewidth]{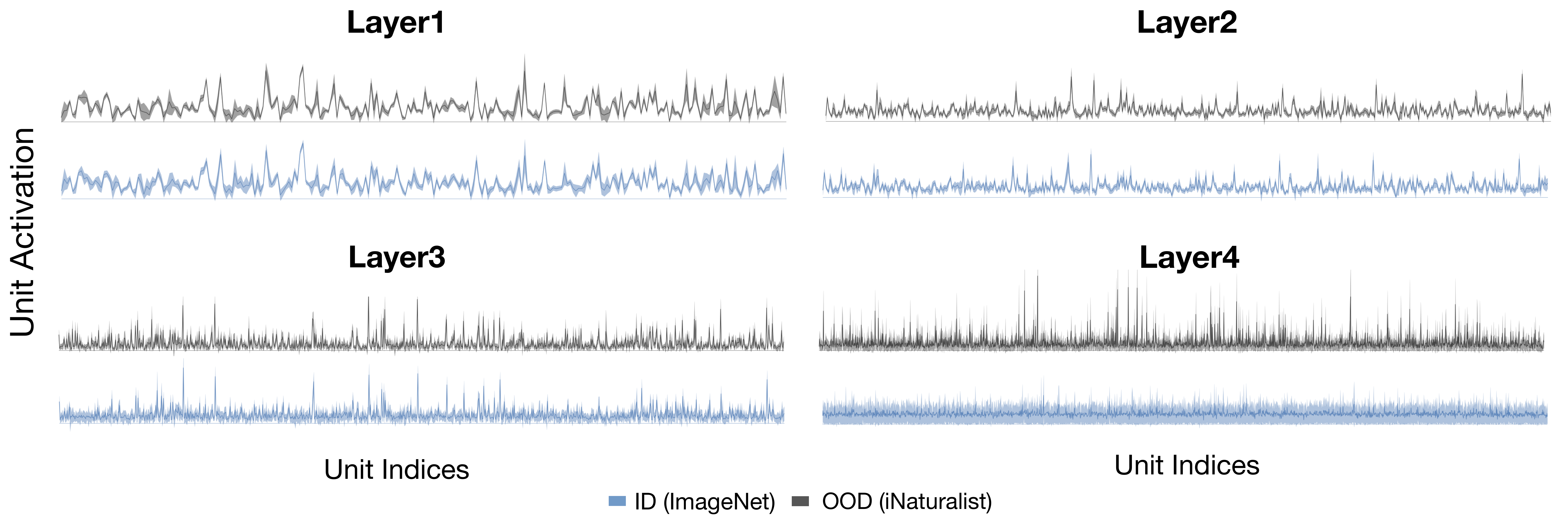}
	\end{center}
	\caption[The distribution of per-unit activations in the penultimate layer for ImageNet and iNaturalist on different layers.]{\small The distribution of per-unit activations in the penultimate layer for ID (ImageNet, blue) and OOD (iNaturalist, gray) on different layers. 
	Each layer (denoted by \texttt{layer 1} - \texttt{layer 4})  corresponds to the output of each residual block in ResNet-50~\citep{he2016identity}. 
	}
	\vspace{-0.4cm}
	\label{fig:react_diff_layers}
\end{figure}

\begin{table}[htb]
\caption[Ablation study of applying ReAct on different layers.]{ \small Ablation study of applying ReAct on different layers. We used ResNet-18~\citep{he2016deep} pre-trained on CIFAR-100 and ResNet-50 pre-trained on ImageNet. $\uparrow$ indicates larger values are better and $\downarrow$ indicates smaller values are better. All values are percentaged over multiple OOD test datasets described in Section~\ref{sec:react_common_benchmark} and Section~\ref{sec:react_imagenet}.}
\centering
\scalebox{0.85}{
\begin{tabular}{lllll}
\toprule
\multirow{2}{*}{Layers of applying ReAct} & \multicolumn{2}{c}{ID: CIFAR-100} & \multicolumn{2}{c}{ID: ImageNet} \\ \cline{2-5}
 & FPR95 $\downarrow$ & AUROC $\uparrow$ & FPR95 $\downarrow$ & AUROC $\uparrow$ \\
\midrule
Layer1 & 90.86 & 68.17 & 84.83 & 74.88 \\
Layer2 & 84.12 & 75.32 & 76.25 & 79.37 \\
Layer3 & 73.4 & 80.91 & 63.87 & 86.46 \\
Layer4 (ReAct) & \textbf{59.61} & \textbf{87.48} & \textbf{31.43} & \textbf{92.95} \\
\midrule
No ReAct~\citep{liu2020energy} & 71.93 & 82.82 & 58.41 & 86.17 \\ \bottomrule
\end{tabular}}
\label{tab:diff_layers}
\end{table}

\subsection{Using True BatchNorm Statistics on OOD Data}
\label{sec:react_bn_ood}
Typically, for a unit activation denoted by $z$, the network estimates the running mean $\mathbb{E}_\text{in}(z)$ and variance $\text{Var}_\text{in}(z)$, over the entire ID training set during training. During inference time, the network applies BatchNorm statistics~\citep{bn2015pmlr} $\mathbb{E}_\text{in}(z)$ and $\text{Var}_\text{in}(z)$, which helps normalize the activations for the test data with the same distribution $\mathcal{D}_\text{in}$:
\begin{align}
    \text{BatchNorm}(z;\gamma, \beta, \epsilon) = \frac{z-\mathbb{E}_\text{in}[z]}{\sqrt{\text{Var}_\text{in}[z]+\epsilon}}\cdot \gamma + \beta
\end{align}

However, our key observation is that using \emph{mismatched} BatchNorm statistics---that are estimated on $\mathcal{D}_\text{in}$ yet blindly applied to the OOD $\mathcal{D}_\text{out}$---can trigger abnormally high unit activations (see bottom of Figure~\ref{fig:react_trainval}). 
As a thought experiment, for OOD data, we instead apply the \emph{true} BatchNorm statistics estimated on a batch of OOD images:
\begin{align}
    \text{BatchNorm}(z;\gamma, \beta, \epsilon) = \frac{z-\mathbb{E}_\text{out}[z]}{\sqrt{\text{Var}_\text{out}[z]+\epsilon}}\cdot \gamma + \beta.
\end{align}
As a result, we observe well-behaved activation patterns with near-constant mean and standard deviations (see the top of Figure~\ref{fig:react_trainval}). Our study therefore reveals one of the fundamental causes for neural networks to produce overconfident predictions for OOD data. Despite the interesting observation, we note that estimating the true BN statistics for OOD poses a strong and impractical assumption of having access to a batch of OOD data during test time. In contrast, using ReAct does not operate under such an assumption and can be applied on any single OOD instance, as well as for neural networks trained with alternative normalization mechanisms (as we show in Section~\ref{sec:react_discussion}). 

\begin{figure}[htb]
	\begin{center}
		\includegraphics[width=0.7\linewidth]{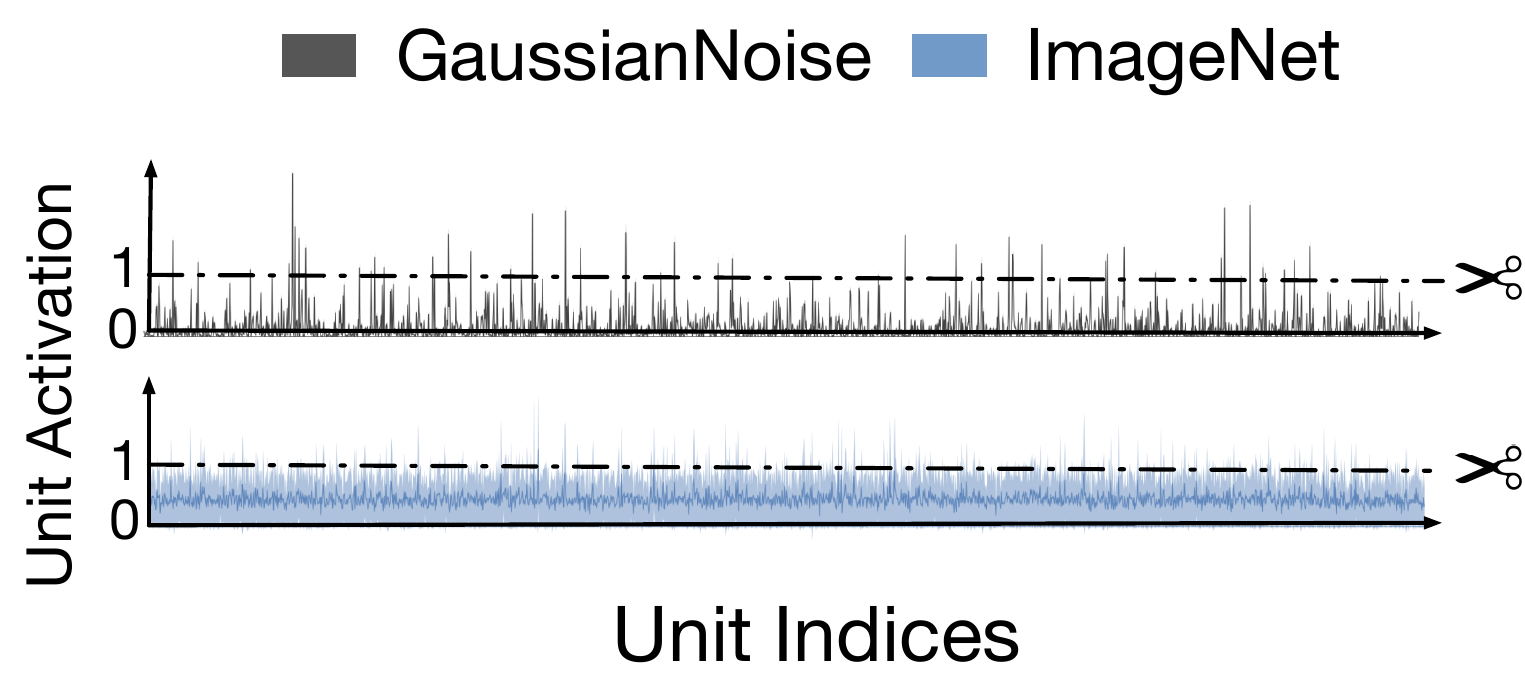}
	\end{center}
	\caption[The distribution of per-unit activations in the penultimate layer for Gaussian Noise and ImageNet.]{\small The distribution of per-unit activations in the penultimate layer for OOD data (Gaussian Noise) and ID data (ImageNet) on ResNet50.}
	\vspace{-0.4cm}
	\label{fig:react_noise_act}
\end{figure}

\subsection{Unit Activation Patterns for Gaussian Noise}
\label{sec:react_noise_act}
We provide the activation patterns for Gaussian noise input (our validation data) in Figure~\ref{fig:react_noise_act}. The experiment is based on ResNet-50 architecture~\citep{he2016identity}. We show that using Gaussian noise as input can lead to overly high unit activations, which is consistent with the observation in Figure~\ref{fig:react_teaser} and Figure~\ref{fig:react_abnormal}.

\begin{sidewaystable}
\caption[Detailed results on six common OOD benchmark datasets.]{\small Detailed results on six common OOD benchmark datasets: \texttt{Textures}~\citep{cimpoi2014describing}, \texttt{SVHN}~\citep{netzer2011reading}, \texttt{Places365}~\citep{zhou2017places}, \texttt{LSUN-Crop}~\citep{yu2015lsun}, \texttt{LSUN-Resize}~\citep{yu2015lsun}, and \texttt{iSUN}~\citep{xu2015turkergaze}. For each ID dataset, we use the same ResNet-18 architecture~\citep{he2016deep} and compare the performance with and without ReAct respectively. $\uparrow$ indicates larger values are better and $\downarrow$ indicates smaller values are better. }
\scalebox{0.60}{
\begin{tabular}{cllllllll}
\toprule
\multicolumn{1}{l}{\multirow{3}{*}{\textbf{ID Dataset}}} & \multicolumn{1}{c}{\multirow{3}{*}{\textbf{Methods}}} & \multirow{2}{*}{\textbf{SVHN}} & \multirow{2}{*}{\textbf{LSUN-Crop}} & \multirow{2}{*}{\textbf{LSUN-Resize}} & \multirow{2}{*}{\textbf{iSUN}} & \multirow{2}{*}{\textbf{Textures}} & \multirow{2}{*}{\textbf{Places365}} & \multirow{2}{*}{\textbf{Average}} \\
\multicolumn{1}{l}{} & \multicolumn{1}{c}{} & & & & & & & \\ \cline{3-9} 
\multicolumn{1}{l}{} & \multicolumn{1}{c}{} & \multicolumn{7}{c}{\textbf{FPR95 $\downarrow$ / AUROC $\uparrow$ / AUPR $\uparrow$ }} \\ \hline
\multirow{8}{*}{CIFAR-10} & MSP & 59.66/91.25/78.84 & 45.21/93.80/80.81 & 51.93/92.73/80.04 & 54.57/92.12/80.01 & 66.45/88.5/79.47 & 62.46/88.64/75.48 & 56.71/91.17/79.11 \\
 & MSP + ReAct & 57.15/91.69/91.77 & 46.37/93.33/93.00 & 46.32/93.61/93.37 & 50.02/92.96/93.26 & 62.85/89.31/92.59 & 60.15/89.28/88.65 & 53.81/91.70/92.11 \\
 & ODIN & 60.37/88.27/89.82 & 7.81/98.58/98.73 & 9.24/98.25/98.51 & 11.62/97.91/98.38 & 52.09/89.17/93.72 & 45.49/90.58/90.55 & 31.10/93.79/94.95 \\
 & ODIN+ReAct & 51.77/88.87/89.09 & 14.99/97.29/97.42 & 6.84/98.65/98.81 & 9.55/98.28/98.62 & 43.81/90.41/94.16 & 45.87/90.73/90.82 & 28.81/94.04/94.82 \\
 & Energy & 54.41/91.22/93.05 & 10.19/98.05/98.33 & 23.45/96.14/96.92 & 27.52/95.59/96.78 & 55.23/89.37/94.01 & 42.77/91.02/90.98 & 35.60/93.57/95.01 \\
 & Energy+ReAct & 49.77/92.18/93.67 & 16.99/97.11/97.48 & 17.94/96.98/97.56 & 20.84/96.46/97.38 & 47.96/91.55/95.40 & 43.97/91.33/91.66 & 32.91/94.27/95.53 \\ \midrule
\multirow{8}{*}{CIFAR-100} & MSP & 81.32/77.74/78.78 & 70.11/83.51/83.02 & 82.46/75.73/76.32 & 82.26/76.16/78.26 & 85.11/73.36/80.79 & 83.06/74.47/73.27 & 80.72/76.83/78.41 \\
 & MSP + ReAct & 74.17/82.3/85.58 & 73.1/82.47/85.25 & 74.73/80.81/83.77 & 73.49/81.45/85.60 & 74.82/80.37/89.03 & 82.37/74.99/76.43 & 75.45/80.4/84.28 \\
 & ODIN & 40.94/93.29/94.49 & 28.72/94.51/94.93 & 79.61/82.13/85.09 & 76.66/83.51/87.35 & 83.63/72.37/82.80 & 87.71/71.46/72.85 & 66.21/82.88/86.25 \\
 &  ODIN+ReAct & 22.87/95.63/96.13 & 36.61/91.45/91.20 & 75.02/85.53/88.42 & 70.21/86.51/89.89 & 66.79/82.69/89.52 & 87.94/69.57/70.01 & 59.91/85.23/87.53 \\
 & Energy & 81.74/84.56/88.39 & 34.78/93.93/94.77 & 73.57/82.99/85.57 & 73.36/83.80/87.40 & 85.87/74.94/84.12 & 82.23/76.68/77.40 & 71.93/82.82/86.28 \\
 & Energy+ReAct & 70.81/88.24/91.07 & 39.99/92.51/93.38 & 54.47/89.56/91.07 & 51.89/90.12/92.29 & 59.15/87.96/93.31 & 81.33/76.49/76.63 & 59.61/87.48/89.63  \\ \bottomrule
\end{tabular}
}
\label{tab:detail-results}
\end{sidewaystable}

%% file: chapters/supp_dice.tex
\subsection{Variance Reduction with Correlated Variables}
\label{sec:dice_corr}

\vspace{0.1cm} \noindent \textbf{Extension of Lemma 2.} We can show variance reduction in a more general case with correlated variables. The variance of output $f_c$ without sparsification is:
$$\mathrm{Var} [f_c] = \sum_{i=1}^{m} \sigma_i^2  + 2 \sum_{1\le i < j \le m} \mathrm{Cov}(v_i, v_j),$$
where $\mathrm{Cov}(\cdot ,\cdot )$ is the covariance. The expression states that the variance is the sum of the diagonal of the covariance matrix plus two times the sum of its upper triangular elements.

Similarly, the variance of output \emph{with} directed sparsification (by taking the top units) is:
$$\mathrm{Var} [f_c^\text{DICE}]= \sum_{i=t+1}^{m} \sigma_i^2  + 2 \sum_{t< i < j \le m} \mathrm{Cov}(v_i, v_j).$$
Therefore, the variance reduction is given by: %
$$\sum_{i=1}^{t} \sigma_i^2  + 2 \sum_{1\le i < j \le m} \mathrm{Cov}(v_i, v_j) - 2\sum_{t< i <j \le m} \mathrm{Cov}(v_i, v_j), $$

\begin{figure*}[ht]
	\begin{center}
		\includegraphics[width=0.7\linewidth]{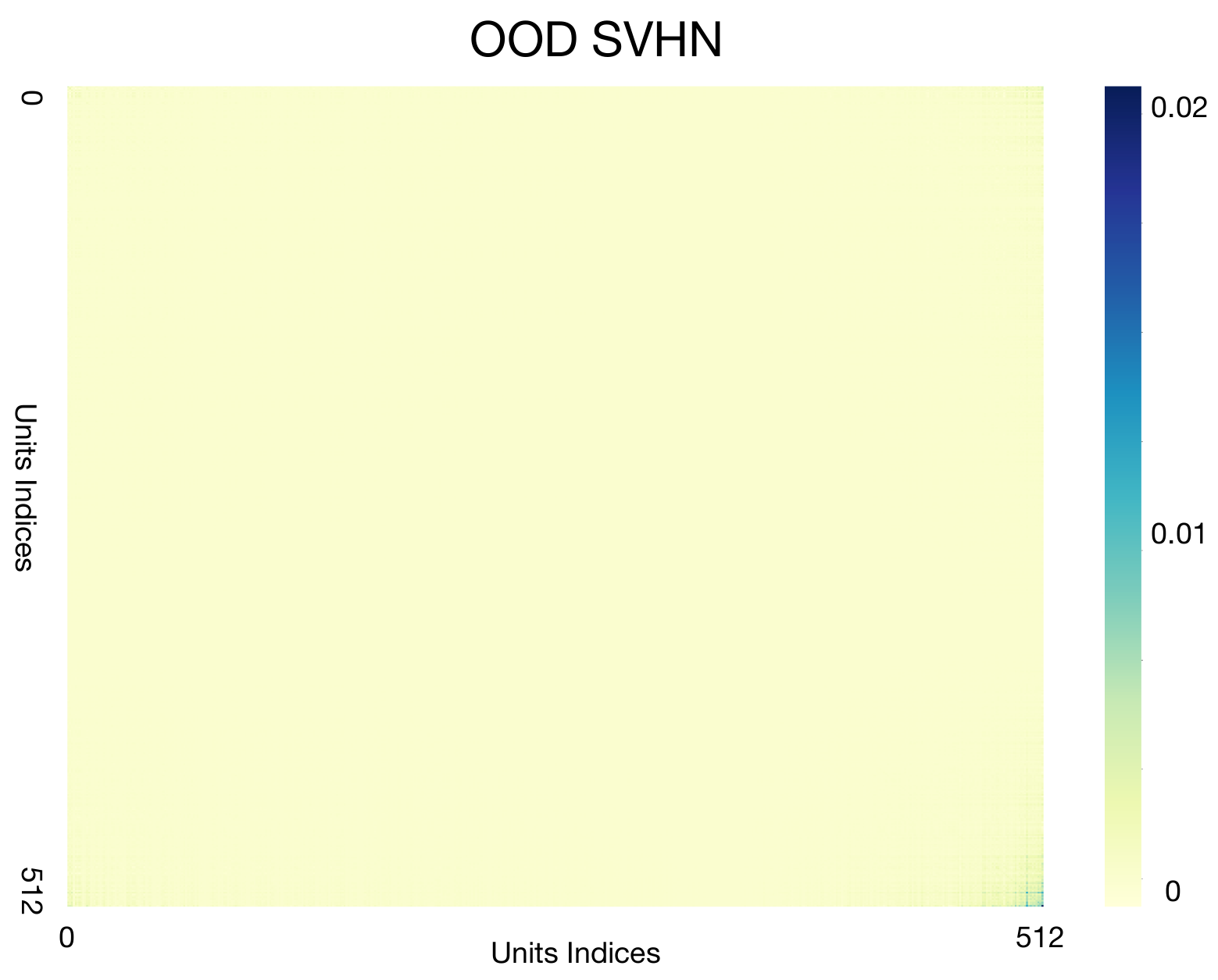}
	\end{center}
	\caption[Covariance matrix of unit contribution estimated on the OOD dataset SVHN.]{Covariance matrix of unit contribution estimated on the OOD dataset SVHN. Model is trained on ID dataset CIFAR-10. The unit indices are sorted from low to high, based on the expectation value of ID's unit contribution (airplane class, same as in Figure~\ref{fig:dice_whytopk}). The matrix primarily consists of elements with 0 value. }
	\label{fig:dice_covmat}
\end{figure*}

We show in Fig.~\ref{fig:dice_covmat} that the covariance matrix of unit contribution $v$ primarily consists of elements of 0, which indicates the independence of variables by large. The covariance matrix is estimated on the CIFAR-10 model with DenseNet-101, which is consistent with our main results in Table~\ref{tab:dice_main-results}. %

Moreover, the summation of non-zero entries in the full matrix (i.e., the second term) is greater than that of the submatrix with top units (i.e., the third term), resulting in a larger variance reduction than in Lemma~\ref{lemma:dicev3}. In the case of OOD data (SVHN), we empirically measure the  variance reduction, where
$\sum_{i=1}^{t} \sigma_i^2  + 2 \sum_{1\le i < j \le m} \mathrm{Cov}(v_i, v_j)$ equals to \textbf{6.8} and $2\sum_{t< i <j \le m} \mathrm{Cov}(v_i, v_j)$ equals to \textbf{2.2}. Therefore, DICE leads to a significant variance reduction effect.

\subsection{Validation Strategy}
\label{sec:dice_val}
We use a validation set of \texttt{Gaussian noise} images, which are generated by sampling from $\mathcal{N}(0,1)$ for each pixel location.
The optimal $p$ is selected from $\{0.1,0.3, 0.5,0.7,0.9,0.99\}$, which is $0.9$ for CIFAR-10/100 and $0.7$ for ImageNet.
We also show in Figure~\ref{fig:dice_sparsity} using Gaussian can already find the near-optimal one averaged over all OOD test datasets considered. 

\subsection{More results on the effect of Sparsity Parameter} 
\label{sec:dice_ood-sparsep}
We characterize  the effect of sparsity parameter $p$ on other ID datasets.  In Table~\ref{tab:dice_k-ablation}, we summarize the OOD detection performance and classification performance for DenseNet trained on CIFAR-10 and ImageNet, where we vary $p=\{0.1, 0.3, 0.5, 0.7, 0.9, 0.99\}$. A similar trend is observed on CIFAR-100 as discussed in the main paper.

\begin{table*}[htb] 
\caption[Effect of varying sparsity parameter $p$ on CIFAR-10 and ImageNet.]{\small Effect of varying sparsity parameter $p$. Results are averaged on the test datasets described in Section~\ref{sec:dice_experiments}.}
\label{tab:dice_k-ablation}
\centering
\footnotesize
\scalebox{0.9}{
\begin{tabular}{l|lll|lll}
\toprule
\multirow{2}{*}{\textbf{Sparsity}}  & \multicolumn{3}{c|}{\textbf{CIFAR-10}} & \multicolumn{3}{c}{\textbf{ImageNet}} \\
  & \textbf{FPR95} $\downarrow$ & \textbf{AUROC} $\uparrow$ & \textbf{Acc.} $\uparrow$ & \textbf{FPR95} $\downarrow$ & \textbf{AUROC} $\uparrow$ & \textbf{Acc.} $\uparrow$ \\ \midrule
$p=0.99$ & 57.57 & 84.29 & 60.81 & 75.79 &  66.07 & 63.28\\
$p=0.9$ & 21.76 & 94.91 & 94.38 & 40.10 & 89.09  & 73.36 \\
$p=0.7$ & 21.76 & 94.91 & 94.35 & 34.75 & 90.77  & 73.82 \\
$p=0.5$ & 21.76 & 94.91 & 94.35  & 34.58 & 90.80 & 73.80\\
$p=0.3$ & 21.75 & 94.91 & 94.35 & 34.70 & 90.69 & 73.57 \\
$p=0.1$ & 21.92 & 94.90 & 94.33 & 40.25 & 89.44 &  73.38 \\ \midrule
$p=0$  & 26.55 & 94.57 & 94.50 & 58.41 & 86.17 & 75.20\\ \bottomrule
\end{tabular}}

\end{table*}

\subsection{Detailed OOD Detection Performance for CIFAR}
\label{sec:dice_detailed-cifar}
We report the detailed performance for all six test OOD datasets for models trained on CIFAR10 and CIFAR-100 respectively in Table~\ref{tab:dice_detail-results-cifar10}
 and Table~\ref{tab:dice_detail-results-cifar100}.

\begin{sidewaystable}
\caption[Detailed results for sparsification baselines on six common OOD benchmark datasets.]{\small Detailed results on six common OOD benchmark datasets: \texttt{Textures}~\citep{cimpoi2014describing}, \texttt{SVHN}~\citep{netzer2011reading}, \texttt{Places365}~\citep{zhou2017places}, \texttt{LSUN-Crop}~\citep{yu2015lsun}, \texttt{LSUN-Resize}~\citep{yu2015lsun}, and \texttt{iSUN}~\citep{xu2015turkergaze}. For each ID dataset, we use the same DenseNet pretrained on \textbf{CIFAR-10}. $\uparrow$ indicates larger values are better and $\downarrow$ indicates smaller values are better.}
\scalebox{0.54}{
\begin{tabular}{llllllllllllllll} \toprule
\multirow{3}{*}{\textbf{Method Type}} & \multirow{3}{*}{\textbf{Method}} & \multicolumn{2}{c}{\textbf{SVHN}} & \multicolumn{2}{c}{\textbf{LSUN-c}} & \multicolumn{2}{c}{\textbf{LSUN-r}} & \multicolumn{2}{c}{\textbf{iSUN}} & \multicolumn{2}{c}{\textbf{Textures}} & \multicolumn{2}{c}{\textbf{Places365}} & \multicolumn{2}{c}{\textbf{Average}} \\ \cline{3-16}
 & & \textbf{FPR95} & \textbf{AUROC} & \textbf{FPR95} & \textbf{AUROC} & \textbf{FPR95} & \textbf{AUROC} & \textbf{FPR95} & \textbf{AUROC} & \textbf{FPR95} & \textbf{AUROC} & \textbf{FPR95} & \textbf{AUROC} & \textbf{FPR95} & \textbf{AUROC} \\ 
 & & $\downarrow$ & $\uparrow$ & $\downarrow$ & $\uparrow$ & $\downarrow$ & $\uparrow$ & $\downarrow$ & $\uparrow$ & $\downarrow$ & $\uparrow$ & $\downarrow$ & $\uparrow$ & $\downarrow$ & $\uparrow$ \\ \midrule
\multirow{5}{*}{Non-Sparse} & MSP  & 47.24 & 93.48 & 33.57 & 95.54 & 42.10 & 94.51 & 42.31 & 94.52 & 64.15 & 88.15 & 63.02 & 88.57 & 48.73 & 92.46 \\
 & ODIN  & 25.29 & 94.57 & 4.70 & 98.86 & 3.09 & 99.02 & 3.98 & 98.90 & 57.50 & 82.38 & 52.85 & 88.55 & 24.57 & 93.71 \\
 & GODIN  & 6.68 & 98.32 & 17.58 & 95.09 & 36.56 & 92.09 & 36.44 & 91.75 & 35.18 & 89.24 & 73.06 & 77.18 & 34.25 & 90.61 \\
 & Mahalanobis  & 6.42 & 98.31 & 56.55 & 86.96 & 9.14 & 97.09 & 9.78 & 97.25 & 21.51 & 92.15 & 85.14 & 63.15 & 31.42 & 89.15 \\
 & Energy  & 40.61 & 93.99 & 3.81 & 99.15 & 9.28 & 98.12 & 10.07 & 98.07 & 56.12 & 86.43 & 39.40 & 91.64 & 26.55 & 94.57 \\ 
 & ReAct  & 41.64 & 93.87 & 5.96 & 98.84 & 11.46 & 97.87 & 12.72 & 97.72 & 43.58 & 92.47 & 43.31 & 91.03 & 26.45 & 94.67 \\
  \midrule
\multirow{5}{*}{Sparse} & Unit-Droput  & 89.16 & 60.96 & 72.97 & 81.33 & 87.03 & 68.78 & 87.29 & 68.07 & 88.53 & 60.10 & 94.82 & 59.18 & 86.63 & 66.40 \\
 & Weight-Droput & 81.34 & 80.03 & 21.06 & 96.15 & 54.70 & 90.33 & 58.88 & 89.80 & 83.34 & 73.31 & 73.42 & 81.10 & 62.12 & 85.12 \\
 & Unit-Pruning & 40.56 & 93.99 & 3.81 & 99.15 & 9.28 & 98.12 & 10.07 & 98.07 & 56.1 & 86.43 & 39.47 & 91.64 & 26.55 & 94.57  \\
 & Weight-Pruning & 28.61 & 95.40 & 3.01 & 99.30 & 8.58 & 98.19 & 9.08 & 98.16 & 49.45 & 88.20 & 46.78 & 89.77 & 24.25 & 94.84 \\
 & {DICE} (ours) 
 & 25.99$^{\pm{5.10}}$ & 95.90$^{\pm{1.08}}$ & 0.26$^{\pm{0.11}}$ & 99.92$^{\pm{0.02}}$ & 3.91$^{\pm{0.56}}$ & 99.20$^{\pm{0.15}}$ & 4.36$^{\pm{0.71}}$ & 99.14$^{\pm{0.15}}$ & 41.90$^{\pm{4.41}}$ & 88.18$^{\pm{1.80}}$ & 48.59$^{\pm{1.53}}$ & 89.13$^{\pm{0.31}}$ & 20.83$^{\pm{1.58}}$ & 95.24$^{\pm{0.24}}$ 
\\ \bottomrule
\end{tabular}}

\label{tab:dice_detail-results-cifar10}
\end{sidewaystable}

\begin{sidewaystable}
\caption[Detailed results for OOD detection baselines on six common OOD benchmark datasets.]{\small Detailed results on six common OOD benchmark datasets: \texttt{Textures}~\citep{cimpoi2014describing}, \texttt{SVHN}~\citep{netzer2011reading}, \texttt{Places365}~\citep{zhou2017places}, \texttt{LSUN-Crop}~\citep{yu2015lsun}, \texttt{LSUN-Resize}~\citep{yu2015lsun}, and \texttt{iSUN}~\citep{xu2015turkergaze}. For each ID dataset, we use the same DenseNet pretrained on \textbf{CIFAR-100}. $\uparrow$ indicates larger values are better and $\downarrow$ indicates smaller values are better. }
\scalebox{0.53}{
\begin{tabular}{llllllllllllllll} \toprule
\multirow{3}{*}{\textbf{Method Type}} & \multirow{3}{*}{\textbf{Method}} & \multicolumn{2}{c}{\textbf{SVHN}} & \multicolumn{2}{c}{\textbf{LSUN-c}} & \multicolumn{2}{c}{\textbf{LSUN-r}} & \multicolumn{2}{c}{\textbf{iSUN}} & \multicolumn{2}{c}{\textbf{Textures}} & \multicolumn{2}{c}{\textbf{Places365}} & \multicolumn{2}{c}{\textbf{Average}} \\ \cline{3-16}
 & & \textbf{FPR95} & \textbf{AUROC} & \textbf{FPR95} & \textbf{AUROC} & \textbf{FPR95} & \textbf{AUROC} & \textbf{FPR95} & \textbf{AUROC} & \textbf{FPR95} & \textbf{AUROC} & \textbf{FPR95} & \textbf{AUROC} & \textbf{FPR95} & \textbf{AUROC} \\ 
 & & $\downarrow$ & $\uparrow$ & $\downarrow$ & $\uparrow$ & $\downarrow$ & $\uparrow$ & $\downarrow$ & $\uparrow$ & $\downarrow$ & $\uparrow$ & $\downarrow$ & $\uparrow$ & $\downarrow$ & $\uparrow$ \\ \midrule
\multirow{5}{*}{\textbf{Non-Sparse}} & MSP & 81.70 & 75.40 & 60.49 & 85.60 & 85.24 & 69.18 & 85.99 & 70.17 & 84.79 & 71.48 & 82.55 & 74.31 & 80.13 & 74.36 \\
 & ODIN & 41.35 & 92.65 & 10.54 & 97.93 & 65.22 & 84.22 & 67.05 & 83.84 & 82.34 & 71.48 & 82.32 & 76.84 & 58.14 & 84.49 \\
 & GODIN & 36.74 & 93.51 & 43.15 & 89.55 & 40.31 & 92.61 & 37.41 & 93.05 & 64.26 & 76.72 & 95.33 & 65.97 & 52.87 & 85.24 \\ 
 & Mahalanobis & 22.44 & 95.67 & 68.90 & 86.30 & 23.07 & 94.20 & 31.38 & 93.21 & 62.39 & 79.39 & 92.66 & 61.39 & 55.37 & 82.73 \\
 & Energy & 87.46 & 81.85 & 14.72 & 97.43 & 70.65 & 80.14 & 74.54 & 78.95 & 84.15 & 71.03 & 79.20 & 77.72 & 68.45 & 81.19 \\ 
 & ReAct & 83.81 & 81.41 & 25.55 & 94.92 & 60.08 & 87.88 & 65.27 & 86.55 & 77.78 & 78.95 & 82.65 & 74.04 & 62.27 & 84.47 \\ \midrule
 \multirow{5}{*}{\textbf{Sparse}} & Unit-Droput & 91.43 & 54.71 & 56.24 & 85.25 & 91.06 & 57.79 & 90.88 & 57.90 & 89.59 & 54.57 & 94.15 & 56.15 & 85.56 & 61.06 \\
 & Weight-Droput & 92.97 & 64.39 & 18.96 & 95.62 & 88.67 & 65.48 & 87.12 & 67.82 & 88.45 & 64.38 & 88.69 & 71.87 & 77.48 & 71.59 \\
 & Unit-Pruning & 87.52 & 81.83 & 14.73 & 97.43 & 70.62 & 80.18 & 74.46 & 79.00 & 84.20 & 71.02 & 79.32 & 77.70 & 68.48 & 81.19 \\
 & Weight-Pruning & 77.99 & 84.14 & 5.17 & 99.05 & 59.42 & 87.13 & 61.80 & 86.09 & 72.68 & 73.85 & 82.53 & 75.06 & 59.93 & 84.22 \\
 & {DICE} (ours) 
  & 54.65$^{\pm{4.94}}$ & 88.84$^{\pm{0.39}}$ & 0.93$^{\pm{0.07}}$ & 99.74$^{\pm{0.01}}$ & 49.40$^{\pm{1.99}}$ & 91.04$^{\pm{1.49}}$ & 48.72$^{\pm{1.55}}$ & 90.08$^{\pm{1.36}}$ & 65.04$^{\pm{0.66}}$ & 76.42$^{\pm{0.35}}$ & 79.58$^{\pm{2.34}}$ & 77.26$^{\pm{1.08}}$ & 49.72$^{\pm{1.69}}$ & 87.23$^{\pm{0.73}}$ 
\\ \bottomrule
\end{tabular}}

\label{tab:dice_detail-results-cifar100}
\end{sidewaystable}

%% file: chapters/supp_knn.tex
%%%%%%%%%%%%%%%%%%%%%%%%%%%%%%%%%%%%%%%%%%%%%%%%%%%%%%%%%%%%%%
%%%%%%%%%%%%%%%%%%%%%%  THEORY SECTION %%%%%%%%%%%%%%%%%%%%%%%
%%%%%%%%%%%%%%%%%%%%%%%%%%%%%%%%%%%%%%%%%%%%%%%%%%%%%%%%%%%%%%

\subsection{Theoretical Analysis}
\label{sec:knn_theory_sup}

\vspace{0.1cm} \noindent \textbf{Proof of Theorem~\ref{th:knn2bayes}.} We now provide the proof sketch for readers to understand the key idea, which revolves around performing the empirical estimation of the probability $\hat{p}(g_i = 1|\*z_i)$. By the Bayesian rule, the probability of $\*z$ being ID data is:
\begin{align*}
     p(g_i = 1|\*z_i) &= 
     \frac{p(\*z_i| g_i = 1) \cdot p(g_i = 1)}{p(\*z_i)} \\ 
     &=\frac{p_{in}(\*z_i)\cdot p(g_i = 1)}
     {p_{in}(\*z_i)\cdot p(g_i = 1) + p_{out}(\*z_i)\cdot p(g_i = 0)} \\
    \hat{p}(g_i = 1|\*z_i) &=   \frac{(1-\varepsilon) \hat p_{in}(\*z_i)}{(1-\varepsilon){\hat p_{in}(\*z_i) + \varepsilon\hat p_{out}(\*z_i)}}.
\end{align*}
Hence, estimating $\hat p(g_i = 1|\*z_i)$ boils down to deriving the empirical estimation of $\hat p_{in}(\*z_i)$ and $\hat p_{out}(\*z_i)$, which we show below respectively.

\vspace{0.1cm} \noindent \textbf{Estimation for $\hat p_{in}(\*z_i)$.}
Recall that $\*z$ is a normalized feature vector in $\mathbb{R}^{m}$. Therefore $\*z$ locates on the surface of a $m$-dimensional unit sphere. 
We denote $B(\*z, r) = \{\*z':\lVert\*z'-\*z\rVert_2 \le r\} \cap \{\lVert\*z'\rVert_2 = 1\}$, which is a set of data points on the unit hyper-sphere and are at most $r$ Euclidean distance away from the center $\*z$. Note that the local dimension of $B(\*z, r)$ is $m-1$. 

Assuming the density satisfies Lebesgue's differentiation theorem, the probability density function can be attained by: 
$$ 
p_{in}(\*z_i) = \lim_{r \rightarrow 0} \frac{p(\*z \in B(\*z_i, r)|g_i = 1)}{|B(\*z_i, r)|}.
$$
In training time, we empirically observe $n$ in-distribution samples $\mathcal{Z}_n = \{\*z'_1, \*z'_2, ..., \*z'_n \}$. We assume each sample $\*z'_j$ is \emph{i.i.d} with a probability mass $\frac{1}{n}$. The empirical point-density for the ID data can be estimated by $k$-NN distance:
\begin{align*}
 \hat{p}_{in}(\*z_i;k, n) &= \frac{p(\*z'_j \in B(\*z_i, r_k(\*z_i))|\*z'_j \in \mathcal{Z}_n)}{|B(\*z_i, r_k(\*z_i))|} \\
 &= \frac{k}{c_bn(r_k(\*z_i))^{m-1}}, 
\end{align*}
where $c_b$ is a constant. The following Lemma~\ref{lemma:l1bound} establishes the convergence rate of the estimator.

\begin{lemma}
\label{lemma:l1bound}
$$\lim_{\frac{k}{n} \rightarrow 0} \hat{p}_{in}(\*z_i;k, n) = p_{in}(\*z_i).$$ 
Specifically,
$$\mathbb{E}[|\hat{p}_{in}(\*z_i;k, n) - p_{in}(\*z_i)|] = o(\sqrt[m-1]{\frac{k}{n}}+ \sqrt{\frac{1}{k}}).$$
\end{lemma}
The proof is given in~\cite{zhao2020analysis}. 

\vspace{0.1cm} \noindent \textbf{Estimation for $\hat p_{out}(\*z_i)$.}
A key challenge in OOD detection is the lack of knowledge on OOD distribution, which can arise universally outside ID data. We thus try to keep our analysis general and reflect the fact that we do not have any strong prior information about OOD. For this reason, we model OOD data with an equal chance to appear outside of the high-density region of ID data. Our theory is thus complementary to our experiments and captures the universality of OOD data. 
Specifically, we denote $$\hat{p}_{out}(\*z_i) = \hat{c}_{0}\mathbf{1}\{\hat{p}_{in}(\*z_i;k, n) < \frac{\beta\varepsilon\hat{c}_{0}}{(1-\beta)(1-\varepsilon)}\}$$
where the threshold is chosen to satisfy the theorem. 

Lastly, our theorem holds by plugging in the empirical estimation of $\hat p_{in}(\*z_i)$ and $\hat p_{out}(\*z_i)$. 

\begin{proof}

\begin{align*}
    &\mathbf{1}\{-r_k(\*z_i) \ge \lambda\} \\
    &~~~~= \mathbf{1}\{\varepsilon c_b n \hat{c}_{0}(r_k(\*z_i))^{m-1} \le \frac{1-\beta}{\beta}(1-\varepsilon)k\} \\
    &~~~~= \mathbf{1}\{\varepsilon c_b n \hat{c}_{0}\mathbf{1}\{\varepsilon c_b n \hat{c}_{0} (r_k(\*z_i))^{m-1} > \frac{1-\beta}{\beta}(1-\varepsilon)k\}(r_k(\*z_i))^{m-1} \le \frac{1-\beta}{\beta}(1-\varepsilon)k\} \\
    &~~~~= \mathbf{1}\{\varepsilon c_b n \hat{c}_{0} \mathbf{1}\{\hat{p}_{in}(\*z_i;k, n) < \frac{\beta\varepsilon\hat{c}_{0}}{(1-\beta)(1-\varepsilon)}\}(r_k(\*z_i))^{m-1} \le \frac{1-\beta}{\beta}(1-\varepsilon)k\} \\
    &~~~~= \mathbf{1}\{\varepsilon c_b n  \hat{p}_{out}(\*z_i)(r_k(\*z_i))^{m-1} \le \frac{1-\beta}{\beta}(1-\varepsilon)k\} \\
    &~~~~= \mathbf{1}\{\frac{k(1 - \varepsilon)}{k(1 - \varepsilon) + \varepsilon c_b n\hat{p}_{out}(\*z_i)(r_k(\*z_i))^{m-1}} \ge \beta\} \\
    &~~~~= \mathbf{1}\{\hat{p}(g_i = 1|\*z_i) \ge \beta\}
\end{align*}

\end{proof}

\subsection{Configurations}
\label{sec:knn_config}

\vspace{0.1cm} \noindent \textbf{Non-parametric methods for anomaly detection.} We provide implementation details of the non-parametric methods in this section. Specifically,

\texttt{\textbf{IForest}}~\citep{liu2008iforest} generates a random forest assuming the test anomaly can be isolated in fewer steps. We use 100 base estimators in the ensemble and each estimator draws 256 samples randomly for training. The number of features to train each base estimator is set to 512.   

\texttt{\textbf{LOF}}~\citep{breunig2000lof} defines an outlier score based on the sample’s $k$-NN distances. We set $k=50$. 

\texttt{\textbf{LODA}}~\citep{2016loda} is an ensemble solution combining multiple weaker binary classifiers. The number of bins for the histogram is set to 10. 

\texttt{\textbf{PCA}}~\citep{shyu2003pca} detects anomaly samples with large values when mapping to the directions with small eigenvalues. We use 50 components for calculating the outlier scores. 

\texttt{\textbf{OCSVM}}~\citep{bernhard2001ocsvm} learns a decision
boundary that corresponds to the desired density level set
of with the kernel function. We use the RBF kernel with $\gamma=\frac{1}{512}$. The upper bound on the fraction of training error is set to 0.5. 

Some of these methods~\citep{bernhard2001ocsvm, shyu2003pca} are specifically designed for anomaly detection scenarios that assume ID data is from one class. We show that $k$-NN distance with the class-aware embeddings can achieve both OOD detection and multi-class classification tasks.

\subsection{Results on Different Architecture}
\label{sec:knn_other_arc}
In the main paper, we have shown that the nearest neighbor approach is competitive on ResNet. In this section, we show in Table~\ref{tab:knn_other-arc} that KNN's strong performance holds on different network architectures DenseNet-101~\citep{huang2017densely}. All the numbers reported are averaged over OOD test datasets described in Section~\ref{sec:knn_common_benchmark}. %

\begin{table*}[b] 
\centering
\vspace{-0.3cm}
\caption[Results on CIFAR-10/100 with DenseNet-101.]{\small \textbf{Comparison results with DenseNet-101.} Comparison with competitive out-of-distribution detection methods. All methods are based on a model trained on {ID data only}.  All values are percentages and are averaged over all OOD test datasets.}
\footnotesize
\scalebox{0.9}{
\begin{tabular}{l|lll|lll}
\toprule
\multirow{2}{*}{\textbf{Method}}  & \multicolumn{3}{c|}{\textbf{CIFAR-10}} & \multicolumn{3}{c}{\textbf{CIFAR-100}} \\
  & \textbf{FPR95} $\downarrow$ & \textbf{AUROC} $\uparrow$ & \textbf{ID ACC} $\uparrow$ & \textbf{FPR95} $\downarrow$ & \textbf{AUROC} $\uparrow$ & \textbf{ID ACC} $\uparrow$ \\ \midrule
MSP & 49.95 & 92.05 & 94.38 & 79.10 & 75.39 & 75.08 \\ 
Energy & 30.16 & 92.44 & 94.38 & 68.03 & 81.40 & 75.08 \\ 
ODIN & 30.02 & 93.86 & 94.38 & 55.96 & 85.16 & 75.08 \\ 
Mahalanobis & 35.88 & 87.56 & 94.38 & 74.57 & 66.03 & 75.08 \\ 
GODIN & 28.98 & 92.48 & 94.22 & 55.38 & 83.76 & 74.50 \\ 
CSI & 70.97 & 78.42 & 93.49 & 79.13 & 60.41 & 68.48 \\ 
SSD+ & 16.21 & 96.96 &  \textbf{94.45} &  43.44 & 88.97  &  \textbf{75.21}\\ 
KNN+ & \textbf{12.16} & \textbf{97.58} & \textbf{94.45}  & \textbf{37.27} & \textbf{89.63} & \textbf{75.21} \\
  \bottomrule
\end{tabular}}
        \vspace{-0.2cm}
        \label{tab:knn_other-arc}
\end{table*}

%% file: chapters/supp_nscl.tex
\subsection{Proof Details for Section~\ref{sec:nscl_method}}
\label{sec:nscl_sup_method_proof}

\subsubsection{Bound Linear Probing Error by Regression Residual}
\label{sec:nscl_sup_cls_bound}

\begin{lemma} (Recap of Lemma~\ref{lemma:nscl_cls_bound})
Denote by $\mathbf{y}(x) \in \mathbb{R}^{|\mathcal{Y}_u|}$  a one-hot vector, whose $y(x)$-th position is 1 and 0 elsewhere. Let
$\mathbf{Y} \in \mathbb{R}^{N_u \times |\mathcal{Y}_u|}$ be a matrix whose rows are stacked by $\mathbf{y}(x)$. We have: 
    $$\mathcal{R}(U^*) \triangleq \underset{{\*M}\in \mathbb{R}^{k \times |\mathcal{Y}_u|}}{\operatorname{min}} \|\mathbf{Y} - U^* \*M \|^2_F \geq \frac{1}{2}\mathcal{E}(f) $$
\label{lemma_sup:cls_bound}
\end{lemma}
\begin{proof}
    Suppose $\Tilde{f}(x) = \sqrt{w_x}f(x)$, we first show that 
\begin{align*}
\| \mathbf{y}(x) - \Tilde{f}(x)^{\top} \*M \|^2  &\geq \frac{1}{2} \mathbbm{1}\left[y(x) \neq h(x;\Tilde{f}, M)\right] 
\end{align*}
If $y(x) = h(x;\Tilde{f}, M)$, it is clear that $\| \mathbf{y}(x) - \Tilde{f}(x)^{\top} \*M \|^2 \geq 0$. If $y(x) \neq h(x;\Tilde{f}, M)$, then there exists another index $y' \neq y(x)$ so that $\Tilde{f}(x)^{\top} \Vec{\mu}_{y'} \geq \Tilde{f}(x)^{\top} \Vec{\mu}_{y(x)}$. Then, 
\begin{align*}
\| \mathbf{y}(x) - \Tilde{f}(x)^{\top} \*M \|_2^2 &\geq (1 - \Tilde{f}(x)^{\top} \Vec{\mu}_{y(x)})^2 + (\Tilde{f}(x)^{\top} \Vec{\mu}_{y'})^2 
\\ &\geq \frac{1}{2} (1 - \Tilde{f}(x)^{\top} \Vec{\mu}_{y(x)} + \Tilde{f}(x)^{\top} \Vec{\mu}_{y'})^2
\\ &\geq \frac{1}{2},
\end{align*}
where the first inequality is by only keeping $y'$-th and $y(x)$-th terms in the $l_2$ norm. We can then prove the lemma by: 
\begin{align*}
    \mathcal{R}(U^*) &= \underset{{\*M}\in \mathbb{R}^{k \times |\mathcal{Y}_u|}}{\operatorname{min}} \|\mathbf{Y} - U^* \*M \|^2_F    \\
    &= \underset{{\*M}\in \mathbb{R}^{k \times |\mathcal{Y}_u|}}{\operatorname{min}} 
    \underset{{x \in \mathcal{X}_u}}{\sum}
    \| \mathbf{y}(x) - \sqrt{w_x}f(x)^{\top} \Sigma_k^{-\frac{1}{2}}\*M \|^2     \\ 
    &= \underset{{\*M}\in \mathbb{R}^{k \times |\mathcal{Y}_u|}}{\operatorname{min}} 
    \underset{{x \in \mathcal{X}_u}}{\sum}
    \| \mathbf{y}(x) - \sqrt{w_x}f(x)^{\top} \*M \|^2     \\ 
    &\geq \frac{1}{2} \underset{{\*M}\in \mathbb{R}^{k \times |\mathcal{Y}_u|}}{\operatorname{min}} 
    \underset{{x \in \mathcal{X}_u}}{\sum} \mathbbm{1}\left[y(x) \neq h(x;\Tilde{f}, M)\right] \\
    &= \frac{1}{2}\mathcal{E}(f),
\end{align*}
 where the second equation is given by $F^* \Sigma_k^{-\frac{1}{2}} = V_k$, and $U^*$ is the last $N_u$ rows of $V_k$, and the last equation is based on the fact that multiplying a scalar value on the output does not change the prediction result ($h(x;f, \*M) = h(x;\Tilde{f}, \*M)$). 
\end{proof}

\subsubsection{Spectral Contrastive Loss}
\label{sec:nscl_proof-nscl}
\begin{theorem}
\label{th:nscl_sup-ncd-scl} (Recap of Theorem~\ref{th:nscl_ncd-scl}) 
We define $\*f_x = \sqrt{w_x}f(x)$ for some function $f$. Recall $\alpha,\beta$ is a hyper-parameter defined in Eq.~\eqref{eq:nscl_def_wxx}. Then minimizing the loss function $\mathcal{L}_{\mathrm{mf}}(F, A)$ is equivalent to minimizing the following loss function for $f$, which we term \textbf{NCD Spectral Contrastive Loss (NSCL)}:
\begin{align}
\begin{split}
    \mathcal{L}_{nscl}(f) &\triangleq - 2\alpha \mathcal{L}_1(f) 
- 2\beta  \mathcal{L}_2(f) \\ & + \alpha^2 \mathcal{L}_3(f) + 2\alpha \beta \mathcal{L}_4(f) +  
\beta^2 \mathcal{L}_5(f),
\label{eq:nscl_sup_def_nscl}
\end{split}
\end{align}
where 
\begin{align*}
    \mathcal{L}_1(f) &= \sum_{i \in \mathcal{Y}_l}\underset{\substack{\bar{x}_{l} \sim \mathcal{P}_{{l_i}}, \bar{x}'_{l} \sim \mathcal{P}_{{l_i}},\\x \sim \mathcal{T}(\cdot|\bar{x}_{l}), x^{+} \sim \mathcal{T}(\cdot|\bar{x}'_l)}}{\mathbb{E}}\left[f(x)^{\top} {f}\left(x^{+}\right)\right], \\ 
    \mathcal{L}_2(f) &= \underset{\substack{\bar{x}_{u} \sim \mathcal{P}_{u},\\x \sim \mathcal{T}(\cdot|\bar{x}_{u}), x^{+} \sim \mathcal{T}(\cdot|\bar{x}_u)}}{\mathbb{E}}
\left[f(x)^{\top} {f}\left(x^{+}\right)\right], \\
    \mathcal{L}_3(f) &= \sum_{i \in \mathcal{Y}_l}\sum_{j \in \mathcal{Y}_l}\underset{\substack{\bar{x}_l \sim \mathcal{P}_{{l_i}}, \bar{x}'_l \sim \mathcal{P}_{{l_j}},\\x \sim \mathcal{T}(\cdot|\bar{x}_l), x^{-} \sim \mathcal{T}(\cdot|\bar{x}'_l)}}{\mathbb{E}}
\left[\left(f(x)^{\top} {f}\left(x^{-}\right)\right)^2\right], \\ 
    \mathcal{L}_4(f) &= \sum_{i \in \mathcal{Y}_l}\underset{\substack{\bar{x}_l \sim \mathcal{P}_{{l_i}}, \bar{x}_u \sim \mathcal{P}_{u},\\x \sim \mathcal{T}(\cdot|\bar{x}_l), x^{-} \sim \mathcal{T}(\cdot|\bar{x}_u)}}{\mathbb{E}}
\left[\left(f(x)^{\top} {f}\left(x^{-}\right)\right)^2\right], \\
    \mathcal{L}_5(f) &= \underset{\substack{\bar{x}_u \sim \mathcal{P}_{u}, \bar{x}'_u \sim \mathcal{P}_{u},\\x \sim \mathcal{T}(\cdot|\bar{x}_u), x^{-} \sim \mathcal{T}(\cdot|\bar{x}'_u)}}{\mathbb{E}}
\left[\left(f(x)^{\top} {f}\left(x^{-}\right)\right)^2\right].
\end{align*}
\end{theorem}
\begin{proof} We can expand $\mathcal{L}_{\mathrm{mf}}(F, A)$ and obtain
\begin{align*}
\mathcal{L}_{\mathrm{mf}}(F, A) &=\sum_{x, x^{\prime} \in \mathcal{X}}\left(\frac{w_{x x^{\prime}}}{\sqrt{w_x w_{x^{\prime}}}}-\*f_x^{\top} \*f_{x^{\prime}}\right)^2 \\
&= \text{const} + \sum_{x, x^{\prime} \in \mathcal{X}}\left(-2 w_{x x^{\prime}} f(x)^{\top} {f}\left(x^{\prime}\right)+w_x w_{x^{\prime}}\left(f(x)^{\top}{f}\left(x^{\prime}\right)\right)^2\right),
\end{align*} 
where $\*f_x = \sqrt{w_x}f(x)$ is a re-scaled version of $f(x)$.
At a high level we follow the proof in ~\citep{haochen2021provable}, while the specific form of loss varies with the different definitions of positive/negative pairs. The form of $\mathcal{L}_{nscl}(f)$ is derived from plugging $w_{xx'}$  and $w_x$. 

Recall that $w_{xx'}$ is defined by
\begin{align*}
w_{x x^{\prime}} &= \alpha \sum_{i \in \mathcal{Y}_l}\mathbb{E}_{\bar{x}_{l} \sim {\mathcal{P}_{l_i}}} \mathbb{E}_{\bar{x}'_{l} \sim {\mathcal{P}_{l_i}}} \mathcal{T}(x | \bar{x}_{l}) \mathcal{T}\left(x' | \bar{x}'_{l}\right)+ \beta \mathbb{E}_{\bar{x}_{u} \sim {\mathcal{P}_u}} \mathcal{T}(x| \bar{x}_{u}) \mathcal{T}\left(x'| \bar{x}_{u}\right) ,
\end{align*}
and $w_{x}$ is given by 
\begin{align*}
w_{x } &= \sum_{x^{\prime}}w_{xx'} \\ &=\alpha \sum_{i \in \mathcal{Y}_l}\mathbb{E}_{\bar{x}_{l} \sim {\mathcal{P}_{l_i}}} \mathbb{E}_{\bar{x}'_{l} \sim {\mathcal{P}_{l_i}}} \mathcal{T}(x | \bar{x}_{l}) \sum_{x^{\prime}} \mathcal{T}\left(x' | \bar{x}'_{l}\right)+ \beta \mathbb{E}_{\bar{x}_{u} \sim {\mathcal{P}_u}} \mathcal{T}(x| \bar{x}_{u}) \sum_{x^{\prime}} \mathcal{T}\left(x' | \bar{x}_{u}\right) \\
&= \alpha \sum_{i \in \mathcal{Y}_l}\mathbb{E}_{\bar{x}_{l} \sim {\mathcal{P}_{l_i}}} \mathcal{T}(x | \bar{x}_{l}) + \beta \mathbb{E}_{\bar{x}_{u} \sim {\mathcal{P}_u}} \mathcal{T}(x| \bar{x}_{u}). 
\end{align*}

Plugging $w_{x x^{\prime}}$ we have, 

\begin{align*}
    -2 \sum_{x, x^{\prime} \in \mathcal{X}} w_{x x^{\prime}} f(x)^{\top} {f}\left(x^{\prime}\right) &= -2 \sum_{x, x^{+} \in \mathcal{X}} w_{x x^{+}} f(x)^{\top} {f}\left(x^{+}\right) 
    \\ &= -2\alpha  \sum_{i \in \mathcal{Y}_l}\mathbb{E}_{\bar{x}_{l} \sim {\mathcal{P}_{l_i}}} \mathbb{E}_{\bar{x}'_{l} \sim {\mathcal{P}_{l_i}}} \sum_{x, x^{\prime} \in \mathcal{X}} \mathcal{T}(x | \bar{x}_{l}) \mathcal{T}\left(x' | \bar{x}'_{l}\right) f(x)^{\top} {f}\left(x^{\prime}\right) 
    \\ &~~~~ -2 \beta \mathbb{E}_{\bar{x}_{u} \sim {\mathcal{P}_u}} \sum_{x, x^{\prime}} \mathcal{T}(x| \bar{x}_{u}) \mathcal{T}\left(x'| \bar{x}_{u}\right) f(x)^{\top} {f}\left(x^{\prime}\right) 
    \\ &= -2\alpha  \sum_{i \in \mathcal{Y}_l}\underset{\substack{\bar{x}_{l} \sim \mathcal{P}_{{l_i}}, \bar{x}'_{l} \sim \mathcal{P}_{{l_i}},\\x \sim \mathcal{T}(\cdot|\bar{x}_{l}), x^{+} \sim \mathcal{T}(\cdot|\bar{x}'_l)}}{\mathbb{E}}  \left[f(x)^{\top} {f}\left(x^{+}\right)\right] 
    \\ &~~~~ - 2\beta 
    \underset{\substack{\bar{x}_{u} \sim \mathcal{P}_{u},\\x \sim \mathcal{T}(\cdot|\bar{x}_{u}), x^{+} \sim \mathcal{T}(\cdot|\bar{x}_u)}}{\mathbb{E}}
\left[f(x)^{\top} {f}\left(x^{+}\right)\right] 
\\ &= - 2\alpha  \mathcal{L}_1(f) 
- 2\beta  \mathcal{L}_2(f)
\end{align*}

Plugging $w_{x}$ and $w_{x'}$ we have, 

\begin{align*}
    &\sum_{x, x^{\prime} \in \mathcal{X}}w_x w_{x^{\prime}}\left(f(x)^{\top}{f}\left(x^{\prime}\right)\right)^2 = \sum_{x, x^{-} \in \mathcal{X}}w_x w_{x^{-}}\left(f(x)^{\top}{f}\left(x^{-}\right)\right)^2 \\
    &=\sum_{x, x^{\prime} \in \mathcal{X}} \left( \alpha \sum_{i \in \mathcal{Y}_l}\mathbb{E}_{\bar{x}_{l} \sim {\mathcal{P}_{l_i}}} \mathcal{T}(x | \bar{x}_{l}) + \beta \mathbb{E}_{\bar{x}_{u} \sim {\mathcal{P}_u}} \mathcal{T}(x| \bar{x}_{u}) \right) \cdot \\
    &~~~~~~~~~~~~~~~~\left(\alpha  \sum_{j \in \mathcal{Y}_l}\mathbb{E}_{\bar{x}'_{l} \sim {\mathcal{P}_{l_j}}} \mathcal{T}(x^{-} | \bar{x}'_{l}) + \beta \mathbb{E}_{\bar{x}'_{u} \sim {\mathcal{P}_u}} \mathcal{T}(x^{-}| \bar{x}'_{u}) \right) \left(f(x)^{\top}{f}\left(x^{-}\right)\right)^2 \\
    &= \alpha ^ 2 \sum_{x, x^{-} \in \mathcal{X}}  \sum_{i \in \mathcal{Y}_l}\mathbb{E}_{\bar{x}_{l} \sim {\mathcal{P}_{l_i}}} \mathcal{T}(x | \bar{x}_{l}) \sum_{j \in \mathcal{Y}_l}\mathbb{E}_{\bar{x}'_{l} \sim {\mathcal{P}_{l_j}}} \mathcal{T}(x^{-} | \bar{x}'_{l})\left(f(x)^{\top}{f}\left(x^{-}\right)\right)^2 \\
    &~~~~+ 2\alpha \beta \sum_{x, x^{-} \in \mathcal{X}} \sum_{i \in \mathcal{Y}_l}\mathbb{E}_{\bar{x}_{l} \sim {\mathcal{P}_{l_i}}} \mathcal{T}(x | \bar{x}_{l})  \mathbb{E}_{\bar{x}_{u} \sim {\mathcal{P}_u}} \mathcal{T}(x^{-}| \bar{x}_{u}) \left(f(x)^{\top}{f}\left(x^{-}\right)\right)^2  \\
    &~~~~+ \beta^2 \sum_{x, x^{-} \in \mathcal{X}}  \mathbb{E}_{\bar{x}_{u} \sim {\mathcal{P}_u}} \mathcal{T}(x| \bar{x}_{u}) \mathbb{E}_{\bar{x}'_{u} \sim {\mathcal{P}_u}} \mathcal{T}(x^{-}| \bar{x}'_{u}) \left(f(x)^{\top}{f}\left(x^{-}\right)\right)^2 \\
    &= \alpha^2 \sum_{i \in \mathcal{Y}_l}\sum_{j \in \mathcal{Y}_l}\underset{\substack{\bar{x}_l \sim \mathcal{P}_{{l_i}}, \bar{x}'_l \sim \mathcal{P}_{{l_j}},\\x \sim \mathcal{T}(\cdot|\bar{x}_l), x^{-} \sim \mathcal{T}(\cdot|\bar{x}'_l)}}{\mathbb{E}}
\left[\left(f(x)^{\top} {f}\left(x^{-}\right)\right)^2\right] \\
    &~~~~+ 2\alpha\beta
    \sum_{i \in \mathcal{Y}_l}\underset{\substack{\bar{x}_l \sim \mathcal{P}_{{l_i}}, \bar{x}_u \sim \mathcal{P}_{u},\\x \sim \mathcal{T}(\cdot|\bar{x}_l), x^{-} \sim \mathcal{T}(\cdot|\bar{x}_u)}}{\mathbb{E}}
\left[\left(f(x)^{\top} {f}\left(x^{-}\right)\right)^2\right] \\    &~~~~+ \beta^2
     \underset{\substack{\bar{x}_u \sim \mathcal{P}_{u}, \bar{x}'_u \sim \mathcal{P}_{u},\\x \sim \mathcal{T}(\cdot|\bar{x}_u), x^{-} \sim \mathcal{T}(\cdot|\bar{x}'_u)}}{\mathbb{E}}
\left[\left(f(x)^{\top} {f}\left(x^{-}\right)\right)^2\right] 
\\& = \alpha^2 \mathcal{L}_3(f) + 2\alpha\beta \mathcal{L}_4(f) + \beta^2\mathcal{L}_5(f).
\end{align*}
\end{proof}

\newpage
\subsection{Proof for Eigenvalue in Toy Example }
\label{sec:nscl_proof-eigen}
Before we present the proof of Theorem~\ref{th:nscl_toy_extreme}, Theorem\ref{th:nscl_toy_general} and Lemma~\ref{th:nscl_toy_harmful}, we first present the following lemma~\ref{lemma:nscl_sup_eigprob} which extensively explore the order and the form of eigenvectors of the general form $T(t)$. 
Note that $T_1$ and $T_2$ are special cases of the following $T(t)$ with $t \in [\tau_0, \tau_c]$:

\begin{equation*}
T(t)=\left[\begin{array}{ccccc}
\tau_1 & t & t & \tau_0 & \tau_0 \\
t & \tau_1 & \tau_{c} & \tau_{s} & \tau_0  \\
t & \tau_{c} & \tau_1 & \tau_0 & \tau_{s}  \\
\tau_0 & \tau_{s} & \tau_0 & \tau_1 & \tau_{c}  \\
\tau_0 & \tau_0 & \tau_{s} & \tau_{c} & \tau_1  \\
\end{array}\right],
\end{equation*}
where $t$ indicates the strength of the connection between labeled data and a novel class in unlabeled data.

\begin{lemma}
    Assume $\tau_1 = 1$, $\tau_0 = 0$, $\tau_c < \tau_s < 1.5\tau_c$, $\Bar{t} = \sqrt{\frac{2(\tau_s-\tau_c)^2\tau_c}{2\tau_c - \tau_s}}$, let $a(\lambda)  = {\lambda-1 \over 2t}$ and $b(\lambda)  = {\tau_{s}(\lambda-1) \over 2(\lambda -1 -\tau_{c})t}$ are real value functions, the matrix $T(t)$'s eigenvectors (not necessarily $l_2$-normalized) and its eigenvalues are the following:
    
    (Case 1): If  $t \in (\Bar{t}, \tau_c]$,  
    \begin{align*}
    \begin{array}{ll}
        v_1 = [1, a(\lambda_1), a(\lambda_1), b(\lambda_1), b(\lambda_1)]^{\top}, &\lambda_1 > 1 + \tau_{s} + \tau_{c}, \\
        v_2 = [1, a(\lambda_2), a(\lambda_2), b(\lambda_2), b(\lambda_2)]^{\top}, &\lambda_2 \in [1 + \tau_{s} - \tau_{c}, 1 + \tau_{c}) \\
        v_3 = [0, -1, 1, -1, 1]^{\top}, &\lambda_3 = 1 + \tau_{s} - \tau_{c}, \\
        v_4 = [1, a(\lambda_4), a(\lambda_4), b(\lambda_4), b(\lambda_4)]^{\top}, &\lambda_4 \in (1 - \tau_{s} - \tau_{c},  1) \\
        v_5 = [0, 1, -1, -1, 1]^{\top}, &\lambda_5 = 1 - \tau_{s} - \tau_{c}, \\
    \end{array}
    \end{align*}
    
    (Case 2): If  $t \in (0, \Bar{t})$, 
        \begin{align*}
    \begin{array}{ll}
        v_1 = [1, a(\lambda_1), a(\lambda_1), b(\lambda_1), b(\lambda_1)]^{\top}, &\lambda_1 > 1 + \tau_{s} + \tau_{c}, \\
        v_2 = [0, -1, 1, -1, 1]^{\top}, &\lambda_2 = 1 + \tau_{s} - \tau_{c}, \\
        v_3 = [1, a(\lambda_3), a(\lambda_3), b(\lambda_3), b(\lambda_3)]^{\top}, &\lambda_3 \in [1, 1 + \tau_{s} - \tau_{c}) \\
        v_4 = [1, a(\lambda_4), a(\lambda_4), b(\lambda_4), b(\lambda_4)]^{\top}, &\lambda_4 \in (1 - \tau_{s} - \tau_{c},1) \\
        v_5 = [0, 1, -1, -1, 1]^{\top}, &\lambda_5 = 1 - \tau_{s} - \tau_{c}, \\
    \end{array}
    \end{align*}

        (Case 3): If  $t = 0$, 
        \begin{align*}
    \begin{array}{ll}
        v_1 = [0, 1, 1, 1, 1]^{\top}, &\lambda_1 = 1 + \tau_{s} + \tau_{c}, \\
        v_2 = [0, -1, 1, -1, 1]^{\top}, &\lambda_2 = 1 + \tau_{s} - \tau_{c}, \\
        v_3 = [1, 0, 0, 0, 0]^{\top}, &\lambda_3 = 1 \\
        v_4 = [0, 1, 1, -1, -1]^{\top}, &\lambda_4 = 1 - \tau_{s} + \tau_{c} \\
        v_5 = [0, 1, -1, -1, 1]^{\top}, &\lambda_5 = 1 - \tau_{s} - \tau_{c}, \\
    \end{array}
    \end{align*}
    \label{lemma:nscl_sup_eigprob}
\end{lemma}

\begin{proof}
For $t=0$, Case 3, we can verify by direct calculation. 

Now for Case 1 and Case 2, we consider $t\in (0, \tau_{c})$. For  any $i \in [5]$, denote $\hat{\lambda}_i$ as unordered eigenvalue and $\hat{v_i}$ is its corresponding eigenvector.
We can direct verify that
\begin{align}
    \hat{\lambda}_1 =& 1 + \tau_{s} - \tau_{c}\\
    \hat{\lambda}_2 = & 1 - \tau_{s} - \tau_{c},
\end{align}
are two eigenvalues of $\Tilde{A}_t$ and 
\begin{align}
    \hat{v}_1 =& [0, -1, 1, -1, 1]^{\top}\\
    \hat{v}_2 = &[0, 1, -1, -1, 1]^{\top}, 
\end{align}
are two corresponding eigenvectors. Now, we prove for $i \in \{3,4,5\}$,  $\hat{v}_i = [1, a(\hat{\lambda}_i), a(\hat{\lambda}_i), b(\hat{\lambda}_i), b(\hat{\lambda}_i)]^{\top}$ are eigenvector for $\hat{\lambda}_i$.
For $i \in \{3,4,5\}$ we only need to show
\begin{align}
    \begin{cases}
    1+2t a(\hat{\lambda}_i) & = \hat{\lambda}_i \\
    t + (1+\tau_{c})a(\hat{\lambda}_i) + \tau_{s}b(\hat{\lambda}_i) & = \hat{\lambda}_i a(\hat{\lambda}_i) \\
    \tau_{s}a(\hat{\lambda}_i) + (1+\tau_{c})b(\hat{\lambda}_i) & = \hat{\lambda}_i b(\hat{\lambda}_i). 
    \end{cases}
\end{align}
Equivalently to
\begin{align}
    \begin{cases}
    1+2t a(\hat{\lambda}_i) -\hat{\lambda}_i & = 0 \\
    t + (1+\tau_{c}+\tau_{s}-\hat{\lambda}_i)(a(\hat{\lambda}_i)d + b(\hat{\lambda}_i)) & = 0 \\
    t + (1+\tau_{c}-\tau_{s}-\hat{\lambda}_i)(a(\hat{\lambda}_i) - b(\hat{\lambda}_i)) & = 0 .
    \end{cases}
\end{align}
Let $z_i = \hat{\lambda}_i - 1$. Equivalently to
\begin{align}
    \begin{cases}
    1+2t a(\hat{\lambda}_i) -\hat{\lambda}_i & = 0 \\
    (\hat{\lambda}_i -1 -\tau_{c})b(\hat{\lambda}_i) -\tau_{s}a(\hat{\lambda}_i)  & = 0\\
    z_i^3-2\tau_{c}z_i^2+(\tau_{c}^2-\tau_{s}^2-2t^2)z_i+2\tau_{c}t^2 &=0 .
    \end{cases}
\end{align}
Let $g(z) = z^3-2\tau_{c}z^2+(\tau_{c}^2-\tau_{s}^2-2t^2)z+2\tau_{c}t^2$, we can verify that $g(-\infty)<0, ~~ g(-\tau_{c}-\tau_{s}) = -4\tau_{c}(\tau_{c}+\tau_{s})^2+4t^2\tau_{c}+2t^2\tau_{s} < 0, ~~ g(0) = 2\tau_{c}t^2 > 0, ~~ g(\tau_{c}) = -\tau_{s}^2\tau_{c}<0 , ~~g(\tau_{c}+\tau_{s}) = -2\tau_{s}t^2 < 0 , ~~g(+\infty)>0$. Thus, we have three solutions and satisfying $1-\tau_{c}-\tau_{s} < \hat{\lambda}_5 < 1 < \hat{\lambda}_4 < 1+\tau_{c}< 1+\tau_{c}+\tau_{s}<\hat{\lambda}_3$. As $\hat{\lambda}_i \neq 1+\tau_{c}$ for $i \in \{3,4,5\}$, thus,  equivalently to
\begin{align}
    \begin{cases}
    a(\hat{\lambda}_i)  & = {\hat{\lambda}_i-1 \over 2t} \\
    b(\hat{\lambda}_i)  & = {\tau_{s}(\hat{\lambda}_i-1) \over 2(\hat{\lambda}_i -1 -\tau_{c})t}\\
    (\hat{\lambda}_i-1)^3-2\tau_{c}(\hat{\lambda}_i-1)^2+(\tau_{c}^2-\tau_{s}^2-2t^2)(\hat{\lambda}_i-1)+2\tau_{c}t^2 &=0 .
    \end{cases}
\end{align}

When $t > \Bar{t}$, we have $g(\tau_{s} - \tau_{c})>0$. Thus, we have $1-\tau_{c}-\tau_{s} < \hat{\lambda}_5 < 1+\tau_{s} - \tau_{c} < \hat{\lambda}_4 < 1+\tau_{c}+\tau_{s}<\hat{\lambda}_3$. By reorder, we finish Case 1.

When $t < \Bar{t}$, we have $g(\tau_{s} - \tau_{c})<0$. Thus, we have $1-\tau_{c}-\tau_{s} < \hat{\lambda}_5 < 1 < \hat{\lambda}_4 < 1+\tau_{s} - \tau_{c}< 1+\tau_{c}+\tau_{s}<\hat{\lambda}_3$. By reorder the eigenvectors w.r.t the size of eigenvalues, we finish Case 2.
\end{proof}

\begin{theorem}
   (Recap of Theorem~\ref{th:nscl_toy_extreme}) Assume $\tau_1 = 1$, $\tau_0 = 0$, $\tau_s < 1.5\tau_c$. We have
$$
U^*_1=\left[\begin{array}{ccccc}
 a_1 & a_1 & b_1 & b_1 \\
 a_2 & a_2 & b_2 & b_2 \\
\end{array}\right]^{\top},
$$
where $a_1,b_1$ are some positive real numbers, and $a_2,b_2$ has different signs. 
$$U^*_2=\left\{\begin{array}{ll}   
\frac{1}{2}\left[\begin{array}{cccc}
1 & 1 & 1 & 1\\ 1 & 1 & -1 & -1\\
\end{array}\right]^{\top}, & \text{if } \tau_{s} < \tau_{c}, \\
\frac{1}{2}\left[\begin{array}{cccc}
1 & 1 & 1 & 1\\ -1 & 1 & -1 & 1\\
\end{array}\right]^{\top}, & \text{if } \tau_{s} > \tau_{c}, 
\end{array}\right. $$
With label vector $\Vec{y} = \{1,1,0,0\}$, we have 
\begin{equation}
    \mathcal{R}(U^*_1, \Vec{y}) = 0, \mathcal{R}(U^*_2, \Vec{y}) = \left\{\begin{array}{ll}    
    0, & \text{if } \tau_{s} < \tau_{c}\\
     1, &  \text{if } \tau_{s} > \tau_{c}.
    \end{array}\right. 
\end{equation}
    \label{th:nscl_sup_toy_extreme}
\end{theorem}

\begin{proof}
    In the Case 1 and Case 3 of Lemma~\ref{lemma:nscl_sup_eigprob}, we have shown the $U^*_1$ and $U^*_2$ case when $\tau_s > \tau_c$ respectively. In this proof, we just need to show the case when $\tau_s < \tau_c$. For $U^*_2$ and $\tau_s < \tau_c$,  since  $t = 0$, we can directly prove by giving the eigenvectors with order: 
        \begin{align*}
    \begin{array}{ll}
        v_1 = [0, 1, 1, 1, 1]^{\top}, &\lambda_1 = 1 + \tau_{s} + \tau_{c}, \\
        v_2 = [0, 1, 1, -1, -1]^{\top}, &\lambda_2 = 1 - \tau_{s} + \tau_{c} \\
        v_3 = [1, 0, 0, 0, 0]^{\top}, &\lambda_3 = 1 \\
        v_4 = [0, -1, 1, -1, 1]^{\top}, &\lambda_4 = 1 + \tau_{s} - \tau_{c}, \\
        v_5 = [0, 1, -1, -1, 1]^{\top}, &\lambda_5 = 1 - \tau_{s} - \tau_{c}, \\
    \end{array}
    \end{align*}
    
    For $U_1^*$, one can see that in the Case 1 of Lemma~\ref{lemma:nscl_sup_eigprob}, we still have $\lambda_2 > \lambda_3$ since $\tau_s<1.5\tau_c<2\tau_c$ holds. Therefore the order of $v_2$ and $v_3$ does not change. Then $U_1^*$ is the concatenation of the last four dimensions of $v_2$ and $v_1$.
    
    Now we would like to show that $a_1, b_1$ are positive and $a_2, b_2$ have different signs. We have shown in Lemma~\ref{lemma:nscl_sup_eigprob} that $a(\lambda)  = {\lambda-1 \over 2t}$ and $b(\lambda)  = {\tau_{s}(\lambda-1) \over 2(\lambda -1 -\tau_{c})t}$. Since $a_1 = a(\lambda_1)$ and $b_1 = b(\lambda_1)$, one can show that $a_1 > 0,b_1>0$ since $\lambda_1 > 1 + \tau_s + \tau_c$. For $\lambda_2 \in [1 + \tau_{s} - \tau_{c}, 1 + \tau_{c})$, it is clear that $a_2 = a(\lambda_2) > 0 > b(\lambda_2) = b_2$ when $\tau_{s} > \tau_{c}$, and conversely we have $a_2 = a(\lambda_2) < 0 < b(\lambda_2) = b_2$ when $\tau_{s} < \tau_{c}$. So $a_2$ and $b_2$ have different signs in both cases. 

    Recall $\mathcal{R}(U^*, \Vec{y})$ is defined as:
    $$ \mathcal{R}(U^*, \Vec{y}) =  \underset{{\Vec{\mu}}\in \mathbb{R}^{k}}{\operatorname{min}} \|\Vec{y} - U^* \Vec{\mu} \|^2_2, $$
    Let $\Vec{\mu} = [\frac{b_2}{a_1b_2 - a_2b_1}, \frac{-b_1}{a_1b_2 - a_2b_1}]^{\top}$, $\mathcal{R}(U_1^*, \Vec{y}) = 0$. If $\tau_s < \tau_c$, let $\Vec{\mu} = [1, 1]^{\top}$, then $\mathcal{R}(U_2^*, \Vec{y}) = 0$. If $\tau_s > \tau_c$, $\Vec{\mu}^* = U_2^{*\top} \Vec{y} = [1, 0]^{\top}$ is the minimizer and we have $\mathcal{R}(U_2^*, \Vec{y}) = 1$. 
    
\end{proof}

\begin{theorem}
     (Recap of Theorem~\ref{th:nscl_toy_general}) Assume  $\tau_1 = 1$, $\tau_0 = 0$, $1.5\tau_c > \tau_s > \tau_c$. Let $\Bar{t} = \sqrt{\frac{2(\tau_s-\tau_c)^2\tau_c}{2\tau_c - \tau_s}}$, $r: \mathbb{R} \mapsto (0,1) $ as a real value function, we have 
     \begin{equation}
         \mathcal{R}(U^*_t, \Vec{y}) = \left\{\begin{array}{ll}    
     0, &  \text{if } t \in (\Bar{t}, \tau_s), \\
    r(t), & \text{if } t \in (0, \Bar{t}) \\ 
     1, &  \text{if } t = 0. \end{array}\right. 
     \end{equation}
    \label{th:nscl_sup_toy_general}
\end{theorem}
\begin{proof}
    According to Lemma~\ref{lemma:nscl_sup_eigprob}, if $t \in (\Bar{t}, \tau_s)$, $$
U^*_t=\left[\begin{array}{ccccc}
 a_1 & a_1 & b_1 & b_1 \\
 a_2 & a_2 & b_2 & b_2 \\
\end{array}\right]^{\top},
$$
where $a_1,b_1$ are some positive real numbers, and $a_2,b_2$ has different signs. Let $\Vec{\mu} = [\frac{b_2}{a_1b_2 - a_2b_1}, \frac{-b_1}{a_1b_2 - a_2b_1}]^{\top}$, $\mathcal{R}(U_t^*, \Vec{y}) = 0$. If $t = 0$, $\mathcal{R}(U_t^*, \Vec{y}) = 0$, which is proved in Theorem~\ref{th:nscl_sup_toy_extreme} when $\tau_s > \tau_c$. If $t \in (0, \Bar{t})$, as shown in Lemma~\ref{lemma:nscl_sup_eigprob}, we have 
$$
U^*_t=\left[\begin{array}{ccccc}
 {\lambda_1-1 \over 2t} & {\lambda_1-1 \over 2t} & {\tau_{s}(\lambda_1-1) \over 2(\lambda_1 -1 -\tau_{c})t} & {\tau_{s}(\lambda_1-1) \over 2(\lambda_1 -1 -\tau_{c})t} \\
 -1 & 1 & -1 & 1 \\
\end{array}\right]^{\top},
$$ where $\lambda_1 > 0$. $\Vec{\mu}_* = (U^{*\top}_tU^*_t)^{\dag}U^{*\top}_t\Vec{y} = [\frac{{\lambda_1-1 \over 2t}}{({\lambda_1-1 \over 2t})^2 + ({\tau_{s}(\lambda_1-1) \over 2(\lambda_1 -1 -\tau_{c})t})^2}, 0]^{\top}$, then: 

$$\mathcal{R}(U_t^*, \Vec{y}) = \frac{2\tau_s^2}{({\lambda_1 - 1 - \tau_c})^2 + \tau_s^2} = r(\lambda_1) \in (0,1). $$
Note that $\lambda_1$ is a value dependent on $t$, therefore $r(\lambda_1)$ can be represented as $r(t)$.

\end{proof}

\begin{lemma}  (Recap of Lemma~\ref{th:nscl_toy_harmful})
    If  $\tau_{s} < \tau_{c} < 1.5\tau_s$,  
$
    \mathcal{R}(U^*_3, \Vec{y}) = 1, \mathcal{R}(U^*_2, \Vec{y}) = 0. 
$
    \label{th:nscl_sup_toy_harmful}
\end{lemma} 
\begin{proof}
    When $\mathcal{X}_{l}^{\text{case 3}} \triangleq \{X_{\textcolor{gray}{\cube{0.5}}, \textcolor{gray}{c_3}}\}  (\text{gray cube})$, we have     
\begin{equation*}
T_3=\left[\begin{array}{ccccc}
\tau_1 & \tau_{s} & \tau_0 & \tau_{s} & \tau_0 \\
\tau_{s} & \tau_1 & \tau_{c} & \tau_{s} & \tau_0  \\
\tau_0 & \tau_{c} & \tau_1 & \tau_0 & \tau_{s}  \\
\tau_{s} & \tau_{s} & \tau_0 & \tau_1 & \tau_{c}  \\
\tau_0 & \tau_0 & \tau_{s} & \tau_{c} & \tau_1  \\
\end{array}\right],
\end{equation*}
    
    Follow the same proof in Lemma~\ref{lemma:nscl_sup_eigprob}, one can show that $$
U^*_3=\left[\begin{array}{ccccc}
 a_1 &  b_1 &  a_1 & b_1 \\
 a_2 &  b_2 &  a_2 & b_2 \\
\end{array}\right]^{\top},
$$
where $a_1,b_1$ are some positive real numbers, and $a_2,b_2$ has different signs. Note that $U^*_3$ forms the same linear span as 
$$ \frac{1}{2}\left[\begin{array}{cccc}
1 & 1 & 1 & 1\\ -1 & 1 & -1 & 1\\
\end{array}\right]^{\top}.$$ Therefore, we have $\mathcal{R}(U^*_3, \Vec{y}) = 1$ as proved in Theorem~\ref{th:nscl_sup_toy_extreme}.
\end{proof}

\subsection{Additional Details for Section~\ref{sec:nscl_theory_main}}

This section acts as an expanded version of Section~\ref{sec:nscl_theory_main}. We will first show in Section~\ref{sec:nscl_sup_no_approx} with the background and proof for Theorem~\ref{th:nscl_no_approx} with the original adjacency matrix $\Dot{A}$. Then we present the analysis based on the approximation matrix $\Bar{A}$ in Section~\ref{sec:nscl_sup_with_approx}. Finally, we show the formal proof of our main Theorem~\ref{th:nscl_main} in Section~\ref{sec:nscl_sup_main_proof}. The proof of Theorem~\ref{th:nscl_main} requires two important ingredients (Lemma~\ref{lemma:nscl_sup_error_bound_approx} and Lemma~\ref{lemma:nscl_sup_bound_kappa}) with proof deferred in Section~\ref{sec:nscl_sup_error_bound_approx} and Section~\ref{sec:nscl_sup_kappa_proof}   respectively.  

\subsubsection{Sufficient and Necessary Condition for Perfect Residual}

\label{sec:nscl_sup_no_approx}

We first present the formal analysis in Theorem~\ref{th:nscl_sup_no_approx} which is an extended version of  Theorem~\ref{th:nscl_no_approx} without approximation and we start with the recap of definitions. 

\vspace{0.1cm} \noindent \textbf{Notations.}  Recall that $V^* \in \mathbb{R}^{N \times k}$ is defined as the top-$k$ singular vectors of $\Dot{A}$ and we split the eigen-matrix into two parts for labeled and unlabeled samples respectively:
$${V}^*=\left[\begin{array}{cc} 
{L}^* \in \mathbb{R}^{N_l \times k}\\ {U}^* \in \mathbb{R}^{N_u \times k}
\end{array}\right] = \left[\begin{array}{cccc}
{l}_1 & {l}_2 & \cdots &{l}_k \\ {u}_1 & {u}_2 & \cdots & {u}_k
\end{array}\right]$$
for labeled and unlabeled samples respectively. Then we let $V^{\flat} \in \mathbb{R}^{N \times (N-k)}$ be the remaining singular vectors of $\Dot{A}$ except top-$k$. Similarly, we split $V^{\flat}$ into two parts:
$${V}^{\flat}=\left[\begin{array}{cc} 
{L}^{\flat} \in \mathbb{R}^{N_l \times (N-k)} \\ {U}^{\flat} \in \mathbb{R}^{N_u \times (N-k)}
\end{array}\right] = \left[\begin{array}{cccc}
{l}_{k+1} & {l}_{k+2} & \cdots & {l}_{N} \\ {u}_{k+1} & {u}_{k+2} & \cdots & {u}_{N}
\end{array}\right].$$
We can also split the matrix $\Dot{A}$ at the $N_l$-th row and the $N_l$-th column and we obtain $A_{ll} \in \mathbb{R}^{N_l \times N_l},A_{ul} \in \mathbb{R}^{N_u \times N_l},A_{uu} \in \mathbb{R}^{N_u \times N_u}$ with
$$\Dot{A} = \left[\begin{array}{cc} 
A_{ll} & A^{\top}_{ul} \\ 
 A_{ul} & A_{uu}
\end{array}\right].$$

\begin{theorem}
    (\textbf{No approximation})  
    Denote the projection matrix $\mathsf{P}_{L^{\flat}} = L^{\flat\top}(L^{\flat}L^{\flat\top})^{\dag}L^{\flat}$, where $^{\dag}$ denotes the Moore-Penrose inverse. For any labeling vector $\Vec{y} \in \{0,1\}^{N_u}$, we have
    \begin{equation}
        \mathcal{R}(U^*, \Vec{y}) \leq \|(I-\mathsf{P}_{L^{\flat}}) U^{\flat\top} \Vec{y}\|^2_2.
    \label{eq:nscl_sup_R_bound_no_approx}
    \end{equation}
    The sufficient and necessary condition for $\mathcal{R}(U^*, \Vec{y}) = 0$ is $\Vec{\omega} \in \mathbb{R}^{N_l}$ such that   
    \begin{equation}
        \forall i = k+1, \ldots, N, \langle \Vec{y}^{\top} (\sigma_i I - A_{uu})^{\dag}A_{ul}, l_i \rangle =  \langle\Vec{\omega}, l_i \rangle
    \label{eq:nscl_sup_complex_condition}
    \end{equation}       
    where $\sigma_i$ is the $i$-th largest eigenvalue of $\Dot{A}$. 
    \label{th:nscl_sup_no_approx}
\end{theorem}

\begin{proof}

Define $\Vec{y}' = [\Vec{\zeta}^{\top} , \Vec{y}^{\top}]^{\top}$ as an extended labeling vector, where $\Vec{\zeta} \in \mathbb{R}^{N_l}$ can be a ``placeholder'' vector with any values. We have 
    \begin{align*}
        \mathcal{R}\left(U^*, \vec{y}\right) &= \min_{\Vec{\mu} \in \mathbb{R}^{k }} \|\Vec{y} - U^* \Vec{\mu} \|^2_2  \\
        & = \min_{\Vec{\mu} \in \mathbb{R}^{k }, \Vec{\zeta}\in \mathbb{R}^{N_l}} \|\Vec{y}' - V^* \Vec{\mu} \|^2_2 \\ 
        & =  \min_{\Vec{\zeta}\in \mathbb{R}^{N_l}} \|\Vec{y}' - V^*V^{*\top} \Vec{y}'\|^2_2 \\
        &=  \min_{\Vec{\zeta}\in \mathbb{R}^{N_l}} \|V^{\flat\top} \Vec{y}'\|^2_2 \\
        & = \min_{\Vec{\zeta}\in \mathbb{R}^{N_l}} \| L^{\flat\top} \Vec{\zeta} + U^{\flat\top} \Vec{y}\|^2_2 \\
         & = \|(I-L^{\flat\top}(L^{\flat}L^{\flat\top})^{\dag}L^{\flat}) U^{\flat\top} \Vec{y}\|^2_2.
    \end{align*}

The sufficient and necessary condition for $\mathcal{R}(U^*, \Vec{y}) = 0$ is: $$\exists \vec{\omega} \in \mathbb{R}^{N_l}, \forall i = k+1, \ldots, N,  u_i^{\top} \Vec{y} = l_i^{\top}\Vec{\omega}.$$

We then look into the relationship between $l_i$ and $u_i$.
Since 
$$\left[\begin{array}{cc} 
A_{ll} & A^{\top}_{ul} \\ 
 A_{ul} & A_{uu} 
\end{array}\right]\left[\begin{array}{c} 
 l_i \\ 
 u_i 
\end{array}\right] = \sigma_i \left[\begin{array}{c} 
 l_i \\ 
 u_i 
\end{array}\right],$$
we have the following results: 
$$u_i = (\sigma_i I - A_{uu})^{\dag}A_{ul}l_i.$$

So the sufficient and necessary condition becomes: there exists $\Vec{\omega} \in \mathbb{R}^{N_l}$ such that   
    \begin{equation}
        \forall i = k+1, \ldots, N,\langle \Vec{y}^{\top} (\sigma_i I - A_{uu})^{\dag}A_{ul}, l_i \rangle =  \langle\Vec{\omega}, l_i \rangle,
    \label{eq:nscl_complex_condition}
    \end{equation}       
    where $\sigma_i$ is the $i$-th largest singular value of $\Dot{A}$. 
\end{proof}

\vspace{0.1cm} \noindent \textbf{Interpretation of Theorem~\ref{th:nscl_sup_no_approx}.} The bound of residual in Ineq.~\eqref{eq:nscl_R_bound_no_approx} composed of two projections:  $U^{\flat\top}$ and $(I-\mathsf{P}_{L^{\flat}})$. If we only consider $\|U^{\flat\top}\Vec{y}\|^2_2$, it is equivalent to $\Vec{y}^{\top} (I - U^{*}U^{*\top}) \Vec{y}$ which indicates the information in $\Vec{y}$ that is not covered by the learned representation $U^{*}$. Then multiplying the second projection matrix $(I-\mathsf{P}_{L^{\flat}})$ further reduces the residual by considering the information from labeled data, since $\mathsf{P}_{L^{\flat}}$ is a projection matrix that projects a vector to the linear span of $L^{\flat}$. In the extreme case, when $U^{\flat\top} \Vec{y}$ fully lies in the linear span of $L^{\flat}$, the residual $\mathcal{R}(U^*, \Vec{y})$ becomes 0. 
To provide further insights about Eq.~\eqref{eq:nscl_sup_complex_condition}, we analyze in a simplified setting by approximating $\Dot{A}$ in the next section. 

\subsubsection{Analysis with Approximation} 
\label{sec:nscl_sup_with_approx}
In Theorem~\ref{th:nscl_sup_no_approx}, we put an analysis on how $L^{\flat}$ can influence the residual function. However, $L^{\flat}$ is a matrix with $N_l$ rows, so it is hard to quantitatively understand the effect of $N_l$ labeled samples individually. We resort to viewing the labeled samples as a whole. Our idea is motivated by the Stochastic Block Model (SBM)~\citep{holland1983stochastic} model, which analyzes the probability between different communities instead of individual values. In our case, we aim to analyze the probability vector $\eta_u \in \mathbb{R}^{N_u}$ denoting the chance of each unlabeled data point having the same augmentation view as one of the samples from the known class. The relationship between $\eta_u$ and $A_{uu}$ is then of our interest. 
Specifically, we define $\Bar{A}$ with values at $(i, j )$ be the following: 

\begin{equation}
    \Bar{A}_{x_ix_j} = \left\{\begin{array}{cc}             \Dot{A}_{x_i x_j} & \text{if } x_i \in \mathcal{X}_u, x_j \in \mathcal{X}_u,   \\
        \mathbb{E}_{x' \in \mathcal{X}_l} \Dot{A}_{x_i x'} & \text{if } x_i \in \mathcal{X}_u, x_j \in \mathcal{X}_l,   \\
        \mathbb{E}_{x' \in \mathcal{X}_l} \Dot{A}_{x' x_j } & \text{if } x_i \in \mathcal{X}_l, x_j \in \mathcal{X}_u,   \\
        \mathbb{E}_{x', x'' \in \mathcal{X}_l} \Dot{A}_{x' x''} & \text{if } x_i \in \mathcal{X}_l, x_j \in \mathcal{X}_l.   \\
    \end{array} \right.
    \label{eq:nscl_sup_abar_def}
\end{equation}

The probability is estimated by taking the average. It is equivalent to multiplying matrix $P$ and $P^{\top}$ on left and right side, where $P \in \mathbb{R}^{N \times N}$ is given by: 
$$P=\left[\begin{array}{cc} 
 \frac{1}{N_l}\mathbf{1}_{N_l \times N_l}  & \mathbf{0}_{N_l \times N_u} \\
 \mathbf{0}_{N_u \times N_l} &  I_{N_u} 
\end{array}\right],$$

where $\mathbf{1}_{n\times m}$ and $\mathbf{0}_{n\times m}$ represent matrix filled with 1 and 0 respectively with shape $n\times m$. Then we can write $\bar{A} \in \mathbb{R}^{N \times N}$, the approximated version of $A$, as follows: 
$$\bar{A}= PAP^{\top} = \left[\begin{array}{cc} 
\eta_l\mathbf{1}_{N_l \times N_l} & \mathbf{1}_{N_l \times 1}\Vec{\eta}_u^{\top} \\ 
\Vec{\eta}_u \mathbf{1}_{1 \times N_l} & A_{uu},
\end{array}\right],$$
where $\eta_l \in \mathbb{R}$ and $\Vec{\eta}_u \in \mathbb{R}^{N_u\times 1}$. Our analysis can then focus on how $\eta_u$ influences the representation space learned by $A_{uu}$. 
Similar to Section~\ref{sec:nscl_sup_no_approx}, we define the top-$k$ and the remainder singular vectors with corresponding splits as :
$$\Bar{V}^*=\left[\begin{array}{cc} 
\Bar{L}^* \\ \Bar{U}^*
\end{array}\right] = \left[\begin{array}{cccc}
\Bar{l}_1 & \Bar{l}_2 & \cdots &\Bar{l}_k \\ \Bar{u}_1 & \Bar{u}_2 & \cdots & \Bar{u}_k
\end{array}\right],$$
$$\Bar{V}^{\flat}=\left[\begin{array}{cc} 
\Bar{L}^{\flat} \\ \Bar{U}^{\flat}
\end{array}\right] = \left[\begin{array}{cccc}
\Bar{l}_{k+1} & \Bar{l}_{k+2} & \cdots & \Bar{l}_{N} \\ \Bar{u}_{k+1} & \Bar{u}_{k+2} & \cdots & \Bar{u}_{N}
\end{array}\right].$$
Note that due to the special structure of $\Bar{A}$ with $N_l$ duplicated rows and columns, the eigenvector $\Bar{V}$ has a special structure as we demonstrate in the next Lemma~\ref{lemma:nscl_sup_l_form}. We defer the proof to Section~\ref{sec:nscl_sup_l_form_proof}.

\begin{lemma}
   Since $A_{uu}$ is symmetric and has large diagonal values, we assume $A_{uu}$ is a positive semi-definite matrix. $\Bar{L}^*$ is stacked by the same row such that 
$ \Bar{L}^* = \frac{\mathbf{1}_{N_l \times 1}}{N_l} \Bar{\mathfrak{l}}^{*\top},$ where $\Bar{\mathfrak{l}}^{*} \in \mathbb{R}^{k}$ and that $\Bar{L}^\flat$ has the following form: 
$$\Bar{L}^\flat = \left[\begin{array}{cccc}  \frac{\mathbf{1}_{N_l\times 1}}{N_l} \Bar{\mathfrak{l}}^{\prime\top} & \Bar{l}_{N - \Theta + 1} & ... & \Bar{l}_{N}
\end{array}\right],$$ where $\Theta$ is the rank of the null space for $A_{uu} - \frac{\eta_u\eta_u^{\top}}{\eta_l}$,  
 $\Bar{\mathfrak{l}}^{\prime} \in \mathcal{R}^{N-k-\Theta}$ with non-zero values, and $\Bar{l}_{N - \Theta + 1}, ..., \Bar{l}_{N}$ are all perpendicular to $\mathbf{1}_{N_l}$. 
\label{lemma:nscl_sup_l_form}
\end{lemma}

By property in Lemma~\ref{lemma:nscl_sup_l_form}, we define: 
\begin{equation}
    \Bar{\mathfrak{l}}^{\flat} \triangleq  \Bar{L}^{\flat\top} \mathbf{1}_{ N_l \times 1} = \left[\begin{array} {cccc}\Bar{\mathfrak{l}}^{\prime\top} &  0 &  ... & 0\end{array}\right]^\top \in \mathbb{R}^{N-k}.
\end{equation}

\begin{definition}
    To ease the notation, we let $\mathcal{I} \triangleq \{k+1, k+2, ..., N-\Theta\}$ and we mainly discuss $i \in \mathcal{I}$.
\end{definition}
These definitions facilitate the presentation of the following Theorem~\ref{th:nscl_sup_with_approx}.
\begin{theorem}
    (\textbf{With approximation})  Denote $\mathfrak{T}(\Vec{y}) = \frac{\|\Bar{U}^{\flat\top} \Vec{y}\|_2}{\|\Vec{y}\|_2}$ and  $\kappa(\Vec{y}) = \cos(\Bar{U}^{\flat\top} \Vec{y},\Bar{\mathfrak{l}}^{\flat})$, where $\cos$ measures the cosine distance between two vectors. Let  $\sigma_i$ as the $i$-th largest eigenvalue of $\Dot{A}$ and $\Bar{\sigma}_i$ is for $\Bar{A}$. 
    For a labeling vector $\Vec{y} \in \{0,1\}^{N_u}$, we have
    \begin{equation}
         \mathcal{R}(\Bar{U}^*, \Vec{y}) = \frac{N_u}{|\mathcal{Y}_u|} (1-\kappa(\Vec{y})^2) \mathfrak{T}(\Vec{y})^2. 
        % \label{eq:nscl_R_bound_with_approx}
    \end{equation}
    If the ignorance degree $\mathfrak{T}(\Vec{y})$ is non-zero, the sufficient and necessary condition for $\mathcal{R}(\Bar{U}^*, \Vec{y}) = 0$: there exists $\omega \in \mathbb{R}$ such that  
    \begin{equation}
        \forall i \in \mathcal{I},  \Vec{y}^{\top} (\Bar{\sigma}_i I - A_{uu})^{\dag}\Vec{\eta}_u= \omega.
    \label{eq:nscl_easy_condition}
    \end{equation}    
    \label{th:nscl_sup_with_approx}
\end{theorem}

\begin{proof}

Define $\Vec{y}' = [\zeta\mathbf{1}_{1\times N_l} , \Vec{y}^{\top}]^{\top}$ as an extended labeling vector where $\zeta$  is any real number. We have 
    \begin{align*}
        \mathcal{R}\left(\Bar{U}^*, \vec{y}\right) &= \min_{\Vec{\mu} \in \mathbb{R}^{k }} \|\Vec{y} - \Bar{U}^* \Vec{\mu} \|^2_2  \\
        &= \min_{\Vec{\mu} \in \mathbb{R}^{k }, \zeta\in \mathbb{R}} \{ \|\Vec{y} - \Bar{U}^* \Vec{\mu} \|^2_2 + \|(\zeta - \Bar{\mathfrak{l}}^{*\top}\Vec{\mu})\mathbf{1}_{1\times N_l}\|_2^2 \} \\
        &= \min_{\Vec{\mu} \in \mathbb{R}^{k }, \zeta \in \mathbb{R}} \|\Vec{y}' - \Bar{V}^* \Vec{\mu} \|^2_2 \\ 
        & =  \min_{\zeta \in \mathbb{R}} \|\Vec{y}' - \Bar{V}^*\Bar{V}^{*\top} \Vec{y}'\|^2_2 \\
        &=  \min_{\zeta \in \mathbb{R}} \|\Bar{V}^{\flat\top} \Vec{y}'\|^2_2 \\
        & = \min_{\zeta \in \mathbb{R}} \|  \zeta \Bar{L}^{\flat\top} \mathbf{1}_{N_l\times 1}+ \Bar{U}^{\flat\top} \Vec{y}\|^2_2 \\
        &= \min_{\zeta \in \mathbb{R}} \|  \zeta \Bar{\mathfrak{l}}^{\flat}+ \Bar{U}^{\flat\top} \Vec{y}\|^2_2 \\
         & = \|(I - \frac{\Bar{\mathfrak{l}}^{\flat}\Bar{\mathfrak{l}}^{\flat\top}}{\|\Bar{\mathfrak{l}}^{\flat}\|^2_2}) \Bar{U}^{\flat\top} \Vec{y}\|^2_2 \\
         & = (1-\kappa(\Vec{y})^2) \|\Bar{U}^{\flat\top} \Vec{y}\|^2_2 \\
         & = \frac{N_u}{|\mathcal{Y}_u|} (1-\kappa(\Vec{y})^2) \mathfrak{T}(\Vec{y})^2.
    \end{align*}

We then look into the components of $\Bar{\mathfrak{l}}^{\flat}$ and $\Bar{U}^{\flat}$.
According to Lemma~\ref{lemma:nscl_sup_l_form}, when $i > N - \Theta$, we have:
\begin{equation}\Bar{\mathfrak{l}}^{\flat} = \left[\begin{array} {cccc}\Bar{\mathfrak{l}}^{\prime\top} &  0 &  ... & 0\end{array}\right]^\top = [\begin{array}{ccccccc}(\Bar{\mathfrak{l}}^{\flat})_{k+1} & (\Bar{\mathfrak{l}}^{\flat})_{k+2} & \cdots & (\Bar{\mathfrak{l}}^{\flat})_{N - \Theta} & 0  \cdots & 0 \end{array}].
\label{eq:nscl_sup_l_def}
\end{equation}

And the sufficient and necessary condition for $\mathcal{R}(\Bar{U}^*, \Vec{y})$ to be minimized by $\Bar{\mathfrak{l}}^{\flat}$ is: \begin{equation}
\exists \omega \in \mathbb{R}, \forall i \in \mathcal{I}, \Bar{u}_{i}^{\flat\top} \Vec{y} = \omega (\Bar{\mathfrak{l}}^{\flat})_{i}.
\label{eq:nscl_sup_uywl}
\end{equation}

Note that for $ i \in \mathcal{I}$,
$$\left[\begin{array}{cc} 
\eta_l\mathbf{1}_{N_l \times N_l} & \mathbf{1}_{N_l \times 1}\Vec{\eta}_u^{\top} \\ 
\Vec{\eta}_u \mathbf{1}_{1 \times N_l} & A_{uu}
\end{array}\right]\left[\begin{array}{c} 
 \Bar{l}_i \\ 
 \Bar{u}_i 
\end{array}\right] = \Bar{\sigma}_i \left[\begin{array}{c} 
 \Bar{l}_i \\ 
 \Bar{u}_i 
\end{array}\right].$$
Also since $(\Bar{\mathfrak{l}}^{\flat})_i = \mathbf{1}_{1 \times N_l} \Bar{l}_i \in \mathbb{R}$,
we have the following results: 
$$\Bar{u}_i = (\Bar{\sigma}_i I - A_{uu})^{\dag}\Vec{\eta}_u (\Bar{\mathfrak{l}}^{\flat})_i.$$

Thus, the sufficient and necessary condition~\eqref{eq:nscl_sup_uywl} becomes: there exists $\omega \in \mathbb{R}$ such that   
    \begin{equation}
        \forall i \in \mathcal{I}, \Vec{y}^{\top}(\Bar{\sigma}_i I - A_{uu})^{\dag}\Vec{\eta}_u= \omega.
    % \label{eq:nscl_sup_easy_condition}
    \end{equation}     
\end{proof}

\subsubsection{Proof of Lemma~\ref{lemma:nscl_sup_l_form}} \label{sec:nscl_sup_l_form_proof}
\begin{proof}
To understand the structure of $\Bar{U}$ and $\Bar{L}$, we consider the eigenvalue problem:
    $$\left[\begin{array}{cc} 
\eta_l\mathbf{1}_{N_l \times N_l} & \mathbf{1}_{N_l \times 1}\Vec{\eta}_u^{\top} \\ 
\Vec{\eta}_u \mathbf{1}_{1 \times N_l} & A_{uu}
\end{array}\right]\left[\begin{array}{c} 
 \Bar{l}_i \\ 
 \Bar{u}_i 
\end{array}\right] = \Bar{\sigma}_i \left[\begin{array}{c} 
 \Bar{l}_i \\ 
 \Bar{u}_i 
\end{array}\right].$$
In the non-trivial case, $\eta_l \neq 0, \Vec{\eta}_u \neq \mathbf{0}_{N_l}$ , we have the following two equations: 
\begin{align*}
    \eta_l\mathbf{1}_{N_l \times 1} \mathbf{1}_{1 \times N_l}  \Bar{l}_i + \mathbf{1}_{N_l \times 1}\Vec{\eta}_u^{\top}  \Bar{u}_i &= \Bar{\sigma}_i \Bar{l}_i \\
    (\Bar{\sigma}_i I - A_{uu})\Bar{u}_i  &= \Vec{\eta}_u \mathbf{1}_{1 \times N_l} \Bar{l}_i.
\end{align*}

\vspace{0.1cm} \noindent \textbf{(Case 1)} When $\Bar{\sigma}_i \neq 0$, then $\Bar{l}_i$ has $N_l$ duplicated scalar values $\frac{ \Vec{\eta}_u^{\top}  \Bar{u}_i}{\Bar{\sigma}_i - N_l\eta_l}$ for the first equation to satisfy.  

\vspace{0.1cm} \noindent \textbf{(Case 2)} When $\Bar{\sigma}_i = 0$, then by combing the two equations, we have: 
$$A_{uu}\Bar{u}_i = \frac{\Vec{\eta}_u\Vec{\eta}^\top_u}{\eta_l} \Bar{u}_i. $$ 
If $A_{uu} - \frac{\Vec{\eta}_u\Vec{\eta}^\top_u}{\eta_l}$ is a full rank matrix, then $\Bar{u}_i = \mathbf{0}_{N_u}$, and by the first equation $\mathbf{1}_{1 \times N_l}\Bar{l}_i = 0$. If $A_{uu} - \frac{\Vec{\eta}_u\Vec{\eta}^\top_u}{\eta_l}$ is a deficiency matrix and $\text{rank}(A_{uu} - \frac{\Vec{\eta}_u\Vec{\eta}^\top_u}{\eta_l}) \ge \text{rank}(A_{uu})$\footnote{When $\text{rank}(A_{uu} - \frac{\Vec{\eta}_u\Vec{\eta}^\top_u}{\eta_l}) < \text{rank}(A_{uu})$, it means that $\eta_u$ happens to cancel out one of the direction in $A_{uu}$. Such an event has zero probability almost sure in reality. We do not consider this case in our proof. }, then $\Bar{u}_i$ lies in the null space formed by $\Vec{\eta}_u$ and $A_{uu}$ jointly, then $\Vec{\eta}^{\top}_u \Bar{u}_i = 0$, we still have $\mathbf{1}_{1 \times N_l}\Bar{l}_i = 0$. 

Therefore when $i \in \{1,\dots, k\}$, $\Bar{\sigma}$ is non-zero values, so that $\Bar{L}^*$ is stacked by the same row such that 
$ \Bar{L}^* = \frac{\mathbf{1}_{N_l \times 1}}{N_l} \Bar{\mathfrak{l}}^{*\top},$ where $\Bar{\mathfrak{l}}^{*} \in \mathbb{R}^{k}$. 
For $i \in \{k+1,\dots, N\}$,
$\Bar{L}^\flat$ has the following form: 
$$\Bar{L}^\flat = \left[\begin{array}{cccc}  \frac{\mathbf{1}_{N_l\times 1}}{N_l} \Bar{\mathfrak{l}}^{\prime\top} & \Bar{l}_{N - \Theta + 1} & ... & \Bar{l}_{N}
\end{array}\right],$$ where $\Theta$ is the rank of the null space for $A_{uu} - \frac{\eta_u\eta_u^{\top}}{\eta_l}$,  
 $\Bar{\mathfrak{l}}^{\prime} \in \mathcal{R}^{N-k-\Theta},$ and $\Bar{l}_{N - \Theta + 1}, ..., \Bar{l}_{N}$ are all perpendicular to $\mathbf{1}_{N_l}$.
\end{proof}

\subsubsection{Proof for the Main Theorem~\ref{th:nscl_main}}
\label{sec:nscl_sup_main_proof}
In this section, we provide the main proof of Theorem~\ref{th:nscl_main}. For reader's convenience, we provide the recap version in Theorem~\ref{th:nscl_sup_main} by omitting the definition claim, where the detailed definition of $A_{ul}, A_{ll}, q_i, \Bar{U}^{\flat\top}, \Bar{\mathfrak{l}}^{\flat}, \Vec{\eta}_u$ is in Section~\ref{sec:nscl_sup_with_approx}. 

The proof of Theorem~\ref{th:nscl_main} consists of four steps. Firstly, $\mathcal{E}(f)$ is bounded by $\mathcal{R}(U^*)$ as we show in Lemma~\ref{lemma:nscl_cls_bound}. Secondly, the residual $\mathcal{R}\left(U^*, \vec{y}\right)$ of the original representation can be approximated by the residual $\mathcal{R}\left(\Bar{U}^*, \vec{y}\right)$ analyzed in Section~\ref{sec:nscl_sup_with_approx}.  Thirdly, the approximation error bound is in the order of $\frac{\|\Dot{A} - \bar{A}\|_2}{\sigma_{k} - \sigma_{k+1}}  $ as shown in Section~\ref{sec:nscl_sup_error_bound_approx}. Finally, we show that the coverage measurement $\kappa(\Vec{y})$ can be lower bounded in Section~\ref{sec:nscl_sup_kappa_proof}.

\begin{theorem} (Recap of Theorem~\ref{th:nscl_main}) Based on the assumptions made in Lemma~\ref{lemma:nscl_sup_error_bound_approx}, Lemma~\ref{lemma:nscl_sup_w_bound} and Lemma~\ref{lemma:nscl_sup_bound_kappa}.
The linear probing error is bounded by:
    \begin{equation}
        \mathcal{E}(f) \lesssim \frac{2N_u}{|\mathcal{Y}_u|}\left(\sum_i^{|\mathcal{Y}_u|} \mathfrak{T}(\Vec{y}_i)(1-\kappa(\Vec{y}_i)^2) + \frac{\|\Dot{A} - \bar{A}\|_2}{\sigma_{k} - \sigma_{k+1}} \right),
        % \label{eq:nscl_R_bound_with_approx}
    \end{equation}
    where for single labeling vector $\Vec{y}$, $$\kappa(\Vec{y}) = \cos(\Bar{U}^{\flat\top} \Vec{y},\Bar{\mathfrak{l}}^{\flat}) 
    \gtrsim \min_{i > k, j > k} \frac{2\sqrt{\frac{\Vec{y}^{\top}q_i}{\Vec{\eta}_u^{\top}q_i}\frac{\Vec{y}^{\top}q_j}{\Vec{\eta}_u^{\top}q_j}}}{\frac{\Vec{y}^{\top}q_i}{\Vec{\eta}_u^{\top}q_i}+\frac{\Vec{y}^{\top}q_j}{\Vec{\eta}_u^{\top}q_j}}.$$
    \label{th:nscl_sup_main}
\end{theorem}

\begin{proof}
    According to Lemma~\ref{lemma:nscl_cls_bound}, we have 
        $$\mathcal{E}(f) \leq 2\mathcal{R}(U^*) = 2\sum_{i \in \mathcal{Y}_u} \mathcal{R}(U^*, \Vec{y}_i),$$
    where we can view each $\Vec{y}_i$ separately. For simplicity, we use $\Vec{y}$ in the following proof. 
    As show in Section~\ref{sec:nscl_sup_with_approx}, $\mathcal{R}(U^*, \Vec{y})$ can be approximately estimated by $\mathcal{R}(\Bar{U}^*, \Vec{y}_i) = (1-\kappa(\Vec{y})^2) \|\Bar{U}^{\flat\top} \Vec{y}_i\|^2_2 = \mathfrak{T}(\Vec{y}_i)(1-\kappa(\Vec{y})^2) \|\Vec{y}_i\|^2_2$.  Such approximation bound is given by 
$$\mathcal{R}(U^*, \Vec{y}) \lesssim  \mathcal{R}(\Bar{U}^*, \Vec{y}) + \frac{2\|\Dot{A} - \bar{A}\|_2}{\sigma_{k} - \sigma_{k+1}}   \|\Vec{y}\|_2^2,$$
    as shown in Lemma~\ref{lemma:nscl_sup_error_bound_approx} in Section~\ref{sec:nscl_sup_error_bound_approx}. Putting things together, we have 
    \begin{equation*}
        \mathcal{E}(f) \lesssim 2\sum_i^{|\mathcal{Y}_u|} \mathfrak{T}(\Vec{y}_i)(1-\kappa(\Vec{y})^2) \|\Vec{y}_i\|^2_2  + \frac{2\|\Dot{A} - \bar{A}\|_2}{\sigma_{k} - \sigma_{k+1}}  \|\Vec{y}_i\|_2^2.
        % \label{eq:nscl_R_bound_with_approx}
    \end{equation*}
    If the sample size in the novel class is balanced, we have $\|\Vec{y}\|^2_2 = \frac{N_u}{|\mathcal{Y}_u|}$, we have: 
    \begin{equation*}
        \mathcal{E}(f) \lesssim \frac{2N_u}{|\mathcal{Y}_u|}\left( \sum_i^{|\mathcal{Y}_u|} \mathfrak{T}(\Vec{y}_i)(1-\kappa(\Vec{y})^2) + \frac{\|\Dot{A} - \bar{A}\|_2}{\sigma_{k} - \sigma_{k+1}} \right), 
    \end{equation*}
    Finally, the lower bound of $\kappa$ is given by Lemma~\ref{lemma:nscl_sup_bound_kappa} and proved in Section~\ref{sec:nscl_sup_kappa_proof}. 
\end{proof}

\subsubsection{Error Bound by Approximation }
\label{sec:nscl_sup_error_bound_approx}
We see in Section~\ref{sec:nscl_sup_with_approx} that we use the approximated version $\Bar{U}^*$ instead of the actual feature representation $U^*$, which creates a gap. In this section, we will present a formal analysis on  the gap between the induced residuals $\mathcal{R}(U^*, \Vec{y})$ and $\mathcal{R}(\Bar{U}^*, \Vec{y})$.

\begin{lemma}
        When $\|\Dot{A} - \bar{A}\|_2 < \frac{1}{2}({\sigma}_{k} - {\sigma}_{k+1})$ and $|\mathcal{Y}_u| \triangleq \mathbb{E}_{i\in \mathcal{I}} (1 - \|\Bar{u}_i\|^2_2)$ is a non-zero value\footnote{Note that $|\mathcal{Y}_u| = 0$ happens in an extreme case that $\forall i \in \mathcal{I}, \|\Bar{l}_i\|_2^2 = 0 $ which means the extra knowledge is purely irrelevant to the feature representation. Specifically, this could happen when $A_{ul}$ (defined in Section~\ref{sec:nscl_sup_no_approx}) is a zero matrix.}, we have
        $$\mathcal{R}(U^*, \Vec{y}) \lesssim  \mathcal{R}(\Bar{U}^*, \Vec{y}) + 2\frac{\|\Dot{A} - \bar{A}\|_2}{\sigma_{k} - \sigma_{k+1}}  \|\Vec{y}\|_2^2.$$
\label{lemma:nscl_sup_error_bound_approx}
\end{lemma}
\begin{proof}
    Recall that $\Vec{y}' = [\zeta\mathbf{1}_{1\times N_l} , \Vec{y}^{\top}]^{\top}$ is  an extended labeling vector where $\zeta$ is any real number defined in the proof of Theorem~\ref{th:nscl_sup_with_approx}. We let $\zeta^* = \arg \min_{\zeta \in \mathbb{R}} \|\Bar{V}^{\flat\top} \Vec{y}'\|^2_2 $ so that $\Bar{y}^* = [\zeta^*\mathbf{1}_{1\times N_l} , \Vec{y}^{\top}]$. 
We then define $\delta \triangleq \min\{\sigma_{k} - \bar{\sigma}_{k+1},  \bar{\sigma}_{k} - \sigma_{k+1} \}$,
\begin{align*}
    \mathcal{R}(U^*, \Vec{y}) = &  \min_{\zeta \in \mathbb{R}} \|{V}^{\flat\top} \Vec{y}'\|^2_2  \\
    = & \min_{\zeta \in \mathbb{R}} {\Vec{y}^{'\top}} {V}^{\flat} {V}^{\flat\top} \Vec{y}' \\
    = &  \min_{\zeta \in \mathbb{R}} ({\Vec{y}^{'\top}} \Bar{V}^{\flat} \Bar{V}^{\flat\top} \Vec{y}' + {\Vec{y}^{'\top}} {V}^{\flat} {V}^{\flat\top} \Vec{y}' - {\Vec{y}^{'\top}} \Bar{V}^{\flat} \Bar{V}^{\flat\top} \Vec{y}')\\
    \le & \mathcal{R}(\Bar{U}^*, \Vec{y}) +|{\Bar{y}^{*\top}} ({V}^{\flat} {V}^{\flat\top} - \Bar{V}^{\flat} \Bar{V}^{\flat\top}) \Bar{y}^*|\\
    \le & \mathcal{R}(\Bar{U}^*, \Vec{y}) +\|{V}^{\flat} {V}^{\flat\top} - \Bar{V}^{\flat} \Bar{V}^{\flat\top}\| \|\Bar{y}^*\|^2_2\\
    = & \mathcal{R}(\Bar{U}^*, \Vec{y}) +\|{V}^{\flat\top}\Bar{V}^{*}\| \|\Bar{y}^*\|^2_2\\
    \le & \mathcal{R}(\Bar{U}^*, \Vec{y}) +{\|\Dot{A} - \bar{A}\|_2\over \delta} \|\Bar{y}^*\|^2_2\\
    \le & \mathcal{R}(\Bar{U}^*, \Vec{y}) +{2\|\Dot{A} - \bar{A}\|_2\over \sigma_{k} - \sigma_{k+1} } \|\Bar{y}^*\|^2_2,
\end{align*}
where the second last inequality is from Davis-Kahan theorem on subspace distance $\|{V}^{\flat} {V}^{\flat\top} - \Bar{V}^{\flat} \Bar{V}^{\flat\top}\| = \|{V}^{\flat\top}\Bar{V}^{*}\| = \|\Bar{V}^{\flat\top}{V}^{*}\|$, and the last inequality is from Weyl’s inequality so that $\delta \ge ({\sigma}_{k} - {\sigma}_{k+1}) - \|\Dot{A} - \bar{A}\|_2 \ge \frac{1}{2}({\sigma}_{k} - {\sigma}_{k+1})$.

We then investigate the magnitude order of $ \|\Bar{y}^*\|^2_2$. Note that $\|\Bar{y}^*\|^2_2 = \|\Vec{y}\|_2^2 + N_l (\zeta^*) ^2$ and $\zeta^* = {{\Bar{\mathfrak{l}}^{\flat\top}} \Bar{U}^{\flat\top} \Vec{y} \over  {\|\Bar{\mathfrak{l}}^{\flat}\|^2_2}}$ according to the proof of Theorem~\ref{th:nscl_sup_with_approx}. Then, 
\begin{align*}
    \|\Bar{y}^*\|^2_2  &= \|\Vec{y}\|_2^2 + {N_l ({\Bar{\mathfrak{l}}^{\flat\top}} \Bar{U}^{\flat\top} \Vec{y})^2 \over  {\|\Bar{\mathfrak{l}}^{\flat}\|^4_2}}\\
    &= \|\Vec{y}\|_2^2 + \frac{N_l \kappa(\Vec{y})^2 \|\Bar{U}^{\flat\top} \Vec{y}\|^2_2}{\|\Bar{\mathfrak{l}}^{\flat}\|^2_2}  \\
    &= \|\Vec{y}\|_2^2\left(1 + \frac{N_l \kappa(\Vec{y})^2 \mathfrak{T}(\Vec{y})^2}{\|\Bar{\mathfrak{l}}^{\flat}\|^2_2}\right)  \\
    &= \|\Vec{y}\|_2^2\left(1 + \frac{\kappa(\Vec{y})^2 \mathfrak{T}(\Vec{y})^2}{\sum_{i=k+1}^{N - \Theta} (1 - \|\Bar{u}_i\|^2_2)}\right), 
\end{align*}
where the last equation is given by  Lemma~\ref{lemma:nscl_sup_l_form} when $i > N - \Theta$, $(\Bar{\mathfrak{l}}^{\flat})_i = 0$ and also by the fact that when $i \in \mathcal{I}$, $1 -\|\Bar{u}_i\|_2^2 = \|\Bar{l}_i\|_2^2 = N_l (\frac{(\Bar{\mathfrak{l}}^{\flat})_i}{N_l}) ^ 2 = (\Bar{\mathfrak{l}}^{\flat})_i^2 / N_l$. Then by the assumption that $|\mathcal{Y}_u|$ is non-zero, we have 
\begin{align*}
    \|\Bar{y}^*\|^2_2  &= \|\Vec{y}\|_2^2(1 + \frac{\kappa(\Vec{y})^2 \mathfrak{T}(\Vec{y})^2}{(N - \Theta - k)|\mathcal{Y}_u|}) \lesssim \|\Vec{y}\|_2^2(1 + O(\frac{1}{N})).
\end{align*}
By plugging back $\|\Bar{y}^*\|^2_2$, we have         
$$\mathcal{R}(U^*, \Vec{y}) \lesssim  \mathcal{R}(\Bar{U}^*, \Vec{y}) + \frac{2\|\Dot{A} - \bar{A}\|_2}{\sigma_{k} - \sigma_{k+1}}   \|\Vec{y}\|_2^2.$$
\end{proof}

\subsubsection{Analysis on the Coverage Measurement $\kappa(\Vec{y})$}
\label{sec:nscl_sup_kappa_proof}
So far we have shown in Theorem~\ref{th:nscl_sup_with_approx} that the sufficient and necessary condition for a zero residual is when the coverage measurement $\kappa(\Vec{y}) = \cos(\Bar{U}^{\flat\top} \Vec{y},\Bar{\mathfrak{l}}^{\flat})$ equals to one. %
In this section, we provide a deeper analysis on $\kappa(\Vec{y})$ in a less restrictive case.%,

Recall that we have proved in Theorem~\ref{th:nscl_sup_with_approx} that the sufficient and necessary condition for $\kappa(\Vec{y}) = 1$ is: 
    \begin{equation}
        \exists \omega \in \mathbb{R}, \forall i \in \mathcal{I}, \Vec{y}^{\top}(\Bar{\sigma}_i I - A_{uu})^{\dag}\Vec{\eta}_u= \omega.
    % \label{eq:nscl_sup_easy_condition}
    \end{equation}     
In a general case, we consider $\omega_i$ which is variant on $i$: 
$$\omega_i  \triangleq \Vec{y}^{\top}(\Bar{\sigma}_i I - A_{uu})^{\dag}\Vec{\eta}_u.$$ 

Our discussion on $\kappa(\Vec{y})$ is based on the following definitions:

\begin{definition}
Let $q_j$ and $d_j$ as the $j$-th eigenvector/eigenvalue of $A_{uu}$.   Then we define $\Tilde{\mathbf{y}}_j \triangleq \Vec{y}^\top q_j$ and $\Tilde{\boldsymbol{\eta}}_j \triangleq \Vec{\eta}_u^\top q_j$. 
\label{def:sup_y_eta}
\end{definition}

Before showing the bound on $\kappa(\Vec{y})$, we first show the following Lemma~\ref{lemma:nscl_sup_closed_form} and Lemma~\ref{lemma:nscl_sup_w_bound} which is the important ingredient needed to derive the lower bound of $\kappa(\Vec{y})$. We defer the proof to Section~\ref{sec:nscl_sup_closed_form} and Section~\ref{sec:nscl_sup_w_bound} respectively.

\begin{lemma}
    Let $\Omega \in \mathbb{R}^{(N - \Theta - k) \times (N - \Theta - k)}$ be the diagonal matrix with $\Omega_{i'i'} = \omega_i$ ($i' = i - k$ to be aligned with the indexing of $\omega_i$). For any vector $\mathfrak{l} \in \mathbb{R}^{N - \Theta - k}$, we have the following inequality:
        \begin{equation*}
                1 \geq \frac{\mathfrak{l}^\top  \Omega \mathfrak{l} }{\|\Omega \mathfrak{l}  \|_2\|\mathfrak{l} \|_2}  \geq 
    \min_{i,j \in \mathcal{I}} \frac{2\sqrt{\omega_i\omega_j}}{\sqrt{\omega_j} + \sqrt{\omega_i}},
        \end{equation*}
    A sufficient and necessary condition for $\frac{\mathfrak{l}^\top  \Omega \mathfrak{l} }{\|\Omega \mathfrak{l}  \|_2\|\mathfrak{l} \|_2}$ being 1 for all $\mathfrak{l}$ is to let  $\omega_i$ be the same for all $i \in \mathcal{I}$.
    \label{lemma:nscl_sup_closed_form}
\end{lemma}

\begin{lemma}
Assume $\eta_u$ is upper bounded by a small value $\frac{1}{M}$: 
 $\max_{j = 1 ... N_u}(\Vec{\eta}_u)_j = \frac{1}{M}.$\footnote{Such assumption is used to align the magnitude later in the proof between $\Vec{y} \in [0, 1]$ and $\Vec{\eta}_u \in [0, \frac{1}{M}]$ for the value range.} For each indexing pair $i \in \mathcal{I}$ and $i' \in \mathcal{I}$ with order $\omega_i < \omega_{i'}$, we have
    $$ \frac{\omega_{i}}{\omega_{i'}} \gtrsim \frac{\Vec{y}^\top q_i}{\Vec{\eta}_u^\top q_i} / \frac{\Vec{y}^\top q_{i'}}{\Vec{\eta}_u^\top q_{i'}}.$$
    \label{lemma:nscl_sup_w_bound}
\end{lemma}

Putting the ingredients together, we can finally derive an analytical lower bound of $\kappa(\Vec{y})$ in Lemma~\ref{lemma:nscl_sup_bound_kappa} based on the angle of $\Vec{y}$ / $\Vec{\eta}_u$ to each eigenvector of $A_{uu}$.  
\begin{lemma}
W.o.l.g, we let $\omega > 0$  and assume that $\omega_i > 0, \forall i \in \mathcal{I}$ so that perturbation of $\omega_i$ to $\omega$ to be not significant  enough to change the sign of $\omega$. we have:
     $$\kappa(\Vec{y}) = \cos(\Bar{U}^{\flat\top} \Vec{y},\Bar{\mathfrak{l}}^{\flat}) 
    \gtrsim \min_{i > k, j > k} \frac{2\sqrt{\frac{\Vec{y}^{\top}q_i}{\Vec{\eta}_u^{\top}q_i}\frac{\Vec{y}^{\top}q_j}{\Vec{\eta}_u^{\top}q_j}}}{\frac{\Vec{y}^{\top}q_i}{\Vec{\eta}_u^{\top}q_i}+\frac{\Vec{y}^{\top}q_j}{\Vec{\eta}_u^{\top}q_j}},$$
    \label{lemma:nscl_sup_bound_kappa}
\end{lemma}
\begin{proof}
    Recall that 
$$\Bar{u}_i = (\Bar{\sigma}_i I - A_{uu})^{\dag}\Vec{\eta}_u (\Bar{\mathfrak{l}}^{\flat})_i, $$ 
we consider the specific form of $\kappa(\Vec{y})$, 
\begin{align*}
    \kappa(\Vec{y}) &= \cos \left(\bar{U}^{\flat \top} \vec{y}, \Bar{\mathfrak{l}}^{\flat}\right) \\ 
    &= \frac{\sum^{N}_{i=k+1}\omega_i (\Bar{\mathfrak{l}}^{\flat})^2_i}{\sqrt{\sum^{N}_{i=k+1}\omega^2_i(\Bar{\mathfrak{l}}^{\flat})^2_i}\sqrt{\sum^{N}_{i=k+1}(\Bar{\mathfrak{l}}^{\flat})^2_i}} \\ 
    &= \frac{\sum_{i \in \mathcal{I}}\omega_i (\Bar{\mathfrak{l}}^{\flat})^2_i}{\sqrt{\sum_{i \in \mathcal{I}}\omega^2_i(\Bar{\mathfrak{l}}^{\flat})^2_i}\sqrt{\sum_{i \in \mathcal{I}}(\Bar{\mathfrak{l}}^{\flat})^2_i}} \\ 
    &= \frac{\Bar{\mathfrak{l}}^{\prime\top}  \Omega \Bar{\mathfrak{l}}^{\prime} }{\|\Omega \Bar{\mathfrak{l}}^{\prime} \|_2\|\Bar{\mathfrak{l}}^{\prime}\|_2} ,
\end{align*}
where  $\Omega \in \mathbb{R}^{N_u - k - \Theta}$ is a diagonal matrix defined in Lemma~\ref{lemma:nscl_sup_closed_form}, and $\Bar{\mathfrak{l}}^{\prime}$ is  defined in Eq.~\eqref{eq:nscl_sup_l_def}. According to Lemma~\ref{lemma:nscl_sup_closed_form}, we have 

\begin{align*}
    \kappa(\Vec{y}) &= \frac{\Bar{\mathfrak{l}}^{\prime\top}  \Omega \Bar{\mathfrak{l}}^{\prime} }{\|\Omega \Bar{\mathfrak{l}}^{\prime} \|_2\|\Bar{\mathfrak{l}}^{\prime}\|_2} \\
    & \geq 
    \min_{i,j \in \mathcal{I}} \frac{2\sqrt{\omega_i\omega_j}}{\sqrt{\omega_j} + \sqrt{\omega_i}} \\
    & =     \min_{i,j \in \mathcal{I}} \frac{2}{\sqrt{\frac{\omega_j}{\omega_i}} + \sqrt{\frac{\omega_i}{\omega_j}}},
\end{align*}
Then by Lemma~\ref{lemma:nscl_sup_w_bound} and by the fact that $\frac{2}{t+\frac{1}{t}}$ is a monotonically increasing function when $t \in (0, 1)$: 
\begin{align*}
    \kappa(\Vec{y}) &\geq \min_{i,j \in \mathcal{I}} \frac{2}{\sqrt{\frac{\omega_j}{\omega_i}} + \sqrt{\frac{\omega_i}{\omega_j}}} \\
    &\gtrsim \min_{i,j \in \mathcal{I}}  \frac{2}{\sqrt{\frac{\Vec{y}^\top q_i}{\Vec{\eta}_u^\top q_i} / \frac{\Vec{y}^\top q_{j}}{\Vec{\eta}_u^\top q_{j}}} + \sqrt{ \frac{\Vec{y}^\top q_{j}}{\Vec{\eta}_u^\top q_{j}}/\frac{\Vec{y}^\top q_i}{\Vec{\eta}_u^\top q_i}}} \\
    &> \min_{i > k, j > k} \frac{2\sqrt{\frac{\Vec{y}^{\top}q_i}{\Vec{\eta}_u^{\top}q_i}\frac{\Vec{y}^{\top}q_j}{\Vec{\eta}_u^{\top}q_j}}}{\frac{\Vec{y}^{\top}q_i}{\Vec{\eta}_u^{\top}q_i}+\frac{\Vec{y}^{\top}q_j}{\Vec{\eta}_u^{\top}q_j}}.
\end{align*}
\end{proof}

\subsubsection{Proof for Lemma~\ref{lemma:nscl_sup_closed_form}}
\label{sec:nscl_sup_closed_form}
\begin{proof}
    Consider the function $g(\mathfrak{l}) = \frac{\mathfrak{l}^\top  \Omega \mathfrak{l} }{\|\Omega \mathfrak{l}  \|_2\|\mathfrak{l} \|_2}$, the directional derivative $\partial g(\mathfrak{l})/\partial \mathfrak{l}$ is given by: 
        \begin{align*}
            \frac{\partial g(\mathfrak{l})}{\partial \mathfrak{l}} &= \frac{2 \Omega \mathfrak{l} \|\Omega\mathfrak{l}\|_2 \|\mathfrak{l}\|_2 - \Omega^2\mathfrak{l}\frac{\|\mathfrak{l}\|_2}{\|\Omega\mathfrak{l}\|_2} \mathfrak{l}^\top \Omega\mathfrak{l} - \mathfrak{l} \frac{\|\Omega\mathfrak{l}\|_2}{\|\mathfrak{l}\|_2}\mathfrak{l}^\top \Omega\mathfrak{l} }{\|\Omega\mathfrak{l}\|^2_2\|\mathfrak{l}\|^2_2}.
        \end{align*}
    The condition for $\partial g(\mathfrak{l})/\partial \mathfrak{l} = 0$ is 
    \begin{align*}
        2\Omega\mathfrak{l}  = \Omega^2 \mathfrak{l} \frac{\mathfrak{l}^\top \Omega\mathfrak{l}}{\|\Omega\mathfrak{l}\|^2_2} + \mathfrak{l}  \frac{\mathfrak{l}^\top \Omega\mathfrak{l}}{\|\mathfrak{l}\|^2_2}.
    \end{align*}
    Note that the first condition to satisfy this equation is to let $\mathfrak{l}$ as the eigenvectors of $2\Omega - \Omega^2 \frac{\mathfrak{l}^\top \Omega\mathfrak{l}}{\|\Omega\mathfrak{l}\|^2_2}$ which is a diagonal matrix. Then one of the solutions sets is $\mathfrak{l} = c\*e_j$ where $c$ is any non-zero scalar value and $\*e_j$ is the unit vector with $j$-th value 1 and 0 elsewhere. Note that this solution set corresponds to the maximum value of $g(\mathfrak{l})$ which is 1. We are then looking into the local minimum value of  $g(\mathfrak{l})$ by another solution set. We consider another solution set by considering the following matrix as deficiency:  
        \begin{align*}
        \Gamma \triangleq 2\Omega - \Omega^2 \frac{\mathfrak{l}^\top \Omega\mathfrak{l}}{\|\Omega\mathfrak{l}\|^2_2} - \frac{\mathfrak{l}^\top \Omega\mathfrak{l}} {\|\mathfrak{l}\|^2_2} I,
    \end{align*}
    where $\mathfrak{l}$ lies in the null space of this matrix.     If we let $\varrho = \frac{\|\mathfrak{l}\|_2}{\|\Omega\mathfrak{l}\| _2}$, we have: 
    \begin{align*}
        \Gamma = 2\Omega - \varrho g(\hat{\mathfrak{l}}) \Omega^2  - \varrho^{-1} g(\hat{\mathfrak{l}})  I
    \end{align*}
    and $$\Gamma_{i'i'} = 2\omega_i - \varrho g(\hat{\mathfrak{l}}) \omega_i^2 - \varrho^{-1} g(\hat{\mathfrak{l}}),$$
    where $i'$ is indexed starting from 1 and $i$ is indexed starting from $k$.    Note that $\Gamma_{i'i'}$ only has two zero roots. If we consider all $\omega_i$(s) in $\Omega$ to be different,  $\Gamma$ can have at most two zero values in the diagonal. Let $\omega_a, \omega_b$ as two roots of  $2\omega - \varrho g(\hat{\mathfrak{l}}) \omega^2 - \varrho^{-1} g(\hat{\mathfrak{l}})$, we have: 
    $$\varrho\omega_a + (\varrho\omega_a)^{-1} = \varrho\omega_b + (\varrho\omega_b)^{-1} = \frac{2}{g(\hat{\mathfrak{l}})}$$
    $$\varrho = \frac{\sqrt{\omega_b}}{\sqrt{\omega_a}}, g(\hat{\mathfrak{l}}) = \frac{2}{\sqrt{\frac{\omega_b}{\omega_a}} + \sqrt{\frac{\omega_a}{\omega_b}}}, $$
    which corresponds to one local minimal with the indexing pair $(a,b)$. By enumerating all the indexing pairs, we have the global minimum of $g(\mathfrak{l})$: 
     $$ g(\mathfrak{l}^*) = \min_{i,j \in \mathcal{I}} \frac{2\sqrt{\omega_i\omega_j}}{\sqrt{\omega_j} + \sqrt{\omega_i}}.$$
    Note that when some $\omega_i$, $\omega_j$ are identical, this is a special case where the local minimum is equal to the maximum 1. Therefore a sufficient and necessary condition for $g(\mathfrak{l}) = 1$ is to let $\omega_i$ be the same for all $i \in \mathcal{I}$.
    
\end{proof}

\subsubsection{Proof for Lemma~\ref{lemma:nscl_sup_w_bound}}
\label{sec:nscl_sup_w_bound}
\begin{proof}
    
We can write $\omega_i$ by $\Tilde{\mathbf{y}}$ and $\Tilde{\boldsymbol{\eta}}$  in Definition~\ref{def:sup_y_eta}: 
\begin{align*}
    \omega_i &= \Vec{y}^{\top}(\Bar{\sigma}_i I - A_{uu})^{\dag}\Vec{\eta}_u \\ 
    &= \sum_{j\in \mathcal{I}} \frac{(\Vec{y}^\top q_j)(\Vec{\eta}_u^\top q_j)}{\Bar{\sigma}_i - d_j} + \sum^{N_u}_{j = N - \Theta + 1} \frac{(\Vec{y}^\top q_j)(\Vec{\eta}_u^\top q_j)}{\Bar{\sigma}_i} \\
    &= \sum_{j\in \mathcal{I}} \frac{\Tilde{\mathbf{y}}_j\Tilde{\boldsymbol{\eta}}_j}{\Bar{\sigma}_i - d_j} + \frac{1}{\Bar{\sigma}_i}\sum^{N_u}_{j = N - \Theta + 1} \Tilde{\mathbf{y}}_j\Tilde{\boldsymbol{\eta}}_j.
\end{align*}
We then look into the value of $\Bar{\sigma}_i$ by solving the eigenvalue problem:  
\begin{align*}
    \left[\begin{array}{cc} 
\eta_l\mathbf{1}_{N_l \times N_l} & \mathbf{1}_{N_l \times 1}\Vec{\eta}_u^{\top} \\ 
\Vec{\eta}_u \mathbf{1}_{1 \times N_l} & A_{uu}
\end{array}\right]\left[\begin{array}{c} 
 \Bar{l}_i \\ 
 \Bar{u}_i 
\end{array}\right] &= \Bar{\sigma}_i \left[\begin{array}{c} 
 \Bar{l}_i \\ 
 \Bar{u}_i 
\end{array}\right] \\ \Longleftrightarrow \ 
\eta_l\mathbf{1}_{N_l \times N_l}  \Bar{l}_i + \mathbf{1}_{N_l \times 1}\Vec{\eta}_u^{\top}  \Bar{u}_i &= \Bar{\sigma}_i \Bar{l}_i \\ \Longleftrightarrow \ 
\mathbf{1}_{N_l \times 1} \eta_l (\Bar{\mathfrak{l}}^{\flat})_i + \mathbf{1}_{N_l \times 1}\Vec{\eta}_u^{\top}  \Bar{u}_i &= \mathbf{1}_{N_l \times 1} \Bar{\sigma}_i \frac{1}{N_l}(\Bar{\mathfrak{l}}^{\flat})_i \\ \Longleftrightarrow \ 
\eta_l (\Bar{\mathfrak{l}}^{\flat})_i + \Vec{\eta}_u^{\top}  \Bar{u}_i &= \Bar{\sigma}_i \frac{1}{N_l}(\Bar{\mathfrak{l}}^{\flat})_i  \\ \quad\quad \Longleftrightarrow
\eta_l (\Bar{\mathfrak{l}}^{\flat})_i + \Vec{\eta}_u^{\top} (\Bar{\sigma}_i I - A_{uu})^{\dag}\Vec{\eta}_u (\Bar{\mathfrak{l}}^{\flat})_i &= \Bar{\sigma}_i \frac{1}{N_l}(\Bar{\mathfrak{l}}^{\flat})_i \\ \Longleftrightarrow \ 
\eta_l + \Vec{\eta}_u^{\top} (\Bar{\sigma}_i I - A_{uu})^{\dag}\Vec{\eta}_u &=  \frac{\Bar{\sigma}_i}{N_l} \\ \Longleftrightarrow \ 
\eta_l + \sum_{j\in \mathcal{I}} \frac{\Tilde{\boldsymbol{\eta}}^2_j}{\Bar{\sigma}_i - d_j}   &=  \frac{\Bar{\sigma}_i}{N_l} \\
\end{align*}

Note that we get a $(|\mathcal{I}| + 1)$-th degree polynomials of $\Bar{\sigma}_i$ with $(|\mathcal{I}| + 1)$ roots. By observation, we see that there is one root significantly large ($\approx N_l \eta_l$) since $N_l$ and other $|\mathcal{I}|$ roots are very close to each $d_j$. Based on this intuition, we approximately view it as a unary quadratic equation: 

$$
\eta_l + \phi_i + \frac{\Tilde{\boldsymbol{\eta}}^2_i}{\Bar{\sigma}_i - d_i}  =  \frac{\Bar{\sigma}_i}{N_l},$$

where we  let $\phi_i \triangleq \sum_{j\in \mathcal{I}, j \neq i} \frac{\Tilde{\boldsymbol{\eta}}^2_j}{\Bar{\sigma}_i - d_j}$. We then proceed by solving this unary quadratic equation by viewing $\phi_i$ as a variable. 
\begin{align*} \Bar{\sigma}_i (\Bar{\sigma}_i - d_i) &= N_l \eta_l (\Bar{\sigma}_i - d_i) + N_l\phi_i (\Bar{\sigma}_i - d_i) + N_l \Tilde{\boldsymbol{\eta}}^2_i \\ \Longleftrightarrow \ 
\Bar{\sigma}_i^2  &= (d_i + N_l (\eta_l + \phi_i))\Bar{\sigma}_i + N_l (\Tilde{\boldsymbol{\eta}}^2_i -  (\eta_l + \phi_i) d_i ) \\ \Longleftrightarrow \ 
\Bar{\sigma}_i &= \frac{d_i + N_l (\eta_l + \phi_i)}{2} \pm \sqrt{\frac{(d_i + N_l (\eta_l + \phi_i))^2}{4} + N_l (\Tilde{\boldsymbol{\eta}}^2_i -  (\eta_l + \phi_i) d_i )} \\ 
\Longleftrightarrow \ 
\Bar{\sigma}_i &= \frac{d_i + N_l (\eta_l + \phi_i)}{2} \pm \sqrt{\frac{(N_l (\eta_l + \phi_i) - d_i)^2}{4} + N_l \Tilde{\boldsymbol{\eta}}^2_i} \\ 
\Longleftrightarrow \ 
\Bar{\sigma}_i &= \frac{d_i + N_l (\eta_l + \phi_i)}{2} \pm \left(\frac{N_l (\eta_l + \phi_i) - d_i}{2} + \frac{N_l \Tilde{\boldsymbol{\eta}}^2_i}{\frac{N_l (\eta_l + \phi_i) - d_i}{2} + \sqrt{\frac{(N_l (\eta_l + \phi_i) - d_i)^2}{4} + N_l \Tilde{\boldsymbol{\eta}}^2_i}}\right) \\ 
\Longleftrightarrow \ 
\Bar{\sigma}_i &= \frac{d_i + N_l (\eta_l + \phi_i)}{2} \pm \left(\frac{N_l (\eta_l + \phi_i) - d_i}{2} + \frac{1}{\frac{\eta_l + \phi_i - \frac{d_i}{N_l}}{2 \Tilde{\boldsymbol{\eta}}^2_i} + \sqrt{(\frac{\eta_l + \phi_i - \frac{d_i}{N_l}}{2 \Tilde{\boldsymbol{\eta}}^2_i})^2 + 1}}\right) \\ \Longleftrightarrow \ 
\Bar{\sigma}_i &= \frac{d_i + N_l (\eta_l + \phi_i)}{2} \pm \left(\frac{N_l (\eta_l + \phi_i) - d_i}{2} + \frac{ \Tilde{\boldsymbol{\eta}}^2_i}{\eta_l + \phi_i - \frac{d_i}{N_l}} - O((\frac{ \Tilde{\boldsymbol{\eta}}^2_i}{\eta_l + \phi_i})^2)\right)
\end{align*}

Here we see that $\Bar{\sigma}_i$ has two approximated solutions: in the first case, when $\pm$ becomes $+$, $\Bar{\sigma}_i \approx N_l \eta_l$ which is the unique very large solution as we mentioned. Another solution is by picking $\pm$ as $-$, we then have $\Bar{\sigma}_i \approx d_i - \frac{ \Tilde{\boldsymbol{\eta}}^2_i}{\eta_l + \phi_i - \frac{d_i}{N_l}}$. The second case is what we are using in this proof since we are looking at the indexing of $\omega_i$ with $i \in \mathcal{I}$, which is beyond top-$k$. 

For each indexing pair $i$ and $i'$ with order $\omega_i < \omega_{i'}$,
we plug in the solution of $\Bar{\sigma}_i$ and $\Bar{\sigma}_i'$ respectively:  
\begin{align*}
    \frac{\omega_{i}}{\omega_{i'}}  &= \frac{ \sum_{j \in \mathcal{I}} \frac{\tilde{\mathbf{y}}_j \tilde{\boldsymbol{\eta}}_j}{d_j - \bar{\sigma}_{i}} + \frac{1}{\Bar{\sigma}_i'}\sum^{N_u}_{j = N - \Theta + 1} \Tilde{\mathbf{y}}_j\Tilde{\boldsymbol{\eta}}_j}{\sum_{j \in \mathcal{I}} \frac{\tilde{\mathbf{y}}_j \tilde{\boldsymbol{\eta}}_j}{d_j - \bar{\sigma}_{i'}} + \frac{1}{\Bar{\sigma}_i''}\sum^{N_u}_{j = N - \Theta + 1} \Tilde{\mathbf{y}}_j\Tilde{\boldsymbol{\eta}}_j} \\
    &= \frac{ \frac{\tilde{\mathbf{y}}_{i} \tilde{\boldsymbol{\eta}}_{i}}{d_{i} - \bar{\sigma}_{i}} + \sum_{j \in \mathcal{I}, j\neq i} \frac{\tilde{\mathbf{y}}_j \tilde{\boldsymbol{\eta}}_j}{d_j - \bar{\sigma}_{i}}+ \frac{1}{\Bar{\sigma}_i'}\sum^{N_u}_{j = N - \Theta + 1} \Tilde{\mathbf{y}}_j\Tilde{\boldsymbol{\eta}}_j}
    {\frac{\tilde{\mathbf{y}}_{i'} \tilde{\boldsymbol{\eta}}_{i'}}{d_{i'} - \bar{\sigma}_{i'}} + \sum_{j \in \mathcal{I}, j\neq i'} \frac{\tilde{\mathbf{y}}_j \tilde{\boldsymbol{\eta}}_j}{d_j - \bar{\sigma}_{i'}}+ \frac{1}{\Bar{\sigma}_i''}\sum^{N_u}_{j = N - \Theta + 1} \Tilde{\mathbf{y}}_j\Tilde{\boldsymbol{\eta}}_j} \\
    &= \frac{ \frac{\tilde{\mathbf{y}}_{i}}{\tilde{\boldsymbol{\eta}}_{i}} (\eta_l + \phi_{i}) + \tilde{\mathbf{y}}_{i}\tilde{\boldsymbol{\eta}}_{i}(O((\frac{ \Tilde{\boldsymbol{\eta}}^2_{i}}{\eta_l + \phi_i})^2) - O(\frac{1}{N_l}))  + \sum_{j \in \mathcal{I}, j\neq i} \frac{\tilde{\mathbf{y}}_j \tilde{\boldsymbol{\eta}}_j}{d_j - \bar{\sigma}_{i}}+ \frac{1}{\Bar{\sigma}_i'}\sum^{N_u}_{j = N - \Theta + 1} \Tilde{\mathbf{y}}_j\Tilde{\boldsymbol{\eta}}_j}
    { \frac{\tilde{\mathbf{y}}_{i'}}{\tilde{\boldsymbol{\eta}}_{i'}} (\eta_l + \phi_{i'}) + \tilde{\mathbf{y}}_{i'}\tilde{\boldsymbol{\eta}}_{i'}(O((\frac{ \Tilde{\boldsymbol{\eta}}^2_{i'}}{\eta_l + \phi_i'})^2) - O(\frac{1}{N_l}))  + \sum_{j \in \mathcal{I}, j\neq i'} \frac{\tilde{\mathbf{y}}_j \tilde{\boldsymbol{\eta}}_j}{d_j - \bar{\sigma}_{i'}}+ \frac{1}{\Bar{\sigma}_i''}\sum^{N_u}_{j = N - \Theta + 1} \Tilde{\mathbf{y}}_j\Tilde{\boldsymbol{\eta}}_j} \\ 
    &= \frac{ \frac{\tilde{\mathbf{y}}_{i}}{\tilde{\boldsymbol{\eta}}_{i}} \eta_l + \tilde{\mathbf{y}}_{i}\tilde{\boldsymbol{\eta}}_{i}(O((\frac{ \Tilde{\boldsymbol{\eta}}^2_{i}}{\eta_l + \phi_i})^2) - O(\frac{1}{N_l}))  + \sum_{j \in \mathcal{I}, j\neq i} \frac{1}{d_j - \bar{\sigma}_{i}}\tilde{\boldsymbol{\eta}}_j (\tilde{\mathbf{y}}_j + \tilde{\mathbf{y}}_{i}\frac{\tilde{\boldsymbol{\eta}}_j}{\tilde{\boldsymbol{\eta}}_{i}}) + \frac{1}{\Bar{\sigma}_i'}\sum^{N_u}_{j = N - \Theta + 1} \Tilde{\mathbf{y}}_j\Tilde{\boldsymbol{\eta}}_j } 
    {  \frac{\tilde{\mathbf{y}}_{i'}}{\tilde{\boldsymbol{\eta}}_{i'}} \eta_l + \tilde{\mathbf{y}}_{i'}\tilde{\boldsymbol{\eta}}_{i'}(O((\frac{ \Tilde{\boldsymbol{\eta}}^2_{i'}}{\eta_l + \phi_i'})^2) - O(\frac{1}{N_l}))  + \sum_{j \in \mathcal{I}, j\neq i'} \frac{1}{d_j - \bar{\sigma}_{i'}}\tilde{\boldsymbol{\eta}}_j (\tilde{\mathbf{y}}_j + \tilde{\mathbf{y}}_{i'}\frac{\tilde{\boldsymbol{\eta}}_j}{\tilde{\boldsymbol{\eta}}_{i'}})+ \frac{1}{\Bar{\sigma}_i''}\sum^{N_u}_{j = N - \Theta + 1} \Tilde{\mathbf{y}}_j\Tilde{\boldsymbol{\eta}}_j}.
\end{align*}

According to assumption that $\eta_u$ is bounded by $\frac{1}{M}$, we align the magnitude between $\Vec{y}$ and $\Vec{\eta}_u$ by  defining $\Vec{\eta}_u' = M\Vec{\eta}_u$ which is now also in the range of $[0, 1]$ like $\Vec{y}$. Then we also scale the following terms: $\Tilde{\boldsymbol{\eta}}' = M\Tilde{\boldsymbol{\eta}}$. Therefore we can simplify the equation to be:

\scalebox{0.8}{
\begin{minipage}{\linewidth}
\begin{align*}
    \frac{\omega_{i}}{\omega_{i'}} &= \frac{ M\frac{\tilde{\mathbf{y}}_{i}}{\tilde{\boldsymbol{\eta}}'_{i}} \eta_l + \frac{1}{M}\tilde{\mathbf{y}}_{i}\tilde{\boldsymbol{\eta}}'_{i}(O(\frac{1}{M^4}(\frac{ \tilde{\boldsymbol{\eta}}'^2_{i}}{\eta_l + \phi_i})^2) - O(\frac{1}{N_l}))  + \frac{1}{M}\sum_{j \in \mathcal{I}, j\neq i} \frac{1}{d_j - \bar{\sigma}_{i}}\tilde{\boldsymbol{\eta}}'_j (\tilde{\mathbf{y}}_j + \tilde{\mathbf{y}}_{i}\frac{\tilde{\boldsymbol{\eta}}'_j}{\tilde{\boldsymbol{\eta}}'_{i}})+ \frac{1}{M\Bar{\sigma}_i'}\sum^{N_u}_{j = N - \Theta + 1} \Tilde{\mathbf{y}}_j\Tilde{\boldsymbol{\eta}}'_j} 
    {  M \frac{\tilde{\mathbf{y}}_{i'}}{\tilde{\boldsymbol{\eta}}'_{i'}} \eta_l + \frac{1}{M} \tilde{\mathbf{y}}_{i'}\tilde{\boldsymbol{\eta}}'_{i'}(O(\frac{1}{M^4}(\frac{ \tilde{\boldsymbol{\eta}}'^2_{i'}}{\eta_l + \phi_i'})^2) - O(\frac{1}{N_l}))  + \frac{1}{M} \sum_{j \in \mathcal{I}, j\neq i'} \frac{1}{d_j - \bar{\sigma}_{i'}}\tilde{\boldsymbol{\eta}}'_j (\tilde{\mathbf{y}}_j + \tilde{\mathbf{y}}_{i'}\frac{\tilde{\boldsymbol{\eta}}'_j}{\tilde{\boldsymbol{\eta}}'_{i'}}) + \frac{1}{M\Bar{\sigma}_i''}\sum^{N_u}_{j = N - \Theta + 1} \Tilde{\mathbf{y}}_j\Tilde{\boldsymbol{\eta}}'_j} \\
    &= \frac{ \frac{\tilde{\mathbf{y}}_{i}}{\tilde{\boldsymbol{\eta}}'_{i}} \eta_l + \tilde{\mathbf{y}}_{i}\tilde{\boldsymbol{\eta}}'_{i}(O(\frac{1}{M^6}) - O(\frac{1}{M^2 N_l}))  + \frac{1}{M^2}\sum_{j \in \mathcal{I}, j\neq i} \frac{1}{d_j - d_{i} + O(\frac{1}{M^2})}\tilde{\boldsymbol{\eta}}'_j (\tilde{\mathbf{y}}_j + \tilde{\mathbf{y}}_{i}\frac{\tilde{\boldsymbol{\eta}}'_j}{\tilde{\boldsymbol{\eta}}'_{i}})+ \tilde{\mathbf{y}}_{i}\frac{\tilde{\boldsymbol{\eta}}'_j}{\tilde{\boldsymbol{\eta}}'_{i}})+ \frac{1}{M^2\Bar{\sigma}_i'}\sum^{N_u}_{j = N - \Theta + 1} \Tilde{\mathbf{y}}_j\Tilde{\boldsymbol{\eta}}'_j} 
    {  \frac{\tilde{\mathbf{y}}_{i'}}{\tilde{\boldsymbol{\eta}}'_{i'}} \eta_l + \tilde{\mathbf{y}}_{i'}\tilde{\boldsymbol{\eta}}'_{i'}(O(\frac{1}{M^6}) - O(\frac{1}{M^2 N_l}))  + \frac{1}{M^2} \sum_{j \in \mathcal{I}, j\neq i'} \frac{1}{d_j - d_{i'} + O(\frac{1}{M^2})}\tilde{\boldsymbol{\eta}}'_j (\tilde{\mathbf{y}}_j + \tilde{\mathbf{y}}_{i'}\frac{\tilde{\boldsymbol{\eta}}'_j}{\tilde{\boldsymbol{\eta}}'_{i'}})+ \frac{1}{M^2\Bar{\sigma}_i''}\sum^{N_u}_{j = N - \Theta + 1} \Tilde{\mathbf{y}}_j\Tilde{\boldsymbol{\eta}}'_j} \\
    & = \frac{ \frac{\tilde{\mathbf{y}}_{i}}{\tilde{\boldsymbol{\eta}}'_{i}} \eta_l + O(\frac{1}{M^2})} 
    {  \frac{\tilde{\mathbf{y}}_{i'}}{\tilde{\boldsymbol{\eta}}'_{i'}} \eta_l + O(\frac{1}{M^2})}, \\
    &~~~
\end{align*}
\end{minipage}
}
where we simply regard the remaining term with a magnitude much smaller than M. Note that M can be viewed as the magnitude gap of $\frac{\max_i(\Vec{y})_i}{\max_i(\Vec{\eta}_u)_i}$. In our case, $\max_i(\Vec{y})_i$ is set to 1. However, one can always multiply $\Vec{y}$ with a large constant to make M significantly large without changing the residual analysis in the main theorem. In summary, we have 
$$ \frac{\omega_{i}}{\omega_{i'}} \gtrsim \frac{ \frac{\tilde{\mathbf{y}}_{i}}{\tilde{\boldsymbol{\eta}}'_{i}} \eta_l } 
    {  \frac{\tilde{\mathbf{y}}_{i'}}{\tilde{\boldsymbol{\eta}}'_{i'}} \eta_l } =  \frac{\tilde{\mathbf{y}}_{i}}{\tilde{\boldsymbol{\eta}}_{i}} / \frac{\tilde{\mathbf{y}}_{i'}}{\tilde{\boldsymbol{\eta}}_{i'}} = \frac{\Vec{y}^\top q_i}{\Vec{\eta}_u^\top q_i} / \frac{\Vec{y}^\top q_{i'}}{\Vec{\eta}_u^\top q_{i'}}.$$

\end{proof}

\newpage
\subsection{Experimental Details}
\subsubsection{ Details of Training Configurations}
\label{sec:nscl_sup_exp_cifar}
For a fair comparison, we use ResNet-18~\citep{he2016deep} as the backbone for all methods. We add a trainable two-layer MLP projection head that projects the feature from the penultimate layer to an embedding space $\mathbb{R}^{k}$ ($k = 1000$). We use the same data augmentation strategies as SimSiam~\citep{chen2021exploring,haochen2021provable}. We train our model $f(\cdot)$ for 1200 epochs by NCD Spectral Contrastive Loss defined in Eq.~\eqref{eq:nscl_def_nscl}.  We set $\alpha=0.0225$ and $\beta=2$.
We use SGD with momentum 0.95 as an optimizer with cosine annealing (lr=0.03), weight decay 5e-4, and batch size 512. 
We also conduct a sensitivity analysis of the hyper-parameters in Figure~\ref{fig:nscl_hyper}. The performance comparison for each hyper-parameter is reported by fixing other hyper-parameters. The results suggest that the novel class discovery performance of NSCL is  stable when $\alpha$, $\beta$ in a reasonable range and with different learning rates. 

\nocite{sun2017faster,sun2019adaptive}

\begin{figure}[htb]
    \centering
    \includegraphics[width=0.95\linewidth]{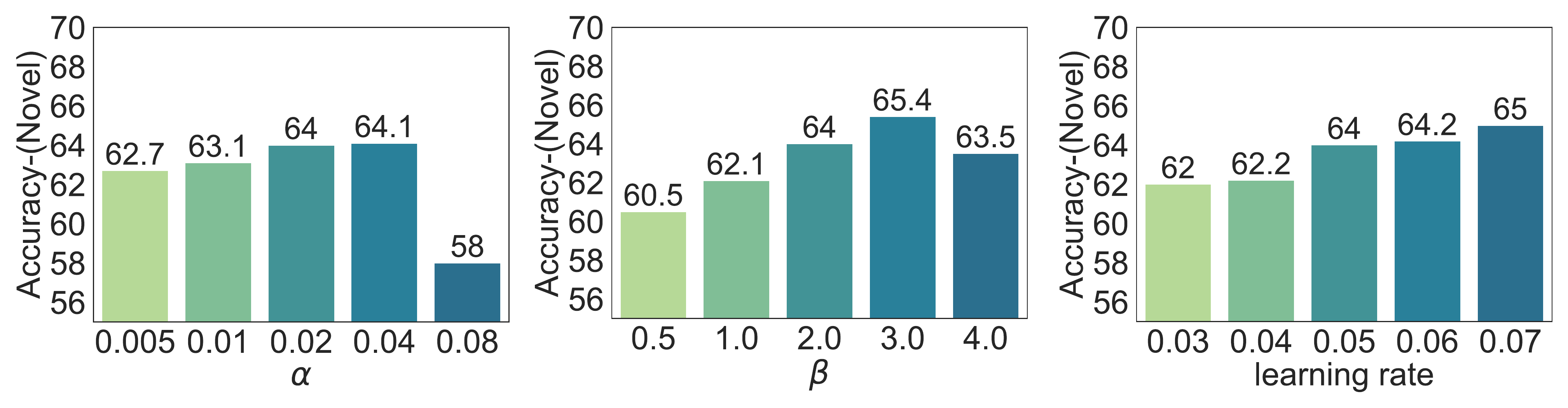}
    \caption[Sensitivity analysis of hyper-parameters.]{Sensitivity analysis of hyper-parameters $\alpha$, $\beta$, and learning rate. We use the training split of CIFAR-100-50/50, and report the novel class accuracy.}
    \vspace{-0.4cm}
    \label{fig:nscl_hyper}
\end{figure}

\subsubsection{Experimental Details of Toy Example}
\label{sec:nscl_sup_exp_vis}
\vspace{0.1cm} \noindent \textbf{Recap of set up}. In Section~\ref{sec:nscl_theory_setup} we consider a toy example that helps illustrate the core idea of our theoretical findings. Specifically, the example aims to cluster 3D objects of different colors and shapes, generated by a 3D rendering software~\citep{johnson2017clevr} with user-defined properties including colors, shape, size, position, etc. 

In what follows, we define two data configurations and corresponding graphs, where the labeled data is correlated  with the attribute of unlabeled data (\textbf{case 1}) vs. not (\textbf{case 2}). For both cases, we have an unlabeled dataset containing red/blue cubes/spheres as: 

$$\mathcal{X}_u \triangleq \{X_{\textcolor{red}{\cube{0.6}}, \textcolor{red}{c_1}}, X_{\textcolor{red}{\sphere{0.5}{red}}, \textcolor{red}{c_1}}, X_{\textcolor{blue}{\cube{0.6}}, \textcolor{blue}{c_2}}, X_{\textcolor{blue}{\sphere{0.5}{blue}}, \textcolor{blue}{c_2}}\}.$$

In the first case, we let the labeled data $\mathcal{X}_{l}^{\text{case 1}}$ be strongly correlated with the target class (red color) in unlabeled data:
$$\mathcal{X}_{l}^{\text{case 1}} \triangleq \{X_{\textcolor{red}{\cylinder{0.4}}, \textcolor{red}{c_1}}\} (\text{red cylinder}).$$  
In the second case, we use gray cylinders which have no overlap in either shape and color: 
$$\mathcal{X}_{l}^{\text{case 2}} \triangleq \{X_{\textcolor{gray}{\cylinder{0.4}}, \textcolor{gray}{c_3}}\}  (\text{gray cylinder}).$$ 
Putting it together, our entire training dataset is 
$\mathcal{X}^{\text{case 1}} = \mathcal{X}_{l}^{\text{case 1}} \cup \mathcal{X}_{u}$ or  $\mathcal{X}^{\text{case 2}} = \mathcal{X}_{l}^\text{case 2} \cup \mathcal{X}_{u}$.

\vspace{0.1cm} \noindent \textbf{Experimental details for Figure~\ref{fig:nscl_toy_vis}}. For training, we rendered 2500 samples for each type of data (4 types in $\mathcal{X}_u$ and 1 type in $\mathcal{X}_l$). In total, we have 12500 samples for both $\mathcal{X}^{\text{case 1}}$ and $\mathcal{X}^{\text{case 2}}$. For training, we use the same data augmentation strategy as in SimSiam~\citep{chen2021exploring}. We use ResNet18 and train the model for 40 epochs (sufficient for convergence) with a fixed learning rate of 0.005, using NSCL defined in Eq.~\eqref{eq:nscl_def_nscl}.  We set $\alpha=0.04$ and $\beta=1$, respectively. Our visualization is by PyTorch implementation of UMAP~\citep{umap}, with parameters $(\texttt{n\_neighbors=30,                min\_dist=1.5, spread=2, metric=euclidean})$.

%% file: chapters/supp_sorl.tex
\subsection{Technical Details of Spectral Open-world Representation Learning}
\label{sec:sorl_proof-SORL}
\begin{theorem}
\label{th:sorl_sup-orl-scl} (Recap of Theorem~\ref{th:sorl_orl-scl}) 
We define $\*f_x = \sqrt{w_x}f(x)$ for some function $f$. Recall $\eta_{u},\eta_{l}$ are two hyper-parameters defined in Eq.~\eqref{eq:sorl_def_wxx}. Then minimizing the loss function $\mathcal{L}_{\mathrm{mf}}(F, A)$ is equivalent to minimizing the following loss function for $f$, which we term \textbf{Spectral Open-world Representation Learning (SORL)}:
\begin{align}
\begin{split}
    \mathcal{L}_\text{SORL}(f) \triangleq & - 2\eta_{u} \mathcal{L}_1(f) 
- 2\eta_{l}  \mathcal{L}_2(f) \\ & +  \eta_{u}^2 \mathcal{L}_3(f) + 2\eta_{u} \eta_{l} \mathcal{L}_4(f) +  
\eta_{l}^2 \mathcal{L}_5(f),
\label{eq:sorl_sup_def_sorl}
\end{split}
\end{align}
where 
\begin{align*}
    \mathcal{L}_1(f) &= \sum_{i \in \mathcal{Y}_l}\underset{\substack{\bar{x}_{l} \sim \mathcal{P}_{{l_i}}, \bar{x}'_{l} \sim \mathcal{P}_{{l_i}},\\x \sim \mathcal{T}(\cdot|\bar{x}_{l}), x^{+} \sim \mathcal{T}(\cdot|\bar{x}'_l)}}{\mathbb{E}}\left[f(x)^{\top} {f}\left(x^{+}\right)\right] , \\
    \mathcal{L}_2(f) &= \underset{\substack{\bar{x}_{u} \sim \mathcal{P},\\x \sim \mathcal{T}(\cdot|\bar{x}_{u}), x^{+} \sim \mathcal{T}(\cdot|\bar{x}_u)}}{\mathbb{E}}
\left[f(x)^{\top} {f}\left(x^{+}\right)\right], \\
    \mathcal{L}_3(f) &= \sum_{i,j \in \mathcal{Y}_l}\underset{\substack{\bar{x}_l \sim \mathcal{P}_{{l_i}}, \bar{x}'_l \sim \mathcal{P}_{{l_j}},\\x \sim \mathcal{T}(\cdot|\bar{x}_l), x^{-} \sim \mathcal{T}(\cdot|\bar{x}'_l)}}{\mathbb{E}}
\left[\left(f(x)^{\top} {f}\left(x^{-}\right)\right)^2\right], \\
    \mathcal{L}_4(f) &= \sum_{i \in \mathcal{Y}_l}\underset{\substack{\bar{x}_l \sim \mathcal{P}_{{l_i}}, \bar{x}_u \sim \mathcal{P},\\x \sim \mathcal{T}(\cdot|\bar{x}_l), x^{-} \sim \mathcal{T}(\cdot|\bar{x}_u)}}{\mathbb{E}}
\left[\left(f(x)^{\top} {f}\left(x^{-}\right)\right)^2\right], \\
    \mathcal{L}_5(f) &= \underset{\substack{\bar{x}_u \sim \mathcal{P}, \bar{x}'_u \sim \mathcal{P},\\x \sim \mathcal{T}(\cdot|\bar{x}_u), x^{-} \sim \mathcal{T}(\cdot|\bar{x}'_u)}}{\mathbb{E}}
\left[\left(f(x)^{\top} {f}\left(x^{-}\right)\right)^2\right].
\end{align*}
\end{theorem}
\begin{proof} We can expand $\mathcal{L}_{\mathrm{mf}}(F, A)$ and obtain
\begin{align*}
\mathcal{L}_{\mathrm{mf}}(F, A) = &\sum_{x, x^{\prime} \in \mathcal{X}}\left(\frac{w_{x x^{\prime}}}{\sqrt{w_x w_{x^{\prime}}}}-\*f_x^{\top} \*f_{x^{\prime}}\right)^2 \\
= & \text{const} + \sum_{x, x^{\prime} \in \mathcal{X}}\left(-2 w_{x x^{\prime}} f(x)^{\top} {f}\left(x^{\prime}\right)+w_x w_{x^{\prime}}\left(f(x)^{\top}{f}\left(x^{\prime}\right)\right)^2\right),
\end{align*} 
where $\*f_x = \sqrt{w_x}f(x)$ is a re-scaled version of $f(x)$.
At a high level, we follow the proof in ~\citep{haochen2021provable}, while the specific form of loss varies with the different definitions of positive/negative pairs. The form of $\mathcal{L}_\text{SORL}(f)$ is derived from plugging $w_{xx'}$  and $w_x$. 

Recall that $w_{xx'}$ is defined by
\begin{align*}
w_{x x^{\prime}} &= \eta_{u} \sum_{i \in \mathcal{Y}_l}\mathbb{E}_{\bar{x}_{l} \sim {\mathcal{P}_{l_i}}} \mathbb{E}_{\bar{x}'_{l} \sim {\mathcal{P}_{l_i}}} \mathcal{T}(x | \bar{x}_{l}) \mathcal{T}\left(x' | \bar{x}'_{l}\right)+ \eta_{l} \mathbb{E}_{\bar{x}_{u} \sim {\mathcal{P}}} \mathcal{T}(x| \bar{x}_{u}) \mathcal{T}\left(x'| \bar{x}_{u}\right) ,
\end{align*}
and $w_{x}$ is given by 
\begin{align*}
w_{x } &= \sum_{x^{\prime}}w_{xx'} \\ &=\eta_{u} \sum_{i \in \mathcal{Y}_l}\mathbb{E}_{\bar{x}_{l} \sim {\mathcal{P}_{l_i}}} \mathbb{E}_{\bar{x}'_{l} \sim {\mathcal{P}_{l_i}}} \mathcal{T}(x | \bar{x}_{l}) \sum_{x^{\prime}} \mathcal{T}\left(x' | \bar{x}'_{l}\right)+ \eta_{l} \mathbb{E}_{\bar{x}_{u} \sim {\mathcal{P}}} \mathcal{T}(x| \bar{x}_{u}) \sum_{x^{\prime}} \mathcal{T}\left(x' | \bar{x}_{u}\right) \\
&= \eta_{u} \sum_{i \in \mathcal{Y}_l}\mathbb{E}_{\bar{x}_{l} \sim {\mathcal{P}_{l_i}}} \mathcal{T}(x | \bar{x}_{l}) + \eta_{l} \mathbb{E}_{\bar{x}_{u} \sim {\mathcal{P}}} \mathcal{T}(x| \bar{x}_{u}). 
\end{align*}

Plugging in $w_{x x^{\prime}}$ we have, 

\begin{align*}
    &-2 \sum_{x, x^{\prime} \in \mathcal{X}} w_{x x^{\prime}} f(x)^{\top} {f}\left(x^{\prime}\right) \\
    = & -2 \sum_{x, x^{+} \in \mathcal{X}} w_{x x^{+}} f(x)^{\top} {f}\left(x^{+}\right) 
    \\ = & -2\eta_{u}  \sum_{i \in \mathcal{Y}_l}\mathbb{E}_{\bar{x}_{l} \sim {\mathcal{P}_{l_i}}} \mathbb{E}_{\bar{x}'_{l} \sim {\mathcal{P}_{l_i}}} \sum_{x, x^{\prime} \in \mathcal{X}} \mathcal{T}(x | \bar{x}_{l}) \mathcal{T}\left(x' | \bar{x}'_{l}\right) f(x)^{\top} {f}\left(x^{\prime}\right) \\
    & -2 \eta_{l} \mathbb{E}_{\bar{x}_{u} \sim {\mathcal{P}}} \sum_{x, x^{\prime}} \mathcal{T}(x| \bar{x}_{u}) \mathcal{T}\left(x'| \bar{x}_{u}\right) f(x)^{\top} {f}\left(x^{\prime}\right) 
    \\ = & -2\eta_{u}  \sum_{i \in \mathcal{Y}_l}\underset{\substack{\bar{x}_{l} \sim \mathcal{P}_{{l_i}}, \bar{x}'_{l} \sim \mathcal{P}_{{l_i}},\\x \sim \mathcal{T}(\cdot|\bar{x}_{l}), x^{+} \sim \mathcal{T}(\cdot|\bar{x}'_l)}}{\mathbb{E}}  \left[f(x)^{\top} {f}\left(x^{+}\right)\right] \\
    & - 2\eta_{l} 
    \underset{\substack{\bar{x}_{u} \sim \mathcal{P},\\x \sim \mathcal{T}(\cdot|\bar{x}_{u}), x^{+} \sim \mathcal{T}(\cdot|\bar{x}_u)}}{\mathbb{E}}
\left[f(x)^{\top} {f}\left(x^{+}\right)\right] \\ =& - 2\eta_{u}  \mathcal{L}_1(f) 
- 2\eta_{l}  \mathcal{L}_2(f).
\end{align*}

Plugging $w_{x}$ and $w_{x'}$ we have, 

\begin{align*}
    &\sum_{x, x^{\prime} \in \mathcal{X}}w_x w_{x^{\prime}}\left(f(x)^{\top}{f}\left(x^{\prime}\right)\right)^2 \\
    = &  \sum_{x, x^{-} \in \mathcal{X}}w_x w_{x^{-}}\left(f(x)^{\top}{f}\left(x^{-}\right)\right)^2 \\
    = & \sum_{x, x^{\prime} \in \mathcal{X}} \left( \eta_{u} \sum_{i \in \mathcal{Y}_l}\mathbb{E}_{\bar{x}_{l} \sim {\mathcal{P}_{l_i}}} \mathcal{T}(x | \bar{x}_{l}) + \eta_{l} \mathbb{E}_{\bar{x}_{u} \sim {\mathcal{P}}} \mathcal{T}(x| \bar{x}_{u}) \right)  \\&~~~~~~~~~~~\cdot \left(\eta_{u}  \sum_{j \in \mathcal{Y}_l}\mathbb{E}_{\bar{x}'_{l} \sim {\mathcal{P}_{l_j}}} \mathcal{T}(x^{-} | \bar{x}'_{l}) + \eta_{l} \mathbb{E}_{\bar{x}'_{u} \sim {\mathcal{P}}} \mathcal{T}(x^{-}| \bar{x}'_{u}) \right) \left(f(x)^{\top}{f}\left(x^{-}\right)\right)^2 \\
    = & \eta_{u} ^ 2 \sum_{x, x^{-} \in \mathcal{X}}  \sum_{i \in \mathcal{Y}_l}\mathbb{E}_{\bar{x}_{l} \sim {\mathcal{P}_{l_i}}} \mathcal{T}(x | \bar{x}_{l}) \sum_{j \in \mathcal{Y}_l}\mathbb{E}_{\bar{x}'_{l} \sim {\mathcal{P}_{l_j}}} \mathcal{T}(x^{-} | \bar{x}'_{l})\left(f(x)^{\top}{f}\left(x^{-}\right)\right)^2 \\
    &+ 2\eta_{u} \eta_{l} \sum_{x, x^{-} \in \mathcal{X}} \sum_{i \in \mathcal{Y}_l}\mathbb{E}_{\bar{x}_{l} \sim {\mathcal{P}_{l_i}}} \mathcal{T}(x | \bar{x}_{l})  \mathbb{E}_{\bar{x}_{u} \sim {\mathcal{P}}} \mathcal{T}(x^{-}| \bar{x}_{u}) \left(f(x)^{\top}{f}\left(x^{-}\right)\right)^2  \\
    &+ \eta_{l}^2 \sum_{x, x^{-} \in \mathcal{X}}  \mathbb{E}_{\bar{x}_{u} \sim {\mathcal{P}}} \mathcal{T}(x| \bar{x}_{u}) \mathbb{E}_{\bar{x}'_{u} \sim {\mathcal{P}}} \mathcal{T}(x^{-}| \bar{x}'_{u}) \left(f(x)^{\top}{f}\left(x^{-}\right)\right)^2 \\
    = & \eta_{u}^2 \sum_{i \in \mathcal{Y}_l}\sum_{j \in \mathcal{Y}_l}\underset{\substack{\bar{x}_l \sim \mathcal{P}_{{l_i}}, \bar{x}'_l \sim \mathcal{P}_{{l_j}},\\x \sim \mathcal{T}(\cdot|\bar{x}_l), x^{-} \sim \mathcal{T}(\cdot|\bar{x}'_l)}}{\mathbb{E}}
\left[\left(f(x)^{\top} {f}\left(x^{-}\right)\right)^2\right] \\
&+ 2\eta_{u}\eta_{l}
    \sum_{i \in \mathcal{Y}_l}\underset{\substack{\bar{x}_l \sim \mathcal{P}_{{l_i}}, \bar{x}_u \sim \mathcal{P},\\x \sim \mathcal{T}(\cdot|\bar{x}_l), x^{-} \sim \mathcal{T}(\cdot|\bar{x}_u)}}{\mathbb{E}}
\left[\left(f(x)^{\top} {f}\left(x^{-}\right)\right)^2\right] \\ &+ \eta_{l}^2
     \underset{\substack{\bar{x}_u \sim \mathcal{P}, \bar{x}'_u \sim \mathcal{P},\\x \sim \mathcal{T}(\cdot|\bar{x}_u), x^{-} \sim \mathcal{T}(\cdot|\bar{x}'_u)}}{\mathbb{E}}
\left[\left(f(x)^{\top} {f}\left(x^{-}\right)\right)^2\right] 
\\ = & \eta_{u}^2 \mathcal{L}_3(f) + 2\eta_{u}\eta_{l} \mathcal{L}_4(f) + \eta_{l}^2\mathcal{L}_5(f).
\end{align*}
\end{proof}

\newpage
\subsection{Technical Details for Toy Example}
\label{sec:sorl_sup_toy}

\subsubsection{Calculation Details for Figure~\ref{fig:sorl_toy_setting}.}

We first recap the toy example, which illustrates the core idea of our theoretical findings. Specifically, the example aims to distinguish 3D objects with different shapes, as shown in Figure~\ref{fig:sorl_toy_setting}. These images are generated by a 3D rendering software~\citep{johnson2017clevr} with user-defined properties including colors, shape, size, position, etc. 
% We are interested in contrasting the representations (in the form of singular vectors) when  the label information is either incorporated in training or not. 

\noindent \textbf{Data design.} Suppose the training samples come from three types, $\mathcal{X}_{\cube{1}}$, $\mathcal{X}_{\sphere{0.7}{gray}}$, $\mathcal{X}_{\cylinder{0.6}}$. Let $\mathcal{X}_{\cube{1}}$ be the sample space with \textbf{known} class, and $\mathcal{X}_{\sphere{0.7}{gray}}, \mathcal{X}_{\cylinder{0.6}}$ be the sample space with \textbf{novel} classes. Further, the two novel classes are constructed to have different relationships with the known class. 
Specifically, we construct the toy dataset with 6 elements as shown in Figure~\ref{fig:sorl_sup_toy}(a).

\textbf{Augmentation graph.} Based on the data design, we formally define the augmentation graph, which encodes the probability of augmenting a source image $\bar{x}$ to the augmented view $x$:
\begin{align}
    \mathcal{T}\left(x \mid \bar{x} \right)=
    \left\{\begin{array}{ll}
    \tau_{1} & \text { if }  \text{color}(x) = \text{color}(\bar{x}), \text{shape}(x) = \text{shape}(\bar{x}); \\
    \tau_{c} & \text { if }  \text{color}(x) = \text{color}(\bar{x}), \text{shape}(x) \neq \text{shape}(\bar{x}); \\
    \tau_{s} & \text { if }  \text{color}(x) \neq \text{color}(\bar{x}), \text{shape}(x) = \text{shape}(\bar{x}); \\
    \tau_{0} & \text { if }  \text{color}(x) \neq \text{color}(\bar{x}), \text{shape}(x) \neq \text{shape}(\bar{x}). \\
    \end{array}\right.
    \label{eq:sorl_sup_def_edge}
\end{align}

According to the definition above,  the corresponding augmentation matrix $T$ with each element formed by $\mathcal{T}(\cdot \mid \cdot)$ is given in Figure~\ref{fig:sorl_sup_toy}(b). We proceed by showing the details to derive $A^{(u)}$ and $A$ using $T$. 
%%%%%%%%%%%%%%%%%%%%%% Figure toy setting  %%%%%%%%%%%%%%%%%%%%%% 
\begin{figure*}[htb]
    \centering
\includegraphics[width=0.9\linewidth]{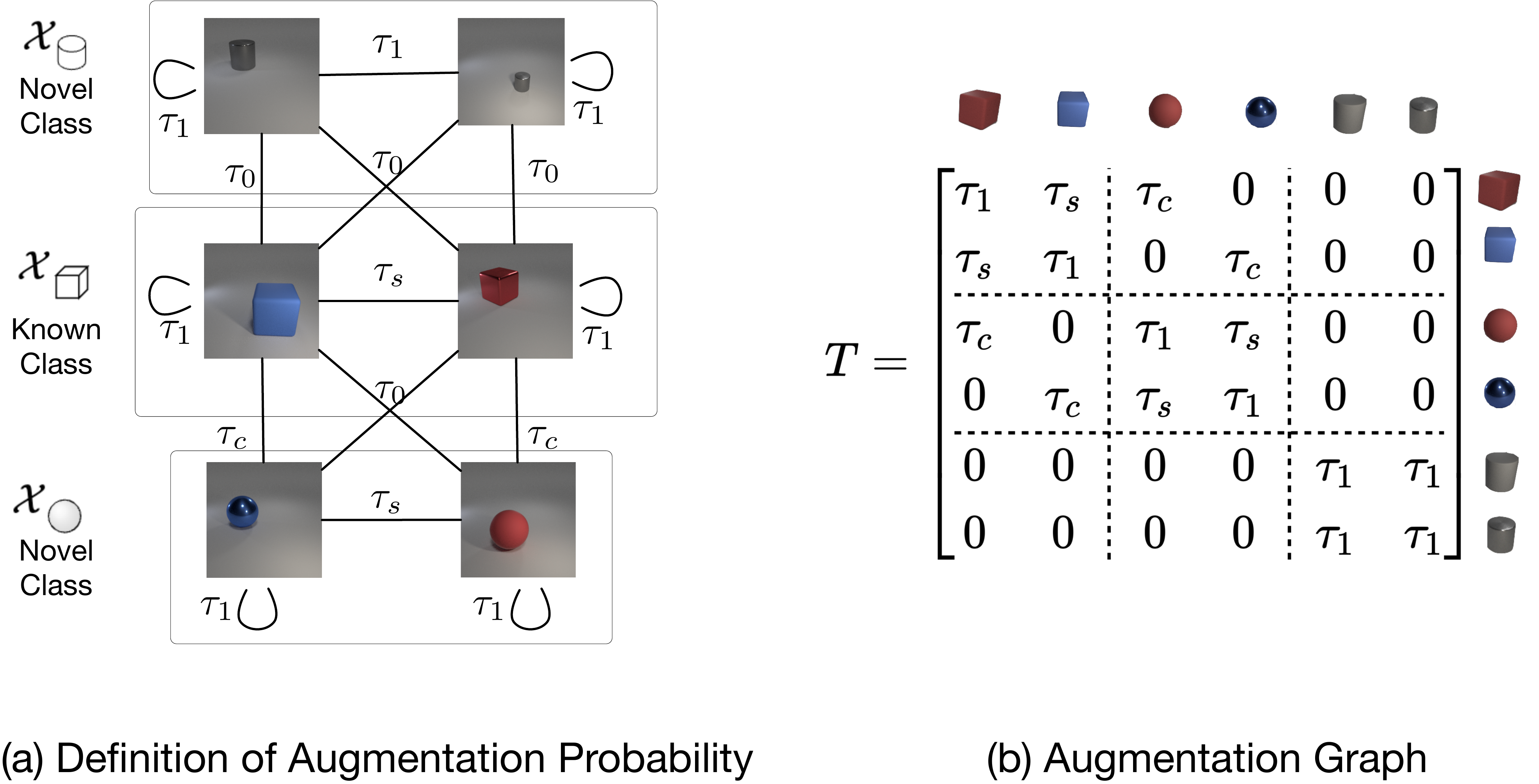}
    \caption[An illustrative example for theoretical analysis (appendix).]{An illustrative example for theoretical analysis. We consider a 6-node graph with one known class (cube) and two novel classes (sphere, cylinder). (a) The augmentation probabilities between nodes are defined by their color and shape in Eq.~\eqref{eq:sorl_sup_def_edge}. (b) The augmentation matrices $T$ derived by Eq.~\eqref{eq:sorl_sup_def_edge} where we let $\tau_0=0$.  }
    \vspace{-0.2cm}
    \label{fig:sorl_sup_toy}
\end{figure*}
%%%%%%%%%%%%%%%%%%%%%%%%%%%%%%%%%%%%%%%%%%%%%%%%%%%%%%%%%%%%%%%%% 

\textbf{Derivation details for $A^{(u)}$ and $A$.} 
Recall that each element of $A^{(u)}$ is formed by 
$
w^{(u)}_{x x^{\prime}} = \mathbb{E}_{\bar{x} \sim {\mathcal{P}}}  \mathcal{T}(x| \bar{x}) \mathcal{T}\left(x'| \bar{x}\right).
$
In this toy example, one can then see that $A^{(u)} = \frac{1}{6} T T^{\top}$ since augmentation matrix $T$ is defined that each element $T_{x\bar{x}}=\mathcal{T}(x| \bar{x})$. Note that $T$ is explicitly given in Figure~\ref{fig:sorl_sup_toy}(b) and then if we let $\eta_u = 6$, we have the close-from: 

$$\eta_u A^{(u)}= T^2 = \left[\begin{array}{cccccc}\tau_1^2+\tau_s^2+\tau_c^2 & 2 \tau_1\tau_s & 2 \tau_1\tau_c & 2 \tau_c \tau_s & 0 & 0 \\ 2 \tau_1\tau_s & \tau_1^2+\tau_s^2+\tau_c^2 & 2 \tau_c \tau_s & 2 \tau_1\tau_c & 0 & 0 \\  2 \tau_1\tau_c & 2 \tau_c \tau_s & \tau_1^2+\tau_s^2+\tau_c^2 & 2 \tau_1\tau_s & 0 & 0 \\ 2 \tau_c \tau_s & 2 \tau_1\tau_c & 2 \tau_1\tau_s & \tau_1^2+\tau_s^2+\tau_c^2 & 0 & 0 \\  0 & 0 & 0 & 0 & 2 \tau_1^2 & 2 \tau_1^2 \\ 0 & 0 & 0 & 0 & 2 \tau_1^2 & 2 \tau_1^2\end{array}\right].$$

We then derive the second part $A^{(l)}$ whose element is given by:
$$w^{(l)}_{x x^{\prime}} \triangleq \sum_{i \in \mathcal{Y}_l}\mathbb{E}_{\bar{x}_{l} \sim {\mathcal{P}_{l_i}}} \mathbb{E}_{\bar{x}'_{l} \sim {\mathcal{P}_{l_i}}} \mathcal{T}(x | \bar{x}_{l}) \mathcal{T}\left(x' | \bar{x}'_{l}\right).$$ 
Such a form can be simplified in Section~\ref{sec:sorl_theory} by defining $\mathfrak{l} \in \mathbb{R}^{N}, (\mathfrak{l})_x = \mathbb{E}_{\bar{x}_{l} \sim {\mathcal{P}_{l_1}}} \mathcal{T}(x | \bar{x}_{l})$ and by letting $|\mathcal{Y}_l|=1$. 
In this toy example, the known class only has two elements, so $\mathfrak{l} = \frac{1}{2}(T_{:, 1}+T_{:, 2})$ (average of $T$'s 1st \& 2nd column), we then have:
$$
A^{(l)} = \mathfrak{l} \mathfrak{l}^{\top} = \frac{1}{4}\left[\begin{array}{cccccc}\left(\tau_1+\tau_s\right)^2 & \left(\tau_1+\tau_s\right)^2 & \tau_c\left(\tau_1+\tau_s\right) & \tau_c\left(\tau_1+\tau_s\right) & 0 & 0 \\ \left(\tau_1+\tau_s\right)^2 & \left(\tau_1+\tau_s\right)^2 & \tau_c\left(\tau_1+\tau_s\right) & \tau_c\left(\tau_1+\tau_s\right) & 0 & 0 \\  \tau_c\left(\tau_1+\tau_s\right) & \tau_c\left(\tau_1+\tau_s\right) & \tau_c^2 & \tau_c^2 & 0 & 0 \\ \tau_c\left(\tau_1+\tau_s\right) & \tau_c\left(\tau_1+\tau_s\right) & \tau_c^2 & \tau_c^2 & 0 & 0 \\  0 & 0 & 0 & 0 & 0 & 0 \\ 0 & 0 & 0 & 0 & 0 & 0\end{array}\right].
$$
Finally, if we let $\eta_{l} = 4$ and $A =  \eta_{u} A^{(u)} + \eta_{l} A^{(l)}$, we have the full results in Figure~\ref{fig:sorl_toy_setting}.

\subsubsection{Calculation Details for Figure~\ref{fig:sorl_toy_result}.}
In this section, we present the analysis of eigenvectors and their orders for toy examples shown in Figure~\ref{fig:sorl_toy_setting}. In Theorem~\ref{th:sorl_sup_toy_label} we present the spectral analysis for the adjacency matrix with additional label information while in Theorem~\ref{th:sorl_sup_toy_base}, we show the spectral 
 analysis for the unlabeled case. 

\begin{theorem}\label{th:sorl_sup_toy_label} Let 
$$ \eta_u A^{(u)}  = \left[\begin{array}{cccccc}
\tau_1^2+\tau_s^2+\tau_c^2 & 2\tau_1\tau_s & 2\tau_1\tau_c & 2\tau_c\tau_s & 0 & 0\\
    2\tau_1\tau_s & \tau_1^2+\tau_s^2+\tau_c^2 & 2\tau_c\tau_s & 2\tau_1\tau_c & 0 & 0\\
    2\tau_1\tau_c & 2\tau_c\tau_s & \tau_1^2+\tau_s^2+\tau_c^2 & 2\tau_1\tau_s & 0 & 0\\
    2\tau_c\tau_s & 2\tau_1\tau_c & 2\tau_1\tau_s & \tau_1^2+\tau_s^2+\tau_c^2 & 0 & 0\\
    0 & 0 & 0 & 0 & 2\tau_1^2 & 2\tau_1^2\\
    0 & 0 & 0 & 0 & 2\tau_1^2 & 2\tau_1^2\\
\end{array}\right],$$ 

$$  A  = \eta_u A^{(u)} + \left[\begin{array}{cccccc}
    (\tau_1+\tau_s)^2 & (\tau_1+\tau_s)^2 & \tau_c(\tau_1+\tau_s) & \tau_c(\tau_1+\tau_s) & 0 & 0\\
    (\tau_1+\tau_s)^2 & (\tau_1+\tau_s)^2 & \tau_c(\tau_1+\tau_s) & \tau_c(\tau_1+\tau_s) & 0 & 0\\
    \tau_c(\tau_1+\tau_s) & \tau_c(\tau_1+\tau_s) & \tau_c^2 & \tau_c^2 & 0 & 0\\
    \tau_c(\tau_1+\tau_s) & \tau_c(\tau_1+\tau_s) & \tau_c^2 & \tau_c^2 & 0 & 0\\
    0 & 0 & 0 & 0 & 0 & 0\\
    0 & 0 & 0 & 0 & 0 & 0\\
\end{array}\right], $$ 
and we assume that $1 \gg {\tau_{c} \over \tau_1} > {\tau_{s} \over \tau_1} > 0$, $ {4 \over 9} \tau_c \le \tau_s \le \tau_c$ and $\tau_1 + \tau_c + \tau_s = 1$. 

Let $\lambda_1, \lambda_2, \lambda_3$ and $v_1,v_2,v_3$ be the largest three eigenvalues and their corresponding eigenvectors of ${D^{{{-{1\over 2}}}}} {A} {D^{{{-{1\over 2}}}}}$, which is the normalized adjacency matrix of $A$. Then the concrete form of $\lambda_1, \lambda_2, \lambda_3$ and $v_1,v_2,v_3$ can be approximately given by: 
\begin{align*}
    &\hat{\lambda}_1 = 1, ~~\hat{\lambda}_2 = 1, ~~\hat{\lambda}_3 = 1-{16\over 3} {\tau_c \over \tau_1}, \\
    & \hat{v}_1 = [0,0,0,0,1,1], \\
    & \hat{v}_2 = [\sqrt{3},\sqrt{3},1,1,0,0], \\
    &\hat{v}_3 = [1,1,-\sqrt{3},-\sqrt{3},0,0].
\end{align*}
Note that the approximation gap can be tightly bounded. Specifically, for $i \in \{1,2,3\}$, we have $|\lambda_i - \hat{\lambda}_i| \le O(({\tau_c \over \tau_1})^2 )$ and $\|\sin(U, \hat{U})\footnote{The $\sin$ operation measures the distance of two matrices with orthonormal columns, which is usually used in the subspace distance. See more in \url{https://trungvietvu.github.io/notes/2020/DavisKahan}.}\|_F \le O({\tau_c \over \tau_1})$, where  $U = [v_1,v_2,v_3], \hat{U} = [\hat{v}_1, \hat{v}_2, \hat{v}_3]$.
\end{theorem}
\begin{proof}
By $\tau_1 + \tau_c + \tau_s = 1$ and $1 \gg {\tau_{c} \over \tau_1} > {\tau_{s} \over \tau_1} > 0$, 
we define the following equation which approximates the corresponding terms up to error $O(({\tau_c \over \tau_1})^2)$: 
$$ A \approx \widehat{A}  =  \tau_1^2 \left[\begin{array}{cccccc}
    2+2{\tau_s \over \tau_1} & 1+4{\tau_s \over \tau_1} & 3{\tau_c \over \tau_1} & {\tau_c \over \tau_1} & 0 & 0\\
    1+4{\tau_s \over \tau_1} & 2+2{\tau_s \over \tau_1} & {\tau_c \over \tau_1} & 3{\tau_c \over \tau_1} & 0 & 0\\
    3{\tau_c \over \tau_1} & {\tau_c \over \tau_1} & 1 & 2{\tau_s \over \tau_1} & 0 & 0\\
    {\tau_c \over \tau_1} & 3{\tau_c \over \tau_1} & 2{\tau_s \over \tau_1} & 1 & 0 & 0\\
    0 & 0 & 0 & 0 & 2 & 2\\
    0 & 0 & 0 & 0 & 2 & 2\\
\end{array}\right].$$ 

\begin{equation*}
\resizebox{1.1\linewidth}{!}{$ D \approx \widehat{D}  =  \tau_1^2  diag\left(\left[3\left(1+2{\tau_s \over \tau_1}+{4\over 3}{\tau_c \over \tau_1} \right), 3\left(1+2{\tau_s \over \tau_1}+{4\over 3}{\tau_c \over \tau_1} \right), 1+ 2{\tau_s \over \tau_1}+4{\tau_c \over \tau_1}, 1+ 2{\tau_s \over \tau_1}+4{\tau_c \over \tau_1}, 4,  4 \right]\right). $}
\end{equation*} 
\begin{equation*}
\resizebox{1.1\linewidth}{!}{$ D^{{{-{1\over 2}}}} \approx \widehat{D^{{{-{1\over 2}}}}}  =  {1\over \tau_1}  diag\left(\left[\sqrt{3}\left(1-{\tau_s \over \tau_1}-{2\over 3}{\tau_c \over \tau_1} \right), \sqrt{3}\left(1-{\tau_s \over \tau_1}-{2\over 3}{\tau_c \over \tau_1} \right), 1- {\tau_s \over \tau_1}-2{\tau_c \over \tau_1}, 1-{\tau_s \over \tau_1}-2{\tau_c \over \tau_1}, 2,  2 \right]\right).$}
\end{equation*} 

\begin{align*}
     &D^{{{-{1\over 2}}}} A D^{{{-{1\over 2}}}}\approx \widehat{D^{{{-{1\over 2}}}}} \widehat{A} \widehat{D^{{{-{1\over 2}}}}}  \\ &=   \left[\begin{array}{cccccc}
    {2\over 3}\left(1-{\tau_s \over \tau_1} - {4\over 3}{\tau_c \over \tau_1}\right) & {1\over 3}\left(1+2{\tau_s \over \tau_1} - {4\over 3}{\tau_c \over \tau_1}\right) & \sqrt{3}{\tau_c \over \tau_1} & {1\over \sqrt{3}}{\tau_c \over \tau_1} & 0 & 0\\
    {1\over 3}\left(1+2{\tau_s \over \tau_1} - {4\over 3}{\tau_c \over \tau_1}\right) & {2\over 3}\left(1-{\tau_s \over \tau_1} - {4\over 3}{\tau_c \over \tau_1}\right) & {1\over \sqrt{3}}{\tau_c \over \tau_1} & \sqrt{3}{\tau_c \over \tau_1} & 0 & 0\\
    \sqrt{3}{\tau_c \over \tau_1} & {1\over \sqrt{3}}{\tau_c \over \tau_1} & 1-2{\tau_s \over \tau_1} - 4{\tau_c \over \tau_1} & 2{\tau_s \over \tau_1} & 0 & 0\\
    {1\over \sqrt{3}}{\tau_c \over \tau_1} & \sqrt{3}{\tau_c \over \tau_1} & 2{\tau_s \over \tau_1} & 1-2{\tau_s \over \tau_1} - 4{\tau_c \over \tau_1} & 0 & 0\\
    0 & 0 & 0 & 0 & {1\over 2} & {1\over 2}\\
    0 & 0 & 0 & 0 & {1\over 2} & {1\over 2}\\
\end{array}\right].
\end{align*}

And we have 
\begin{align*}
& \left\|{D^{{{-{1\over 2}}}}} {A} {D^{{{-{1\over 2}}}}} - \widehat{D^{{{-{1\over 2}}}}} \widehat{A} \widehat{D^{{{-{1\over 2}}}}}\right\|_2 \\
\le & \left\|{D^{{{-{1\over 2}}}}} {A} {D^{{{-{1\over 2}}}}} - \widehat{D^{{{-{1\over 2}}}}} \widehat{A} \widehat{D^{{{-{1\over 2}}}}}\right\|_F \\
\le &  O(({\tau_c \over \tau_1})^2 ).
\end{align*}
Let $\hat{\lambda}_a, \dots, \hat{\lambda}_f $ be six eigenvalues of $\widehat{D^{{{-{1\over 2}}}}} \widehat{A} \widehat{D^{{{-{1\over 2}}}}}$, and $\hat{v}_a, \dots, \hat{v}_f $ be corresponding eigenvectors. By direct calculation we have 
$$
\hat{\lambda}_a = 1, ~~\hat{\lambda}_b = 1, ~~\hat{\lambda}_c = 1-{16\over 3} {\tau_c \over \tau_1}, ~~\hat{\lambda}_d = 0
$$ 
and corresponding eigenvectors as 
\begin{align*}
    & \hat{v}_a = [0,0,0,0,1,1], \\
    & \hat{v}_b = [\sqrt{3},\sqrt{3},1,1,0,0], \\
    &\hat{v}_c = [1,1,-\sqrt{3},-\sqrt{3},0,0], \\
    &\hat{v}_d = [0,0,0,0,1,-1].
\end{align*}
For the remaining two eigenvectors, by the symmetric property, they have the formula 
\begin{align*}
    & \hat{v}_e = [\alpha({\tau_s \over \tau_1},{\tau_c \over \tau_1}), -\alpha({\tau_s \over \tau_1},{\tau_c \over \tau_1}), \beta({\tau_s \over \tau_1},{\tau_c \over \tau_1}), -\beta({\tau_s \over \tau_1},{\tau_c \over \tau_1}), 0, 0], \\
    & \hat{v}_f = [\beta({\tau_s \over \tau_1},{\tau_c \over \tau_1}), -\beta({\tau_s \over \tau_1},{\tau_c \over \tau_1}), -\alpha({\tau_s \over \tau_1},{\tau_c \over \tau_1}), \alpha({\tau_s \over \tau_1},{\tau_c \over \tau_1}), 0, 0],
\end{align*}
where $\alpha, \beta$ are some real functions. Then, by solving 
\begin{align*}
    & \widehat{D^{{{-{1\over 2}}}}} \widehat{A} \widehat{D^{{{-{1\over 2}}}}} \hat{v}_e = \hat{\lambda}_e\hat{v}_e \\
    & \widehat{D^{{{-{1\over 2}}}}} \widehat{A} \widehat{D^{{{-{1\over 2}}}}} \hat{v}_f = \hat{\lambda}_f\hat{v}_f, 
\end{align*}
we get 
\begin{align*}
    \hat{\lambda}_e = {1\over 9} \left( \sqrt{(3- 12 {\tau_s \over \tau_1} - 16 {\tau_c \over \tau_1})^2   + 108({\tau_c \over \tau_1})^2  } -24{\tau_s \over \tau_1} -20{\tau_c \over \tau_1} + 6\right)\\
    \hat{\lambda}_f = {1\over 9} \left( -\sqrt{(3- 12 {\tau_s \over \tau_1} - 16 {\tau_c \over \tau_1})^2   + 108({\tau_c \over \tau_1})^2  } -24{\tau_s \over \tau_1} -20{\tau_c \over \tau_1} + 6\right).
\end{align*}
Now, we show that $\hat{\lambda}_c > \hat{\lambda}_e$. By ${\tau_c \over \tau_1} \ll 1$ and $ {4 \over 9} \tau_c \le \tau_s \le \tau_c$
\begin{align*}
    \hat{\lambda}_c \ge \hat{\lambda}_e \Leftrightarrow~ & 3  + 24 {\tau_s \over \tau_1} - 28 {\tau_c \over \tau_1} \ge \sqrt{(3- 12 {\tau_s \over \tau_1} - 16 {\tau_c \over \tau_1})^2   + 108({\tau_c \over \tau_1})^2  } \\
    \Leftrightarrow~ & 36 ({\tau_s \over \tau_1})^2 + 35({\tau_c \over \tau_1})^2  - 144  {\tau_s \over \tau_1}{\tau_c \over \tau_1} + 18{\tau_s \over \tau_1} - 6 {\tau_c \over \tau_1} \ge 0.
\end{align*}
Thus, we have $1 = \hat{\lambda}_a = \hat{\lambda}_b > \hat{\lambda}_c > \hat{\lambda}_e > \hat{\lambda}_f > \hat{\lambda}_d = 0$. Moreover, we also have
\begin{align*}
    \hat{\lambda}_c - \hat{\lambda}_e = & 1-{16\over 3} {\tau_c \over \tau_1} - {1\over 9} \left( \sqrt{(3- 12 {\tau_s \over \tau_1} - 16 {\tau_c \over \tau_1})^2   + 108({\tau_c \over \tau_1})^2  } -24{\tau_s \over \tau_1} -20{\tau_c \over \tau_1} + 6\right)\\
    \ge & \Omega\left({\tau_c \over \tau_1}\right).
\end{align*}

Let $\hat{\lambda}_1 = \hat{\lambda}_a, \hat{\lambda}_2 = \hat{\lambda}_b, \hat{\lambda}_3 = \hat{\lambda}_c$. 
Then, by Weyl's Theorem, for $i \in \{1,2,3\}$, we have 
$$|\lambda_i - \hat{\lambda}_i| \le \left\|{D^{{{-{1\over 2}}}}} {A} {D^{{{-{1\over 2}}}}} - \widehat{D^{{{-{1\over 2}}}}} \widehat{A} \widehat{D^{{{-{1\over 2}}}}}\right\|_2 \le O(({\tau_c \over \tau_1})^2 ).  $$
By Davis-Kahan theorem, we have 
\begin{align*}
    \|\sin(U, \hat{U})\|_F \le { O(({\tau_c \over \tau_1})^2 )\over \Omega\left({\tau_c \over \tau_1}\right)} \le O({\tau_c \over \tau_1}).
\end{align*}
We finish the proof. 
\end{proof}

\begin{theorem}\label{th:sorl_sup_toy_base}
Recall $\eta_u A^{(u)}$ is defined in Theorem~\ref{th:sorl_sup_toy_label}.
Assume $1 \gg {\tau_{c} \over \tau_1} > {\tau_{s} \over \tau_1} > 0$ and $\tau_1 + \tau_c + \tau_s = 1$. Let $\lambda_1^{(u)}, \lambda_2^{(u)}, \lambda_3^{(u)}$ and $v_1^{(u)},v_2^{(u)},v_3^{(u)}$ be the largest three eigenvalues and their corresponding eigenvectors of $ {D^{(u){{-{1\over 2}}}}} {(\eta_u A^{(u)})} {D^{(u){{-{1\over 2}}}}}$, which is the normalized adjacency matrix of ${\eta_u A^{(u)}}$. Let 
\begin{align*}
    & \hat{\lambda}^{(u)}_1 = 1, ~~\hat{\lambda}^{(u)}_2 = 1, ~~\hat{\lambda}^{(u)}_3 = 1-4 {\tau_s \over \tau_1}, \\
    & \hat{v}^{(u)}_1 = [0,0,0,0,1,1], \\
    & \hat{v}^{(u)}_2 = [1,1,1,1,0,0], \\
    &\hat{v}^{(u)}_3 = [1,-1,1,-1,0,0].
\end{align*}
Let $U^{(u)} = [v^{(u)}_1,v^{(u)}_2,v^{(u)}_3], \hat{U}^{(u)} = [\hat{v}^{(u)}_1, \hat{v}^{(u)}_2, \hat{v}^{(u)}_3]$. Then, for $i \in \{1,2,3\}$, we have $|\lambda^{(u)}_i - \hat{\lambda}^{(u)}_i| \le O(({\tau_c \over \tau_1})^2 )$ and $\|\sin(U^{(u)}, \hat{U}^{(u)})\|_F \le O({\tau_c^2 \over \tau_1(\tau_c -\tau_s)})$.
\end{theorem}
\begin{proof}
Similar to the proof of Theorem~\ref{th:sorl_sup_toy_label},  up to error $O(({\tau_c \over \tau_1})^2)$, we have the following equation,

$$  \widehat{\eta_u A^{(u)}}  =  \tau_1^2 \left[\begin{array}{cccccc}
    1 & 2{\tau_s \over \tau_1} & 2{\tau_c \over \tau_1} & 0 & 0 & 0\\
    2{\tau_s \over \tau_1} & 1 & 0 & 2{\tau_c \over \tau_1} & 0 & 0\\
    2{\tau_c \over \tau_1} & 0 & 1 & 2{\tau_s \over \tau_1} & 0 & 0\\
    0 & 2{\tau_c \over \tau_1} & 2{\tau_s \over \tau_1} & 1 & 0 & 0\\
    0 & 0 & 0 & 0 & 2 & 2\\
    0 & 0 & 0 & 0 & 2 & 2\\
\end{array}\right].$$ 

$$
\widehat{D^{(u)}} = \tau_1^2 
 diag\left(\left[1+2{\tau_s \over \tau_1}+2{\tau_c \over \tau_1}, 1+2{\tau_s \over \tau_1}+2{\tau_c \over \tau_1}, 1+2{\tau_s \over \tau_1}+2{\tau_c \over \tau_1}, 1+2{\tau_s \over \tau_1}+2{\tau_c \over \tau_1}, 4,  4 \right]\right).
$$

$$
\widehat{D^{(u)-{1\over 2}}} = {1\over \tau_1} diag\left(\left[1-{\tau_s \over \tau_1}-{\tau_c \over \tau_1}, 1-{\tau_s \over \tau_1}-{\tau_c \over \tau_1}, 1-{\tau_s \over \tau_1}-{\tau_c \over \tau_1}, 1-{\tau_s \over \tau_1}-{\tau_c \over \tau_1}, 2,  2 \right]\right).
$$

\begin{align*}
    &\widehat{D^{(u)-{1\over 2}}} \widehat{\eta_u A^{(u)}} \widehat{D^{(u)-{1\over 2}}}  =\\   
 &~~\left[\begin{array}{cccccc}
    1-2{\tau_s \over \tau_1}-2{\tau_c \over \tau_1} & 2{\tau_s \over \tau_1} & 2{\tau_c \over \tau_1} & 0 & 0 & 0\\
    2{\tau_s \over \tau_1} & 1-2{\tau_s \over \tau_1}-2{\tau_c \over \tau_1} & 0 & 2{\tau_c \over \tau_1} & 0 & 0\\
    2{\tau_c \over \tau_1} & 0 & 1-2{\tau_s \over \tau_1}-2{\tau_c \over \tau_1} & 2{\tau_s \over \tau_1} & 0 & 0\\
    0 & 2{\tau_c \over \tau_1} & 2{\tau_s \over \tau_1} & 1-2{\tau_s \over \tau_1}-2{\tau_c \over \tau_1} & 0 & 0\\
    0 & 0 & 0 & 0 & {1\over 2} & {1\over 2}\\
    0 & 0 & 0 & 0 & {1\over 2} & {1\over 2}\\  
\end{array}\right].\end{align*}

Let $\hat{\lambda}^{(u)}_1, \dots, \hat{\lambda}^{(u)}_6 $ be six eigenvalue of $ \widehat{D^{(u)-{1\over 2}}} \widehat{\eta_u A^{(u)}} \widehat{D^{(u)-{1\over 2}}}$, and $\hat{v}^{(u)}_1, \dots, \hat{v}^{(u)}_6 $ be corresponding eigenvectors. By direct calculation we have 
$$
\hat{\lambda}^{(u)}_1 = 1, ~~\hat{\lambda}^{(u)}_2 = 1, ~~\hat{\lambda}^{(u)}_3 = 1-4 {\tau_s \over \tau_1}, ~~\hat{\lambda}^{(u)}_4 = 1-4 {\tau_c \over \tau_1}, ~~\hat{\lambda}^{(u)}_5 = 1-4 {\tau_s \over \tau_1} -4 {\tau_c \over \tau_1}, ~~\hat{\lambda}^{(u)}_6 = 0
$$ 
and corresponding eigenvector as 
\begin{align*}
    & \hat{v}^{(u)}_1 = [0,0,0,0,1,1], \\
    & \hat{v}^{(u)}_2 = [1,1,1,1,0,0], \\
    &\hat{v}^{(u)}_3 = [1,-1,1,-1,0,0], \\
    &\hat{v}^{(u)}_4 = [1,1,-1,-1,0,0], \\
    &\hat{v}^{(u)}_5 = [1,-1,-1,1,0,0], \\
    &\hat{v}^{(u)}_6 = [0,0,0,0,1,-1].
\end{align*}

Then, by Weyl's Theorem, for $i \in \{1,2,3\}$, we have 
$$|\lambda^{(u)}_i - \hat{\lambda}^{(u)}_i| \le \left\|{D^{(u){{-{1\over 2}}}}} {\eta_u A^{(u)}} {D^{(u){{-{1\over 2}}}}} - \widehat{D^{(u)-{1\over 2}}} \widehat{\eta_u A^{(u)}} \widehat{D^{(u)-{1\over 2}}} \right\|_2 \le O(({\tau_c \over \tau_1})^2 ).  $$
By Davis-Kahan theorem, we have 
\begin{align*}
    \|\sin(U^{(u)}, \hat{U}^{(u)})\|_F \le { O(({\tau_c \over \tau_1})^2 )\over 4({\tau_c \over \tau_1} - {\tau_s \over \tau_1})} \le O({\tau_c^2 \over \tau_1(\tau_c -\tau_s)}).
\end{align*}
We finish the proof. 
\end{proof}

\newpage
\subsection{Technical Details for Main Theory}

\subsubsection{Matrix Form of K-means and the Derivative}
\label{sec:sorl_sup_matrix_kmeans}
Recall that we defined the K-means clustering measure of features in Sec.~\ref{sec:sorl_theory}:
\begin{equation}
    \mathcal{M}_{kms}(\Pi, Z) = \sum_{\pi\in \Pi}  \sum_{i \in \pi}\left\|\*z_i-\boldsymbol{\mu}_\pi\right\|^2 / \sum_{\pi\in \Pi}  |\pi|\left\|\boldsymbol{\mu}_\pi-\boldsymbol{\mu}_\Pi\right\|^2,
\label{eq:sorl_sup_def_kms_measure}
\end{equation}
where the numerator measures the intra-class distance:
\begin{equation}
    \mathcal{M}_{intra}(\Pi, Z) = \sum_{\pi\in \Pi}  \sum_{i \in \pi}\left\|\*z_i-\boldsymbol{\mu}_\pi \right\|^2, 
\label{eq:sorl_sup_def_align_measure}
\end{equation}
and the denominator measures the inter-class distance: 
\begin{equation}
    \mathcal{M}_{inter}(\Pi, Z) =  \sum_{\pi\in \Pi}  |\pi|\left\|\boldsymbol{\mu}_\pi-\boldsymbol{\mu}_\Pi\right\|^2.
\label{eq:sorl_sup_def_sep_measure}
\end{equation}
We will show next how to convert the intra-class and the inter-class measures into a matrix form, which is desirable for analysis. 

\vspace{0.1cm} \noindent \textbf{Intra-class measure.} Note that the $K$-means intra-class measure can be rewritten in a matrix form:

$$\mathcal{M}_{intra}(\Pi, Z)= \|Z - H_\Pi Z\|^2_F,$$

where $H_\Pi$ is a matrix to convert $Z$ to mean vectors w.r.t clusters defined by $\Pi$. Without losing the generality, we assume $Z$ is ordered according to the partition in $\Pi$ --- first $|\pi_1|$ vectors are in $\pi_1$, next $|\pi_2|$ vectors are in $\pi_2$, etc. Then  $H_\Pi$ is given by: 

$$H_\Pi = \left[\begin{array}{cccc}
    \frac{1}{|\pi_1|}\mathbf{1}_{|\pi_1| \times |\pi_1|} & \mathbf{0} & ... & \mathbf{0}\\
     \mathbf{0} & \frac{1}{|\pi_2|}\mathbf{1}_{|\pi_2| \times |\pi_2|} & ... & \mathbf{0} \\
     ... & ... & ... & ... \\
     \mathbf{0} & \mathbf{0}  & ... & \frac{1}{|\pi_k|}\mathbf{1}_{|\pi_k| \times |\pi_k|}
\end{array}\right].$$

 Going further, we have:
\begin{align*}
    \mathcal{M}_{intra}(\Pi, Z) &= \|Z - H_\Pi Z\|^2_F \\
    &= \operatorname{Tr}((I-H_\Pi)^2ZZ^{\top}) \\
    &= \operatorname{Tr}((I-2H_\Pi + H_\Pi^2)ZZ^{\top}) \\
    &= \operatorname{Tr}((I-H_\Pi)ZZ^{\top}).
\end{align*}

\vspace{0.1cm} \noindent \textbf{Inter-class measure.}
The inter-class  measure can be equivalently given by: 
$$\mathcal{M}_{inter}(\Pi, Z)= \|H_\Pi Z - \frac{1}{N}\mathbf{1}_{N \times N} Z\|^2_F,$$
where $H_{\Pi}$ is defined as above. And we can also derive:
\begin{align*}
    \mathcal{M}_{inter}(\Pi, Z) &= \|H_\Pi Z - \frac{1}{N}\mathbf{1}_{N \times N} Z\|^2_F\\
    &= \operatorname{Tr}((H_\Pi - \frac{1}{N}\mathbf{1}_{N \times N})^2ZZ^{\top})\\
    &= \operatorname{Tr}((H_\Pi^2 - \frac{2}{N}H_\Pi \mathbf{1}_{N \times N} + \frac{1}{N^2} \mathbf{1}_{N \times N}^2)ZZ^{\top})\\
    &= \operatorname{Tr}((H_\Pi - \frac{1}{N} \mathbf{1}_{N \times N} )ZZ^{\top}).
\end{align*}

\subsubsection{K-means Measure Has the Same Order as K-means Error}
\label{sec:sorl_sup_cls_err}

\begin{theorem} 
We define the $\xi_{\pi \rightarrow \pi'}$ as the index of samples that is from class division $\pi$ however is closer to $\boldsymbol{\mu}_{\pi'}$ than $\boldsymbol{\mu}_{\pi}$. 
In other word, $\xi_{\pi \rightarrow \pi'} = \{i: i \in \pi, \|\*z_i - \boldsymbol{\mu}_{\pi}\|_2 \geq \|\*z_i - \boldsymbol{\mu}_{\pi'}\|_2\}$. Assuming $|\xi_{\pi \rightarrow \pi'}| > 0$, we define below the  clustering error ratio from $\pi$ to $\pi'$ as $\mathcal{E}_{\pi \rightarrow \pi'}$ and the overall cluster error ratio $\mathcal{E}_{\Pi, Z}$ as the \textbf{Harmonic Mean} of $\mathcal{E}_{\pi \rightarrow \pi'}$ among all class pairs:
$$
\mathcal{E}_{\Pi, Z} = C(C-1)/\left(\sum_{\substack{\pi \neq \pi'\\ \pi, \pi' \in \Pi}}\frac{1}{\mathcal{E}_{\pi \rightarrow \pi'}}\right), \text{where }\mathcal{E}_{\pi \rightarrow \pi'} = \frac{|\xi_{\pi \rightarrow \pi'}|}{|\pi'| + |\pi|}.
$$
The K-means measure $\mathcal{M}_{kms}(\Pi, Z)$ has the same order as the Harmonic Mean of the cluster error ratio between all cluster pairs:
     $$\mathcal{E}_{\Pi, Z} = O(\mathcal{M}_{kms}(\Pi, Z)).$$
    \vspace{-0.5cm}
\label{th:sorl_sup_cluster_err_bound}
\end{theorem}
\begin{proof}
We have the following inequality for $i \in \xi_{\pi \rightarrow \pi'}$:
$$ 4 \|\*z_i - \boldsymbol{\mu}_{\pi}\|^2_2 \geq  2\|\*z_i - \boldsymbol{\mu}_{\pi}\|^2_2 + 2\|\*z_i - \boldsymbol{\mu}_{\pi'}\|^2_2 \geq \|\boldsymbol{\mu}_{\pi} - \boldsymbol{\mu}_{\pi'}\|^2_2. $$

Then we have: 
\begin{align*}
\mathcal{M}_{intra}(\Pi, Z) &=  \sum_{\pi \in \Pi}\sum_{i \in \pi}  \|\*z_i - \boldsymbol{\mu}_{\pi}\|^2_2 \\ &\geq 
    \sum_{i \in \pi}  \|\*z_i - \boldsymbol{\mu}_{\pi}\|^2_2 \\ &\geq  \sum_{i \in \xi_{\pi \rightarrow \pi'}}  \|\*z_i - \boldsymbol{\mu}_{\pi}\|^2_2 
    \\ &\geq  \frac{1}{4}\sum_{i \in \xi_{\pi \rightarrow \pi'}} 
 \|\boldsymbol{\mu}_{\pi} - \boldsymbol{\mu}_{\pi'}\|^2_2 \\ 
 &= \frac{1}{4}|\xi_{\pi \rightarrow \pi'}| \|\boldsymbol{\mu}_{\pi} - \boldsymbol{\mu}_{\pi'}\|^2_2.
\end{align*}

Note that the inter-class measure can be decomposed into the summation of cluster center distances:
\begin{align*}
    \mathcal{M}_{inter}(\Pi, Z) &=  \sum_{\pi \in \Pi}  |\pi|\left\|\boldsymbol{\mu}_{\pi} -\boldsymbol{\mu}_\Pi\right\|_2^2 
    \\&= \sum_{\pi\in \Pi}  \frac{|\pi|}{N^2} \left\|(\sum_{\pi' \in \Pi} |\pi'|)\boldsymbol{\mu}_{\pi} -\sum_{\pi' \in \Pi} |\pi'| \boldsymbol{\mu}_{\pi'} \right\|_2^2 
    \\& \leq \frac{C}{N^2} \sum_{\pi\in \Pi} |\pi| \sum_{\pi' \in \Pi} |\pi'|^2 \left\|\boldsymbol{\mu}_{\pi} -\boldsymbol{\mu}_{\pi'} \right\|_2^2
    \\&= \frac{C}{N^2} \sum_{\pi \neq \pi'} |\pi||\pi'|(|\pi'| + |\pi|) \left\|\boldsymbol{\mu}_{\pi} -\boldsymbol{\mu}_{\pi'} \right\|_2^2,
\end{align*}
where $\sum_{\pi \neq \pi'}$ is enumerating over any two different class partitions in $\Pi$. Combining together, we have:
\begin{align*}
C(C-1) / \left(\sum_{\pi \neq \pi'} \frac{(|\pi'| + |\pi|)}{|\xi_{\pi \rightarrow \pi'}|} \right) & = 
C(C-1) / \left(\sum_{\pi \neq \pi'} \frac{(|\pi'| + |\pi|)\|\boldsymbol{\mu}_{\pi} -\boldsymbol{\mu}_{\pi'} \|_2^2}{|\xi_{\pi \rightarrow \pi'}|\|\boldsymbol{\mu}_{\pi} -\boldsymbol{\mu}_{\pi'} \|_2^2} \right) 
     \\ & \leq 
    C(C-1) / \left(\sum_{\pi \neq \pi'} \frac{|\pi'||\pi|(|\pi'| + |\pi|)\|\boldsymbol{\mu}_{\pi} -\boldsymbol{\mu}_{\pi'} \|_2^2}{N^2|\xi_{\pi \rightarrow \pi'}|\|\boldsymbol{\mu}_{\pi} -\boldsymbol{\mu}_{\pi'} \|_2^2} \right) 
    \\&\leq  C(C-1) / \left( \frac{\mathcal{M}_{inter}(\Pi, Z)}{4C\mathcal{M}_{intra}(\Pi, Z)} \right) 
    \\&= O(\mathcal{M}_{kms}(\Pi, Z)).
\end{align*}

\end{proof}

\subsubsection{Proof of Theorem~\ref{th:sorl_main}}
\label{sec:sorl_sup_proof_main}
We start by providing more details to supplement Sec.~\ref{sec:sorl_feature}.

\noindent \textbf{Matrix perturbation by adding labels.} Recall that we define in Eq.~\ref{eq:sorl_adj_orl} that the adjacency matrix  is the unlabeled one $A^{(u)}$ plus the perturbation of the label information $A^{(l)}$:
\begin{equation*}
    A = \eta_{u} A^{(u)} +  \eta_{l} A^{(l)}.
\end{equation*}
We study the perturbation from two aspects: (1) The direction of the perturbation which is given by $A^{(l)}$, (2) The perturbation magnitude $\eta_{l}$.
We first consider the perturbation direction  $A^{(l)}$ and recall that we defined the concrete form in Eq.~\ref{eq:sorl_def_wxx_b}:
\begin{align*}
A_{x x^{\prime}}^{(l)} = w^{(l)}_{x x^{\prime}} \triangleq \sum_{i \in \mathcal{Y}_l}\mathbb{E}_{\bar{x}_{l} \sim {\mathcal{P}_{l_i}}} \mathbb{E}_{\bar{x}'_{l} \sim {\mathcal{P}_{l_i}}} \mathcal{T}(x | \bar{x}_{l}) \mathcal{T}\left(x' | \bar{x}'_{l}\right). 
\end{align*}
For simplicity, we consider $|\mathcal{Y}_l| = 1$ in this theoretical analysis. Then we observe that $A_{x x^{\prime}}^{(l)}$ is a rank-1 matrix can be written as 
\begin{align*}
A_{x x^{\prime}}^{(l)} = \mathfrak{l} \mathfrak{l}^{\top},  
\end{align*}
where $\mathfrak{l} \in \mathbb{R}^{N\times 1}$ with $(\mathfrak{l})_x = \mathbb{E}_{\bar{x}_{l} \sim {\mathcal{P}_{l_1}}} \mathcal{T}(x | \bar{x}_{l})$. And we define $D_{l} \triangleq diag(\mathfrak{l})$.  

\noindent \textbf{The perturbation function of representation.} We then consider a more generalized form for the adjacency matrix: 
\begin{equation*}
    A(\delta) \triangleq \eta_u A^{(u)} + \delta \mathfrak{l} \mathfrak{l}^{\top}.
\end{equation*}
where we treat the adjacency matrix as a function of the  
``labeling perturbation'' degree $\delta$. It is clear that  $A(0) = \eta_{u} A^{(u)} $ which is the scaled adjacency matrix for the unlabeled case and that $A(\eta_{l}) = A$. When we let the adjacency matrix be a function of $\delta$, the normalized form and the derived feature representation should also be the function of $\delta$. We proceed by defining these terms.

Without losing the generality, 
we let $diag( \mathbf{1}_N^{\top}A(0)) = I_{N}$ which means the node in the unlabeled graph has equal degree. We then have: 
$$D(\delta) \triangleq diag(\mathbf{1}_N^{\top}A(\delta)) = I_N + \delta D_{l}.$$

The normalized adjacency matrix is given by: $$\Dot{A}(\delta) \triangleq D(\delta)^{-\frac{1}{2}}A(\delta)D(\delta)^{-\frac{1}{2}}.$$

For feature representation $Z(\delta)$, it is derived from the top-$k$ SVD components of $\Dot{A}(\delta)$. 
Specifically, we have:
$$Z(\delta)Z(\delta)^{\top} = D(\delta)^{-\frac{1}{2}} \Dot{A}_k(\delta) D(\delta)^{-\frac{1}{2}} = D(\delta)^{-\frac{1}{2}} \sum_{j=1}^{k} \lambda_j(\delta) \Phi_j(\delta)
    D(\delta)^{-\frac{1}{2}},$$ 
where we define $\Dot{A}_k(\delta)$ as the top-$k$ SVD components of $\Dot{A}(\delta)$ and can be further written as $\Dot{A}_k(\delta) = \sum_{j=1}^{k} \lambda_j(\delta) \Phi_j(\delta)$. Here the $\lambda_j(\delta)$ is the $j$-th singular value and $\Phi_j(\delta)$ is the $j$-th singular projector ($\Phi_j(\delta) = v_j(\delta)v_j(\delta)^{\top}$) defined by the $j$-th singular vector $v_j(\delta)$. 
\textbf{For brevity, when $\delta=0$, we remove the suffix $(0)$} since it is equivalent to the unperturbed version of notations. For example, we let 
$$\Tilde{A}(0) = \Tilde{A}^{(u)}, Z(0) = Z^{(u)}, \lambda_i(0) = \lambda_i^{(u)}, v_i(0) = v_i^{(u)}, \Phi_i(0) = \Phi_i^{(u)}.$$

\begin{theorem} (Recap of Theorem~\ref{th:sorl_main}) Denote $V_{\varnothing}^{(u)} \in \mathbb{R}^{N \times (N-k)}$ as the \textit{null space} of $V_k^{(u)}$ and $\Tilde{A}_k^{(u)} = V_k^{(u)} \Sigma_k^{(u)} V_k^{(u)\top}$ as the rank-$k$ approximation for $\Tilde{A}^{(u)}$.  
 Given $\delta, \eta_{1} > 0 $ and let 
 $\mathcal{G}_k$ as the spectral gap between $k$-th and $k+1$-th singular values of $\Tilde{A}^{(u)}$, we have: 
\begin{align*}
    &\Delta_{kms}(\delta) =  \delta \eta_{1} \operatorname{Tr} \left( \Upsilon \left(V_k^{(u)} V_k^{(u)\top} \mathfrak{l} \mathfrak{l}^{\top}(I +   V_\varnothing^{(u)}  V_\varnothing^{(u)\top}) - 2\Tilde{A}_k^{(u)} diag(\mathfrak{l}) \right)\right) + O(\frac{1}{\mathcal{G}_k} + \delta^2),
\end{align*}
where $diag(\cdot)$ converts the vector to the corresponding diagonal matrix and $\Upsilon \in \mathbb{R}^{N\times N}$ is a matrix encoding the \textbf{ground-truth clustering structure} in the way that $\Upsilon_{xx'} > 0$ if $x$ and $x'$ has the same label and  $\Upsilon_{xx'} < 0$ otherwise.
\label{th:sorl_sup_main}
\end{theorem}

\begin{proof}
%%%%%%%%%%%%%%%%%%%  Derivative  Begins %%%%%%%%%%%%%%%%%%%
As we shown in Sec~\ref{sec:sorl_sup_matrix_kmeans}, we can now also write the K-means measure as the function of perturbation:  
$$\mathcal{M}_{kms}(\delta) = \frac{\operatorname{Tr}((I-H_\Pi)Z(\delta)Z(\delta)^{\top})}{\operatorname{Tr}((H_\Pi - \frac{1}{N} \mathbf{1}_{N \times N} )Z(\delta)Z(\delta)^{\top})}. $$
The proof is directly given by the following Lemma~\ref{lemma:sup_kms_derivative}.
\end{proof}

\begin{lemma} Let $\eta_{1}, \eta_{2}$ be two real values and $\Upsilon  = (1+\eta_{2})H_\Pi - I - \frac{\eta_{2}}{N} \mathbf{1}_{N}\mathbf{1}_{N}^{\top}$. Let  the spectrum gap $\mathcal{G}_k = \frac{\lambda^{(u)}_k}{\lambda^{(u)}_{k+1}}$, 
we have the derivative of the K-means measure evaluated at $\delta=0$: 
\begin{align*}
    &[\mathcal{M}_{kms}(\delta)]'\Bigr|_{\delta=0} = 
    \\&~~- \eta_{1} \operatorname{Tr} \left( \Upsilon \left(V_k^{(u)} V_k^{(u)\top} \mathfrak{l} \mathfrak{l}^{\top} - 2\Tilde{A}_k^{(u)} D_{l} +  V_k^{(u)} V_k^{(u)\top} \mathfrak{l}\mathfrak{l}^{\top} V_\varnothing^{(u)}  V_\varnothing^{(u)\top}\right)\right) + O(\frac{1}{\mathcal{G}_k}). 
\end{align*}
\label{lemma:sup_kms_derivative}
\end{lemma}

The proof for Lemma~\ref{lemma:sup_kms_derivative} is lengthy.  We postpone it to Sec.~\ref{sec:sorl_sup_proof_kms_derivative}.

\subsubsection{Proof of Theorem~\ref{th:sorl_main_simp}}
\label{sec:sorl_sup_proof_main_simp}
We start by showing the justification of the assumptions made in Theorem~\ref{th:sorl_main_simp}. 

\begin{assumption}
We assume the spectral gap $\mathcal{G}_k$ is large. Such an assumption is commonly used in theory works using spectral analysis~\citep{shen2022connect,joseph2016impact}.
\label{ass:spectral-gap}
\end{assumption}
\begin{assumption}
We assume $\mathfrak{l}$ lies in the linear span of $V_k^{(u)}$. \textit{i.e.}, $V_k^{(u)} V_k^{(u)\top} \mathfrak{l} = \mathfrak{l},   V_{\varnothing}^{(u)\top}\mathfrak{l} = 0$. The goal of this assumption is to simplify \\ $(V_k^{(u)} V_k^{(u)\top} \mathfrak{l} \mathfrak{l}^{\top} +  V_k^{(u)} V_k^{(u)\top} \mathfrak{l}\mathfrak{l}^{\top} V_\varnothing^{(u)}  V_\varnothing^{(u)\top})$ to $\mathfrak{l}\mathfrak{l}^{\top}$. 
\label{ass:tl_in_vk_space}
\end{assumption}
\begin{assumption}
For any $\pi_c \in \Pi$, $\forall i, j \in \pi_c, \mathfrak{l}_{(i)} =  \mathfrak{l}_{(j)} =: \mathfrak{l}_{\pi_c}$. Recall that the $\mathfrak{l}_{(i)}$ means the connection between the $i$-th sample to the labeled data. Here we can view $\mathfrak{l}_{\pi_c}$ as the \textit{ connection between class $c$ to the labeled data}. 
\label{ass:tl_cls}
\end{assumption}

\begin{theorem} (Recap of Theorem~\ref{th:sorl_main_simp}.) 
With Assumption~\ref{ass:spectral-gap},~\ref{ass:tl_in_vk_space} and ~\ref{ass:tl_cls}.
 Given $\delta, \eta_{1}, \eta_{2} > 0 $, we have: 
\begin{align*}
    &\Delta_{kms}(\delta) \geq  \delta \eta_{1}\eta_{2} \sum_{\pi_c \in \Pi}  |\pi_c| \mathfrak{l}_{\pi_c} \Delta_{\pi_c}(\delta),
\end{align*}
where 
\begin{equation*}
    \Delta_{\pi_c}(\delta) = (\mathfrak{l}_{\pi_c} - \frac{1}{N}) - 2(1-\frac{|\pi_c|}{N})(\mathbb{E}_{i \in \pi_c} \mathbb{E}_{j \in \pi_c}\*z_i^{\top}\*z_j - \mathbb{E}_{i \in \pi_c} \mathbb{E}_{j \notin \pi_c}\*z_i^{\top}\*z_j).
\end{equation*}
\label{th:sorl_sup_main_simp}
\end{theorem}
\begin{proof}
    The proof is directly given by Lemma~\ref{lemma:sup_kms_derivative_simp} and plugging the definition of $\Delta_{kms}(\delta)$.
\end{proof}

\begin{lemma} With Assumption~\ref{ass:spectral-gap}~\ref{ass:tl_in_vk_space} and ~\ref{ass:tl_cls},
we have the derivative of K-means measure with the upper bound: 
\begin{align*}
    [\mathcal{M}_{kms}(\delta)]'\Bigr|_{\delta=0} &\leq - \eta_{1}\eta_{2} \sum_{\pi \in \Pi}  |\pi| \mathfrak{l}_{\pi} \left((\mathfrak{l}_{\pi} - \frac{1}{N}) - 2(\boldsymbol{\mu}_\pi^{\top}\boldsymbol{\mu}_\pi - \boldsymbol{\mu}_\pi^{\top}\boldsymbol{\mu}_{\Pi})\right). 
\end{align*}
\label{lemma:sup_kms_derivative_simp}
\end{lemma}
\begin{proof}
By Assumption~\ref{ass:spectral-gap}~\ref{ass:tl_in_vk_space} and ~\ref{ass:tl_cls} and Theorem~\ref{th:sorl_main}, we have
\begin{align*}
    \frac{1}{\eta_{1}}[\mathcal{M}_{kms}(\delta)]'\Bigr|_{\delta=0} &= - \operatorname{Tr} \left( \Upsilon \left(V_k^{(u)} V_k^{(u)\top} \mathfrak{l} \mathfrak{l}^{\top} - 2\Tilde{A}_k^{(u)} D_{l}\right)\right)  \\
     &= - \operatorname{Tr} \left( \Upsilon \left(\mathfrak{l} \mathfrak{l}^{\top} - 2\Tilde{A}_k^{(u)} D_{l}\right)\right) \\
 &= - \operatorname{Tr} \left(\left(
     (1+\eta_{2})H_\Pi - I - \frac{\eta_{2}}{N} \mathbf{1}_{N}\mathbf{1}_{N}^{\top}\right)\left(\mathfrak{l} \mathfrak{l}^{\top} - 2\Tilde{A}_k^{(u)} D_{l}\right)\right)   \\
&= (1+\eta_{2})\mathcal{M}'_{H} + \mathcal{M}'_{I} + \eta_{2}\mathcal{M}'_{\mathbf{1}},
\end{align*}
where 
 \begin{align*}
    \mathcal{M}'_{H}
     &= - \operatorname{Tr} \left(
     H_\Pi \left(\mathfrak{l} \mathfrak{l}^{\top} - 2\Tilde{A}_k^{(u)} D_{l}\right)\right) \\
     &= - \sum_{\pi \in \Pi} \left( |\pi| (\mathbb{E}_{i\in \pi}\mathfrak{l}_{(i)})^2 - \frac{2}{|\pi|}\sum_{i \in \pi}\sum_{j \in \pi}\mathfrak{l}_{(i)} \Tilde{A}_{k,(i,j)}^{(u)} \right)\\
     &= - \sum_{\pi \in \Pi} \left( |\pi| \mathfrak{l}_{\pi}^2 - 2|\pi| \mathfrak{l}_{\pi}\mathbb{E}_{(i,j) \in \pi \times \pi} \*z_i^{\top}\*z_j\right) \\
      &= - \sum_{\pi \in \Pi}  |\pi| \mathfrak{l}_{\pi}(\mathfrak{l}_{\pi} - 2\boldsymbol{\mu}_\pi^{\top}\boldsymbol{\mu}_\pi),
\end{align*}

 \begin{align*}
    \mathcal{M}'_{I}
     &= \operatorname{Tr} \left(
     \left(\mathfrak{l} \mathfrak{l}^{\top} - 2\Tilde{A}_k^{(u)} D_{l}\right)\right) 
     \\
      &= \sum_{\pi \in \Pi}  |\pi| \mathfrak{l}_{\pi}(\mathfrak{l}_{\pi} - 2\mathbb{E}_{i \in \pi} \*z_i^{\top}\*z_i),
\end{align*}
and
 \begin{align*}
    \mathcal{M}'_{\mathbf{1}}
     &= \operatorname{Tr} \left(
     \frac{1}{N} \mathbf{1}_{N}\mathbf{1}_{N}^{\top}\left(\mathfrak{l} \mathfrak{l}^{\top} - 2\Tilde{A}_k^{(u)} D_{l}\right)\right) 
     \\
      &= \frac{1}{N} - 2 \sum_{\pi \in \Pi} \sum_{i \in \pi} \mathfrak{l}_{(i)} \mathbb{E}_{ j \in [N]} \*z_i^{\top}\*z_j
    \\
      &= \frac{1}{N} - 2 \sum_{\pi \in \Pi} |\pi| \mathfrak{l}_{\pi} \boldsymbol{\mu}_{\pi}^{\top}\boldsymbol{\mu}_{\Pi}.
\end{align*}

We observe that 
\begin{align*}
    \mathcal{M}'_{I} +  \mathcal{M}'_{H}&= - \sum_{\pi \in \Pi}  |\pi| \mathfrak{l}_{\pi}(\mathfrak{l}_{\pi} - 2\boldsymbol{\mu}_\pi^{\top}\boldsymbol{\mu}_\pi) + \sum_{\pi \in \Pi}  |\pi| \mathfrak{l}_{\pi}(\mathfrak{l}_{\pi} - 2\mathbb{E}_{i \in \pi} \*z_i^{\top}\*z_i) 
     \\ &= 2\sum_{\pi \in \Pi}  |\pi| \mathfrak{l}_{\pi}(\|\mathbb{E}_{i \in \pi}\*z_i\|_2^2 - \mathbb{E}_{i \in \pi} \|\*z_i\|_2^2)
     \\ &\leq 0,
\end{align*}

where the last inequality is by Jensen's Inequality.  
We then have
\begin{align*}
    \frac{1}{\eta_{1}\eta_{2}}[\mathcal{M}_{kms}(\delta)]'\Bigr|_{\delta=0} 
    &\leq  \mathcal{M}'_{H} + \mathcal{M}'_{\mathbf{1}} 
    \\ &= - \sum_{\pi \in \Pi}  |\pi| \mathfrak{l}_{\pi}(\mathfrak{l}_{\pi} - 2\boldsymbol{\mu}_\pi^{\top}\boldsymbol{\mu}_\pi) + \frac{1}{N} - 2 \sum_{\pi \in \Pi} |\pi| \mathfrak{l}_{\pi} \boldsymbol{\mu}_{\pi}^{\top}\boldsymbol{\mu}_{\Pi}
    \\ &= \frac{1}{N} - \sum_{\pi \in \Pi}  |\pi| \mathfrak{l}_{\pi}(\mathfrak{l}_{\pi} - 2(\boldsymbol{\mu}_\pi^{\top}\boldsymbol{\mu}_\pi - \boldsymbol{\mu}_\pi^{\top}\boldsymbol{\mu}_{\Pi}))
        \\ &= - \sum_{\pi \in \Pi}  |\pi| \mathfrak{l}_{\pi}((\mathfrak{l}_{\pi} - \frac{1}{N} ) - 2(\boldsymbol{\mu}_\pi^{\top}\boldsymbol{\mu}_\pi - \boldsymbol{\mu}_\pi^{\top}\boldsymbol{\mu}_{\Pi})).
\end{align*}
\end{proof}

\subsubsection{Proof of Lemma~\ref{lemma:sup_kms_derivative}}
\label{sec:sorl_sup_proof_kms_derivative}

\textbf{Notation Recap:} 
  We define $\Dot{A}_k(\delta)$ as the top-$k$ SVD components of $\Dot{A}(\delta)$ and can be further written as $\Dot{A}_k(\delta) = \sum_{j=1}^{k} \lambda_j(\delta) \Phi_j(\delta)$. Here the $\lambda_j(\delta)$ is the $j$-th singular value and $\Phi_j(\delta)$ is the $j$-th singular projector ($\Phi_j(\delta) = v_j(\delta)v_j(\delta)^{\top}$) defined by the $j$-th singular vector $v_j(\delta)$. 
\textbf{For brevity, when $\delta=0$, we remove the suffix $(0)$} since it is equivalent to the unperturbed version of notations. For example, we let 
$$\Tilde{A}(0) = \Tilde{A}^{(u)}, Z(0) = Z^{(u)}, \lambda_i(0) = \lambda_i^{(u)}, v_i(0) = v_i^{(u)}, \Phi_i(0) = \Phi_i^{(u)}.$$

\begin{proof}
By the derivative rule, we have, 
\begin{align*}
    \mathcal{M}_{kms}'(\delta) &= \frac{1}{\mathcal{M}_{inter}(\Pi, Z)} \mathcal{M}_{intra}'(\delta) - \frac{\mathcal{M}_{intra}(\Pi, Z)}{\mathcal{M}_{inter}(\Pi, Z)^2} \mathcal{M}_{inter}'(\delta) \\
    &= \eta_{1} \mathcal{M}_{intra}'(\delta) - \eta_{1}\eta_{2} \mathcal{M}_{inter}'(\delta) \\
    &= \eta_{1} \left(\operatorname{Tr}((I_{\Pi}-H_\Pi)[Z(\delta)Z(\delta)^{\top}]') - \eta_{2} \operatorname{Tr}((H_\Pi - \frac{1}{N} \mathbf{1}_{N \times N} )[Z(\delta)Z(\delta)^{\top}]')\right) \\
    &= \eta_{1} \left(\operatorname{Tr}((I_{\Pi}+\frac{\eta_{2}}{N} \mathbf{1}_{N \times N}-(\eta_{2}+1)H_\Pi)[Z(\delta)Z(\delta)^{\top}]') \right)  \\
    &= - \eta_{1} \left(\operatorname{Tr}(\Upsilon [Z(\delta)Z(\delta)^{\top}]') \right) \\
    &= - \eta_{1} \sum_{j=1}^{k} \operatorname{Tr}(\Upsilon [D(\delta)^{-\frac{1}{2}}  \lambda_j(\delta) \Phi_j(\delta)
    D(\delta)^{-\frac{1}{2}}]'),\\
\end{align*}
where we let $\eta_{1} = \frac{1}{\mathcal{M}_{inter}(\Pi, Z)}$, $\eta_{2} = \frac{\mathcal{M}_{intra}(\Pi, Z)}{\mathcal{M}_{inter}(\Pi, Z)}$ and $\Upsilon  = (1+\eta_{2})H_\Pi - I_{\Pi} - \frac{\eta_{2}}{N} \mathbf{1}_{N}\mathbf{1}_{N}^{\top}$. We proceed by showing the calculation of $[D(\delta)^{-\frac{1}{2}}]'$, $[\lambda_j(\delta)]'$ and $[\Phi_j(\delta)]'$.

Since $D(\delta) = I + \delta D_{l},$ then $[D(\delta)^{-\frac{1}{2}}]'\Bigr|_{\delta=0} = -\frac{1}{2}D_{l}$. To calculate $[\lambda_j(\delta)]'$ and $[\Phi_j(\delta)]'$, we first need: 
\begin{align*}
[\Dot{A}(\delta)]'\Bigr|_{\delta=0} &= [D(\delta)^{-\frac{1}{2}} A(\delta) D(\delta)^{-\frac{1}{2}}]' \\
&= [D(\delta)^{-\frac{1}{2}}]' \Tilde{A}^{(u)} +  [A(\delta)]' + \Tilde{A}^{(u)} [D(\delta)^{-\frac{1}{2}}]' \\
&= -\frac{1}{2}D_{l} \Tilde{A}^{(u)} + \mathfrak{l}\mathfrak{l}^{\top} -\frac{1}{2} \Tilde{A}^{(u)} D_{l}.
\end{align*}

Then, according to Equation (3) in ~\cite{greenbaum2020first}, we have: 

\begin{align*}
[\lambda_j(\delta)]'\Bigr|_{\delta=0} &= \operatorname{Tr}(\Phi_j^{(u)}[\Dot{A}(\delta)]') \\
&= \operatorname{Tr}(\Phi_j^{(u)}(-\frac{1}{2}D_{l} \Tilde{A}^{(u)} + \mathfrak{l}\mathfrak{l}^{\top} -\frac{1}{2} \Tilde{A}^{(u)} D_{l})) \\
&= \operatorname{Tr}((-\frac{\lambda_j^{(u)}}{2}D_{l}\Phi_j^{(u)} + \Phi_j^{(u)}\mathfrak{l}\mathfrak{l}^{\top} -\frac{\lambda_j^{(u)}}{2} \Phi_j^{(u)}D_{l})) \\
&= \operatorname{Tr}(\Phi_j^{(u)}(\mathfrak{l}\mathfrak{l}^{\top} - \lambda_j^{(u)} D_{l})).
\end{align*}

According to Equation (10) in ~\cite{greenbaum2020first}, we have: 

\begin{align*}
[\Phi_j(\delta)]'\Bigr|_{\delta=0} &= (\lambda_j^{(u)}I_N - \Dot{A}^{(u)})^{\dagger} [\Dot{A}(\delta)]' \Phi_j^{(u)}+ \Phi_j^{(u)}[\Dot{A}(\delta)]' (\lambda_j^{(u)}I_N - \Dot{A}^{(u)})^{\dagger} \\
&= \sum^{N}_{i \neq j} \frac{1}{\lambda_j^{(u)} - \lambda_i^{(u)}} (\Phi_i^{(u)}[\Dot{A}(\delta)]' \Phi_j^{(u)}+ \Phi_j^{(u)}[\Dot{A}(\delta)]' \Phi_i^{(u)}) \\
&= \sum^{N}_{i \neq j} \frac{1}{\lambda_j^{(u)} - \lambda_i^{(u)}} (\Phi_i^{(u)}(-\frac{1}{2}D_{l} \Tilde{A}^{(u)} + \mathfrak{l}\mathfrak{l}^{\top} -\frac{1}{2} \Tilde{A}^{(u)} D_{l}) \Phi_j^{(u)}+ \Phi_j^{(u)}(...)\Phi_i^{(u)} ) \\
&= \sum^{N}_{i \neq j} \frac{1}{\lambda_j^{(u)} - \lambda_i^{(u)}} (\Phi_i^{(u)}(\mathfrak{l}\mathfrak{l}^{\top} -\frac{\lambda_j^{(u)}+\lambda_i^{(u)}}{2}D_{l}) \Phi_j^{(u)}+ \Phi_j^{(u)}(\mathfrak{l}\mathfrak{l}^{\top} -\frac{\lambda_j^{(u)}+\lambda_i^{(u)}}{2}D_{l})\Phi_i^{(u)} ). 
\end{align*}

Now we calculate the derivative of the $K$-means loss: 

\begin{align*}
    &\frac{1}{\eta_{1}}[\mathcal{M}_{kms}(\delta)]'\Bigr|_{\delta=0} \\
    &= - \sum_{j=1}^{k} [\operatorname{Tr}(\Upsilon D(\delta)^{-\frac{1}{2}}  \lambda_j(\delta) \Phi_j(\delta)
    D(\delta)^{-\frac{1}{2}})]'\Bigr|_{\delta=0} \\
    &= - \sum_{j=1}^{k} \operatorname{Tr}\left(\Upsilon \left([D(\delta)^{-\frac{1}{2}}]'  \lambda_j^{(u)} \Phi_j^{(u)}+  \lambda_j^{(u)} \Phi_j^{(u)}
    [D(\delta)^{-\frac{1}{2}}]' +  [\lambda_j(\delta)]' \Phi_j^{(u)}+    \lambda_j^{(u)}[\Phi_j(\delta)]'\right)\right) \\
    &= \sum_{j=1}^{k} \operatorname{Tr}\left(\Upsilon \left(\frac{\lambda_j^{(u)}}{2} D_l  \Phi_j^{(u)}+\frac{\lambda_j^{(u)}}{2} \Phi_j^{(u)} D_l -[\lambda_j(\delta)]' \Phi_j^{(u)}-   \lambda_j^{(u)}[\Phi_j(\delta)]'\right)\right) \\
    &= \mathcal{M}_{a}' + \mathcal{M}_{b}' + \mathcal{M}_{c}',
\end{align*}
where 
\begin{align*}
    \mathcal{M}_{a}' = \sum_{j=1}^{k} \frac{\lambda_j^{(u)}}{2}\operatorname{Tr}\left(\Upsilon \left( D_l  \Phi_j^{(u)}+ \Phi_j^{(u)}    D_l\right)\right), 
\end{align*}
\begin{align*}
    \mathcal{M}_{b}' &= - \sum_{j=1}^{k} \operatorname{Tr}\left(\Upsilon [\lambda_j(\delta)]' \Phi_j^{(u)} \right) 
    = -\sum_{j=1}^{k} \operatorname{Tr}\left( (\mathfrak{l}\mathfrak{l}^{\top} - \lambda_j^{(u)} D_{l})\Phi_j^{(u)} \right)\operatorname{Tr}\left( \Upsilon \Phi_j^{(u)} \right) \\
    &= -\sum_{j=1}^{k} \operatorname{Tr}\left( (\mathfrak{l}\mathfrak{l}^{\top} - \lambda_j^{(u)} D_{l})\Phi_j^{(u)}\Upsilon \Phi_j^{(u)} \right), 
\end{align*}
\begin{align*}
    &\mathcal{M}_{c}' 
    \\ &= - \sum_{j=1}^{k} \operatorname{Tr}\left(\Upsilon    \lambda_j^{(u)}[\Phi_j(\delta)]'\right)\\
    &= - \sum_{j=1}^{k} \operatorname{Tr} \left(  \sum^{N}_{i \neq j} \frac{\lambda_j^{(u)}}{\lambda_j^{(u)} - \lambda_i^{(u)}} (\Upsilon \Phi_i^{(u)}(\mathfrak{l}\mathfrak{l}^{\top} -\frac{\lambda_j^{(u)}+\lambda_i^{(u)}}{2}D_{l}) \Phi_j^{(u)}+ \Upsilon \Phi_j^{(u)}(\mathfrak{l}\mathfrak{l}^{\top} -\frac{\lambda_j^{(u)}+\lambda_i^{(u)}}{2}D_{l})\Phi_i^{(u)} ) \right)\\
    &= - \sum_{j=1}^{k} \operatorname{Tr} \left(  \sum^{N}_{i \neq j} \frac{\lambda_j^{(u)}}{\lambda_j^{(u)} - \lambda_i^{(u)}} \left((\Phi_j^{(u)}\Upsilon \Phi_i^{(u)}+ \Phi_i^{(u)}\Upsilon \Phi_j^{(u)} ) (\mathfrak{l}\mathfrak{l}^{\top} -\frac{\lambda_j^{(u)}+\lambda_i^{(u)}}{2}D_{l}) \right) \right)
    \\ 
    &= - \sum_{j=1}^{k} \operatorname{Tr} \left(  \sum_{i \neq j, i \le k} \frac{\lambda_j^{(u)}}{\lambda_j^{(u)} - \lambda_i^{(u)}}  \left((\Phi_j^{(u)}\Upsilon \Phi_i^{(u)}+ \Phi_i^{(u)}\Upsilon \Phi_j^{(u)} ) (\mathfrak{l}\mathfrak{l}^{\top} -\frac{\lambda_j^{(u)}+\lambda_i^{(u)}}{2}D_{l}) \right) \right)\\
    & ~~~~ - \sum_{j=1}^{k} \operatorname{Tr} \left(  \sum^{N}_{i = k+1} \frac{\lambda_j^{(u)}}{\lambda_j^{(u)} - \lambda_i^{(u)}} \left((\Phi_j^{(u)}\Upsilon \Phi_i^{(u)}+ \Phi_i^{(u)}\Upsilon \Phi_j^{(u)} ) (\mathfrak{l}\mathfrak{l}^{\top} -\frac{\lambda_j^{(u)}+\lambda_i^{(u)}}{2}D_{l}) \right) \right)    \\
    &= - \sum_{j=1}^{k} \operatorname{Tr} \left(  \sum_{i < j} \left(\frac{\lambda_j^{(u)}}{\lambda_j^{(u)} - \lambda_i^{(u)}} + \frac{\lambda_i^{(u)}}{\lambda_i^{(u)} - \lambda_j^{(u)}}\right) \left((\Phi_j^{(u)}\Upsilon \Phi_i^{(u)}+ \Phi_i^{(u)}\Upsilon \Phi_j^{(u)} ) (\mathfrak{l}\mathfrak{l}^{\top} -\frac{\lambda_j^{(u)}+\lambda_i^{(u)}}{2}D_{l}) \right) \right)\\
    & ~~~~ - \sum_{j=1}^{k} \operatorname{Tr} \left(  \sum^{N}_{i = k+1} \frac{\lambda_j^{(u)}}{\lambda_j^{(u)} - \lambda_i^{(u)}} \left((\Phi_j^{(u)}\Upsilon \Phi_i^{(u)}+ \Phi_i^{(u)}\Upsilon \Phi_j^{(u)} ) (\mathfrak{l}\mathfrak{l}^{\top} -\frac{\lambda_j^{(u)}+\lambda_i^{(u)}}{2}D_{l}) \right) \right)
    \\ 
    &= - \sum_{j=1}^{k} \operatorname{Tr} \left(  \sum_{i < j} \left((\Phi_j^{(u)}\Upsilon \Phi_i^{(u)}+ \Phi_i^{(u)}\Upsilon \Phi_j^{(u)} ) (\mathfrak{l}\mathfrak{l}^{\top} -\frac{\lambda_j^{(u)}+\lambda_i^{(u)}}{2}D_{l}) \right) \right)\\
    & ~~~~ - \sum_{j=1}^{k} \operatorname{Tr} \left(  \sum^{N}_{i = k+1} \frac{\lambda_j^{(u)}}{\lambda_j^{(u)} - \lambda_i^{(u)}} \left((\Phi_j^{(u)}\Upsilon \Phi_i^{(u)}+ \Phi_i^{(u)}\Upsilon \Phi_j^{(u)} ) (\mathfrak{l}\mathfrak{l}^{\top} -\frac{\lambda_j^{(u)}+\lambda_i^{(u)}}{2}D_{l}) \right) \right)
    \\
    &= - \sum_{j=1}^{k} \operatorname{Tr} \left(  \sum_{i \neq j, i \le k} \frac{1}{2} \left((\Phi_j^{(u)}\Upsilon \Phi_i^{(u)}+ \Phi_i^{(u)}\Upsilon \Phi_j^{(u)} ) (\mathfrak{l}\mathfrak{l}^{\top} -\frac{\lambda_j^{(u)}+\lambda_i^{(u)}}{2}D_{l}) \right) \right)\\
    & ~~~~ - \sum_{j=1}^{k} \operatorname{Tr} \left(  \sum^{N}_{i = k+1} \frac{\lambda_j^{(u)}}{\lambda_j^{(u)} - \lambda_i^{(u)}} \left((\Phi_j^{(u)}\Upsilon \Phi_i^{(u)}+ \Phi_i^{(u)}\Upsilon \Phi_j^{(u)} ) (\mathfrak{l}\mathfrak{l}^{\top} -\frac{\lambda_j^{(u)}+\lambda_i^{(u)}}{2}D_{l}) \right) \right).
\end{align*}
Thus, we have:
\begin{align*}
    &\mathcal{M}_{b}' + \mathcal{M}_{c}' \\ 
    &= - \sum_{j=1}^{k} \operatorname{Tr} \left(  \sum^{k}_{ i = 1} \frac{1}{2} \left((\Phi_j^{(u)}\Upsilon \Phi_i^{(u)}+ \Phi_i^{(u)}\Upsilon \Phi_j^{(u)} ) (\mathfrak{l}\mathfrak{l}^{\top} -\frac{\lambda_j^{(u)}+\lambda_i^{(u)}}{2}D_{l}) \right) \right)\\
    & ~~~~ - \sum_{j=1}^{k} \operatorname{Tr} \left(  \sum^{N}_{i = k+1} \frac{\lambda_j^{(u)}}{\lambda_j^{(u)} - \lambda_i^{(u)}} \left((\Phi_j^{(u)}\Upsilon \Phi_i^{(u)}+ \Phi_i^{(u)}\Upsilon \Phi_j^{(u)} ) (\mathfrak{l}\mathfrak{l}^{\top} -\frac{\lambda_j^{(u)}+\lambda_i^{(u)}}{2}D_{l}) \right) \right),
\end{align*}
\begin{align*}
    \mathcal{M}_{a}' = & \sum_{j=1}^{k} \frac{\lambda_j^{(u)}}{2}\operatorname{Tr}\left(\Upsilon \left( D_l  \Phi_j^{(u)}+ \Phi_j^{(u)}
    D_l\right)\right) \\
    = & \sum_{j=1}^{k} \frac{\lambda_j^{(u)}}{2}\operatorname{Tr}\left(\left(   \Phi_j^{(u)}\Upsilon  + 
     \Upsilon  \Phi_j^{(u)} \right)D_l\right)\\
    = & \sum_{j=1}^{k} \frac{\lambda_j^{(u)}}{2}\operatorname{Tr}\left(\left(   \Phi_j^{(u)}\Upsilon  \sum^{N}_{i=1}\Phi_i^{(u)}+ 
     \sum^{N}_{i=1}\Phi_i^{(u)}\Upsilon  \Phi_j^{(u)} \right)D_l\right)\\
    = & \sum_{j=1}^{k} \operatorname{Tr}\left( \sum^{N}_{i=1} \frac{\lambda_j^{(u)}}{2}\left(   \Phi_j^{(u)}\Upsilon  \Phi_i^{(u)}+ 
     \Phi_i^{(u)}\Upsilon  \Phi_j^{(u)} \right)D_l\right).
\end{align*}

Then $[\mathcal{M}_{kms-all}(\delta)]'\Bigr|_{\delta=0} / \eta_{1}$ is given by:
\begin{align*}
     &\mathcal{M}_{a}' + \mathcal{M}_{b}' + \mathcal{M}_{c}' \\ 
     &= - \sum_{j=1}^{k} \operatorname{Tr} \left(  \sum^{k}_{ i = 1} \frac{1}{2} \left((\Phi_j^{(u)}\Upsilon \Phi_i^{(u)}+ \Phi_i^{(u)}\Upsilon \Phi_j^{(u)} ) (\mathfrak{l}\mathfrak{l}^{\top} -\frac{3\lambda_j^{(u)}+\lambda_i^{(u)}}{2}D_{l}) \right) \right)\\
    & ~~~~ - \sum_{j=1}^{k} \operatorname{Tr} \left(  \sum^{N}_{i = k+1} \frac{\lambda_j^{(u)}}{\lambda_j^{(u)} - \lambda_i^{(u)}} \left((\Phi_j^{(u)}\Upsilon \Phi_i^{(u)}+ \Phi_i^{(u)}\Upsilon \Phi_j^{(u)} ) (\mathfrak{l}\mathfrak{l}^{\top} -{\lambda_j^{(u)}}D_{l}) \right) \right) \\
    &= - \sum_{j=1}^{k} \operatorname{Tr} \left(  \sum^{k}_{ i = 1} \frac{1}{2} \left((\Phi_j^{(u)}\Upsilon \Phi_i^{(u)}+ \Phi_i^{(u)}\Upsilon \Phi_j^{(u)} ) (\mathfrak{l}\mathfrak{l}^{\top} -{2\lambda_j^{(u)}}D_{l}) \right) \right)\\
    & ~~~~ - \sum_{j=1}^{k} \operatorname{Tr} \left(  \sum^{N}_{i = k+1} \frac{\lambda_j^{(u)}}{\lambda_j^{(u)} - \lambda_i^{(u)}} \left((\Phi_j^{(u)}\Upsilon \Phi_i^{(u)}+ \Phi_i^{(u)}\Upsilon \Phi_j^{(u)} ) (\mathfrak{l}\mathfrak{l}^{\top} -{\lambda_j^{(u)}}D_{l}) \right) \right) 
    \\
    &= - \sum_{j=1}^{k}   \sum^{k}_{ i = 1} v_i^{(u)\top}\Upsilon v_j^{(u)} \cdot v_i^{(u)\top}(\mathfrak{l}\mathfrak{l}^{\top} -{2\lambda_j^{(u)}}D_{l})v_j^{(u)} \\
    & ~~~~ - \sum_{j=1}^{k}   \sum^{N}_{i = k+1} \frac{2\lambda_j^{(u)}}{\lambda_j^{(u)} - \lambda_i^{(u)}}  v_i^{(u)\top}\Upsilon v_j^{(u)} \cdot   v_i^{(u)\top} (\mathfrak{l}\mathfrak{l}^{\top} -{\lambda_j^{(u)}}D_{l}) v_j^{(u)}. 
\end{align*}
We can represent $\frac{\lambda_j^{(u)}}{\lambda_j^{(u)} - \lambda_i^{(u)}} = 1 + \sum_{p=1}^{\infty} (\frac{\lambda_i^{(u)}}{\lambda_j^{(u)}})^p$. Denote the residual term as : 
$$\mathcal{M}'_{e} = -\sum_{j=1}^{k}   \sum^{N}_{i = k+1}   \sum_{p=1}^{\infty} 2(\frac{\lambda_i^{(u)}}{\lambda_j^{(u)}})^p v_i^{(u)\top}\Upsilon v_j^{(u)} \cdot   v_i^{(u)\top} (\mathfrak{l}\mathfrak{l}^{\top} -{\lambda_j^{(u)}}D_{l}) v_j^{(u)} = O(\frac{1}{\mathcal{G}_k}).$$ 
We then have:

\begin{align*}
    &\frac{1}{\eta_{1}}[\mathcal{M}_{kms-all}(\delta)]'\Bigr|_{\delta=0} \\ &= - \operatorname{Tr} (V_k^{(u)\top} \Upsilon V_k^{(u)} \cdot V_k^{(u)\top} \mathfrak{l}\mathfrak{l}^{\top} V_k^{(u)})  + 2 \operatorname{Tr} (V_k^{(u)\top} \Upsilon V_k^{(u)} \cdot \Sigma_k^{(u)} V_k^{(u)\top}D_{l}V_k^{(u)}) \\
    & ~~~~ -  2\operatorname{Tr} (V_{\varnothing}^{(u)\top} \Upsilon V_k^{(u)} \cdot V_k^{(u)\top} \mathfrak{l}\mathfrak{l}^{\top} V_{\varnothing}^{(u)}) + 2\operatorname{Tr} (V_{\varnothing}^{(u)\top} \Upsilon V_k^{(u)} \cdot \Sigma_k^{(u)} V_k^{(u)\top}D_{l}V_{\varnothing}^{(u)}) + \mathcal{M}'_e  \\
    &= - \operatorname{Tr} ( \Upsilon V_k^{(u)} V_k^{(u)\top} \mathfrak{l} \mathfrak{l}^{\top} V_k^{(u)} V_k^{(u)\top}) +  2\operatorname{Tr} (\Upsilon \Tilde{A}_k^{(u)} D_{l}V_k^{(u)} V_k^{(u)\top})\\
    & ~~~~ -  2\operatorname{Tr} (\Upsilon V_k^{(u)} V_k^{(u)\top} \mathfrak{l}\mathfrak{l}^{\top} (I_N - V_k^{(u)}  V_k^{(u)\top})) +  2\operatorname{Tr} (\Upsilon \Tilde{A}_k^{(u)} D_{l}(I_N - V_k^{(u)}  V_k^{(u)\top})) + \mathcal{M}'_e  \\
    &= - 2 \operatorname{Tr} ( \Upsilon V_k^{(u)} V_k^{(u)\top} \mathfrak{l} \mathfrak{l}^{\top}) +  2 \operatorname{Tr} (\Upsilon \Tilde{A}_k^{(u)} D_{l}) + \operatorname{Tr} (\Upsilon V_k^{(u)} V_k^{(u)\top} \mathfrak{l}\mathfrak{l}^{\top} V_k^{(u)}  V_k^{(u)\top}) + \mathcal{M}'_e  \\
    &= - 2\operatorname{Tr} \left( \Upsilon \left(V_k^{(u)} V_k^{(u)\top} \mathfrak{l} \mathfrak{l}^{\top} - \Tilde{A}_k^{(u)} D_{l} - \frac{1}{2} V_k^{(u)} V_k^{(u)\top} \mathfrak{l}\mathfrak{l}^{\top} V_k^{(u)}  V_k^{(u)\top}\right)\right) + \mathcal{M}'_e  \\
    &= - \operatorname{Tr} \left( \Upsilon \left(V_k^{(u)} V_k^{(u)\top} \mathfrak{l} \mathfrak{l}^{\top} - 2\Tilde{A}_k^{(u)} D_{l} +  V_k^{(u)} V_k^{(u)\top} \mathfrak{l}\mathfrak{l}^{\top} V_\varnothing  V_\varnothing^{\top}\right)\right) + O(\frac{1}{\mathcal{G}_k}).
\end{align*}
\end{proof}

\newpage
\subsection{Analysis on Other Contrastive Losses}
\label{sec:sorl_sup_simclr_analysis}
In this section, we discuss the extension of our graphic-theoretic analysis to one of the most common contrastive loss functions -- SimCLR~\citep{chen2020simclr}. SimCLR loss is an extended version of InfoNCE loss~\citep{van2018cpc} that achieves great empirical success and inspires a proliferation of follow-up works~\citep{khosla2020supcon,vaze22gcd,caron2020swav,he2019moco,zbontar2021barlow,bardes2021vicreg,chen2021exploring}. Specifically, SupCon~\citep{khosla2020supcon} extends SimCLR to the supervised setting.  GCD~\citep{vaze22gcd} and OpenCon~\citep{sun2023opencon} further leverage the SupCon and SimCLR losses, and are tailored to the open-world representation learning setting considering both labeled and unlabeled data.

At a high level, we consider a general form of 
the SimCLR and its extensions (including SupCon, GCD, OpenCon) as: 
\begin{equation}
 \mathcal{L}_{\text {gnl}}(f;\mathcal{P}_{+})=-\frac{1}{\tau}\underset{(x, x^+) \sim \mathcal{P}_+}{\mathbb{E}}\left[f(x)^{\top} f(x^+)\right] 
 \quad+\underset{x\sim \mathcal{P}}{\mathbb{E}}\left[\log \left(\underset{\substack{x'\sim \mathcal{P}\\x\neq x'}}{\mathbb{E}} e^{f(x')^{\top} f(x) / \tau}\right)\right],
\end{equation}
where we let the $\mathcal{P}_{+}$ as the distribution of \textbf{positive pairs} defined in Section~\ref{sec:sorl_graph_def}. In SimCLR~\citep{chen2020simclr}, the positive pairs are purely sampled in the \textit{unlabeled case (u)} while SupCon~\citep{khosla2020supcon} considers the \textit{labeled case (l)}. With both labeled and unlabeled data, GCD~\citep{vaze22gcd} and OpenCon~\citep{sun2023opencon} sample positive pairs in both cases. 

In this section, we investigate an alternative form that eases the theoretical analysis (also applied in~\citep{wang2020understanding}):
\begin{align}
 \widehat{\mathcal{L}}_{\text {gnl }}(f;\mathcal{P}_{+})= & -\frac{1}{\tau}\underset{(x, x^+) \sim \mathcal{P}_+}{\mathbb{E}}\left[f(x)^{\top} f(x^+)\right] 
 \quad+\log \left(\underset{\substack{x,x'\sim \mathcal{P}\\x\neq x'}}{\mathbb{E}} e^{f(x')^{\top} f(x) / \tau}\right) \\
 \geq & \mathcal{L}_{\text {gnl }}(f;\mathcal{P}_{+}),
\end{align}
which serves an upper bound of $\mathcal{L}_{\text {gnl }}(f)$ according to Jensen's Inequality. 

\noindent \textbf{A graph-theoretic view.}
Recall in Section~\ref{sec:sorl_graph_def}, we define the graph $G(\mathcal{X}, w)$ with 
vertex set $\mathcal{X}$ and edge weights $w$. Each entry of adjacency matrix $A$ is given by $w_{x x'}$, which denotes the marginal probability of generating the pair for any two augmented data $x, x' \in \mathcal{X}$: 
\begin{align*}
    w_{x x^{\prime}} = \eta_{u} w^{(u)}_{x x^{\prime}} + \eta_{l} w^{(l)}_{x x^{\prime}},
\end{align*}
and $w_{x}$ measures the degree of node $x$:
\begin{align*}
    w_{x} = \sum_{x^{\prime}} w_{x x^{\prime}}.
\end{align*}

One can view the difference between SimCLR and its variants in the following way: 
(1) SimCLR~\citep{chen2020simclr} corresponds to $\eta_l = 0$ when there is no labeled case; (2) SupCon~\citep{khosla2020supcon} corresponds to $\eta_u = 0$ when only labeled case is considered. (3) GCD~\citep{vaze22gcd} and OpenCon~\citep{sun2023opencon} correspond to the cases when $\eta_u,\eta_l$ are both non-zero due to the availability of both labeled and unlabeled data.

With the define marginal probability of sampling positive pairs $w_{xx'}$ and the marginal probability of sampling a single sample $w_x$, we have: 

\begin{align*}
    \widehat{\mathcal{L}}_{\text {gnl }}(Z; G(\mathcal{X},w))  &= -\frac{1}{\tau} \sum_{x,x'\in \mathcal{X}} w_{xx'} f(x)^{\top} f\left(x'\right) + \log \left( \sum_{\substack{x,x'\in \mathcal{X} \\ x \neq x'}} w_{x}w_{x'}e^{f(x')^{\top} f(x) / \tau} \right) 
      \\ &=  -\frac{1}{\tau} \operatorname{Tr} (Z^{\top}AZ) + \log \operatorname{Tr} \left( (D\mathbf{1}_{N}\mathbf{1}_{N}^{\top}D - D^2) \exp(\frac{1}{\tau}ZZ^T)\right). 
\end{align*}

When $\tau$ is large:

\begin{align*}
    \widehat{\mathcal{L}}_{\text {simclr }}(Z; G(\mathcal{X},w))  
     &\approx -\frac{1}{\tau} \operatorname{Tr} (Z^{\top}AZ) + \log \operatorname{Tr} \left( (D\mathbf{1}_{N}\mathbf{1}_{N}^{\top}D - D^2) (\mathbf{1}_{N}\mathbf{1}_{N}^{\top} + \frac{1}{\tau}ZZ^T) \right) \\
     &=  -\frac{1}{\tau} \operatorname{Tr} (Z^{\top}AZ) + \log (1 + \frac{\frac{1}{\tau}\operatorname{Tr}(Z^{\top}(D\mathbf{1}_{N}\mathbf{1}_{N}^{\top}D - D^2)Z)}{\operatorname{Tr}(D)^2 - \operatorname{Tr}(D^2)}) + \text{const} \\
     &\approx  -\frac{1}{\tau} \operatorname{Tr} (Z^{\top}AZ) + \frac{\frac{1}{\tau}\operatorname{Tr}(Z^{\top}(D\mathbf{1}_{N}\mathbf{1}_{N}^{\top}D - D^2)Z)}{\operatorname{Tr}(D)^2 - \operatorname{Tr}(D^2)} + \text{const} \\
     &= - \frac{1}{\tau} \operatorname{Tr} \left( Z^{\top}(A - \frac{D\mathbf{1}_{N}\mathbf{1}_{N}^{\top}D - D^2}{\operatorname{Tr}(D)^2 - \operatorname{Tr}(D^2)}) Z \right) + \text{const}.
\end{align*}

If we further consider the constraint that the $Z^{\top}Z = I$, minimizing $\widehat{\mathcal{L}}_{\text {simclr }}(Z; G(\mathcal{X},w))$ boils down to the eigenvalue problem such that $Z$ is formed by the top-$k$ eigenvectors of matrix $(A - \frac{D\mathbf{1}_{N}\mathbf{1}_{N}^{\top}D - D^2}{\operatorname{Tr}(D)^2 - \operatorname{Tr}(D^2)})$. Recall that our main analysis for Theorem~\ref{th:sorl_main} and Theorem~\ref{th:sorl_main_simp} is based on the insight that the feature space is formed by the top-$k$ eigenvectors of the normalized adjacency matrix $D^{-\frac{1}{2}}AD^{-\frac{1}{2}}$. Viewed in this light, the same analysis could be applied to the SimCLR loss as well, which only differs in the concrete matrix form. We do not include the details in this thesis but leave it as future work.

\newpage
\subsection{Additional Experiments Details}

\subsubsection{Experimental Details of Toy Example}
\label{sec:sorl_sup_exp_vis}
\textbf{Recap of set up}. In Section~\ref{sec:sorl_theory_toy} we consider a toy example that helps illustrate the core idea of our theoretical findings. Specifically, the example aims to cluster 3D objects of different colors and shapes, generated by a 3D rendering software~\citep{johnson2017clevr} with user-defined properties including colors, shape, size, position, etc.  Suppose the training samples come from three shapes, $\mathcal{X}_{\cube{1}}$, $\mathcal{X}_{\sphere{0.7}{gray}}$, $\mathcal{X}_{\cylinder{0.6}}$. Let $\mathcal{X}_{\cube{1}}$ be the sample space with \textbf{known} class, and $\mathcal{X}_{\sphere{0.7}{gray}}, \mathcal{X}_{\cylinder{0.6}}$ be the sample space with \textbf{novel} classes. Further, the two novel classes are constructed to have different relationships with the known class. 
Specifically, the toy dataset contains elements with 5 unique types:

$$\mathcal{X} = \mathcal{X}_{\cube{1}} \cup \mathcal{X}_{\sphere{0.7}{gray}} \cup \mathcal{X}_{\cylinder{0.6}},$$
where
$$\mathcal{X}_{\cube{1}} = \{x_{\textcolor{red}{\cube{1}}}, x_{\textcolor{blue}{\cube{1}}} \},$$
$$\mathcal{X}_{\sphere{0.7}{gray}} = \{x_{\textcolor{red}{\sphere{0.7}{gray}}}, x_{\textcolor{blue}{\sphere{0.7}{gray}}}\},$$
$$\mathcal{X}_{\cylinder{0.6}} = \{x_{\textcolor{gray}{\cylinder{0.6}}}\}.$$

\textbf{Experimental details for Figure~\ref{fig:sorl_toy_result}(b)}. We rendered 2500 samples for each type of data. In total, we have 12500 samples. For known class $\mathcal{X}_{\cube{1}}$, we randomly select $50\%$ as labeled data and treat the rest as unlabeled.  
For training, we use the same data augmentation strategy as in SimSiam~\citep{chen2021exploring}. We use ResNet18 and train the model for 40 epochs (sufficient for convergence) with a fixed learning rate of 0.005, using SORL defined in Eq.~\eqref{eq:sorl_def_SORL}.  We set $\eta_l=0.2$ and $\eta_u=1$, respectively. Our visualization is by PyTorch implementation of UMAP~\citep{umap}, with parameters $(\texttt{n\_neighbors=30, min\_dist=1.5, spread=2, metric=euclidean})$.

\subsubsection{Experimental Details for Benchmarks}
\label{sec:sorl_sup_exp_details}
\vspace{0.1cm} \noindent \textbf{Training settings.}
For a fair comparison, we use ResNet-18~\citep{he2016deep} as the backbone for all methods. Similar to~\citep{cao2022openworld}, we pre-train the backbone using the unsupervised Spectral Contrastive Learning~\citep{haochen2021provable} for 1200 epochs. The configuration for the pre-training stage is consistent with ~\citep{haochen2021provable}. Note that the pre-training stage does not incorporate any label information. At the training stage,  we follow the same practice in~\citep{sun2023opencon,cao2022openworld}, and train our model $f(\cdot)$ by only updating the parameters of the last block of ResNet. In addition, we add a trainable two-layer MLP projection head that projects the feature from the penultimate layer to an embedding space $\mathbb{R}^{k}$ ($k = 1000$). We use the same data augmentation strategies as SimSiam~\citep{chen2021exploring,haochen2021provable}. We train for 400 epochs by SORL defined in Eq.~\eqref{eq:sorl_def_SORL}. For CIFAR-10, we set $\eta_l=0.5, \eta_u=1$ while for CIFAR-100, we set $\eta_l=0.15, \eta_u=3$.
We use SGD with momentum 0.9 as an optimizer with cosine annealing (lr=0.05), weight decay 5e-4, and batch size 512.

\vspace{0.1cm} \noindent \textbf{Evaluation settings.}
At the inference stage, we evaluate the performance in a transductive manner (evaluate on $\mathcal{D}_u$). We run a semi-supervised K-means algorithm as proposed in~\citep{vaze22gcd}. 
We follow the evaluation strategy in~\citep{cao2022openworld} and report the following metrics: (1) classification accuracy on known classes, (2) clustering accuracy  on the novel data, and (3) overall accuracy on all classes. The accuracy
of the novel classes is measured by solving an optimal assignment problem using the Hungarian algorithm~\citep{kuhn1955hungarian}. When reporting accuracy on all classes, we solve optimal assignments using both known and
novel classes.

%% file: chapters/supp_opencon.tex
\subsection{Preliminaries of  Contrastive Losses}
\label{sec:opencon_lull}

Recall in the main context of this chapter, we provide a general form of the per-sample contrastive loss:

\begin{equation*}
    \mathcal{L}_\phi\big(\bx;\tau,\mathcal{P}(\bx),\mathcal{N}(\bx)\big) = -\frac{1}{|\mathcal{P}(\bx)|}\sum_{\bz^{+} \in \mathcal{P}(\bx)} \log \frac{\exp(\bz^{\top} \cdot \bz^+ / \tau)}{\sum_{\bz^- \in \mathcal{N}(\bx)} \exp (\bz^{\top} \cdot \bz^- / \tau)},
\end{equation*}

where $\tau$ is the temperature parameter, $\bz$ is the $L_2$ normalized embedding of $\bx$, $\mathcal{P}(\bx)$ is the positive set of embeddings \emph{w.r.t.}  $\bz$, and $\mathcal{N}(\bx)$ is the negative set of embeddings. 

In this section, we provide a detailed definition of Supervised Contrastive Loss (SupCon)~\citep{khosla2020supcon} and Self-supervised Contrastive Loss (SimCLR)~\citep{chen2020simclr}.

\vspace{0.1cm} \noindent \textbf{Supervised Contrastive Loss.} For a mini-batch $\mathcal{B}_l$ with samples drawn from $\mathcal{D}_l$, we apply two random augmentations for each sample and generate a multi-viewed batch $\tilde{\mathcal{B}}_l$. We denote the embeddings of the multi-viewed batch as $\mathcal{A}_l$, where the cardinality $|\mathcal{A}_l|$  = 2$|\mathcal{B}_l|$. For any sample $\bx$ in the mini-batch $\tilde{\mathcal{B}}_l$, the positive and negative set of embeddings are as follows: 

\begin{align*}
    \mathcal{P}_{l}(\mathbf{x}) &=\left\{\mathbf{z}^{\prime} \mid \mathbf{z}^{\prime} \in\left\{\mathcal{A}_{l} \backslash \mathbf{z}\right\}, y^{\prime}=y\right\} \\
\mathcal{N}_{l}(\mathbf{x}) &=\mathcal{A}_{l} \backslash \mathbf{z},
\end{align*}

where $y$ is the ground-truth label of  $\bx$,  and $y^{\prime}$ is the predicted label for the corresponding sample of $\bz^{\prime}$. Formally, the supervised contrastive loss is defined as:

\begin{equation*}
    \mathcal{L}_l = \mathcal{L}_\phi\big(\bx;\tau_l,\mathcal{P}_l(\bx),\mathcal{N}_l(\bx)\big),
\end{equation*}
where $\tau_l$ is the temperature.

\vspace{0.1cm} \noindent \textbf{Self-Supervised Contrastive Loss.} For a mini-batch $\mathcal{B}_u$ with samples drawn from unlabeled dataset $\mathcal{D}_u$, we apply two random augmentations for each sample and generate a multi-viewed batch $\tilde{\mathcal{B}}_u$. We denote the embeddings of the multi-viewed batch as $\mathcal{A}_u$, where the cardinality $|\mathcal{A}_u|$  = 2$|\mathcal{B}_u|$. For any sample $\bx$ in the mini-batch $\tilde{\mathcal{B}}_u$, the positive and negative set of embeddings is as follows: 

\begin{align*}
    \mathcal{P}_{u}(\mathbf{x}) &=\left\{\mathbf{z}^{\prime} \mid \mathbf{z}^{\prime} = \phi(\bx^{\prime}), \bx^{\prime} \text{ is augmented from the same sample as } \bx \right\} \\
\mathcal{N}_{u}(\mathbf{x}) &=\mathcal{A}_{u} \backslash \mathbf{z}
\end{align*}

The self-supervised contrastive loss is then defined as:
\begin{equation*}
    \mathcal{L}_u = \mathcal{L}_\phi\big(\bx;\tau_u,\mathcal{P}_u(\bx),\mathcal{N}_u(\bx)\big),
\end{equation*}
where $\tau_u$ is the temperature.

\newpage
\subsection{Algorithm}

%%%%%%%%%%%%%%%%%%%%%%%%  MAIN ALGORITHM %%%%%%%%%%%%%%%%%%%%%%%%%%%%

Below we summarize the full algorithm of open-world contrastive learning. The notation of $\mathcal{B}_u$, $\mathcal{B}_l$, $\mathcal{A}_u$, $\mathcal{A}_l$ is defined in Appendix~\ref{sec:opencon_lull}.
\begin{algorithm}[htb]
\begin{algorithmic}
  \STATE {\textbf{Input:}} Labeled set $\mathcal{D}_{l}=\left\{\bx_{i} , y_{i}\right\}_{i=1}^{n}$ and unlabeled set $\mathcal{D}_{u}=\left\{\bx_{i}\right\}_{i=1}^{m}$, neural network encoder $\phi$, randomly initialized prototypes $\bM$.\\
   
  \STATE \textbf{Training Stage}:
        \REPEAT
          \STATE \textbf{Data Preparation}:
        \STATE Sample a mini-batch of labeled data $\mathcal{B}_l = \{\bx_i, y_i\}_{i=1}^{b_l}$ and unlabeled data $\mathcal{B}_u = \{\bx_i\}_{i=1}^{b_u}$\\
        \STATE Generate augmented batch  and extract normalized embedding set $\mathcal{A}_l, \mathcal{A}_u$%
	    \STATE \textbf{OOD detection}:
	        \STATE Calculate OOD detection threshold $\lambda$ by $\mathcal{A}_l$
	        \STATE Separate $\mathcal{A}_{n}$ from $\mathcal{A}_{u}$
	    \STATE \textbf{Positive/Negative Set Selection}:
    	    \STATE Assign pseudo-labels $\hat{y}_i$ by prototypes for each sample in $\mathcal{A}_n$
    	    \STATE Obtain $\mathcal{P}_n, \mathcal{N}_n$ from $\mathcal{A}_n$
	    \STATE \textbf{Back-propagation}:
	    \STATE Calculate loss $\mathcal{L}_\text{OpenCon}$
	    \STATE Update network $\phi$ using the gradients.
	  	  \STATE \textbf{Prototype Update}:  
	  	  \STATE Update prototype vectors with Equation~\ref{eq:opencon_mov_avg}
        \UNTIL{Convergence}
  \STATE 
  
\end{algorithmic}
\caption{Open-world Contrastive Learning}
\label{alg:main}
\end{algorithm}

\subsection{Additional Theoretical Details}

\subsubsection{Theoretical Justification of OOD Detection for OpenCon}
\label{sec:opencon_whyood}

In this section, we theoretically show that OOD detection helps open-world representation learning by reducing the lower bound of loss $\underset{\bx \in \mathcal{D}_n}{\mathbb{E}} \mathcal{L}_{n}(\bx) $. We start with the definition of the supervised loss of the Mean Classifier, which provides the lower bound in Lemma~\ref{lemma:lowerbound}.

\begin{definition}
\label{def:meancls}
(Mean Classifier)  the mean classifier is a linear layer with weight matrix $\bM^*$ whose $c$-th row is the mean $\tilde{\boldsymbol{\mu}}_c$ of representations of inputs with class $c$: $\tilde{\boldsymbol{\mu}}_{c}=\underset{\bx \in  \mathcal{S}(c)}{\mathbb{E}}[\phi(\bx)]$, where $\mathcal{S}(c)$ defined in Appendix~\ref{sec:opencon_em} is the set of samples with predicted label $c$. The average supervised loss of its mean
classifier is:

\begin{equation}
    \mathcal{L}_{sup}^*:=-\underset{c^{+},  c^{-} \in \mathcal{Y}_\text{all}^{2}}{\mathbb{E}}\left[  \underset{\bx \in \mathcal{S}(c^+)}{\mathbb{E}} \phi(\bx)(\tilde{\boldsymbol{\mu}}_{c^+} - \tilde{\boldsymbol{\mu}}_{c^-}) \mid c^{+} \neq c^{-}\right]
\end{equation}
\end{definition}

\begin{lemma}
\label{lemma:lowerbound}
Let $\gamma = p(c^+ = c^-), c^{+},  c^{-} \in \mathcal{Y}_\text{all}^{2} $, it holds that
$$
\underset{\bx \in \mathcal{D}_n}{\mathbb{E}} \mathcal{L}_{n}(\bx) \geq \frac{1-\gamma}{\tau} \mathcal{L}_{s u p}^{*} 
$$
\end{lemma}

\begin{proof}
\begin{align*}
    & \underset{\bx \in \mathcal{D}_n}{\mathbb{E}} \mathcal{L}_{n}(\bx)  \\
    &= \underset{\bx \in \mathcal{D}_n}{\mathbb{E}} -\frac{1}{|\mathcal{P}(\bx)|}\sum_{\bz^{+} \in \mathcal{P}(\bx)} \log \frac{\exp(\bz^{\top} \cdot \bz^+ / \tau)}{\sum_{\bz^- \in \mathcal{N}(\bx)} \exp (\bz \cdot \bz^- / \tau)} 
    \\ &= \underset{\bx \in \mathcal{D}_n}{\mathbb{E}} \left[ -\frac{1}{|\mathcal{P}(\bx)|}\sum_{\bz^{+} \in \mathcal{P}(\bx)} (\bz^{\top} \cdot \bz^+ / \tau) + \frac{1}{|\mathcal{P}(\bx)|} \sum_{\bz^{+}\in \mathcal{P}(\bx)} \log \sum_{\bz^- \in \mathcal{N}(\bx)} \exp (\bz^{\top} \cdot \bz^- / \tau)\right]
    \\ & \overset{\mathrm{(a)}}{\approx} - \underset{c^+ \in \mathcal{Y}_\text{all}} {\mathbb{E}} \ \underset{\bx, \bx^+ \in \mathcal{S}^2(c^+) } {\mathbb{E}} \left[ \phi(\bx)^{\top} \cdot \phi(\bx^+) / \tau - \log \underset{c^- \in \mathcal{Y}_\text{all}, \bx \in \mathcal{S}(c^-) } {\mathbb{E}} \exp (\phi(\bx)^{\top} \cdot \phi(\bx^-) / \tau) \right] 
    \\ & \overset{\mathrm{(b)}}{\geq} - \underset{c^+ \in \mathcal{Y}_\text{all}} {\mathbb{E}} \ \underset{\bx, \bx^+ \in \mathcal{S}^2(c^+) } {\mathbb{E}} \left[ \phi(\bx)^{\top} \cdot \phi(\bx^+) / \tau -  \underset{c^- \in \mathcal{Y}_\text{all}, \bx \in \mathcal{S}(c^-) } {\mathbb{E}} \phi(\bx)^{\top} \cdot \phi(\bx^-) / \tau \right] 
    \\ &= - \underset{c^+, c^- \in \mathcal{Y}_\text{all}} {\mathbb{E}} \ \underset{\bx \in \mathcal{S}(c^+) } {\mathbb{E}}   \phi(\bx)(\tilde{\boldsymbol{\mu}}_{c^+} - \tilde{\boldsymbol{\mu}}_{c^-}) / \tau
    \\ &= p(c^+ \neq c^-) 
    \cdot \underset{c^{+},  c^{-} \in \mathcal{Y}_\text{all}^{2}}{\mathbb{E}} \left[  -\underset{\bx \in \mathcal{S}(c^+)}{\mathbb{E}} \phi(\bx)(\tilde{\boldsymbol{\mu}}_{c^+} - \tilde{\boldsymbol{\mu}}_{c^-}) / \tau \mid c^{+} \neq c^{-}\right] + p(c^+ \neq c^-) \cdot 0
    \\ &= \frac{1-\gamma}{\tau} \mathcal{L}_{s u p}^{*},
\end{align*}
where in (a) we approximate the summation over the positive/negative set by taking the expectation over the positive/negative sample in set $\mathcal{S}(c)$ (defined in Appendix~\ref{sec:opencon_em}) and in (b) we apply the Jensen Inequality since the $\log$ is a concave function. 
\end{proof}

In the first step, we show in Lemma~\ref{lemma:lowerbound} that $\underset{\bx \in \mathcal{D}_n}{\mathbb{E}} \mathcal{L}_{n}(\bx) $ is lower-bounded by a constant $\frac{1-\gamma}{\tau}$ times supervised loss $\mathcal{L}_{sup}^*$ defined in Definition~\ref{def:meancls}. Note that $\mathcal{L}_{sup}^*$ is non-positive and close to $-1$ in practice. Then the lower bound of $\underset{\bx \in \mathcal{D}_n}{\mathbb{E}} \mathcal{L}_{n}(\bx) $ has a  positive correlation with $\gamma=p(c^+ = c^-)$. Note that $\gamma$ can be reduced by OOD detection. To explain this:

When we separate novelty samples and form $\mathcal{D}_n$, it has fewer hidden classes than $\mathcal{D}_u$. With fewer hidden classes, the probability of $c^+$ being equal to $c^-$ in random sampling is decreased, and thus reduces the lower bound of the $\underset{\bx \in \mathcal{D}_n}{\mathbb{E}} \mathcal{L}_{n}(\bx) $. 

In summary, OOD detection facilitates open-world contrastive learning by having fewer candidate classes.

\subsubsection{Proof Details}
\label{sec:opencon_proof}

\vspace{0.1cm} \noindent \textbf{Proof of Lemma~\ref{lemma:mstep}.}
\begin{proof}
\begin{align*}
\underset{\phi, \bM}{\operatorname{argmax\ }} \sum_{i=1}^{|\mathcal{D}_n|} \sum_{c \in \mathcal{Y}_\text{all}} q_i(c) \log \frac{p\left(\bx_{i}, c | \phi, \bM\right)}{q_i(c)} 
&\overset{\mathrm{(a)}}{=}\underset{\phi, \bM}{\operatorname{argmax\ }}\sum_{i=1}^{|\mathcal{D}_n|} \sum_{c \in \mathcal{Y}_\text{all}} q_i(c) \log p\left(\bx_{i} | c, \phi, \bM\right)\\
&\overset{\mathrm{(b)}}{=}\underset{\phi, \bM}{\operatorname{argmax\ }}\sum_{i=1}^{|\mathcal{D}_n|} \sum_{c \in \mathcal{Y}_\text{all}} \mathbf{1}\{\hat{y}_{i}=c\} \log p\left(\bx_{i} | c, \phi, \bM\right) \\
&\overset{\mathrm{(d)}}{=}\underset{\phi, \bM}{\operatorname{argmax\ }} \sum_{c \in \mathcal{Y}_\text{all}} \sum_{\bx \in \mathcal{S}(c)} \log p(\bx | c, \phi, \bM) \\
&\overset{\mathrm{(e)}}{=}\underset{\phi, \bM}{\operatorname{argmax\ }} \sum_{c \in \mathcal{Y}_\text{all}} \sum_{\bx \in \mathcal{S}(c)}  \phi(\bx)^{\top} \cdot \boldsymbol{\mu}_{c},
\end{align*}

where equation $(a)$ is given by removing the constant term $q_i(c)\log\frac{p(c)}{q_i(c)}$ in $\operatorname{argmax}$, (b) is by plugging $q_i(c)$, (d) is by reorganizing the index, and (e) is by plugging the vMF density function and removing the constant.
\end{proof}

\vspace{0.1cm} \noindent \textbf{Proof of Lemma~\ref{lemma:opt_phi}.}
\begin{proof}
\begin{align*}
    \underset{\phi}{\operatorname{argmin\ }} \sum_{\bx \in \mathcal{D}_n}  \mathcal{L}_{a}(\bx)
    &= \underset{\phi}{\operatorname{argmin\ }} -
    \sum_{\bx \in \mathcal{D}_n} \frac{1}{ |\mathcal{P}(\bx)|}  \sum_{\bz^{+} \in \mathcal{P}(\bx)} \phi(\bx)^{\top} \cdot \bz^+ 
    \\ &= \underset{\phi}{\operatorname{argmin\ }} -
    \sum_{c \in \mathcal{Y}_\text{all}} \sum_{\bx \in \mathcal{S}(c)} \frac{1}{|\mathcal{S}_c| - 1} \sum_{\boldsymbol{x^+} \in \mathcal{S}(c) \backslash  \bx} \phi(\bx)^{\top} \cdot \phi(\bx^+) 
    \\ &= \underset{\phi}{\operatorname{argmin\ }}  - \sum_{c \in \mathcal{Y}_\text{all}} \sum_{\bx \in \mathcal{S}(c)}
    \frac{1}{|\mathcal{S}_c| - 1} \left(\left(\sum_{\boldsymbol{x^+} \in \mathcal{S}(c) } \phi(\bx)^{\top} \cdot \phi(\bx^+) \right) - 1 \right) 
    \\ & \overset{\mathrm{(a)}}{=} \underset{\phi}{\operatorname{argmin\ }}  - \sum_{c \in \mathcal{Y}_\text{all}} \sum_{\bx \in \mathcal{S}(c)} \eta_{\mathcal{S}_c} \phi(\bx)^{\top} \cdot \boldsymbol{\mu}^*_{c}
    \\  & \overset{\mathrm{(b)}}{\approx}  \underset{\phi}{\operatorname{argmax\ }} \sum_{c \in \mathcal{Y}_\text{all}} \sum_{\bx \in \mathcal{S}(c)}  \phi(\bx)^{\top} \cdot \boldsymbol{\mu}^*_{c} , 
\end{align*}

where in (a) $\eta_{\mathcal{S}_c} = \frac{|\mathcal{S}_c|  }{|\mathcal{S}_c| - 1} \| \mathbb{E}_{\bx \in \mathcal{S}(c)}[\phi(\bx)]\|_2$ is a constant value close to 1, and in (b) we  show the approximation to the optimization target in Equation~\ref{eq:opencon_m-target} with the fixed prototypes. The proof is done by using the equation in Lemma~\ref{lemma:mstep}. 
\end{proof}

\subsection{Additional Experiements}

\subsubsection{Discussion on Using Samples in \texorpdfstring{$\mathcal{D}_u \backslash \mathcal{D}_n$}{Lg}}
\label{sec:opencon_du_slash_dn}

We discussed in Section~\ref{sec:opencon_prototype} that samples from $\mathcal{D}_u \backslash \mathcal{D}_n$ contain indistinguishable data from known and novel classes. In this section, we show that using these samples for prototype-based learning may be undesirable. 

We first show that the overlapping between the novel and known classes in $\mathcal{D}_u \backslash \mathcal{D}_n$ can be non-trivial. In Figure~\ref{fig:opencon_distrib}, we show the distribution plot of the scores $\max_{j \in \mathcal{Y}_l}\boldsymbol{\mu}_{j}^\top \cdot \phi(\bx_i)$. It is notable that there exists a large overlapping area when $\max_{j \in \mathcal{Y}_l}\boldsymbol{\mu}_{j}^\top \cdot \phi(\bx_i) > \lambda$. For visualization clarity, we color the known classes in blue and the novel classes in gray. 

\begin{figure}
    \centering
    \includegraphics[width=0.6\linewidth]{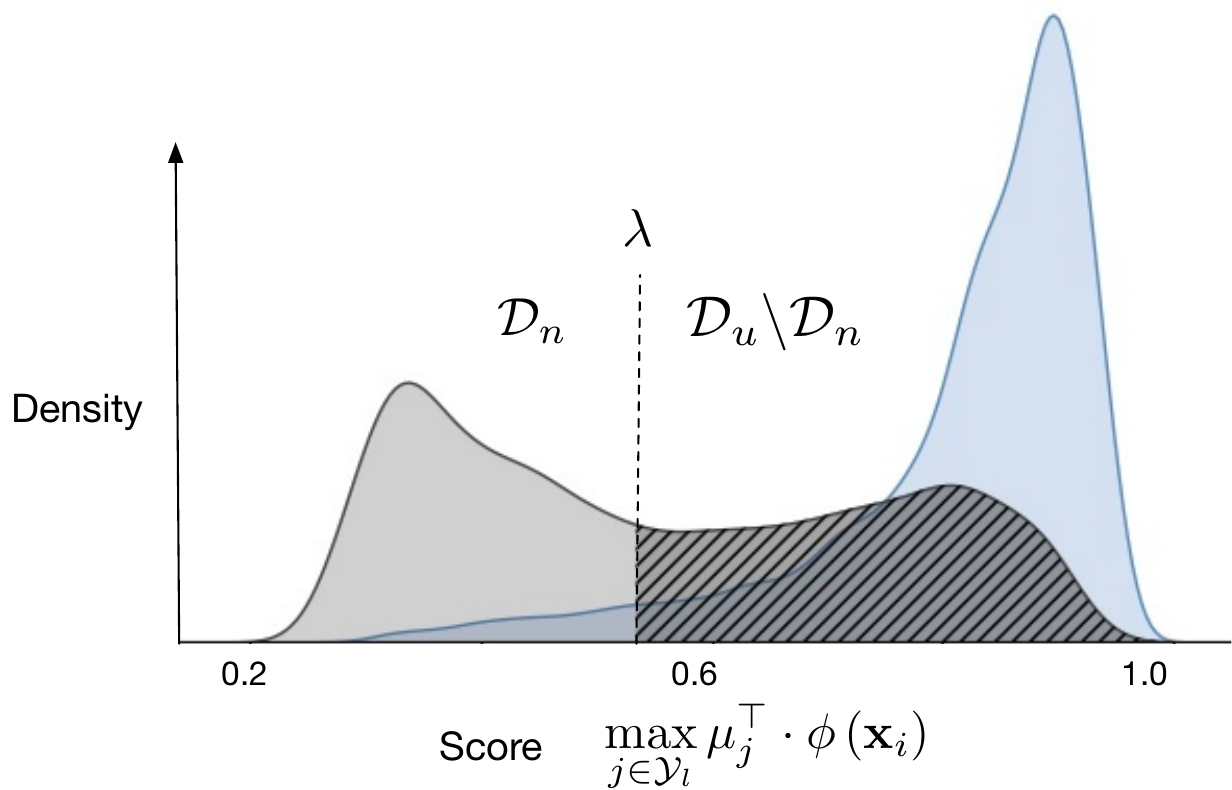}
    \caption[Distribution of OOD score for unlabeled data.]{Distribution of OOD score $\max_{j \in \mathcal{Y}_l}\boldsymbol{\mu}_{j}^\top \cdot \phi(\bx_i)$ for unlabeled data $\mathcal{D}_u$ on CIFAR-100. For visual clarity, gray indicates samples from novel classes and blue indicates samples from known classes. The shaded area highlights samples from novel classes but misidentified as known classes.}
    \label{fig:opencon_distrib}
\end{figure}

We next show that using this part of data will be harmful to representation learning. Specifically, we replace $\mathcal{L}_l$ to be the following loss: 

$$\mathcal{L}_{known} = \sum_{\mathbf{x} \in \tilde{\mathcal{B}}_k} \mathcal{L}_{\phi}\left(\mathbf{x} ; \tau_{k}, \mathcal{P}_{k}(\mathbf{x}), \mathcal{N}_{k}(\mathbf{x})\right), $$

where we define $\mathcal{B}_k$ to be a minibatch with samples drawn from $
\mathcal{D}_l \cup (\mathcal{D}_u \backslash \mathcal{D}_n)$---labeled data with known classes, along with the unlabeled data \emph{predicted as known classes}. And we apply two random augmentations for each sample and generate a multi-viewed batch $\tilde{\mathcal{B}}_k$. We denote the embeddings of the multi-viewed batch as $\mathcal{A}_k$. The positive set of embeddings is as follows: 
\begin{align*}
\mathcal{P}_{k}(\mathbf{x}) &=\left\{\mathbf{z}^{\prime} \mid \mathbf{z}^{\prime} \in\left\{\mathcal{A}_{k} \backslash \mathbf{z}\right\}, \tilde{y}^{\prime}=\tilde{y}\right\} \\
\mathcal{N}_{k}(\mathbf{x}) &=\mathcal{A}_{k} \backslash \mathbf{z},  
\end{align*}

where

$$\tilde{y}_i = \begin{cases} 
\hat{y}_i \text{ (pseudo-label)} & \text { if } \bx_{i} \in \mathcal{D}_u \\ y_{i} \text{ (ground-truth label)}, & \text { if } \bx_{i} \in \mathcal{D}_l
\end{cases}$$

Intuitively, $\mathcal{L}_{known}$ is an extension of $\mathcal{L}_{l}$, where we utilize both labeled and unlabeled data from known classes for representation learning. The final loss now becomes: 

\begin{equation}
\label{eq:opencon_modified_loss}
    \mathcal{L}_\text{Modified} = \lambda_n \mathcal{L}_n + \lambda_k\mathcal{L}_{known} + \lambda_u \mathcal{L}_u,
\end{equation}

We show results in Table~\ref{tab:opencon_lmodified}. Compared to the original loss, the seen accuracy drops by 6.6\%. This finding suggests that using $\mathcal{D}_u \backslash \mathcal{D}_n$ for prototype-based learning is suboptimal.

\begin{table}[htb]
\centering
\caption[Comparison with modified losses.]{Comparison with loss $\mathcal{L}_\text{Modified} $ on CIFAR-100.}
\begin{tabular}{lllllll}
\toprule
\multirow{2}{*}{\textbf{Method}} & \multicolumn{3}{c}{\textbf{CIFAR-100}} \\
 & \textbf{All} & \textbf{Novel} & \textbf{Seen} \\ \midrule
$\mathcal{L}_\text{Modified}$ & 47.7 & 46.4 & 62.4 \\
$\mathcal{L}_\text{OpenCon}$ & \textbf{{53.7}} & \textbf{{48.7}} & \textbf{{{69.0}}} \\ \bottomrule
\end{tabular}
\label{tab:opencon_lmodified}
\end{table}

\subsubsection{Results on CIFAR-10}
\label{sec:opencon_cifar-10}
We show results for CIFAR-10 in Table~\ref{tab:opencon_c10}, where OpenCon consistently outperforms strong baselines, particularly ORCA and GCD. Classes are divided into 50\% known and 50\% novel classes. We then select 50\% of known classes as the labeled dataset and the rest as the unlabeled set. The division is consistent with~\citep{cao2022openworld}, which allows us to compare the performance in a fair setting.

\begin{table}[htb]
\centering
\caption[Results on CIFAR-10 dataset.]{Results on CIFAR-10. Asterisk ($^\star$) denotes that the original method can not recognize seen classes. Dagger ($^\dagger$) denotes the original method can not
detect novel classes (and we had to extend it). Results on GCD, ORCA, and OpenCon  (mean and standard deviation) are averaged over five different runs. The ORCA results are reported by running the official repo~\citep{cao2022git}.}
\scalebox{0.95}{
\begin{tabular}{lllllll}
\toprule
\multirow{2}{*}{\textbf{Method}} & \multicolumn{3}{c}{\textbf{CIFAR-10}} \\
 & \textbf{All}  & \textbf{Novel} & \textbf{Seen}\\ \midrule

$^{\dagger}$\textbf{FixMatch}~\citep{alex2020fixmatch} & 49.5 & 50.4 & 71.5 \\
$^{\dagger}$\textbf{DS$^{3}$L}~\citep{guo2020dsl} & 40.2 & 45.3 & 77.6 \\
$^{\dagger}$\textbf{CGDL}~\citep{sun2020cgdl} & 39.7 & 44.6 & 72.3 \\
$^\star$\textbf{DTC}~\citep{Han2019dtc} & 38.3 & 39.5 & 53.9 \\
$^\star$\textbf{RankStats}~\citep{zhao2021rankstat} & 82.9 & 81.0 & 86.6 \\
$^\star$\textbf{SimCLR}~\citep{chen2020simclr} & 51.7 & 63.4 & 58.3 \\ \hline
\textbf{ORCA}~\citep{cao2022openworld}  & 88.3$ ^{\pm{0.3}} $ & 87.5$ ^{\pm{0.2}} $ & 89.9$ ^{\pm{0.4}} $ \\
\textbf{GCD}~\citep{vaze22gcd} & 87.5$ ^{\pm{0.5}} $ & 86.7$ ^{\pm{0.4}} $ & \textbf{90.1}$ ^{\pm{0.3}} $ \\
\textbf{OpenCon (Ours)} & \textbf{90.4}$ ^{\pm{0.6}} $ & \textbf{91.1}$ ^{\pm{0.1}} $ & 89.3$ ^{\pm{0.2}} $
\\ \bottomrule
\end{tabular}}
\label{tab:opencon_c10}
\end{table}

\subsubsection{More Qualitative Comparisons of Embeddings}
\label{sec:opencon_umap20}

In Figure~\ref{fig:opencon_umap20}, we visualize the feature embeddings for a subset of 20 classes using UMAP~\citep{umap}. This covers more classes than what has been shown in the main context of the chapter (Figure~\ref{fig:opencon_umap}). The model is trained on ImageNet-100. OpenCon produces a more compact and distinguishable embedding space than GCD and ORCA. 

\begin{figure}[htb]
    \centering
    \includegraphics[width=0.97\linewidth]{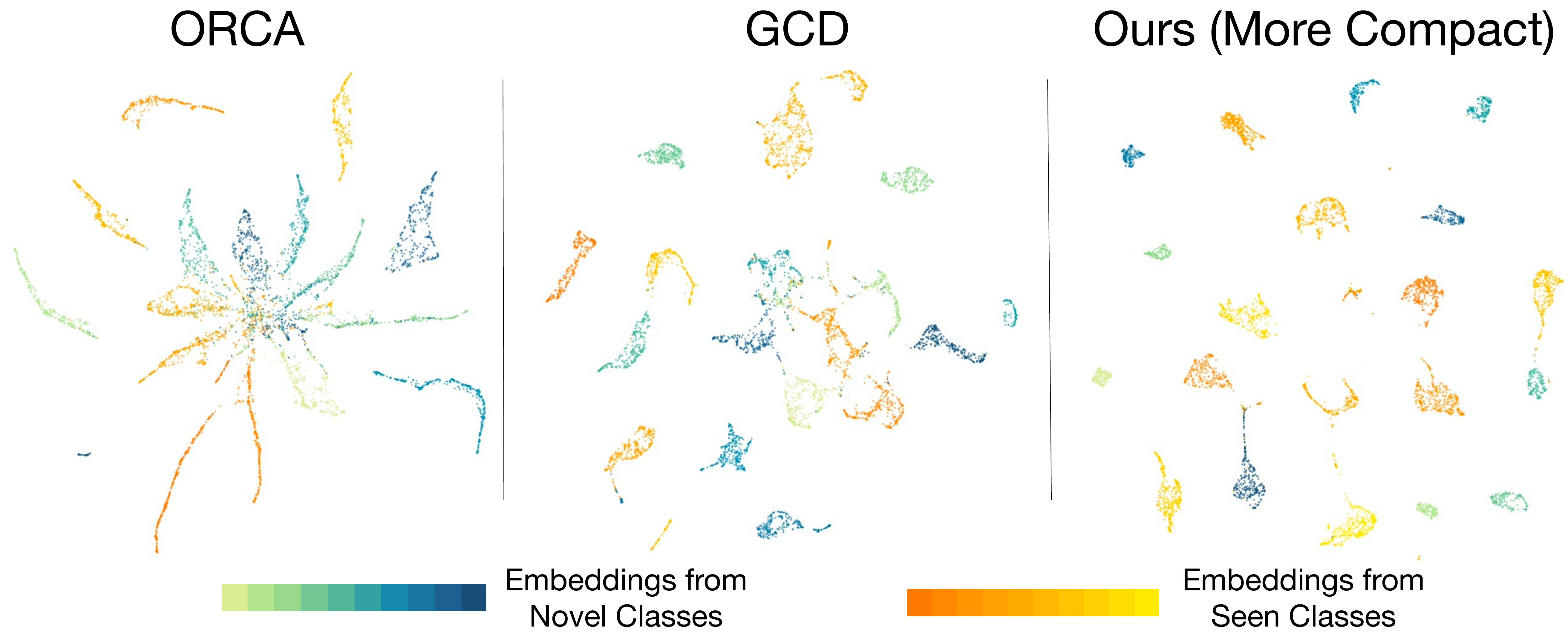}
    \caption[UMAP visualization of the feature embedding from more classes.]{UMAP~\citep{umap} visualization of the feature embedding from 20 classes (10 for seen, 10 for novel). The model is trained on ImageNet-100 with ORCA~\citep{cao2022openworld}, GCD~\citep{vaze22gcd}, and OpenCon (ours).}
    \label{fig:opencon_umap20}
\end{figure}

\subsubsection{Hyperparameters and Sensitivity Analysis}

\label{sec:opencon_hyperparam}

In this section, we introduce the hyper-parameter settings for OpenCon. We also show a validation strategy to determine important hyper-parameters (weight and temperature) in loss $\mathcal{L}_{OpenCon}$ and conduct a sensitivity analysis to show that the validation strategy can select near-optimal hyper-parameters. We start by introducing the basic training setting. 

For CIFAR-100/ImageNet-100, the model is trained for 200/120 epochs with batch-size 512 using stochastic gradient descent with momentum 0.9, and weight decay  $10^{-4}$. The learning rate starts at 0.02 and decays by a factor of 10 at the 50\% and the 75\% training stage. The momentum for prototype updating $\gamma$ is fixed at 0.9. The percentile $p$ for OOD detection is 70\%. We fix the weight for the KL-divergence regularizer to be 0.05. 

Since the label for $\mathcal{D}_u$ is not available, we propose a validation strategy by using labeled data $\mathcal{D}_l$. Specifically, we split the classes in $\mathcal{Y}_l$ equally into two parts: known classes and ``novel'' classes (for which we know the labels). Moreover, $50\%$ samples of the selected known classes are labeled. We further use the new validation dataset to select the best hyper-parameters by grid searching. The selected hyper-parameter groups are summarized in Table~\ref{tab:opencon_hyper}. Note that the only difference between CIFAR-100 and ImageNet-100 settings is the temperature of the self-supervised loss $\mathcal{L}_u$. 

We show the sensitivity of hyper-parameters in Figure~\ref{fig:opencon_hyper}. The performance comparison in the bar plot for each hyper-parameter is reported by fixing other hyper-parameters. We see that our validation strategy successfully selects $\lambda_l$, $\lambda_u$, and $\tau_l$ with the optimal one, and the other three $\lambda_n$, $\tau_u$ and $\tau_n$ are close to the optimal (with <1\% gap in overall accuracy).

\begin{table}[htb]
\centering
\caption[Hyperparameters used for OpenCon.]{Hyperparameters in OpenCon.}
\begin{tabular}{lllllll} \toprule
 & $\lambda_n$  &  $\tau_n$   &  $\lambda_l$  &  $\tau_l$   &    $\lambda_u$  &  $\tau_u$   \\ \midrule
ImageNet-100 & 0.1 & 0.7 & 0.2 & 0.1 & 1 & 0.6\\
CIFAR-100 & 0.1 & 0.7 & 0.2 & 0.1 & 1 & 0.4 \\ \bottomrule
\end{tabular}
\label{tab:opencon_hyper}
\end{table}
\begin{figure}
    \centering
    \includegraphics[width=1\linewidth]{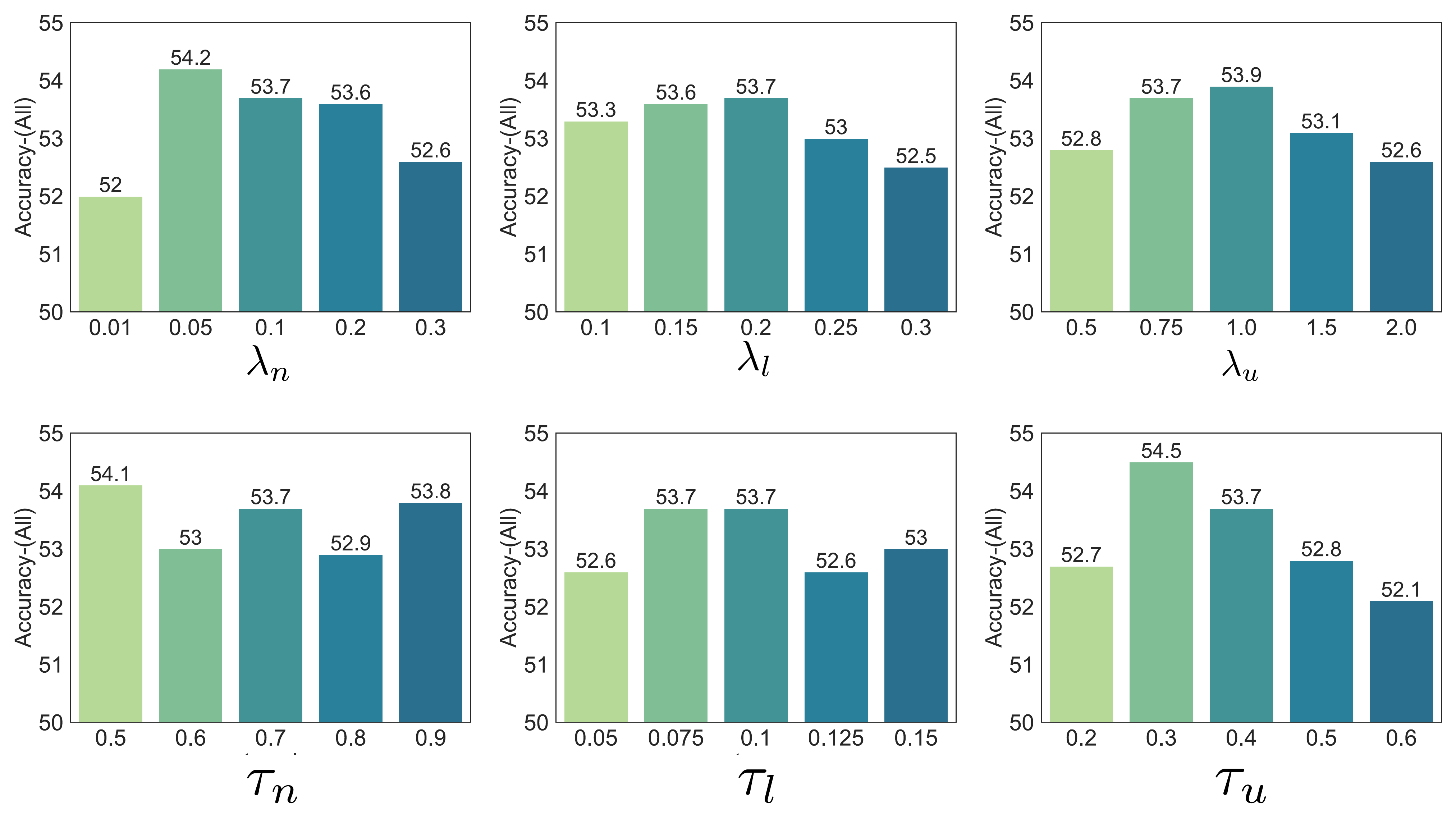}
    \caption[Sensitivity analysis of hyper-parameters.]{Sensitivity analysis of hyper-parameters on CIFAR-100. The overall accuracy is reported on three (weight, temperature) pairs:  ($\lambda_n$, $\tau_n$) for loss $\mathcal{L}_{n}$, ($\lambda_l$, $\tau_l$) for loss $\mathcal{L}_{l}$, and the ($\lambda_u$, $\tau_u$) for loss $\mathcal{L}_{u}$. The middle bar in each plot corresponds to the hyperparameter value used in our main experiments.  }
    \label{fig:opencon_hyper}
\end{figure}

\subsubsection{OOD Detection Comparison}
\label{sec:opencon_ood}
We compare different OOD detection methods in Table~\ref{tab:opencon_ood}. Results show that several popular OOD detection methods produce similar OOD detection performance. Note that Mahalanobis~\citep{lee2018simple} 
require heavier computation which causes an unbearable burden in the training stage. Our method incurs minimal computational overhead.

\begin{table}[htb]
\centering
\caption[Comparison of OOD detection performance with popular methods in OpenCon.]{Comparison of OOD detection performance with popular methods. Results are reported  on CIFAR-100. Samples from 50 known classes are treated as in-distribution (ID) data and samples from the remaining 50 classes are used as out-of-distribution (OOD) data. The OOD detection threshold is estimated on the labeled known classes $\mathcal{D}_l$.}

\begin{tabular}{llllll}
\toprule
 \textbf{Method} & \textbf{FPR95} & \textbf{AUROC} \\ \midrule
{MSP}~\citep{Kevin} & 58.9 & 86.0\\
{Energy}~\citep{liu2020energy} & 57.1 & 87.4\\ 
{Mahalanobis}~\citep{lee2018simple} & 54.6 & 88.7\\ 
{Ours} & 57.1 & 87.4\\ 
\bottomrule
\end{tabular}
\label{tab:opencon_ood}
\end{table}